\documentclass[a4paper]{article}
\usepackage{graphicx}
\usepackage{setspace}
\spacing{1.25}
\usepackage{authblk}

\usepackage[numbers]{natbib}

\usepackage{import}
\usepackage[margin=1in]{geometry}
\makeatletter
\def\thm@space@setup{%
\thm@preskip=1em \thm@postskip=0pt
}
\makeatother

\usepackage[utf8]{inputenc}

\usepackage{amsmath, amssymb, bbold, mathtools}
\usepackage{footnote}
\usepackage{acronym}
\usepackage{subfigure}
\usepackage{enumerate}
\usepackage{hyperref}
\usepackage{multirow}

\usepackage{todonotes}

\acrodef{dro}[DRO]{distributionally robust optimization}
\acrodef{ldt}[LDT]{large deviation theory}
\acrodef{ldp}[LDP]{large deviation principle}
\acrodef{lln}[LLN]{law of large numbers}
\acrodef{kl}[KL]{Kullback-Leibler}
\acrodef{iid}[{i.i.d.\ \!\!}]{independent identically distributed}
\acrodef{qp}[QP]{quadratic program}
\acrodef{qcqp}[QCQP]{quadratically constrained quadratic program}
\acrodef{vod}[VoD]{value of data}
\acrodef{saa}[SAA]{stochastic average approximation}
\acrodef{fl}[FL]{Fenchel-Legendre}

\usepackage{tikz}
\usetikzlibrary{shapes.geometric}
\usetikzlibrary{matrix,positioning}
\usepackage{pgfplots}
\pgfplotsset{compat=1.15}
\usepackage{mathrsfs}
\usetikzlibrary{arrows}
\usepackage{array}

\newcommand{\keywords}[1]{\textbf{Keywords:} #1}

\newcommand{\norm}[1]{\left\|#1\right\|}

\renewcommand{\tfrac}[2]{#1 / #2}

\newcommand{\D}[2]{\mathrm I(#1 , #2 )}

\newcommand{\set}[2]{\left\{ #1\ : \ #2 \right\}}
\newcommand{\tpose}{ ^\top}

\newcommand{\defn}[0]{:=}

\DeclareMathOperator*{\argmin}{arg\,min}
\DeclareMathOperator*{\argmax}{arg\,max}

\newcommand{\mb}{\mathbb}

\usepackage{amsthm}
\theoremstyle{plain}
\newtheorem{theorem}{Theorem}[section]
\newtheorem{lemma}[theorem]{Lemma}
\newtheorem{claim}[theorem]{Claim}

\theoremstyle{definition}
\newtheorem{definition}[theorem]{Definition}
\newtheorem{remark}[theorem]{Remark}

\newtheorem{example}[theorem]{Example}
\newtheorem{proposition}[theorem]{Proposition}

\newcommand{\Eb}{\mathbb{E}}
\newcommand{\Pb}{\mathbb{P}}
\newcommand{\Qb}{\mathbb{Q}}

\newcommand{\cX}{\mathcal{X}}
\newcommand{\cC}{\mathcal{C}}
\newcommand{\cP}{\mathcal{P}}
\newcommand{\cPin}{\cP^{\mathrm{o}}}
\newcommand{\into}{\mathrm{o}}

\newcommand{\Var}{\mathrm{Var}}
\newcommand{\Cov}{\mathrm{Cov}}
\newcommand{\orderleq}{\preceq_{\mathcal{C}}}
\newcommand{\orderpresc}{\preceq_{\hat{\mathcal{X}}}}

\newcommand\numberthis{\addtocounter{equation}{1}\tag{\theequation}}

\newcommand{\loss}{\ell}
\renewcommand{\Re}{\mathbf{R}}
\newcommand{\integ}{\mathbf{N}}

\newcommand{\cKL}{\hat{c}_{\mathrm{KL}}}
\newcommand{\cSVP}{\hat{c}_{\mathrm{V}}}
\newcommand{\deltaSVP}{\hat{\delta}_{\text{V}}}
\newcommand{\cRob}{\hat{c}_{\mathrm{R}}}

\newcommand{\xKL}[1]{\hat{x}_{\mathrm{KL} #1 }}
\newcommand{\xSVP}[1]{\hat{x}_{\mathrm{V} #1}}
\newcommand{\xRob}[1]{\hat{x}_{\mathrm{R} #1}}

\begin{document}

\title{Learning and Decision-Making with Data :\\ Optimal Formulations and Phase Transitions
}



\author[1]{Amine Bennouna}
\author[2]{Bart P.G. Van Parys}
\affil[1]{Operations Research Center, MIT}
\affil[2]{Sloan School of Management, MIT}

\date{}

\maketitle

\begin{abstract}
  We study the problem of designing optimal learning and decision-making formulations when only historical data is available.
  Prior work typically commits to a particular class of data-driven formulation and subsequently tries to establish out-of-sample performance guarantees.
  Following \cite{van2020data} we take here the opposite approach.
  We define first a sensible yardstick with which to measure the quality of any data-driven formulation and subsequently seek to find an ``\textit{optimal}'' such formulation.
  Informally, any data-driven formulation can be seen to balance a measure of proximity of the estimated cost to the actual cost while guaranteeing a level of out-of-sample performance.
  Given an acceptable level of out-of-sample performance, we construct explicitly a data-driven formulation that is uniformly closer to the true cost than any other formulation enjoying the same out-of-sample performance.
  We show the existence of three distinct out-of-sample performance regimes (a superexponential regime, an exponential regime, and a subexponential regime) between which the nature of the optimal data-driven formulation experiences a phase transition.
  The optimal data-driven formulations can be interpreted as a classically robust formulation in the superexponential regime, an entropic distributionally robust formulation in the exponential regime, and finally a variance penalized formulation in the subexponential regime.
  This final observation unveils a surprising connection between these three, at first glance seemingly unrelated, data-driven formulations which until now remained hidden.
\end{abstract}

\section{Data-Driven Decision Making}
\label{sec: data-driven decision making}
We consider decision-making problems in the face of uncertainty where the probability distribution of the uncertainty remains unobserved but rather must be learned from a finite number of independent samples.
Let $\mathcal{X}$ be a compact set of possible decisions and $\xi$ a random variable realizing in a set $\Sigma$ representing the uncertainty.
For a given scenario $i \in \Sigma$ of the uncertainty, and a decision $x \in \mathcal{X}$, the loss incurred for decision $x$ in scenario $i$ is denoted here as $\loss(x, i) \in \Re$. We assume throughout that this loss function is continuous in the decision $x$ for any scenario $i$.
The random variable $\xi$ is distributed according to a probability distribution $\Pb$ in the probability simplex $\mathcal{P}$ over $\Sigma$.
The problem we wish to approximate is
\begin{equation}\label{eq: stochastic opt}
    \min_{x \in \mathcal{X}}~ \mathbb{E}_{\Pb}(\loss(x, \xi)),
\end{equation}
where the cost function $c(x,\Pb) = \mathbb{E}_{\Pb}(\loss(x, \xi))$ represents an expected loss.
Whereas the loss $\loss(x, i)$ of each decision $x$ is known for each scenario $i$, the cost $c(x, \Pb)$ of each decision remains unknown as it is a function of the unknown probability distribution $\Pb$. The described class of problems is rather large as it describes both empirical risk minimization problems in machine learning as well as data-driven decision problems in operations research.

\begin{example}[Machine Learning]\label{exp: Decision-making ML}
Given covariates $(X,Y)\in \Re^n\times \Re$ following an unknown probability distribution $\Pb$, a set of parametrized functions $\{f_{\theta}\}_{\theta \in \Theta}$, and a loss function $L :\Re \times \Re \rightarrow \Re$, the goal is to learn the parameter $\theta$ that minimizes the expected out-of-sample error
\begin{equation}\label{eq: ML problem}
\min_{\theta \in \Theta}~\Eb_{ \Pb}[L(f_{\theta}(X),Y)].
\end{equation}
Such learning problem can be seen as a stochastic optimization problem of the form \eqref{eq: stochastic opt} by letting $x:= \theta$, $\xi:=(X, Y)$, and $c(x,\Pb):= \Eb_{ \Pb}[L(f_{\theta}(X),Y)]$.
\end{example}


{\color{black}
\begin{example}[Newsvendor problem]
\label{exp: newsvendor}
  Consider a newsvendor problem in which the decision maker decides on an inventory level $x \in \cX \subset \Re_+$ in the face of uncertain demand $d \in \Re_+$ following a distribution $\Pb$. Any excess inventory leads to a cost $h>0$ per unit, and any demand that is not satisfied is associated with to an opportunity cost $b>0$ per unit. The newsvendor problem can be formulated as the stochastic optimization problem 
  \begin{equation*}
    \min_{x\in \cX} \, \Eb_{\Pb}[b(d-x)^+ +h(x-d)^+].
  \end{equation*}
\end{example}
}

In practice the unknown distribution is never observed directly but rather must be estimated from historical data.
Problem \eqref{eq: stochastic opt} can hence not be solved directly. Typically, it is assumed instead that we observe independent samples $\{\xi_1, \ldots, \xi_T\} \in \Sigma^T$ identically distributed as the unknown distribution $\Pb$.
The most straightforward|and perhaps the most common| way to formulate a data-driven counterpart to Problem \eqref{eq: stochastic opt} is to simply replace the unknown distribution with its empirical counterpart and to solve instead the problem
\begin{equation*}
  \hat x_{\rm{SAA}}(\hat{\Pb}_T) \in \arg\min_{x\in \cX}\, c(x, \hat{\Pb}_T)
\end{equation*}
where we denote with $\hat{\Pb}_T$ the empirical distribution of the observed data points,
$\hat{\Pb}_T(i) := \frac{1}{T}\sum_{t=1}^T \mathbb{1}_{\xi_t = i}$ for all $i \in \Sigma$.
This formulation is well known in the optimization community as the Sample Average Approximation (SAA) \cite{shapiro2003monte}. When applied to the machine learning problem discussed in Example \ref{exp: Decision-making ML}, this approach reduces to the well known Empirical Risk Minimization (ERM) principle \cite{vapnik2013nature}.
SAA and ERM are motivated by the fact that the historical empirical cost is the simplest and perhaps most natural substitution for the unknown actual cost.
{\color{black}This formulation has been shown to be optimal when the goal is minimizing the \textit{expected} optimality gap. 
More precisely, denote for any data-driven decision $\hat x_T:\cP\to\cX$, $T\in \integ$ its sub-optimality gap as $G(\hat x_T(\hat{\Pb}_T))=c(\hat x_T(\hat{\Pb}_T), \Pb)- \min_{x \in \mathcal{X}}c(x,\Pb)$. 
Under mild technical assumptions asymptotically with $T\to\infty$ we have
  \(
    \Eb_{}[\phi(G(\hat x_{\rm{SAA}}(\hat{\Pb}_T)))] \leq \Eb_{}[\phi(G(\hat x_T(\hat{\Pb}_T)))]
  \)
  for all convex increasing functions $\phi$ where the expectation here is over the data generation \cite{lam2021impossibility}.}

{\color{black}
A decision with good expected performance might however have a nonnegligable probability of having a bad or even catastrophic performance. That is, good {expected} performance does not necessarily imply acceptable practical performance. As an example, consider the newsvendor problem of Example \ref{exp: newsvendor}.
Figure \ref{fig: newsvendor numericals} illustrates the distribution of the relative optimality gap of the SAA solution under the randomness of the data generation. Here, the expected optimality gap is only $\approx 5\%$. Nevertheless, the optimality gap is larger than $17\%$ with probability $0.2$. A decision-maker may seek decisions with not only good expected performance, but also seek to obtain performance guarantee in case a less likely but nevertheless plausible event occurs. 
We seek to provide the decision-maker with a confidence interval on the out-of-sample performance of its prescribed decision for all events whose probability is not too small.
Ultimately, our goal is to study the inherent trade-off between expected performance and such high probability out-of-sample guarantees. Numerous ``robust'' data-driven formulations, built to provide safe decisions, have been suggested in the literature which we review in Section \ref{sec: robust formulations}. A key question is which ``best'' formulation to use for the purpose of data-driven decision-making?
We study this question from the lens of the above mentioned trade-off.
We first formalize this trade-off and then seek to determine in a precise sense ``optimal'' data-driven formulations.
There are two fundamental parts of our problem: a \textit{prediction} problem and a \textit{prescription} problem. We believe that both problems are interesting in their own right and hence we will discuss them separately. Our framework which we now introduce will be a very significant generalization of the results found in \cite{van2020data}. In Section \ref{sec: closely related work}, we detail precisely the key differences and our main novelties. 
}

\begin{figure}
  \centering
  \begin{minipage}{\textwidth}
    \centering
    \begin{tikzpicture}

\def\ratio{0.15}
\def\mean{5}
\def\upperquartile{26.9}
\def\lowerquartile{0.001}
\def\tailquantile{17.62}

\def\boxwidth{0.2}
    
\draw[->] (0,0) -- (5,0);
\draw[thick] (0,\boxwidth*1.5) -- (0,-\boxwidth*1.5); 
\node[left] at (0,0) {Opt}; 

\fill[opacity=0.3,fill=orange!90!black] (\ratio*\lowerquartile,\boxwidth) rectangle (\ratio*\upperquartile,-\boxwidth);
\draw[ultra thick, orange] (\ratio*\mean,-\boxwidth) -- (\ratio*\mean,\boxwidth);
\node[below] at (\ratio*\mean,-\boxwidth) {5\%};
\fill[red, fill opacity=0.1] (\tailquantile*\ratio,\boxwidth*1.5) rectangle (\upperquartile*\ratio,-\boxwidth*1.5);
\draw[thick, red] (\tailquantile*\ratio,\boxwidth*1.5) -- (\tailquantile*\ratio,-\boxwidth*1.5);
\draw[thick, red] (\tailquantile*\ratio,\boxwidth*1.5) -- (\upperquartile*\ratio,\boxwidth*1.5);
\draw[thick, red] (\tailquantile*\ratio,-\boxwidth*1.5) -- (\upperquartile*\ratio,-\boxwidth*1.5);
\node[below, red] at (\tailquantile*\ratio,-\boxwidth*1.5) {17\%};
\end{tikzpicture}
    \caption{{\color{black}Illustration of the distribution of relative optimality gap of SAA ($G(\hat x_{\rm{SAA}}(\hat{\Pb}_T))/\min_{x \in \mathcal{X}}c(x,\Pb)$) on a news vendor problem, under the randomness of the data. Here, $h = 12, b = 1$ and demand $\Tilde{d}$ follows a mixture of two Gaussians $\mathcal{N}(50,5)$ and $\mathcal{N}(100,5)$ with respective weights $0.1$ and $0.9$. The SAA solution is computed with $100$ demand data points, and the presented statistics are computed across $10^7$ random datasets generations. The orange box represents realizations of the optimality gap between quantiles at level $0.01$ and $0.99$. The expected value is at $5\%$. The red box represent values above the quantile at level $0.8$, which is associated with a $17\%$ relative optimality gap.}}
  \end{minipage}
  \label{fig: newsvendor numericals}
\end{figure}

\subsection{Prediction Problems}\label{sec: prediction intro.}

The \textit{prediction} problem consists of estimating the unknown cost of a given decision|the unknown expectation in \eqref{eq: stochastic opt}|based on data. That is, constructing an estimate of the cost function using only the observed data. Such an estimate is called a data-driven predictor. In our setting all statistical information of the observed data can be compressed into its empirical distribution $\hat{\Pb}_T$.
Indeed, the order of the observed data points is of no importance when the data samples are independent.
Hence, in full generality, the data-driven predictor can be written as a function of the decision, the observed empirical distribution and the data size.

\begin{definition}[Predictors]\label{def: predictors.}
A predictor $\hat c$ is a sequence of functions $\{\hat{c}(\cdot,\cdot,T) \in \mathcal{X}\times \mathcal{P} \to \Re\}_{T\in \integ}$. For each distribution of the uncertainty $\Pb \in \cP$, decision $x \in \cX$, and data size $T \in \integ$, $\hat{c}(x,\hat{\Pb}_T,T)$ estimates the true cost $c(x,\Pb) = \Eb_{\Pb}(\loss(x,\xi))$ of decision $x$ under distribution $\Pb$.
\end{definition}

We will restrict our study to the class of smooth predictors $\cC$ which we discuss later. Our goal is to define a notion of optimality for data-driven predictors based on how well they balance two competing properties: out-of-sample {\color{black}guarantee} and accuracy.

\paragraph{Out-of-Sample Guarantees.}

The main issue with the naive estimator used in SAA or ERM formulations is its tendency to overfit the training data and consequently suffer poor, {\color{black}uncontrolled}, out-of-sample performance. 
{\color{black}In particular, in the context of prediction, SAA and ERM may provide an optimistic and unreliable estimate of the out-of-sample cost. This phenomenon is particularly pronounced in machine learning, where it is known as ``overfitting''. In decision analysis, \cite{smith2006optimizer} refers to it as the ``optimizer's curse'', and in finance, \cite{michaud1989markowitz} calls it the ``error maximization effect'' of portfolio optimization. In the context of decision-making, it is desirable to have access to a high quality upper bound on the out-of-sample cost. Indeed, while a lower cost may come as a welcome surprise, underestimating the cost may come with severe financial repercussions. }
Formally, the predictor is desired to have, for all decision $x \in \cX$ and underlying distribution $\Pb \in \cP$, a low \textit{probability of disappointment} 
    \begin{equation}\label{eq: out-of-sample proba}
    \Pb^\infty \left( c(x, \Pb) > \hat c(x, \hat{\Pb}_T,T)\right).
    \end{equation}
{\color{black}As in the rest of the paper, $\Pb^{\infty}$ denotes here the probability distribution of the independent identically distributed data from which the estimate $\hat{\Pb}_T$ is constructed.}
Typically in the decision-making and statistical learning literature, one first considers a particular predictor and subsequently looks to establish out-of-sample disappointment guarantees of the form \eqref{eq: out-of-sample proba} such as\footnote{Some bounds in the literature exhibit an additive constant in the inequality of bounds of the form \eqref{eq: out-of-sample proba}. This is equivalent to the form \eqref{eq: out-of-sample proba} by incorporating the constant in the predictor $\hat{c}$.} those discussed in {\color{black}Section \ref{sec: robust formulations}}.
We take here precisely the opposite approach.
That is, we will impose a desired out-of-sample guarantee and only then seek to determine the ``best'' predictor (in a certain sense) verifying said imposed out-of-sample guarantee.

For a given sequence $(a_T)_{T\geq 1} \in \Re_+^{\integ}$, we are interested in predictors $(\hat{c}(\cdot,\cdot,T))_{T\geq 1}$ verifying the out-of-sample guarantee 
\begin{equation}\label{eq: out-of-sample ganrantee}
    \limsup_{T\to\infty} \frac{1}{a_T} \log \Pb^\infty \left( c(x, \Pb) > \hat c(x, \hat{\Pb}_T,T)\right)\leq -1 \quad \forall x \in \mathcal{X}, \; \forall \Pb \in \mathcal{P}.
\end{equation}
Such guarantee imposes that the estimated cost bounds the out-of-sample cost from above, i.e., $c(x, \Pb) < \hat c(x, \hat{\Pb}_T,T)$, with probability at least $1-e^{-a_T+o(a_T)}$. 
From a statistical point of view, the predictor can be interpreted to provide a one sided confidence interval {\color{black}$(-\infty,\hat{c}(x, \hat{\Pb}_T,T)]$} on the true cost $c(x, \Pb)$.
The sequence $(a_T)_{T\geq 1}$ quantifies the desired speed with which the predictor enjoys stronger out-of-sample performance as the amount of available data grows.
We reiterate the generality of our setting: for \textit{any} sequence $(a_T)_{T\geq 1}$ and associated desired out-of-sample guarantee, we seek to find what is the ``best'' predictor using data that provides the desired out-of-sample guarantee. Let us now precise rigorously what constitutes such ``best'' predictor.

\paragraph{Accuracy.}
It is easy to construct {\color{black}safe estimators} with low probability of disappointment. Indeed, by simply inflating the empirical cost, i.e.,
\[
  \hat{c}(x,\hat{\Pb}_T,T) = c(x, \hat{\Pb}_T) + R
\]
for some large $R > 0$, the probability of disappointment can be made arbitrarily small. The resulting cost estimate is however very conservative, {\color{black}and far from the actual cost,} which is clearly undesirable. {\color{black}In the newsvendor example, this would amount to predicting a very large cost, which is always a safe estimate of the unknown profit.}
{\color{black}Hence, among the predictors providing a desirable out-of-sample guarantee, we seek the closest ones to the unknown out-of-sample cost. There are various ways to define an statistical interesting notion of closeness with which to compare predictors. For example, we could consider the expected bias induced by the predictor $\Eb_{\Pb^\infty
}(\hat{c}(x,\hat{\Pb}_T,T) - c(x,\Pb))$ or expected $L^1$ estimation error $\Eb_{\Pb^\infty
}(|\hat{c}(x,\hat{\Pb}_T,T) - c(x,\Pb)|)$ for all decision $x\in \cX$ and underlying true distribution $\Pb \in \cP$. 
Both notions attempt to quantify how accurate our predictor estimates the unknown out-of-sample cost in \textit{expectation}.
}


{\color{black}
We will consider here a stronger accuracy notion based on the amount predictors inflate or regularize the empirical cost.}
Recall that $(\hat{c}(x,\hat{\Pb}_T,T))_{T\geq 1}$ is used to estimate $c(x,\Pb)$. Hence, the term $(\hat{c}(x,\hat{\Pb}_T,T) - c(x,\hat{\Pb}_T))_{T\geq 1}$ is precisely the amount of regularization added to the empirical cost $c(x,\hat{\Pb}_T)$ by the predictor $\hat{c}(x,\hat{\Pb}_T,T)$. 
We seek predictors that add less regularization uniformly in all decisions $x\in \cX$ and realization of the empirical distribution $\hat{\Pb}_T$.  
We introduce, therefore, the partial order $\preceq_{\mathcal{C}}$ on the set of predictors defined as\footnote{We indicate that dropping the absolute values in the definition of the order $\preceq_{\cC}$ leads to an equivalent order.}
$$
\hat{c}_1 \orderleq \hat{c}_2 \iff
\forall x,\Qb \in \cX\times \cPin \quad
\limsup_{T\to \infty} \frac{|\hat{c}_1(x,\Qb,T)-c(x,\Qb)|}{|\hat{c}_2(x,\Qb,T)-c(x,\Qb)|} \leq 1,%
$$
for $\hat{c}_1,\hat{c}_2 \in \cC$ where $\cPin$ is the interior of $\cP$ and where we follow the convention $\tfrac{0}{0}=1$. Intuitively, $\hat{c}_1$ is preferred to $\hat{c}_2$ if asymptotically, $\hat{c}_1$ adds less regularization in its worst case than $\hat{c}_2$ adds in its best case. That is, informally, for every $x $ and $\Pb$ we have
\(
\sup_{t\geq T}{|\hat{c}_1(x,\Pb,t)-c(x,\Pb)|} 
\lesssim
\inf_{t\geq T}{|\hat{c}_2(x,\Pb,t)-c(x,\Pb)|}. 
\)
Moreover, as seen in the previous section, predictors verifying the out-of-sample guarantee provide a high probability upper bound on the true cost. Hence, when $\hat{c}_1$ and $\hat{c}_2$ verify the out-of-sample guarantee, and $\hat{c}_1 \preceq_{\cC} \hat{c}_2$, $\hat{c}_1$ is intuitively a uniformly better upper bound than $\hat{c}_2$ and hence ought to be preferred.
Notice finally that we can verify easily that $\orderleq$ is a partial order with the equivalence relation $\equiv$ defined as $\hat{c}_1 \equiv \hat{c}_2 \iff |\hat{c}_1(x,\Qb,T)-c(x,\Qb)| \sim |\hat{c}_2(x,\Qb,T)-c(x,\Qb)|, \; \forall x,\Qb \in \cX, \cP$ where $\sim$ denotes asymptotic equivalence when $T \to \infty$. Figure \ref{fig: illustration of guarantee and order.} illustrates this order and the out-of-sample guarantee. We highlight here that the partial order requires less regularization for \textit{every} distribution $\Qb \in \cPin$, in particular, for any realization of the empirical distribution $\hat{\Pb}_T$.
{\color{black}
Finally, we also prove under certain technical assumptions that our proposed order is stronger than the aforementioned alternative orders induces by expected bias and $L^1$ error (see Lemma \ref{lemma: order implies less bias and L1 error.}).
}

\begin{figure}[!htb]
\centering
\begin{minipage}{.45\textwidth}
  \begin{tikzpicture}[scale=0.8]

  \draw[->] (0.3,1) -- (7,1) node[right] {$x,\Pb$};
  \draw[->] (0.5,0.8) -- (0.5,5);

  \draw[thick] plot [smooth,tension=0.9] coordinates{
    (0.5, 2.0)
    (2.5, 1.5)
    (5.0, 2.5)
    (6.5, 2.5)
  };
  \draw (0.5, 2.0) node[left] 
  {$c \;$};

  \draw[thick, color=red] plot [smooth,tension=0.9] coordinates{
    (0.5, 3)
    (1.8, 2)
    (3.5,2)
    (5.5, 2.75)
    (6.5, 2.6)
  };
  \draw (0.5,3) node[left] 
  {${\color{red}\hat{c}_1}$};
  
  \draw[thick, color=blue] plot [smooth,tension=0.9] coordinates{
    (0.5, 4)
    (3,3)
    (5.5, 3.5)
    (6.5, 3.5)
  };
  \draw (0.5, 4) node[left] 
  {${\color{blue}\hat{c}_2}$};
  
  \fill[color=red,fill opacity=0.3]
  plot [smooth,tension=0.99] coordinates{
    (0.5, 3.5)
    (1.8, 2.3)
    (3.5,2.25)
    (4.5,2.6)
    (5.5, 2.85)
    (6.5, 2.9)
  }
  --
  plot [smooth,tension=0.99] coordinates{
    (6.5, 2.3)
    (5.5, 2.55)
    (4.5,2)
    (3.5,1.75)
    (1.8, 1.8)
    (0.5, 2.4)
  }
  -- cycle;

  \fill[color=blue,fill opacity=0.3]
  plot [smooth,tension=0.9] coordinates{
    (0.5, 4.2)
    (1.7,3.75)
    (3.3,3.1)
    (4.5, 3.5)
    (6.5, 3.7)
  }
  --
  plot [smooth,tension=0.99] coordinates{
    (6.5, 3.3)
    (4.5, 2.9)
    (3.3,2.9)
    (1.7,3)
    (0.5, 3.8)
  }
  -- cycle;

\end{tikzpicture}
\end{minipage}
\begin{minipage}{.45\textwidth}
  \begin{tikzpicture}[scale=0.8]

\draw[->] (0.3,1) -- (7,1) node[right] {$x,\Pb$};
\draw[->] (0.5,0.8) -- (0.5,5);

\draw[thick] plot [smooth,tension=0.9] coordinates{
(0.5, 2.0)
(2.5, 1.5)
(5.0, 2.5)
(6.5, 2.5)
 };
   \draw (0.5, 2.0) node[left] 
      {$c \;$};

\draw[thick, color=red] plot [smooth,tension=0.9] coordinates{
(0.5, 3)
(3,2.5)
(5.5, 4.5)
(6.5, 4)
 };
  \draw (0.5,3) node[left] 
      {${\color{red}\hat{c}_1}$};
 
\draw[thick, color=blue] plot [smooth,tension=0.9] coordinates{
(0.5, 4)
(3,3)
(5.5, 3.5)
(6.5, 3.5)
 };
  \draw (0.5, 4) node[left] 
      {${\color{blue}\hat{c}_2}$};
      
\draw[thick, color=green] plot [smooth,tension=0.9] coordinates{
(0.5, 2.5)
(2,2)
(5.5, 3)
(6.5, 3)
 };
  \draw (0.5, 2.5) node[left] 
      {${\color{green}\hat{c}_3}$};
 
\fill[color=red,fill opacity=0.3]
 plot [smooth,tension=0.99] coordinates{
(0.5, 3.2)
(3,2.7)
(4.5, 4)
(6, 4.7)
(6.5, 4.1)
}
--
 plot [smooth,tension=0.99] coordinates{
(6.5, 3.8)
(5.5, 4.3)
(4.5, 3.2)
(2,2.2)
(0.5, 2.8)
}
-- cycle;

\fill[color=blue,fill opacity=0.3]
plot [smooth,tension=0.9] coordinates{
(0.5, 4.2)
(1.7,3.75)
(3.3,3.1)
(4.5, 3.5)
(6.5, 3.7)
 }
--
 plot [smooth,tension=0.99] coordinates{
(6.5, 3.3)
(4.5, 2.9)
(3.3,2.9)
(1.7,3)
(0.5, 3.8)
}
-- cycle;

\fill[color=green,fill opacity=0.3]
plot [smooth,tension=0.9] coordinates{
(0.5, 2.7)
(2,2.2)
(5.5, 3.3)
(6.5, 3.2)
 }
--
plot [smooth,tension=0.9] coordinates{
(6.5, 2.8)
(5.5, 2.7)
(2,1.8)
(0.5, 2.2)
 }
-- cycle;
\end{tikzpicture}
\end{minipage}
\caption{Each colored curve represents the nominal value of a predictor $\hat{c}(x,\Pb,T)$, for a fixed $T$, the black lower curve being the true cost $c(x,\Pb)$. The shaded region represents the random values of $\hat{c}(x,\hat{\Pb}_T,T)$ which occur with high probability $\sim 1-e^{-a_T}$. The predictor $\hat{c}_1$ on the left does not verify the out-of-sample guarantee as there is a set of probability larger than $e^{-a_T}$ where $\hat{c}_1(x,\hat{\Pb}_T,T)<c(x,\Pb)$ (the shaded region below the $c$ curve). The predictor $\hat{c}_2$ on the other hand verifies the out-of-sample as for all values of $\hat{\Pb}_T$ on the high probability set, $\hat{c}_2(x,\hat{\Pb}_T,T)\geq c(x,\Pb)$. The figure on the right illustrates the order $\preceq_{\cC}$. In the figure, $\hat{c}_1$ and $\hat{c}_2$ can not be compared as none is better uniformly than the other. The predictor $\hat{c}_3$ is uniformally closer to $c$ than $\hat{c}_1$ and $\hat{c}_2$, hence $\hat{c}_3 \preceq_{\cC} \hat{c}_1$ and $\hat{c}_3 \preceq_{\cC} \hat{c}_2$. Notice that as both our feasibility and order notions are asymptotic in $T$, the figure here is merely an illustration. }
\label{fig: illustration of guarantee and order.}
\end{figure}
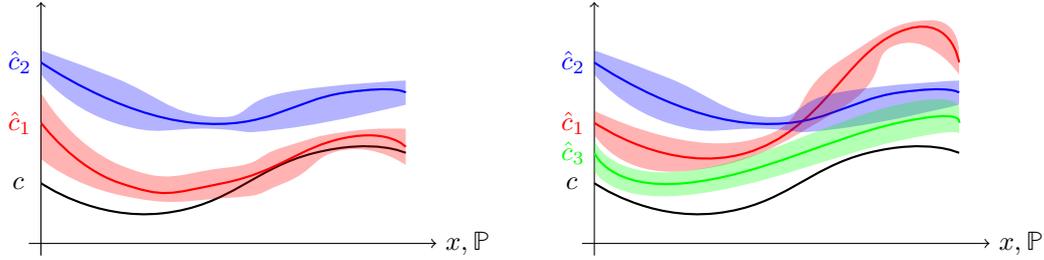

\paragraph{Optimal Prediction.}
{\color{black}There is an inherent trade-off between the accuracy of any predictor and guaranteeing with high probability not to disappoint out-of-sample.}
Due to the fluctuations of the empirical distribution $\hat{\Pb}_T$ around the true distribution $\Pb$, the prediction $\hat{c}(x,\hat{\Pb}_T,T)$ fluctuates around the nominal value $\hat{c}(x,\Pb,T)$. Therefore, the closer $\hat{c}(x,\Pb,T)$ is to the true cost $c(x,\Pb)$, the higher the probability that it disappoints, i.e., $\hat{c}(x,\hat{\Pb}_T,T) \leq c(x,\Pb)$ (see Figure \ref{fig: illustration of guarantee and order.}). Hence, the more preferred a predictor is in terms of our introduced accuracy order $\preceq_{\cC}$, the weaker the out-of-sample guarantees it will verify.

Using the order $\orderleq$, we can formulate the problem of finding the optimal predictor verifying a given out-of-sample guarantee as the following meta-optimization problem:

\begin{equation}
  \label{eq:optimal-predictor}
  \begin{array}{ll}
    \displaystyle\mathop{\text{minimize}}_{\hat c\in \mathcal{C}}{}_{\preceq_{\cC}} & \hat c \\
    \text{subject to} & \displaystyle \limsup_{T\to\infty} \frac{1}{a_T} \log \Pb^\infty \left( c(x, \Pb) > \hat c(x, \hat{\Pb}_T,T)\right)\leq -1 \quad \forall x\in \mathcal{X},\;\Pb\in \cPin.
  \end{array}
\end{equation}

A feasible predictor $\hat{c}$ in problem \eqref{eq:optimal-predictor} is a predictor $\hat{c} \in \cC$ verifying the out-of-sample guarantee \eqref{eq: out-of-sample ganrantee}. Notice that various classes of robust predictors are feasible in problem \eqref{eq:optimal-predictor} such as Wasserstein DRO and SVP, with properly scaled parameters, as indicated by inequalities \eqref{eq: Wasserstein bound} and \eqref{eq: bound for SVP}.

A \textit{weakly optimal} solution to problem \eqref{eq:optimal-predictor} is a feasible predictor $\hat{c}$ such that no other feasible predictor is prefered to $\hat{c}$.
A \textit{strong optimal} solution to problem \eqref{eq:optimal-predictor} is a feasible predictor $\hat{c}$ that is preferred to all feasible predictor $\hat{c}'$, ie $\hat{c} \orderleq \hat{c}'$. Typically, several weakly optimal solutions may be expected to exist, while there rarely exists a strong optimal solution.
When a strong optimal solution $\hat{c}$ exists, it means that for the purpose of data-driven prediction with the desired out-sample guarantee, there is no incentive to use a different predictor than $\hat{c}$ as it is preferred to all other predictors verifying the out-of-sample guarantee, for any uncertainty distribution $\Pb$ and decision $x$. We finally point out that strong optimal solutions are unique in the sense that any two strong optimal solutions $\hat{c}_1$ and $\hat{c}_2$ are necessarily equivalent in the sense that $\hat{c}_1 \equiv \hat{c}_2$.

\subsection{Prescription Problems}

Predicting the cost of any decision $x$ (or parameter choice $\theta$ in the context of machine learning, Example \ref{exp: Decision-making ML}) is typically merely a means to derive an approximation of the optimal solution and its cost.
Out-of-sample guarantees of the form \eqref{eq: out-of-sample ganrantee},
which hold pointwise for any decision $x$, do not generally hold for the prescribed decision with minimal estimated cost
$
\argmin_{x \in \cX} \hat{c}(x,\hat{\Pb}_T,T)
$.
An indirect way to impose out-of-sample guarantees on the prescribed decision is to derive a uniform counterpart to \eqref{eq: out-of-sample ganrantee} which holds uniformly for all decisions.
To do so in statistical learning, one typically restricts the considered function class $\{f_{\theta}\}_{\theta \in \Theta}$ to be of controlled bounded complexity, as measured by covering numbers or VC dimension \cite{vapnik2013nature, vapnik2015uniform}, {\color{black}such as in \cite{lam2019recovering, duchi2021statistics}.}
When the purpose of decision-making is only to derive the optimal solution of problem \eqref{eq: stochastic opt}, seeking such uniform bounds seems unduly restrictive. We seek, therefore, to investigate cost estimates satisfying out-of-sample guarantees precisely at the prescribed decision instead of uniformly over all decisions. In the following, we formalize these observations in what we call the prescription problem.

{\color{black}
Our prescription problem consists of using the observed data $\hat{\Pb}_T$ to choose a decision $\hat{x}_T(\hat{\Pb}_T)\in \cX$ minimizing the true cost $c(\cdot,\Pb)$.
Unlike in the prediction problem, we are only interested in approximating the minimizer and minimum of Problem \eqref{eq: stochastic opt} and not in estimating the cost of any alternative decision.
The decision-maker's interest is two fold: constructing a decision $\hat{x}_T(\hat{\Pb}_T)$  and providing a high probability upper bound guarantee on its out-of-sample performance, say $\hat h(\hat{\Pb}_T,T) \in \Re$. This then constitutes a one sided confidence interval $(-\infty,\hat h(\hat{\Pb}_T,T)]$ on the performance of the prescribed decision. Hence, in general we need to consider pairs of prescriptors and cost estimators $(\hat{x},\hat h)$ to model the prescription problem in full generality. A special commonly considered case is when the prescriptor $\hat{x}$ is a minimizer of a surrogate predictor $\hat{c}$ of the cost function, that is $\hat{x}_T(\hat{\Pb}_T) \in \argmin_{x\in \cX} \hat{c}(x,\hat{\Pb}_T,T)$. Then, the predictor can be used to construct the desired confidence interval with $\hat h(\hat{\Pb}_T,T) := \hat{c}(\hat{x}_T(\hat{\Pb}_T),\hat{\Pb}_T,T)$. While seemingly restrictive, this common special structure does not actually reduce generality. 
In fact,
for any pair $(\hat{x},\hat h)$, consider the predictor $\hat{c}(x,\hat{\Pb}_T,T) = \hat h(\hat{\Pb}_T,T)+\|x-\hat x_T(\hat{\Pb}_T)\|$. Clearly, $\hat{x}_T(\hat{\Pb}_T) \in \argmin_{x\in\cX} \hat{c}(x,\hat{\Pb}_T,T)$ and 
$\hat{c}(\hat{x}_T(\hat{\Pb}_T),\hat{\Pb}_T,T) = \hat h(\hat{\Pb}_T,T)$.
This motivates the following general definition of predictors.
}

\begin{definition}(Prescriptors)
\label{def:prescriptors}  
A precriptor is a sequence of functions $\hat{x}_T(\cdot): \cP \longrightarrow \cX$, $T\in \integ$ such that there exists a predictor $\hat{c} \in \cC$ verifying 
$$
\hat{x}_T(\Pb) \in \argmin_{x \in \cX} \, \hat{c}(x,\Pb,T), \quad \forall \Pb \in \cP, \; T \in \integ.
$$
For each empirical distribution $\hat{\Pb}_T$ and data size $T$, the prescription $\hat{x}_T(\hat{\Pb}_T)$ approximates the optimal solution of $\min_{x\in \cX} c(x,\Pb)$. We denote by $\hat{\cX}$ the set of prescriptors and their associated predictors.
\end{definition}

Akin to the discussion in the predictor problem we will judge any prescriptor in terms of their out-of-sample guarantee and accuracy. Both quantities will be defined similarly as in the prediction problem but crucially only pertain to one particular decision|that is to say the suggested prescription.
 
\paragraph{Out-of-Sample Guarantee.} In the prescription problem estimating the true expected value in Problem \eqref{eq: stochastic opt} is only a means with which to find an associated decision with good out-of-sample cost.
{\color{black}We seek therefore to construct a predictor which provides a high probability upper bound estimate of the true out-of-sample cost of its prescribed solution.}
More precisely, the estimated optimal cost by the predictor $\hat{c}^{\star}(\hat{\Pb}_T,T):= \hat{c}(\hat{x}_T(\hat{\Pb}_T),\hat{\Pb}_T,T)$ is an upper bound on the unknown out-of-sample cost of the prescribed decision $c(\hat{x}_T(\hat{\Pb}_T),\Pb)$ with high probability.
Formally, the estimator is desired to have a low \textit{probability of disappointment}
\begin{equation}\label{eq: prec out-of-sample guarantee prob}
    \Pb^\infty 
    \left(
       c(\hat{x}_T(\hat{\Pb}_T), \Pb) > \hat c^{\star}(\hat{\Pb}_T,T)
    \right).
\end{equation}
where {\color{black}$\hat{c}^{\star}(\Pb,T):= \min_{x\in \cX} \hat{c}(x,\Pb,T)$} for every predictor $\hat{c}$, distribution $\Pb$ and $T\in \integ$.
For a given sequence $(a_T)_{T\geq 1} \in \Re_+^{\integ}$, we are interested in prescriptors verifying the out-of-sample guarantee on the probability of disappointment:
\begin{equation}\label{eq: prescriptor out-of-sample ganrantee}
    \limsup_{T\to\infty} \frac{1}{a_T} \log \Pb^\infty \left( c(\hat{x}_T(\hat{\Pb}_T), \Pb) > \hat c^{\star}(\hat{\Pb}_T,T) \right)\leq -1, \quad \forall  \Pb \in \mathcal{P}.
  \end{equation}
  {\color{black}Such a guarantee can be in equated with a confidence interval on the out-of-sample cost of the prescribed decision, $c(\hat{x}_T(\hat{\Pb}_T), \Pb) \in (-\infty, \hat c^{\star}(\hat{\Pb}_T,T)]$, with coverage probability $1-e^{-a_T+o(a_T)}$. Once again, here we are only interested in an upper bound rather than a two-sided confidence interval, as in decision-making, while lower cost are desirable and we primarily seek to control higher cost scenarios.}
  The sequence $(a_T)_{T\geq 1}$ quantifies the desired speed with which the prescriptor enjoys stronger out-of-sample guarantee as the amount of available data grows. For \textit{any} desired out-of-sample guarantee with probability $1-e^{-a_T+o(a_T)}$, we seek to find what is the ``best'' prescriptor using data that provides the desired out-of-sample guarantee. 

\paragraph{Accuracy.}
{\color{black}While the out-of-sample guarantee ensures the prescription provides a high probability safe estimate of the out-of-sample cost of the prescription, we seek now to define an accuracy notion characterizing how well the prescriptor estimates the true optimal cost $c^{\star}(\Pb) :=\min_{x\in\cX}c(x,\Pb)$. That is, among these safe prescriptors, we seek to determine the closest to the optimal cost.
There are again multiple possible statistical notions to quantify this closeness. One example is how well the estimated optimal cost approximates the true optimal cost in \textit{expectation}, that is by measuring $\Eb_{\Pb^\infty}(|\hat{c}^\star(\hat{\Pb}_T,T) - c^\star(\Pb)|)$. As before, we will consider a yet more discriminating comparison notion.}

We introduce the following partial order $\preceq_{\hat{\mathcal{X}}}$ on prescriptors and their associated predictors
$$
(\hat{c}_1,\hat{x}_1) \preceq_{\hat{\mathcal{X}}} (\hat{c}_2,\hat{x}_2) \iff
\forall \Pb \in \cPin \quad
\limsup_{T\to \infty} \frac{|\hat{c}_1^{\star}(\Pb,T) - c^{\star}(\Pb)|}{|\hat{c}_2^{\star}(\Pb,T) - c^{\star}(\Pb)|} \leq 1
$$
again with the convention $\tfrac{0}{0}=1$.
{\color{black}A prescriptor is preferred over another if its estimation of the optimal cost induces uniformly less regularization to the SAA estimator. In Lemma \ref{lemma: order on prescriptors}, we prove under mild assumptions that this order notion is stronger than the order induced by the expected estimation error mentioned before in which a prescriptor $(\hat{c}_1,\hat{x}_1)$ is preferred to another prescriptor if}
  \[
    \Eb_{\Pb^\infty}(|\hat{c}_2^\star(\hat{\Pb}_T,T) - c^\star(\Pb)|) \leq \Eb_{\Pb^\infty}(|\hat{c}_2^\star(\hat{\Pb}_T,T) - c^\star(\Pb)|).
  \]
Furthermore, notice that when $\hat{c}$ verifies the out-of-sample guarantee, then  $\hat c^{\star}(\hat{\Pb}_T,T) \geq c(\hat{x}_T(\hat{\Pb}_T), \Pb) \geq c^\star(\Pb)$ with high probability. Hence, smaller $\hat c^{\star}(\hat{\Pb}_T,T)$ in our partial order, lead also to closer out-of-sample cost $c(\hat{x}_T(\hat{\Pb}_T), \Pb)$ to the optimal cost $c^\star(\Pb)$ with high probability.
We remark that our proposed order only takes into consideration the accuracy of the cost predictor at the associated prescribed action and is blind to its accuracy with respect to any other alternative decision.
Finally, it should be remarked that $\orderpresc$ is here a partial order with the equivalence relation $\equiv_{\hat{\cX}}$ defined as $(\hat{c}_1,\hat{x}_1) \equiv_{\hat{\cX}} (\hat{c}_2,\hat{x}_2) \iff |\hat{c}_1^{\star}(\Pb,T) - c^{\star}(\Pb)| \sim |\hat{c}_2^{\star}(\Pb,T) - c^{\star}(\Pb)|, \; \forall \Pb \in \cPin$, where $\sim$ denotes asymptotic equivalence as $T \to \infty$.

\paragraph{Optimal Prescription.}
Designing the optimal prescriptor amounts therefore to balancing out-of-sample performance and accuracy at the prescribed decision. This balancing act is formalized as the meta-optimization prescription problem
\begin{equation}
  \label{eq:optimal-prescriptor}
  \begin{array}{ll}
    \displaystyle\mathop{\text{minimize}}_{(\hat c, \hat x)\in \hat{\mathcal{X}}}{}_{\preceq_{\hat{\mathcal{X}}}} & (\hat c, \hat x) \\
    \text{subject to} & \displaystyle \limsup_{T\to\infty} \frac{1}{a_T} \log \Pb^\infty \left( c(\hat{x}_T(\hat{\Pb}_T), \Pb) > \hat c^{\star}(\hat{\Pb}_T,T) \right)\leq -1, \quad \forall \Pb \in \mathcal{P}.
  \end{array}
\end{equation}

We define the notions of feasibility, weak optimality and strong optimality similarly as in the prediction problem. 
We note that feasible solutions of the optimal prediction problem \eqref{eq:optimal-predictor} are not necessarily feasible in the optimal prescription problem \eqref{eq:optimal-prescriptor} as the prediction out-of-sample guarantee does not directly imply the prescription out-of-sample guarantee on the prescribed solution, as discussed earlier.
A key question is whether the prescription problem, where the guarantees are only enforced at the prescribed solution, admits ``better'' solutions for the purpose of prescription than the solutions of the prediction problem. More precisely, does requiring out-of-sample guarantees only at the prescribed decision instead of every decision provides predictors with better prescribed decisions?

{\color{black}
Finally, we point out that studying the prescription problem is technically more challenging than the prediction problem. Indeed, the prediction problem involves studying estimators $(\hat{c}(x,\hat{\Pb}_T,T))_{T\geq 1}$ of the unknown ``constant''  $c(x,\Pb)$ for all fixed $x$. However, the prescription problem involves studying estimators of the form $(\hat{c}(\hat{x}_T(\hat{\Pb}_T),\hat{\Pb}_T,T))_{T\geq 1}$ of the unknown random quantity $\hat{c}(\hat{x}_T(\hat{\Pb}_T),\Pb)$. In the prescription problem, the decision $\hat{x}_T(\hat{\Pb}_T)$ is itself, unlike in the prediction problem, stochastic.}

{\color{black}
\begin{remark}[Oracles]
    A first natural candidate when considering optimal data-driven decision-making formulations is an oracle returning the constant equal to the optimal cost and optimal solution, 
$\hat{c}(x,\hat{\Pb}_T,T) := c(x,\Pb)$ and
$\hat{x}_T(\hat{\Pb}_T) = x^\star(\Pb) \in \argmin_{x\in \cX} c(x,\Pb)$. These are clearly ``optimal'' by any reasonable metric. However, note that the required guarantee in \eqref{eq:optimal-prescriptor} (and also in \eqref{eq:optimal-predictor}) needs to be verified \textit{for all} $\Pb \in \cP^\circ$. Hence, while an ``oracle'' predictor committing to a constant value would be best for the specific distribution $\Pb$ considered, it would fail to verify the out-of-sample guarantee or provide satisfactory accuracy for other out-of-sample distributions $\Pb' \in \cP$. Hence, a feasible formulation must be able to adapt to whatever out-of-sample distribution $\Pb' \in \cP$ using only the observed samples from this distribution $\hat{\Pb}_T$|as expected in real applications. Feasible predictors are subject therefore to the fundamental limitation of the data and its randomness.
\end{remark}
}

\subsection{Summary of Results}
\label{sec:summary-results}

{\color{black}{Investigating the power statistical procedures such as hyptothesis tests or parameter estimators in different statistical regimes is a historically widely established practice. Indeed, the statistical power of a hypothesis tests is often characterzied by the exponential rate with which its type I and type II errors decay to zero in the asymptotic regime $T\to\infty$ as more data becomes available; see for instance \cite[Chapter 7.1]{zeitouni1998large} and references therein.
As another example, the study of the statistical power of a statistical estimator $\hat \theta_T$ which tries to determine an unknown parameter $\theta$ based on $T$ data points in the exponential regime is well established. Indeed, estimators which achieve the fastest exponential worst-case
decay rate
\(
\limsup_{T\to\infty}\frac{1}{T}\log \Pb^\infty ( \| \theta-\hat \theta_T\| > \epsilon )
\)
over all $\Pb$ for some $\epsilon>0$ have historically been denoted Bahadur optimal \cite{bahadur1967rates}.
Pitman optimality on the other hand is concerned with minimizing the variance of $T^{1/2} \| \theta-\hat \theta_T\|$.
For such estimators $\limsup_{T\to\infty} \Pb^\infty ( \|\theta-\hat \theta_T\| > \epsilon_T)<\alpha$ for some $\alpha\in (0, 1)$ if the error tolerance decays as $\epsilon_T=\Theta(T^{-1/2})$.
In the context of our proposed framework the exponential regime of Bahadur corresponds to $a_T=\Theta(T)$ while that of Pitman is associated with $a_T = \Theta(1)$. Solutions to the optimal prediction problem \eqref{eq:optimal-predictor} and the optimal prescription problem \eqref{eq:optimal-prescriptor} will come to depend on the choice of the strength of the imposed out-of-sample performance via the choice of the sequence $(a_T)_{T\geq1}$. The larger $(a_T)_{T\geq1}$, the stronger the required guarantees and consequently the more conservative the feasible predictors and prescriptors.
While previous work such as \cite{van2020data, duchi2021statistics, lam2019recovering} has considered some specific regimes $(a_T)_{T\geq1}$ as we will discuss next, there is at the moment no unified treatment of the entire spectrum of possible asymptotic regimes associated with a generic sequence $(a_T)_{T\geq1}$.
}
{\color{black}{Perhaps surprisingly, we show that such unified treatment is possible and that in each regime there exists strong optimal solutions which we can exhibit. To prove this result we will assume in this paper that the set of events $\Sigma$ has finite cardinality. Extending our results to the more general case of continuous event sets, along the lines of \cite{van2020data}, is a possibility but would require a vastly more technical exposition. We believe that the results in the discrete event case which we present momentarily are sufficiently interesting to warrant this slightly restrictive assumption, and bear all the desired insights.
}

The \textit{exponential} regime in which strong out-of-sample guarantees are imposed with $a_T/T \to r$ ($a_T \sim rT$) for $r>0$ has been studied {\color{black}in a slightly different setting in \cite{van2020data} and \cite{sutter2020general}}. Here the decision-maker desires an exponentially decreasing probability of out-of-sample disappointment. One can show that in our setting in line with what was shown by \cite{van2020data} and \cite{sutter2020general} that the DRO predictor with Kullback-Leibler (KL) ball 
\begin{equation}\label{eq: KL predictor intro}
\hat{c}(x,\hat{\Pb}_T,T) = \sup_{\Pb'\in \Pb} \set{c(x,\Pb')}{\sum_{i\in \Sigma}\hat{\Pb}_T(i) \log \left( \frac{\hat{\Pb}_T(i)}{\Pb'(i)} \right)\leq r}, \quad \forall x \in \cX, \; \forall T \in \integ,
\end{equation}
is strongly optimal for both the optimal prediction and prescription problem. 
{\color{black}{The optimal predictor $\hat{c}$ is in this regime not consistent as more data is observed. Indeed, the considered ambiguity set does not depend on the data size and hence can not shrink as more data is observed.} 
We will study here a wider spectrum of out-of-sample guarantees: the \textit{superexponential} regime in which yet stronger out-of-sample guarantees are imposed, i.e., $a_T/T \to \infty$ ($a_T \gg T$), the \textit{exponential} regime $a_T/T \to r$ ($a_T \sim rT$), as well as the \textit{subexponential} regime where moderate out-of-sample guarantees suffice, i.e., $a_T/T \to 0$, with $a_T \to \infty$ ($a_T \ll T$). We explicitly construct in all three discussed regimes a strongly optimal solution in both the optimal prediction and prescription problems (\ref{eq:optimal-predictor}) and (\ref{eq:optimal-prescriptor}). 

The superexponential regime {\color{black}{is somewhat degenerate as we can prove} that it is necessary to guard against all outcomes no matter the observed data. The robust predictor 
\begin{equation*}
    \hat{c}(x,\hat{\Pb}_T,T) = \sup_{\Pb' \in \cP} c(x,\Pb') \quad \forall x \in \cX,\; \forall T \in \integ,
\end{equation*}
is hence strong optimal in both the optimal prediction and the prescription problems.
In the exponential regime, $a_T\sim rT$, $r>0$, we extend results in \cite{van2020data} to our more general setting and show that the DRO formulation \eqref{eq: KL predictor intro} with KL ambiguity set is strong optimal in both the optimal prediction and prescription problems even when allowed to explicitly depend on the data size.
This implies that consistency is impossible in both the exponential and superexponential regime, and imposing such out-of-sample guarantees yields necessarily rather conservative predictors.

Although perhaps similar at first glance to the other two regimes, the subexponential, $a_T\ll T$, regime is more sophisticated and requires a much finer analysis. {\color{black}The subexponential regime is of interest in problems in which a decision needs to be made for every time $T$ and performance guarantees aggregated over time are of interest. For instance, \cite{garivier2011kl} considers in a bandit setting the regime $a_T=b\log(T)$ for $b>1$.
Suppose indeed that we have the guarantee
\[
 \limsup_{T\to\infty}\Pb^\infty \left( c(x, \Pb) > \hat c(x, \hat{\Pb}_T,T)\right) \leq 1/T^b+o(1/T^b)
\]
with $b>1$ which corresponds to our subexponential regime with $a_T = b \log(T)$. When a decision is implemented for each time $T\geq 1$ then the number of times we disappoint remains bounded even as $T\to\infty$. Indeed, we have that
\(
\Eb_{\Pb^\infty}(\textstyle\sum_{t=1}^T \mb 1\{c(x, \Pb) > \hat c(x, \hat{\Pb}_t,t)]\})\leq \textstyle\sum_{t=1}^T \Pb^\infty(\mb 1\{c(x, \Pb) > \hat c(x, \hat{\Pb}_t,t)\}) < \infty
\)
is here uniformly bounded for all $T\geq 1$. The regime has been considerably less studied and its analysis is the main focus our paper.
}
In this regime we show that consistent predictors become possible and we prove that the SVP formulation of \cite{maurer2009empirical} with predictor
\begin{equation*}
  \hat{c}(x,\hat{\Pb}_T,T)
=
    c (x,\hat{\Pb}_T)
    + 
    \sqrt{\frac{2a_T}{T}\Var_{\hat{\Pb}_T}(\loss(x, \xi))},
    \quad \forall x \in \cX,\; \forall T \in \integ,
\end{equation*}
is strong optimal in both the optimal prediction and prescription problems.
An interesting insight is that this ambiguity set can be identified with the second order term in the Taylor expansion of the KL-divergence ambiguity set of \eqref{eq: KL predictor intro} when $ \Pb' \approx \hat{\Pb}_T$ and  $r = a_T/T$. 
    
We hence show the existence of three distinct out-of-sample performance regimes between which the nature of the optimal data-driven formulation, for the purpose of decision-making and machine learning, experiences a phase transition.
The optimal data-driven formulations can be interpreted as a classical robust formulation in the superexponential regime, a KL distributionally robust formulation in the exponential regime and finally a variance penalized formulation in the subexponential regime; see also Table \ref{tab:summary all}.
This final observation unveils a surprising connection between these three, at first glance seemingly unrelated, data-driven formulations which until now remained hidden.

\renewcommand\arraystretch{2.5}
\begin{table}[ht]
\caption{Summary of the optimal predictors for the three out-of-sample guarantee regimes. }
    \label{tab:summary all}
\begin{center}
\begin{tabular}{|c||c|c|c|} 
\hline
  OOS Guarantee 
& $a_T \gg T$
& $a_T \sim rT$
& $a_T \ll T$  \\
 \hline
 \hline
  \setlength\extrarowheight{-3pt}
    \begin{tabular}{c}
    Optimal\\Ambiguity Set
    \end{tabular}
& $\cP$
& $\Pb' : \displaystyle\sum_{i\in \Sigma} \hat{\Pb}_T(i) \log \left( \frac{\hat{\Pb}_T(i)}{\Pb'(i)} \right)\leq r$
& $\Pb' :  \displaystyle \sum_{i\in \Sigma} \frac{(\Pb'(i)-\hat{\Pb}_T(i))^2}{2\hat{\Pb}_T(i)}
    \leq
 \frac{a_T}{T}$ \\
\hline
Consistency
& No
& No
& Yes \\
  \hline
\end{tabular}
\end{center}
\end{table}


{\color{black}
\paragraph{Limitations:}
We conclude our results section by listing some limitations of our framework. 
The first limitation is in the asymptotic nature of our meta-optimization problem. The considered out-of-sample guarantees and order comparison hold at best in a large sample regime $T\to\infty$. Our work therefore only indicates the form of predictors and prescriptors that enjoy better guarantees as more data becomes available rather than provide an exact formula to use for a finite data size $T$. In a practical problem, our recommendation is to use such the predictors given in Table \ref{tab:summary all}, but to choose the radius and penalization parameter based on finite sample guarantees which exist for all our recommended formulations (see Proposition \ref{prop: predictor finite sample guarantees.} and \cite[Theorem 5]{van2020data}). Our results here indicate that at least in the large sample regime, these finite sample guarantees can not be improved by either more careful analysis nor by considering better predictors.
}

{\color{black}
A second limitation relates to the nature of our considered data-driven decision-making setting. Here, we consider that the decision-maker has no prior on the out-of-sample distribution of the problem at hand, and has only access to data $\hat{\Pb}_T$. In this setting, any distribution $\Pb$ in the probability simplex $\cP$ is a plausible candidate. However, in some settings, the problem warrants some special distributional structure or the decision-maker can make reasonable structural assumptions limiting the possible out-of-sample distributions. Our framework does not take such additional information into account and our resulting optimal predictors may hence be conservative in such settings.
}

\section{Related Work}
\subsection{Optimal Decision-Making}\label{sec: closely related work}
{\color{black}
The study of optimal estimation has a long and distinguished history in statistics \cite{lehmann2006theory}. However, literature on ``optimization aware'' optimal estimation is considerably shorter.
In this section, we review in details closely related work to our problem.
\paragraph{On Optimal Data-Driven Decision-Making Formulations.} 
It can be remarked that the assumption $a_T\to\infty$ we impose even in the subexponential setting as discussed in Section \ref{sec:summary-results} excludes the regime $a_T=\Theta(1)$ asssociated with the study of data-driven formulations with the guaranteee
\begin{equation}
\label{eq:coverage_guarantee}
\limsup_{T\to\infty}\Pb^\infty \left( c(x, \Pb) > \hat c(x, \hat{\Pb}_T,T)\right) \leq \alpha
\end{equation}
for some $\alpha\in (0,1)$. However, this regime is already studied extensively by both \cite{lam2019recovering} and \cite{duchi2021statistics}.
In these works predictors are regarded as statistically ``best'' when they satisfy the out-of-sample guarantee \eqref{eq:coverage_guarantee} with equality.
\cite[Theorem 4]{duchi2021statistics} in particular shows that predictors based on a sufficiently regular $f$-divergence ball with robustness radius taken as $\chi^2_{1-\alpha}/(2T)$ where $\chi^2_{1-\alpha}$ denotes here the $(1-\alpha)$-quantile of the $\chi_1^2$ distribution indeed achieves equality in \eqref{eq:coverage_guarantee}.
  Our work differs from \cite{lam2019recovering,duchi2021statistics} in that our study is not limited to distributional robust predictors and in that our proposed meta-optimization problem provides a disciplined optimality notion. Indeed, even degenerate predictors can provide exact asymptotic out-of-sample guarantees.
  {Consider for example the following family of predictors 
  $$
  \hat{c}_{M,p,q}(x,\hat{\Pb}_T,T) := M \left(1-2 \mathbb{1}\left( T\hat{\Pb}_T(i_0) \; \text{mod} \; q \in \{0, \ldots, p-1\})\right)\right), \quad \forall x,T,
  $$
  where $M>0$ is a large constant, $\tfrac{p}{q}$ is any rational number, and $i_0 \in \Sigma$ is arbitrary.
  We show in Lemma \ref{lemma: exact cov counter example}, Appendix \ref{App: Partial order}, that for any rational $\alpha = p/q$, the predictor $\hat{c}_{M,p,q}$ satisfies (\ref{eq:coverage_guarantee}) with equality. Intuitively, our degenerate predictor $\hat{c}_{M,p,q}$ simply grossly overestimates the unknown out-of-sample cost with probability $1-p/q$ and grossly underestimates the unknown out-of-sample cost the with probability $p/q$. 
  Our meta-optimization problem discriminates between predictors with the help of an accuracy measure.
  At a technical level, while \cite{lam2019recovering} and \cite{duchi2021statistics} use empirical likelihood theory,
  our results rely on the mathematical machinery of large deviation theory \cite{zeitouni1998large} which has proven very useful also in several related fields such as control \cite{jongeneel2021topological, jongeneel2021efficient} and queuing theory \cite{puhalskii2007large}.
}

   A Bayesian perspective is taken by \cite{gupta2015near} to determine the smallest convex ambiguity sets that contain the unknown data-generating distribution with a prescribed level of confidence as the sample size increases. Both the Pearson divergence and KL ambiguity sets with properly scaled radii are optimal in this setting. Attention is restricted in \cite{gupta2015near} to the subclass of distributionally robust predictors with convex ambiguity sets whereas here a much richer class of predictors is considered.

Perhaps the work most closely related to ours is \cite{van2020data} in which a meta-optimization on predictors was first introduced. 
Our work is a substantial generalization of the framework of \cite{van2020data}. First, \cite{van2020data} consider the set of predictors that do not depend on the data size $T$. This is a considerable restriction eliminating several widely used formulations. This includes any regularization formulation with shrinking penalization term, and any distributional robust formulation with decreasing radius. Second,  \cite{van2020data} consider a different partial order, namely $\hat{c}_1 \preceq \hat{c}_2 \iff \hat{c}_1(x, \Pb') \leq \hat{c}_2(x, \Pb')$ for all $x,\Pb'$. When predictors are allowed to depend on the data size, a natural generalization of this order is $\hat{c}_1 \preceq \hat{c}_2 \iff \lim_{T\to\infty} \hat{c}_1(x, \Pb', T) \leq \lim_{T\to\infty} \hat{c}_2(x, \Pb', T)$ for all $x,\Pb'$. However, such order does not discriminate any predictors converging to the optimal cost as they all have the same limit. As \cite{van2020data} only studies the case $a_T = \Theta(T)$, there is no need to capture consistent predictors as they can not be feasible. However, outside this special case this becomes problematic and hence the proofs in \cite{van2020data} do not generalize to our setting, which requires finer analysis and the use of novel tools. We point out however that \cite{sutter2020general} study, within the same framework as \cite{van2020data}, more general data generating processes which may exhibit dependence over time such as Markov chains and auto-regressive models which do not capture in our work.

Finally, a minimax perspective on the optimality of data-driven formulation as in \cite{besbes2023big} can be taken as well. Here formulations are compared by a worst performance metric across all distributions. In our setting, a minimax approach would suggest an accuracy order of the form
  \[
    (\hat{c}_1,\hat{x}_1) 
    \preceq'_{\hat{\mathcal{X}}} 
    (\hat{c}_2,\hat{x}_2) \iff
    \limsup_{T\to \infty} \frac{ \sup_{\Pb \in \cPin}|\hat{c}_1^{\star}(\Pb,T) - c^{\star}(\Pb)|}
    {\sup_{\Pb \in \cPin}|\hat{c}_2^{\star}(\Pb,T) - c^{\star}(\Pb)|} \leq 1.
  \]
  Our suggested accuracy notion provides however stronger optimality results as a strongly optimal solution in our meta-optimization problem \eqref{eq:optimal-prescriptor} has uniformally better accuracy in our case across all distributions simultaneously.
  However, we do note that our results are asymptotic in sharp contrast to most minimax approaches that accomodate finite sample optimality results directly.

}

\subsection{Robust Formulations}\label{sec: robust formulations}
 SAA and ERM have been well documented to yield unreliable decisions which may suffer disappointing out-of-sample performance. To provide robust approaches, two alternatives to simple ERM/SAA have sparked interest in the past years: regularized formulations and distributionally robust formulations. Typically, these approaches attempt at designing predictors satisfying out-of-sample guarantees of the form \eqref{eq: out-of-sample ganrantee}. In particular, all of these formulations are possible solutions of our meta-optimization problem \eqref{eq:optimal-predictor}, and our work seeks to determine which of these formulations has better statistical properties for the purpose of decision-making.
\paragraph{Regularized Formulations.}
The use of regularization to improve prediction is well established and was indeed studied already by \cite{tikhonov1943stability} in the context of ill-posed inverse problems. Likewise, regularization for the benefit of prescription was proposed by \cite{mulvey1995robust} and has been tremendously influential. Decision formulations and their associated prescriptions can be guarded against overfitting by considering regularized cost estimators. 
Such regularized formulations estimate the cost of any decision as
\begin{equation*} 
    \hat{c}(x,\hat{\Pb}_T,T) =  \Eb_{\hat{\Pb}_T}(\ell(x, \xi)) + R(x, \hat{\Pb}_T, T),
\end{equation*}
where the term $R(x, \hat{\Pb}_T, T)$ is referred to as the regularization term.
For instance in the context of machine learning (Example \ref{exp: Decision-making ML}), regularized ERM formulations typically take the form
$
\frac{1}{T} \sum_{t=1}^T L(f_{\theta}(X_t),Y_t) + \lambda_T \mu(\theta)
$
where $\lambda_T \in \mathbf{R}_+$ and $\mu: \Theta \rightarrow \mathbf{R}_+$. Identifying $\mu(\theta)$ for instance as $\|\theta\|_0$, $\|\theta\|_1$, $\|\theta\|_2$ yields respectively sparse regression, LASSO \cite{tibshirani1996regression}, and Ridge \cite{hoerl1970ridge1, hoerl1970ridge2}.
These methods have been shown to enjoy strong out-of-sample performance both in theory \cite{li1986asymptotic, koltchinskii2011nuclear, van2016estimation} as well as in practice.
One particular example of regularization, which will come to play a protagonist role in this paper, is the \textit{Sample Variance Penalization} (SVP) formulation which considers regularized estimates of the form
\begin{equation}\label{eq: exp SVP}
     \Eb_{\hat{\Pb}_T}(\ell(x, \xi)) + \lambda_T \sqrt{\Var_{\hat{\Pb}_T}(\loss(x, \xi))},
\end{equation}
where $\Var_{\hat{\Pb}_T}(\loss(x, \xi))$ denotes the empirical variance of the loss. Although the term SVP was coined by \cite{maurer2009empirical} in the context of machine learning, the benefits of variance regularization were documented much earlier by \cite{Markowitz1952protfolio} in the context of portfolio selection. This regularized formulation is biased towards decisions for which the empirical variance of the cost is low rather than high. In particular, \cite{maurer2009empirical} show that SVP enjoys out-of-sample performance guarantees of the form
\begin{equation}\label{eq: bound for SVP}
    \Pb^{\infty}\left(
        \Eb_{\Pb}(\ell(x, \xi))
        >
        \Eb_{\hat{\Pb}_T}(\ell(x, \xi)) + \sqrt{\frac{2a}{T}} \sqrt{\Var_{\hat{\Pb}_T}(\loss(x, \xi))}
        +  C\frac{a}{T}
    \right)
    \leq 
    2e^{-a},
\end{equation}
for all $a>0$ and $T\in \integ$ with $C>0$ is a constant. Notice that the parameter $a$ controls here the level of out-of-sample performance we can expect the associated SVP estimator to enjoy.
\paragraph{Distributionally Robust Formulations.}
The use of robust formulations to guard against overfitting and to guarantee out-of-sample performance was popularized by \cite{ben2009robust}. During the last decade in particular, Distributionally Robust Optimization (DRO) formulations have witnessed a surge in popularity.
Such formulations consider an ambiguity set $\mathcal{U}_T(\hat{\Pb}_T)\subseteq \cP$ around the empirical distribution $\hat{\Pb}_T$ and protect against overfitting by considering decisions which minimize the worst-case cost over all considered probability distributions in $\mathcal{U}_T(\hat{\Pb}_T)$ instead of merely the empirical one.
That is, they estimate the cost of any decision as
\[
\hat{c}(x,\hat{\Pb}_T,T) =
 \sup_{\Pb' \in \cP} \set{\Eb_{\Pb'}(\loss(x,\xi))}{\Pb' \in \mathcal{U}_T(\hat{\Pb}_T)}.
\]
Evidently, the particular ambiguity set $\mathcal{U}_T(\hat{\Pb}_T)\subseteq \cP$ considered will ultimately determine the statistical properties and computational challenges of the associated robust data-driven formulation.
Much of the early literature \cite{delage2010distributionally, wiesemann2014distributionally, van2016generalized} focused on ambiguity sets consisting of probability measures sharing certain given moments and shape constraints.
More recent approaches \cite{bertsimas2018data} however consider ambiguity sets which are based on a statistical distance  instead.
For instance, \cite{kuhn2019wasserstein, gao2016distributionally} propose to consider ambiguity sets consisting of all distributions at distance at most $r$ to the empirical distribution in the Wasserstein metric. Alternatively, ambiguity sets consisting of all distributions at distance at most $r$ to the empirical distribution as measured by a divergence metric such as the Kullback-Leibler divergence are also well studied \cite{lam2019recovering, duchi2021statistics, gotoh2021calibration}.
 More recently, \cite{bennouna2022holistic} finally propose ambiguity sets which can be interpreted as combining the KL and Levy-Prokhorov distances.
The recent uptick in popularity of such robust formulations is in no small part due to the fact that they are often tractable and enjoy superior statistical guarantees.
\cite[Theorem 3.5]{kuhn2019wasserstein} indicate for instance that the probability that the estimated cost of the Wasserstein DRO formulation does not upper bound the actual unknown cost decays as
\begin{equation}\label{eq: Wasserstein bound}
  \Pb^\infty \left( 
 \Eb_{\Pb}(\loss(x,\xi)) 
  >
  \sup_{\Pb' \in \cP} \set{\Eb_{\Pb'}(\loss(x,\xi))}{\Pb' \in \mathcal{U}_T(\hat{\Pb}_T)} 
  \right)\leq
    C_1 e^{-C_2 T r^{\max(2,\dim\Sigma)}}
\end{equation}
when $r<1$, where $C_1$, $C_2$ are positive constants. 
This guarantee of the form \eqref{eq: out-of-sample ganrantee} shows that DRO approaches such a Wasserstein are also feasible in our optimization problem \eqref{eq:optimal-predictor} with appropriately chosen ambiguity set and radius. In the DRO literature, an important natural question is which choice of ambiguity set enjoys ``best'' such statistical guarantees, which is a main focus of our paper.

{\color{black}
Regularization and robustness appear at first glance to be two distinct ideas with which to encourage a nominal formulation to enjoy better out-of-sample guarantees. However, intimate connections between both ideas have been reported before by \cite{xu2009robustness}.
Furthermore, \cite{gao2017wasserstein} state an equivalence between Wasserstein DRO formulations and a certain gradient-norm regularization formulation. 
The use of $f$-divergence robust optimization as a convex proxy for variance regularization is discussed in \cite{gotoh2018robust, gotoh2021calibration, duchi2016variance, lam2016robust}.
Finally, from the previous discussion it is clear that adding robustness reduces to regularization with
\[
  R(x, \hat{\Pb}_T, T) =  \sup_{\Pb' \in \cP} \set{\Eb_{\Pb'}(\loss(x,\xi))}{\Pb' \in \mathcal{U}_T(\hat{\Pb}_T)} - c(x, \hat{\Pb}_T) \geq 0.
\]
While the addition of robustness or equivalently regularization provide out-of-sample guarantees, it is unclear what amount of regularization or robustness is necessary. More generally,
given multiple formulations which provide out-of-sample guarantees, a natural question is whether any such formulation should be preferred over all others.
Our paper addresses the question of which formulation enjoyed the ``best'' statistical properties for the purpose of data-driven decision-making.
}

\section{Optimal Data-Driven Prediction}\label{sec: predictor}

We study in this section the problem of optimal prediction as formally stated in problem \eqref{eq:optimal-predictor}. We first indicate with the help of a small counterexample that among the set of all possible predictors as defined in Definition \ref{def: predictors.} there exists some pathological predictors, of no practical significance, which might prevent the existence of optimal solutions.

Consider the example illustrated in Figure \ref{fig: counter example 1} where we assume the existence of an optimal predictor $\hat{c}$.
We can derive another predictor $\hat c'$ by subtracting a spike function that vanishes to $0$ on $\{\frac{0}{T}, \frac{1}{T}, \ldots, \frac{T}{T}\}^d$. For all practical purposes hence both predictors are the same as both predict the same cost whatever the empirical distribution $\hat \Pb_T$. In particular, $\hat{c}'$ still verifies the out-of-sample guarantee as it only concerns its values in $\hat{\Pb}_T$. However, the predictor $\hat c'$ can be shown to be a strictly better predictor as quantified by the partial order $\orderleq$.
Such discussed pathological predictors are possible as there is no explicit restriction on the smoothness of the predictors. For instance, in the example of Figure \ref{fig: counter example 1}, the spike that needs to be added has variations that explode with large $T$.
To avoid such pathological predictors, we will consider exclusively regular predictors.

\begin{definition}[Regular Predictors]\label{def: regular pred}
The set of regular predictors $\cC$ is the set of predictors verifying the following two properties\footnote{We recall the definition of these properties in Definition \ref{def: unif boundedness} and \ref{def: Equicontinuity}.}:
\begin{enumerate}
\item The sequence of functions $(\hat{c}(\cdot,\cdot,T))_{T\geq 1}$ is uniformly bounded and equicontinuous.
\item The function $\hat{c}(\cdot,\cdot,T)$ is differentiable in the second argument $\Pb$ for all $T \in \mathbf{N}$. Furthermore, the sequence of its derivatives is uniformly bounded and equicontinuous.
\end{enumerate}
\end{definition}

The differentiability of regular predictors in $\Pb$ is without much loss of generality as the predictor is in essence characterized completely by the values it takes on the discrete support $\{\frac{0}{T},\frac{1}{T}, \ldots,\frac{T}{T}\}^d$. Any non-differentiable predictor can therefore always be substituted by an equivalent predictor that is differentiable by smoothly extending its values from $\{\frac{0}{T},\frac{1}{T}, \ldots,\frac{T}{T}\}^d$ to $\cP$.
The fact that regular predictors should be bounded is also natural in this context as we assume here that the unknown actual cost $c(x, \Pb) < \infty$ remains bounded. The equicontinuity and the uniform boundedness of derivatives condition is precisely what allows us to avoid the pathological asymptotic behaviors of sequences of functions constituting predictors. We will show the existence
of an optimal predictor among the class of regular predictors in each of the three regimes discussed in Section \ref{sec:summary-results}.

\begin{figure}
  \centering
  \begin{minipage}{\textwidth}
    \centering
    \begin{tikzpicture}[scale=0.9]
  \begin{axis}[
    axis x line*=bottom, 
    axis y line*=left, 
    xmin = -2,
    xmax = 2.8,
    ymin = 0.1,
    ymax = 0.9, 
    xticklabels={,,,$\frac{k}{T}$,$\frac{k+1}{T}$,,},
    ytick=\empty,
    yticklabels={,,,$\hat{c}$,,,},
    axis lines = left,
    xlabel=$\Pb(i)$, xlabel style={at={(1,0)}, anchor=north}
    ]


    \plot[color= green!40!black ,thick, opacity=1] plot[domain=0.26:0.84,samples=300, smooth]
    expression{
      0.07*sin(50*x)+0.5 + 0.07*sin(100*x)
      -
      1500*exp(-1/((0.35^2-(x-0.5)^2)))
    };
    
    \plot[color= blue ,thick, opacity=1] 
    plot[samples=300, smooth] expression{0.07*sin(50*x)+0.5 + 0.07*sin(100*x)};

    \plot[color= black ,thick, opacity=1] 
    plot[samples=300, smooth] expression{0.07*sin(30*x)+0.5 - 0.07*sin(60*x)-0.3275};
    
    \draw [<->,green!40!black, opacity=1] (0.2,0.37) -- (0.8,0.37);
    \draw[color=green!40!black, ] (1,0.37) node {$\frac{1}{2T}$};

    \foreach \x in {-1,0,1,2}
    {
      \edef\temp{\noexpand
        \draw [dotted,color = black] (\x,0) -- (\x,0.8)
        ;}
      \temp
    }
    \draw[color=black] (0,0.87) node {$\hat{\Pb}_T(i)$};

    \draw[color=blue] (-1.8,0.48) node {$\hat{c}$};

    \draw[color=black] (-1.8,0.22) node {$c$};
    
    \draw[color=green!40!black] (0.8,0.2) node {$\hat{c}'$};
  \end{axis}
\end{tikzpicture}

    \caption{Illustration of the construction of a pathological predictor dominating a given predictor. The {\color{black}blue} line represents the regular predictor $\hat{c}$, while the {\color{black}green} line represents the perturbed predictor into a pathological predictor $\hat{c}'$.
    Here $i\in \Sigma$, $k\in \{0,\ldots,T\}$ is an integer, and the pointed vertical lines represent the possible values of the empirical distribution $\hat{\Pb}_T(i) \in \{\frac{0}{T}, \ldots, \frac{T}{T}\}$.}
  \end{minipage}
  \label{fig: counter example 1}
\end{figure}
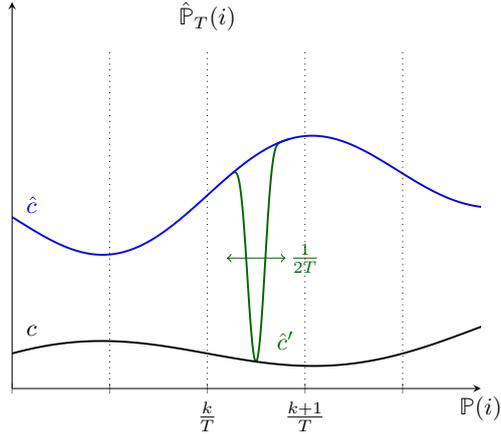

\subsection{The Exponential Regime ($a_T\sim rT$)}\label{sec: pred exp}

We say that the prediction problem is in the exponential regime when the desired out-of-sample disappointment speed satisfies $a_T\sim rT$, i.e., $\lim_{T\to\infty} a_T/T = r > 0$.
That is, admissible predictors can suffer an out-of-sample disappointment probability as defined in Equation \eqref{eq: out-of-sample proba} which decays exponentially at rate $r$ with increasing amount $T$ of observed data points.

In this exponential regime the appropriate distance notion between distributions seems to be the relative entropy sometimes also better known as the KL-divergence.
The relative entropy of a distribution $\Pb\in\cP$ with respect to a distribution $\Pb'\in\cP$ is defined as
\begin{align}\label{eq: KL divergence.}
  \D{\Pb}{\Pb'} = \sum_{i\in\Sigma} \Pb(i) \log\left(\frac{\Pb(i)}{\Pb'(i)}\right),
\end{align}
where we use the conventions $0 \log(0/p)=0$ for any $p\geq 0$ and $p' \log(p'/0)=\infty$ for any $p'> 0$. We can define an associated DRO predictor as
\begin{equation}\label{eq: KL predictor}
  \cKL(x,\Pb,T) 
  =
  \sup_{\Pb' \in \cP} \set{c(x,\Pb')}{\D{\Pb} {\Pb'}\leq r},
  \quad \forall x \in \cX, \; \forall \Pb \in \cP, \; \forall T \in \integ.
\end{equation}

\cite{van2020data} have shown indeed that in the exponential regime this DRO predictor should be preferred to any other predictor which does not {\color{black}explicitly} depend on data size $T$. We generalize this result and show that in fact it should be preferred to any other regular predictor, {\color{black}even when allowed to scale with the data size}. The key component of the proofs establishing this claim is the large deviation property of the empirical distribution $\hat{\Pb}_T$ which we review in Appendix \ref{Appendix: LDT} for completeness.

We first prove that the DRO predictor is in fact regular and enjoys our imposed out-of-sample guarantee.
The full proof of the subsequent result is deferred to Appendix \ref{App: proof feasibility pred exp}.

\begin{proposition}[Feasibility]\label{prop: frasibility pred exp}
  The predictor $\cKL\in \cC$ verifies the out-of-sample guarantee \eqref{eq: out-of-sample ganrantee} when $a_T \sim rT$.
\end{proposition}
\begin{proof}[Sketch of proof]
Let $x,\Pb \in \cX\times \cPin$. Denoting $\Gamma = \set{\Pb' \in \mathcal{P}}{I(\Pb',\Pb) > r}$, we have by definition of $\cKL$, for all $T\in \integ$
  \[
    \hat{\Pb}_T \not \in \Gamma \implies c(x,\Pb) \leq \cKL(x,\hat{\Pb}_T,T).
  \]
Hence, using the Large Deviation Principle, Theorem \ref{thm: LDP finite space}, we get  
  \begin{align*}
    \limsup_{T\to\infty}
    \frac{1}{a_T} \log \Pb^{\infty}
        \left(
                c(x,\Pb) > \cKL(x,\hat{\Pb}_T,T)
        \right)
    \leq& 
        \limsup_{T\to\infty}
    \frac{1}{rT} \log \Pb^{\infty}
    \left(
      \hat{\Pb}_T \in \Gamma
          \right)
    \leq -\frac{1}{r}\inf_{\Pb'\in \bar{\Gamma}}
             I(\Pb',\Pb).
  \end{align*}
 Finally, using the convexity of the relative entropy $I(\cdot,\cdot)$ we show that $\overline{\Gamma}  \subset \set{\Pb' \in \mathcal{P}}{I(\Pb',\Pb) \geq r}$ and therefore $-\inf_{\Pb'\in \bar{\Gamma}}
             I(\Pb',\Pb) \leq -r$.
             
To prove the regularity of $\cKL$, we show its differentiability by rewriting $\cKL$ in a dual form as 
  \begin{equation*}
    \cKL(x,\mb P, T)= \min_{\alpha\geq \max_{i\in \Sigma} \ell(x, i)} ~\{f(\alpha; x,\Pb, T)\defn \alpha - e^{-r} \exp(\textstyle\sum_{i\in \Sigma} \log \left(\alpha-\loss(x, i) \right) {\mb P(i)})\}
  \end{equation*}
and subsequently use the implicit function theorem. The equicontinuity  property holds then simply as the predictor does not depend on $T$.
\end{proof}

We now show that the proposed predictor is strongly optimal in the optimal predictor Problem \eqref{eq:optimal-predictor} and consequently should be preferred to any other regular predictor in the exponential regime. {\color{black}The proof is in Appendix \ref{Appendix: Claims of exp pred thm}.}

\begin{theorem}[Strong Optimality]\label{thm: strong opt exponential regime predictor.}
Consider the exponential regime in which $a_T\sim rT$. The predictor $\cKL$ is feasible in the prediction problem \eqref{eq:optimal-predictor} and
for any predictor $\hat{c} \in \cC$ satisfying the out-of-sample guarantee \eqref{eq: out-of-sample ganrantee}, we have $\cKL \orderleq \hat{c}$. That is, $\cKL$ is a strong optimal predictor in the exponential regime.
\end{theorem}

\begin{remark}
The proof of strong optimality in the exponential regime does not require differentiability and equicontinuity of the derivatives of the predictors. In fact, the optimality results in the exponential (and superexponential) regimes hold even when the considered set of predictors are predictors verifying only the first but not necessarily the second regularity condition in Definition \ref{def: regular pred}.
\end{remark}

Theorem \ref{thm: strong opt exponential regime predictor.} establishes that consistent estimator for which $\lim_{T\to\infty} \hat c(x, \Pb, T) = c(x, \Pb)$ are not compatible with the exponential regime. That is, even the optimal predictor in this regime is typically biased in that we may have $\lim_{T\to\infty} \cKL(x, \Pb, T) > c(x, \Pb)$. This is undesirable as we may hope to recover the unknown cost at least when an increasing amount of data becomes available \cite{bertsimas2018robust}. We can attribute this undesirable outcome by our insistence on imposing an exponentially decaying out-of-sample disappointment as we will show later. We first show however that imposing even stronger out-of-sample guarantees -- perhaps unsurprisingly -- does not alleviate this issue.

\subsection{The Superexponential Regime ($a_T \gg T$)}\label{sec: superexp pred}

We consider now the superexponential regime, in which the desired out-of-sample guarantee speed is stronger than exponential, $a_T\gg T$, i.e., $\lim_{T\to\infty} a_T/T = \infty$.
This implies that admissible predictors suffer an out-of-sample disappointment probability which decays faster than exponential in the number of samples $T$. We will show that, perhaps unsurprisingly, we can not escape fully conservative predictors with such strong guarantees.

Consider a robust predictor taking the worst case scenario of the uncertainty
\begin{equation}\label{eq: overly robust predictor}
  \hat{c}_{R}(x,\Pb,T) 
  =
  \max_{i\in \Sigma} \,\ell(x,i),
  \quad \forall x \in \cX, \; \forall \Pb \in \cP.
\end{equation}
The previous predictor can also be seen as a DRO predictor with the whole simplex $\cP$ as ambiguity set, $\hat{c}_{R}(x,\Pb,T) = \sup_{\Pb' \in \cP} c(x,\Pb')$ for all $x,\Pb$ and $T$.
We remark that this predictor does not actually use the observed data at all and only depends on the support of its potential outcomes instead. We first prove that this robust predictor is indeed regular and enjoys our imposed out-of-sample guarantee.
  
\begin{proposition}[Feasibility]\label{prop: feasibility or predictor}
The predictor $\cRob\in \cC$ verifies the out-of-sample guarantee \eqref{eq: out-of-sample ganrantee} when $a_T \gg T$.
\end{proposition}
\begin{proof}
The predictor $\cRob$ is constant in $\Pb$ and $T$. Therefore, the required regularity conditions follow immediately and $\cRob \in \cC$. Let us now verify the out-of-sample guarantee. Let $(x,\Pb) \in \cX \times \cP$. We have $\hat{c}_{R}(x,\Pb,T) = \sup_{\Pb' \in \cP} c(x,\Pb')$ for all $T\in \integ$, therefore, $\Pb^\infty ( c(x, \Pb) > \hat{c}_{R}(x, \hat{\Pb}_T,T)) =0$ for all $T\in \integ$.
\end{proof}

\begin{theorem}[Strong Optimality]\label{thm: overly strong predictor}
  Consider the superexponential regime in which $a_T \gg T$. The predictor $\cRob$ is feasible in the prediction problem \eqref{eq:optimal-predictor} and
for any predictor $\hat{c} \in \cC$ satisfying the out-of-sample guarantee \eqref{eq: out-of-sample ganrantee}, we have $\cRob \orderleq \hat{c}$. That is, $\cRob$ is a strong optimal predictor in the superexponential regime.
\end{theorem}
We present two proofs of this theorem. The first one, which we sketch next and can be found in Appendix \ref{Appendix: proof of thm superexp pred}, builds upon the result of Theorem \ref{thm: strong opt exponential regime predictor.}. The second one, in Appendix \ref{Appendix: alternative proof of strong opt superexp pred}, is a self contained proof that does not require Theorem \ref{thm: strong opt exponential regime predictor.}.
\begin{proof}[Sketch of Proof]
Proposition \ref{prop: feasibility or predictor} ensures the feasibility of $\cRob$. Let $\hat{c}\in \cC$ be a feasible predictor in Problem \eqref{eq:optimal-predictor} with $a_T\gg T$. Notice that $\hat{c}$ is also feasible for $a_T \sim rT$, for all $r>0$. In fact, verifying a guarantee with a given speed implies verifying all weaker guarantees. Hence Theorem \ref{thm: strong opt exponential regime predictor.} implies that $\cKL \preceq_{\cC} \hat{c}$, for $\cKL$ with an ambiguity set with any $r>0$. The ambiguity set of $\cKL$, $\{\Pb'\in \cP\; : \; \D{\Pb} {\Pb'}\leq r\}$ ``converges" with $r\to \infty$ to $\cP$ which is the ambiguity set of $\cRob$. Hence, intuitively, taking $r\to \infty$ in the inequality $\cKL \preceq_{\cC} \hat{c}$ leads to $\cRob \preceq_{\cC} \hat{c}$. In order to make this line of argument rigorous a slightly more refined approach is required.
\end{proof}

The optimality of the robust predictor shows that for a predictor to satisfy superexponential out-of-sample performance guarantees, it necessarily needs to hedge against all possible distributions of the uncertainty. That is, it needs to take into account the worst-case cost in all scenarios which is independent of the data observed.

The presented optimality results in the exponential and superexponential regimes reveal an interesting insight. In both regimes, the optimal predictors are not consistent. That is, they do not converge to the true cost with increasing data size $T$. This shows that when exponential or stronger guarantees are imposed, consistent data-driven formulations are not possible. In other words, predictors must necessarily remain conservatively biased even when the amount of available data is large.

\subsection{The subexponential Regime ($a_T \ll T$)}\label{sec: subexp pred}

We study now the subexponential regime in which the desired out-of-sample guarantee speed is slower than exponential, $a_T\ll T$, i.e., $\lim_{T\to\infty} a_T/T = 0$.
Admissible predictors may suffer an out-of-sample disappointment probability which decays to zero slower than exponential in the number of samples $T$. While consistent predictors were not possible in the previous exponential and superexponential regimes, we will observe a phase transition when moving into the subexponential regime. Because of the weaker requirements imposed on the out-of-sample performance, we will show that consistent predictors are not only a possibility but a necessity for optimality. The next result indeed indicates that, in the subexponential regime, any weakly optimal predictor in Problem \eqref{eq:optimal-predictor} must necessarily be consistent.

\begin{proposition}[Consistency of weakly optimal predictors]\label{prop: consitency}
Consider the subexponential regime in which $a_T\ll T$. Every weakly optimal predictor in the optimal prediction Problem \eqref{eq:optimal-predictor} is consistent. That is, for every predictor $\hat{c}\in \cC$ verifying the out-of-sample guarantee \eqref{eq: out-of-sample ganrantee}, either $(\hat{c}(\cdot,\cdot,T))_{T\geq1}$ converges point-wise to $c(\cdot,\cdot)$, i.e.,
$$
\lim_{T\to\infty} \hat{c}(x,\Pb,T) = c(x,\Pb), \quad \forall x\in \cX, \; \forall \Pb \in \cP,
$$
or there exists a predictor $\hat{c}'\in \cC$ verifying the out-of-sample guarantee that is strictly preferred to $\hat{c}$, i.e., $\hat{c}' \orderleq \hat{c}$ and $\hat{c}' \not\equiv \hat{c}$.
\end{proposition}

\begin{proof}[Sketch of Proof]
Suppose that there exists a predictor $\hat{c} \in \cC$ which satisfies the out-of-sample guarantee and does not converge point-wise to $c$.
We first show that non-consistency combined with equicontinuity (as $\hat{c} \in \cC$) implies that the predictor is larger than the true cost by a constant gap in an open ball for an infinite number of $T$: there exists $\epsilon>0$, a ball of center $(x_0,\Pb_0)$ and radius $\rho>0$, $\mathcal{B}((x_0,\Pb_0),\rho)$, and an infinite set $\mathcal{T} \subset \integ$ such that
\begin{equation*}
\forall T \in \mathcal{T}, \; 
\forall (x,\Pb) \in \mathcal{B}((x_0,\Pb_0),\rho), \quad
\hat{c}(x_0,\Pb_0,T) - c(x_0,\Pb_0) > \frac{\varepsilon}{4}
.
\end{equation*}
We next show, using the Moderate Deviation Principle, Theorem \ref{thm: MDP finite space}, that it suffices to have a constant gap with the true cost to verify subexponential guarantees. We seek, therefore, to perturb $\hat{c}$ in the ball $\mathcal{B}((x_0,\Pb_0),\rho)$ into a strictly preferred predictor while ensuring to maintain a constant gap with the true cost. We consider $\eta: \cX \times \cP \longrightarrow [0,\frac{\varepsilon}{8}]$ an infinitely differentiable bump function of support $\mathcal{B}((x_0,\Pb_0),\frac{\rho}{2})$ such that $\eta(x_0,\Pb_0) = \frac{\varepsilon}{8}$ and perturb $\hat{c}$ into $\hat{c}'$ defined as $\hat{c}'(x,\Pb,T) = \hat{c}(x,\Pb,T) - \eta(x,\Pb) \mathbf{1}_{T \in \mathcal{T}}$. Figure \ref{fig: proof of consistency} illustrates this construction. Note that it is crucial that $\eta$ is independent of $T$|and more precicely has variations not exploding with $T$|as otherwise $\hat{c}'$ might not preserve the necessary regularity of $\hat{c}$. Non-consistency, combined with equicontinuity provides a sufficient gap with the true cost $c$ to substract a bump function independent of $T$ while preserving a constant gap.
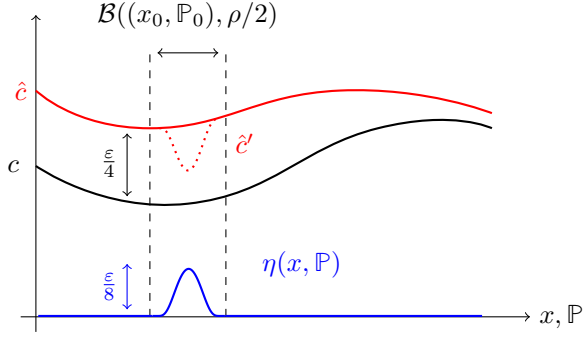
\begin{figure}
  \centering
  \begin{tikzpicture}[scale=1]

  \draw[->] (0.3,0) -- (7,0) node[right] {$x,\Pb$};
  \draw[->] (0.5,-0.2) -- (0.5,4);

  \draw[thick] plot [smooth,tension=0.9] coordinates{
    (0.5, 2.0)
    (2.5, 1.5)
    (5.0, 2.5)
    (6.5, 2.5)
  };
  \draw (0.5, 2.0) node[left] 
  {$c \;$};

  \draw[thick, color=red] plot [smooth,tension=0.9] coordinates{
    (0.5, 3)
    (2, 2.5)
    (4.5, 3)
    (6.5, 2.7)
  };
  \draw (0.5,3) node[left] 
  {${\color{red}\hat{c}}$};

  \draw (3, 2.3) node[right] 
  {${\color{red}\hat{c}'}$};

  \draw[dashed] (2,0) -- (2,3.5);
  \draw[dashed] (3,0) -- (3,3.5);
  \draw[<->] (2.1,3.5) -- (2.9,3.5);  
  \draw (2.5,4) node
  {$\mathcal{B}((x_0,\Pb_0),\rho/2)$};
  
  \draw[<->] (1.7,1.6) -- (1.7,2.43);  
  \draw (1.7,2.03) node[left]
  {$\frac{\varepsilon}{4}$};

  \draw[<->, color=blue] (1.7,0.1) -- (1.7,0.7);  
  \draw[color=blue] (1.7,0.4) node[left]
  {$\frac{\varepsilon}{8}$};

  \draw (4,0.7) node
  {${\color{blue}\eta(x,\Pb)}$};
  \begin{axis}[
    xmin = -0.1,
    xmax = 7,
    ymin = -0.01,
    ymax = 5, 
    xticklabels={,,},
    yticklabels={..},
    axis line style={draw=none},
    tick style={draw=none}
    ] 
    
    \plot[color= blue ,thick, opacity=1] plot[domain=2.005:2.995,samples=300, smooth]
    expression{30*exp(-1/(0.5^2-(x-2.5)^2))};
    \plot[color= blue ,thick]
    plot [smooth,tension=0.9] coordinates{
      (0.45,0)
      (2,0)};
    \plot[color= blue ,thick]
    plot [smooth,tension=0.9] coordinates{
      (3,0)
      (6.5,0)};
      
     \plot[color= red ,thick, dotted] plot[domain=2.1:2.9,samples=300, smooth]
    expression{
    2.19 + (x-2.1)*(2.305-2.19)/(2.9-2.1)
    - 30*exp(-1/(0.5^2-(x-2.5)^2))};
  \end{axis}

\end{tikzpicture}

  \caption{Illustration of the construction of $\hat{c}' \preceq_{\cC} \hat{c}$ and $\hat{c}' \not\equiv \hat{c}$ when $\hat{c}$ is not consistent.}
  \label{fig: proof of consistency}
\end{figure}
\end{proof}

\begin{remark}\label{rmk: consistency of estimator}
Notice that consistency in the sense of Proposition \ref{prop: consitency} implies consistency of the estimator $(\hat{c}(x,\hat{\Pb}_T,T))_{T\geq 1}$ of the true cost $c(x,\Pb) = \Eb_{\Pb}(\loss(x,\xi))$, for each $x\in \cX$ and $\Pb \in \cP$. In fact, point-wise convergence of $(\hat{c}(x,\cdot,T))_{T\geq 1}$ with the equicontinuity of $\hat{c}$ (as $\hat{c} \in \cC$) implies its uniform convergence. Combined with the almost sure convergence of $(\hat{\Pb}_T)_{T\geq 1}$ to $\Pb$, a consequence of the strong law large numbers, it implies almost sure convergence of $(\hat{c}(x,\hat{\Pb}_T,T))_{T\geq 1}$ to $c(x,\Pb)$.
\end{remark}

We will show that a strong optimal solution to the optimal prediction Problem \eqref{eq:optimal-predictor} exists also in the subexponential regime. Consider the \textit{sample variance penalization} (SVP) predictor defined as
\begin{equation}\label{eq: Robust predictor SVP.}
  \cSVP(x,{\mb P},T)
=
    c (x,\Pb)
    + 
    \sqrt{\frac{2a_T}{T}\Var_{\Pb}(\loss(x, \xi))},
    \quad \forall x \in \cX, \; \forall \Pb \in \cP,\; \forall T \in \integ,
\end{equation}
where $\Var_{\Pb}(\loss(x, \xi)) := \mathbb E_{\Pb}((\loss(x, \xi) - \mathbb E_{ \Pb}(\loss(x, \xi)) )^2)$ is the variance of the cost under distribution $\Pb$.
SVP predictors were considered previously by \cite{maurer2009empirical} as an alternative to naive ERM. Their consideration of SVP may be understood through the perspective of a classical ``bias-variance'' trade-off between minimizing an empirical risk and a variance-sensitive regularization term motivated by concentration inequalities such as the one found in Equation \eqref{eq: bound for SVP}. We will show that also in our framework the SVP predictor will play a protagonist role in the subexponential regime.

For all $\Pb \in \cP$ consider the local ellipsoid norm associated to $\Pb$
\begin{equation*}
    \|\Delta\|_{\Pb}^2 := \frac{1}{2} \sum_{i\in \Sigma} \frac{1}{\Pb(i)}\Delta_i^2, \quad \forall \Delta \in \Re^d.
\end{equation*}
Whereas the KL divergence $I(\Pb', \Pb)$ has been shown to be the right notion of distance in the exponential regime between distributions, we will indicate that the local ellipsoidal norm or $\chi^2$-divergence induces the right notion of distance $\Pb' \rightarrow \|\Pb'-\Pb\|_{\Pb}$ in the subexponential regime. We first indicate that the SVP predictor can alternatively be understood as a DRO predictor with ellipsoid uncertainty set. This alternative perspective will also shed light on key properties of SVP such as the tractability of its associated prescription problem which we discuss in Section \ref{sec: prescriptor subexp}.
We prove the next result in Appendix \ref{Appendix: robust-interpr} and show that it also holds in the general case of continuous distributions.

\begin{proposition}[DRO interpretation of SVP]\label{prop: robust predictor is DRO.}
For all $x, \Pb\in \mathcal{X}\times \cPin$ and $T \in \integ$ sufficiently large such that {\color{black}$\sqrt{2a_T/T} \leq \frac{\sqrt{\Var_{\Pb}(\loss(x,\xi))}}{|\loss(x,i) - c(x,\Pb)|}$} for all $i\in \Sigma$, the robust predictor \eqref{eq: Robust predictor SVP.} can be written as
\begin{equation}\label{eq: robust predictor DRO.}
  \cSVP(x,{\mb P},T) = \sup_{\Pb' \in \mathcal{P}}
  \left\{
    c(x,\Pb')\; : \; 
    \norm{\Pb'-\Pb}^2_\Pb
    \leq 
    \frac{a_T}{T}
  \right\},
\end{equation}
The supremum in \eqref{eq: robust predictor DRO.} is in fact attained, i.e.,
\[
\cSVP(x,{\mb P},T) = c(x, \Pb + \sqrt{2a_T/T}\varphi_x(\Pb)) \quad {\rm{for}} \quad \varphi_x(\Pb) := {(\loss(x,\cdot) \odot \Pb - c(x,\Pb)\Pb)}/{\sqrt{\Var_{\Pb}(\loss(x,\xi))}}.
\]
where $\loss(x,\cdot) = (\loss(x,1),\ldots,\loss(x,d))$, $\odot$ denotes the Hadamard product and when $\Var_{\Pb}(\loss(x,\xi))=0$, $\varphi_x(\Pb)$ is taken by convention as any vector in the boundary of\footnote{Take for example the vector $\sqrt{\frac{2\Pb(1)}{1-\Pb(1)}}(e_1 - \Pb)$ where $e_1= (1,0,\ldots,0)^\top$.} $\{\Delta \in \Re^d\; : \; e^\top \Delta =0, \; 2\|\Delta\|_{\Pb}^2\leq a_T/T\}$.
Moreover, $\varphi_x(\Pb)$ verifies the identities $\|\sqrt{2}\varphi_x(\Pb)\|_{\Pb} =1$ and $c(x,\varphi_x(\Pb)) = \sqrt{\Var_{\Pb}(\loss(x,\xi))}$ for all $\Pb \in \cPin$.
\end{proposition}
\begin{proof}[Sketch of Proof]
We first show that the proposed solution of the supremum has cost equal to the SVP formulation $\cSVP$. We then show that any feasible distribution on the supremum problem has cost lower than SVP. The condition on $a_T$ ensures that the proposed solution verifies $\Pb(i) + \sqrt{2a_T/T}\varphi_x(\Pb)(i) \geq 0$ for all $i$. As $\Pb + \sqrt{2a_T/T}\varphi_x(\Pb)$ sums to 1, this implies that $\Pb + \sqrt{2a_T/T}\varphi_x(\Pb) \in \cP$.
\end{proof}

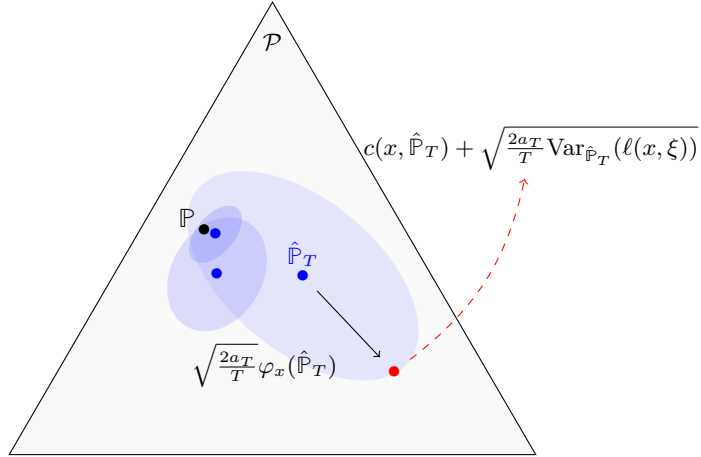
\begin{figure}
    \centering
    \begin{tikzpicture}
  \node[isosceles triangle,
  isosceles triangle apex angle=60,
  draw,
  fill=gray!5,
  rotate=90,
  minimum size =6cm] at (0,0){};
  \draw (0,3.5) node {\small$\cP$};
  
  \fill[blue!30,rotate around={-40:(0.4,0.4)}, opacity = 0.3]
  (0.4,0.4) ellipse (1.8 and 1);
  \node[blue] (PT) at (0.4,0.4) {\textbullet};
  \draw[blue] (0.4,0.4) node[above] {\small$\hat{\Pb}_{T}$};

  \fill[blue!40,rotate around={-120:(-0.73,0.42)}, opacity = 0.3]
  (-0.73,0.42) ellipse (0.8 and 0.6);
  \draw[blue] (-0.73,0.42) node {\textbullet};
  
  \fill[blue!50,rotate around={-130:(-0.75,0.95)}, opacity = 0.3]
  (-0.75,0.95) ellipse (0.45 and 0.25);
  \draw[blue] (-0.75,0.95) node {\textbullet};

  \node (P) at (-0.9,1) {\textbullet};
  \draw (-0.9,1) node[label={[label distance=-0.2cm]100:$\small\Pb$}]{};
  \node[red] (Sup) at (1.6,-0.88) {\textbullet};
  
  \path[->] (PT) edge   node[label=-100:\small$ \sqrt{\frac{2a_T}{T}}\varphi_x(\hat{\Pb}_T)$]{} (Sup) ;
  
  \draw (3.4,2.1) node (sup cost) {\small$c(x,\hat{\Pb}_T) + \sqrt{\frac{2a_T}{T} \Var_{\hat{\Pb}_T}(\loss(x,\xi))}$};
  
  \path[->] (Sup)  edge[bend right = 20, dashed,color=red] node[]{} (sup cost) ;
\end{tikzpicture}

    \caption{Illustration of the DRO expression of the robust predictor in the subexponential regime. The shrinking ellipsoids around the blue points represent the ambiguity set of \eqref{eq: robust predictor DRO.} around $\hat{\Pb}_T$ for increasing values of $T$. The arrow gives the cost at the pointed distribution, attaining the maximum cost in the ellipsoid.}
    \label{fig: robust predictor.}
\end{figure}

Intuitively, the DRO perspective shows that given the observed empirical distribution $\hat{\Pb}_T$, the SVP predictor guards against all distribution in the ellipsoidal ambiguity set $\{ \Pb'\in \cP \; : \; \|\Pb'-\hat{\Pb}_T\|^2_{\hat{\Pb}_T}
\leq 
a_T/T \}$ when the imposed out-of-sample disappointment \eqref{eq: out-of-sample proba} is of order $e^{-a_T}$.
Proposition \ref{prop: robust predictor is DRO.} identifies the distribution $\hat{\Pb}_T + \sqrt{2a_T/T}\varphi_x(\hat{\Pb}_T)$ as the worst-case probability distribution which is still sufficiently likely to have generated the data.
Figure \ref{fig: robust predictor.} illustrates this perspective. {\color{black}We note finally that the connection between $f$-divergence DRO and variance regularization has been well established in prior work \cite{gotoh2021calibration, gotoh2018robust, lam2019recovering, duchi2016variance, duchi2021statistics}.}

We now show that the SVP predictor $\cSVP$ is strong optimal in the optimal prediction Problem \eqref{eq:optimal-predictor} in the subexponential regime. We first show that the predictor $\cSVP$ has the desired regularities of predictors (Definition \ref{def: predictors.}). We defer the proof to Appendix \ref{Appendix: proof of SVP regularity}.

\begin{proposition}[Regularity of $\cSVP$]\label{prop: regularity of SVP}
  The predictor $\cSVP \in \cC$ is regular, i.e., $\cSVP$ verifies the regularity conditions of predictors (Definition \ref{def: predictors.}).
\end{proposition}

We next show that SVP satisfied the desired out-of-sample guarantee \eqref{eq: out-of-sample ganrantee}. We note that existing finite sample guarantees, c.f., \cite[Theorem 1]{maurer2009empirical} and \cite[Theorem 4]{audibert2009exploration} and Proposition \ref{prop: predictor finite sample guarantees.} in this paper, do not directly imply the desired asymptotic guarantee \eqref{eq: out-of-sample ganrantee} for SVP. A finer analysis will hence be required. Interestingly, our proof uses completely different techniques than \cite{maurer2009empirical} and \cite{audibert2009exploration}.
While their proofs rely on concentration inequalities on the empirical standard deviation, our proof uses the DRO form \eqref{eq: robust predictor DRO.} of SVP combined with the Moderate Deviation Principle (Theorem \ref{thm: MDP finite space}) from Large Deviation Theory.  

\begin{proposition}[$\cSVP$ out-of-sample guarantees]\label{prop: feasibility of robust predictor}
The predictor $\cSVP\in \cC$ verifies the out-of-sample guarantee \eqref{eq: out-of-sample ganrantee} when $a_T \ll T$.
\end{proposition}
\begin{proof}
  Let $\Pb \in \cPin$ and $x \in \mathcal{X}$. Let $T_0 \in \integ$ be such that for all $T\geq T_0$ the DRO form \eqref{eq: robust predictor DRO.} of $\cSVP$ holds. Observe that \eqref{eq: robust predictor DRO.} implies that for all $T\geq T_0$
  \[
    \hat{\Pb}_T \in E_T :=\set{\Pb' \in \mathcal{P}}{\norm{\Pb-\Pb'}^2_{\Pb'}\leq \frac{a_T}{T}} \implies c(x,\Pb) \leq \cSVP(x,\hat{\Pb}_T,T).
  \]
  Hence, we have
  \begin{align*}
    \limsup_{T\to\infty}
    \frac{1}{a_T} \log \Pb^{\infty}
        \left(
                c(x,\Pb) > \cSVP(x,\hat{\Pb}_T,T)
        \right)
    \leq& 
        \limsup_{T\to\infty}
    \frac{1}{a_T} \log \Pb^{\infty}
    \left(
      \hat{\Pb}_T \not\in E_T
          \right)\\
    \leq & \limsup_{T\to\infty}
    \frac{1}{a_T} \log \Pb^{\infty}
    \left(
      \hat{\Pb}_T - \Pb \in \sqrt{ \frac{a_T}{T}}\Gamma_T
          \right)           
  \end{align*}
  where $\Gamma_T := \sqrt{T/a_T} (E_T^c -\Pb) = \{\Delta \in \sqrt{T/a_T} \mathcal{P}_0(\Pb) : \norm{\Delta}^2_{\Pb+\Delta\sqrt{a_T/T}}>1\}$, $\mathcal{P}_0(\Pb) = \set{\Pb - \Pb'}{\Pb' \in \mathcal{P}}$.
  The goal now is to apply an MDP (Theorem \ref{thm: MDP finite space}). In other to do that, we need to analyse the asymptotic behavior of the sequence of sets $\Gamma_T$. 
  We show in the following claim that it converges in a precise sense to $\Gamma := \set{\Delta \in \mathcal{P}_{0,\infty}}{\norm{\Delta}^2_{\Pb} > 1}$ where $\mathcal{P}_{0,\infty}:= \{ \Delta \in \Re^d\; : \; e^\top \Delta =0 \}$ is the hyperplane containing differences of distributions. 

\begin{claim}\label{claim: set inclusions for Gammas}
There exists a sequences $(\varepsilon_T)_{T\geq 1}\in \Re^{\integ}_+$ decreasing to $0$ and $T_1\in \integ$ such that $\sqrt{1+\varepsilon_T}\Gamma \subset \Gamma_T \subset \sqrt{1-\varepsilon_T}\Gamma$ for all $T \geq T_1$.
\end{claim}
\begin{proof}
See Appendix \ref{Appendix: Convergence of Ellipsoids}.
\end{proof}

  Let $(\varepsilon_T)_{T\geq 1}$ and $T_1$ be given by Claim \ref{claim: set inclusions for Gammas}. We have $\Gamma_T \subset \sqrt{1-\varepsilon_T}\Gamma$ for all $T \geq \max(T_0,T_1)$. 
  Hence, using the MDP, Theorem \ref{thm: MDP finite space}, we have for all $t \geq \max(T_0,T_1)$
  \begin{align*}
    \limsup_{T\to\infty}
      \frac{1}{a_T} \log \Pb^{\infty}
      \left(
      c(x,\Pb) > \cSVP(x,\hat{\Pb}_T,T)
      \right)
    \leq & \limsup_{T\to\infty}
    \frac{1}{a_T} \log \Pb^{\infty}
    \left(
      \hat{\Pb}_T - \Pb \in \sqrt{\frac{a_T}{T}}\Gamma_{t}
          \right) \\
    \leq& -\inf_{\Delta\in \bar\Gamma_{t}}\norm{\Delta}^2_{\Pb} 
    = -(1-\varepsilon_t) \inf_{\Delta\in \bar\Gamma}\norm{\Delta}^2_{\Pb}
    = -(1-\varepsilon_t)
  \end{align*}
  Hence,
    \(
    \limsup_{T\to\infty}
      \frac{1}{a_T} \log \Pb^{\infty}
      \left(
      c(x,\Pb) > \cSVP(x,\hat{\Pb}_T,T)
      \right)
    \leq -1.
  \)
\end{proof}

We now prove that the SVP predictor \eqref{eq: Robust predictor SVP.} is preferred to any other predictor verifying the out-of-sample guarantee in the subexponential regime establishing therefore strong optimality.

\begin{theorem}[Strong Optimality of $\cSVP$]\label{thm: strong optimality subexp}
Consider the subexponential regime in which $a_T\ll T$. The predictor $\cSVP$ is feasible in the prediction problem \eqref{eq:optimal-predictor} and
for any predictor $\hat{c} \in \cC$ satisfying the out-of-sample guarantee \eqref{eq: out-of-sample ganrantee}, we have $\cSVP \orderleq \hat{c}$. That is, $\cSVP$ is a strong optimal predictor in the subexponential regime.
\end{theorem}

To prove Theorem \ref{thm: strong optimality subexp}, we show that in order to verify an out-of-sample guarantee with speed $(a_T)_{T\geq1}$, a predictor must necessarily add a regularization to the empirical cost larger than $\sqrt{2a_T/T} \sqrt{\Var_{\Pb}(\loss(x,\xi))}$. This quantity is therefore a fundamental minimal amount of regularization for predictors with out-of-sample guarantee.
It also happens to be exactly the regularization added by the SVP predictor.

\begin{proposition}\label{prop: liminf>limsup for a_T.}
Let $\hat{c} \in \mathcal{C}$. If $\hat{c}$ is feasible in \eqref{eq:optimal-predictor} with $a_T \ll T$, then for all $x \in \mathcal{X}$, for all $\Pb \in \cPin$,
\begin{equation*}
    \liminf_{T \to \infty} \sqrt{\frac{T}{a_T}}(\hat{c}(x,\Pb,T)-c(x,\Pb)) \geq \limsup_{T \to \infty} \sqrt{\frac{T}{a_T}}|\cSVP(x,\Pb,T)-c(x,\Pb)| = \sqrt{2\Var_{\Pb}(\loss(x, \xi))}.
\end{equation*}
\end{proposition}
\begin{proof}
Assume for the sake of contradiction that there exists $\hat{c} \in \mathcal{C}$ feasible in \eqref{eq:optimal-predictor} not verifying the result. There hence exists $(x_0, \Pb_0) \in \mathcal{X}\times \cPin$ such that
$$
 \liminf_{T \to \infty} \sqrt{\frac{T}{a_T}}(\hat{c}(x_0,\Pb_0,T)-c(x_0,\Pb_0)) <\limsup_{T \to \infty} \sqrt{\frac{T}{a_T}}|\cSVP(x_0,\Pb_0,T)-c(x_0,\Pb_0)|
$$
We start the proof by showing in the following claim that this inequality extends to an open ball. This will allow us to examine $\hat{c}$ in an open neighborhood. 
In what follows and throughout the proof $\alpha_T \defn 2a_T/T$ for all $T \in \integ$.

\begin{claim}\label{claim: unif conv.}
There exists $\Pb_1 \in \cPin$, $\varepsilon>0$, and an increasing sequence $(t_T)_{T\geq 1} \in \integ^\integ$, such that for all $\varepsilon'>0$, there exists an open ball $\mathcal{B}(\Pb_1,r)$ around $\Pb_1$ of radius $r>0$ such that
$$\hat{c}(x_0,\Pb,t_T) + \varepsilon \sqrt{\alpha_{t_T}}
\leq 
\cSVP(x_0,\Pb,t_T)
+ \varepsilon' \|\Pb - \Pb_1\|,
\quad \forall \Pb \in \mathcal{B}(\Pb_1,r),\;
\forall T \in \integ
.$$
\end{claim}
\begin{proof}
See Appendix \ref{Appendix: proof of long claim storng optimality pred}.
\end{proof}

Let $\Pb_1 \in \cPin$, $\varepsilon>0$, and $(t_T)_{T\geq 1} \in \integ^\integ$ given by Claim \ref{claim: unif conv.}. In the reminder of the proof, we will show that the out-of-sample guarantee \eqref{eq: out-of-sample ganrantee} condition for $\hat{c}$ fails to hold at $(x_0,\Pb_1)$, which contradicts the feasibility of $\hat{c}$. We first construct key elements for the proof.
Set $\varepsilon'>0$ verifying
$\varepsilon' \leq \varepsilon/3$. Let $r>0$ given by Claim \ref{claim: unif conv.} such that
\begin{equation}\label{eq: robustOpt: uniformoty ineq.}
    \hat{c}(x_0,\Pb,t_T) 
    \leq 
    \cSVP(x_0,\Pb,t_T)
     + \varepsilon' \|\Pb - \Pb_1\|
     -\varepsilon \sqrt{\alpha_{t_T}}
    , 
    \quad \forall \Pb \in \mathcal{B}(\Pb_1,r),
    \; \forall T \in \integ.
\end{equation}
Without loss of generality, we can chose $r<1$.

Denote $\varphi:= \varphi_{x_0}$ given by Proposition \ref{prop: robust predictor is DRO.}. The proposition ensures that $\norm{\sqrt{2}\varphi(\Pb_1)}_{\Pb_1}=1$, therefore $\norm{\varphi(\Pb_1)} \neq 0$.
Let $\eta \in (0,1)$ such that
$\eta \sqrt{\Var_{\Pb_1}(\loss(x_0,\xi))}< \varepsilon/6$
and define
$$\Gamma
= 
\{u \in \mathcal{P}_{0,\infty} \; : \; \|u\|<2 \|\varphi(\Pb_1)\|, \;
(1-\eta)c(x_0,u) - \sqrt{\Var_{\Pb_1}(\loss(x_0,\xi))} > -\varepsilon/3\}.$$
The set $\Gamma$ is clearly open in $\mathcal{P}_{0,\infty} := \{\Delta \in \Re^d \; : \; e^\top \Delta =0\}$ and contains $\varphi(\Pb_1)$. 
In fact, $(1-\eta)c(x_0,\varphi(\Pb_1))  = (1- \eta) \sqrt{\Var_{\Pb_1}(\loss(x_0,\xi))} > \sqrt{\Var_{\Pb_1}(\loss(x_0,\xi))} - \varepsilon/6$.
For $u \in \Gamma$, consider the sequence 
$$\Pb_T(u) := \Pb_1 - (1-\eta)\sqrt{\alpha_T} \cdot u,$$ 

for all $T \in \integ$. See Figure \ref{fig: PTu construction} for an illustration of this construction. 

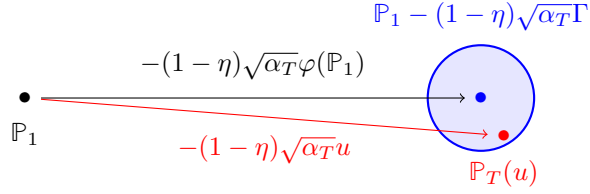
\begin{figure}\label{fig: PTu construction}
        \centering
\begin{tikzpicture}

\node[black, label=below:$\Pb_1$] (P0) at (2,2) {\textbullet};


\filldraw[fill=blue!10!white, draw=blue, thick] (8,2) circle (0.7cm) node [above, label={[label distance=0.5cm]90:${\color{blue}\Pb_1 - (1-\eta)\sqrt{\alpha_T} \Gamma} $}] {} ;

\node[blue](PT) at (8,2) {\textbullet};
\node[red, label=below:${\color{red}\Pb_T(u)}$](PTu) at (8.3,1.5) {\textbullet};

\path[->] (P0) edge   node[label=above:$-(1-\eta)\sqrt{\alpha_T}\varphi(\Pb_1)$]{} (PT) ;

\path[->,red] (P0) edge   node[label=below:$-(1-\eta)\sqrt{\alpha_T}u$]{} (PTu) ;

\end{tikzpicture}
\caption{Illustration of the construction of $\Gamma$ and $\Pb_T(u)$, $u\in \Gamma$. $\Gamma$ is an open set of directions around $\varphi(\Pb_1)$. When $\eta =0$, the blue point represents the distribution $\Pb'_T:= \Pb_1 -\sqrt{\alpha_T} \varphi(\Pb_1)$ such that $ \cSVP(x_0,\Pb'_T,T) = c(x_0,\Pb_1)$.}
\end{figure}

Let $T_0 \in \integ$ such that for all $T\geq T_0$ the DRO expression of $\cSVP$ (Proposition \ref{prop: robust predictor is DRO.}) holds and  $\sqrt{\alpha_{t_T}} \leq r/ (2\|\varphi(\Pb_1)\|)$. The later condition ensures that for $T\geq T_0$ and $u \in \Gamma$, $\Pb_{t_T}(u) \in \mathcal{B}(\Pb_1,r)$, and therefore \eqref{eq: robustOpt: uniformoty ineq.} holds in $\Pb_{t_T}(u)$. We will show that $(\hat{c}(\cdot,\cdot,t_T))_{T\geq 1}$ always disappoints at $(x_0,\Pb_{t_T}(u))_{T\geq 1}$.  That is, $c(x_0,\Pb_1) > \hat{c}(x_0,\Pb_{t_T}(u),t_T)$, for all $T\geq T_0$ and $u\in \Gamma$. We then use the Moderate Deviation Principle, Theorem \ref{thm: MDP finite space}.

Let $T\geq T_0$ and $u \in \Gamma$. Let us examine the sign of $c(x_0,\Pb_1) - \hat{c}(x_0,\Pb_{t_T}(u),t_T)$. $\Pb_{t_T}(u)$ verifies \eqref{eq: robustOpt: uniformoty ineq.}, therefore,
\begin{align}\label{eq: robustOpt: computations.}
    c(x_0,\Pb_1) - \hat{c}(x_0,\Pb_{t_T}(u),t_T) 
    &\geq 
    c(x_0,\Pb_1) - \cSVP(x_0,\Pb_{t_T}(u),t_T) + \varepsilon \sqrt{\alpha_{t_T}} 
    - \varepsilon' \|\Pb_{t_T}(u) - \Pb_1 \|.
\end{align}
We first analyse the term $c(x_0,\Pb_1) - \cSVP(x_0,\Pb_{t_T}(u),t_T)$. We have
\begingroup
\allowdisplaybreaks
\begin{align*}
    c(x_0,\Pb_1) - \cSVP(x_0,\Pb_{t_T}(u),t_T)
    &= 
    c(x_0,\Pb_1) - c(x_0,\Pb_{t_T}(u))
    - \sqrt{\alpha_{t_T}} \sqrt{\Var_{\Pb_{t_T}(u)}(\loss(x_0,\xi))}
    \\
    &= 
    c(x_0,\Pb_1-\Pb_{t_T}(u)) - \sqrt{\alpha_{t_T}} \sqrt{\Var_{\Pb_{t_T}(u)}(\loss(x_0,\xi))}
    \\
    &= (1-\eta)\sqrt{\alpha_{t_T}}c(x_0,u) - \sqrt{\alpha_{t_T}} \sqrt{\Var_{\Pb_{t_T}(u)}(\loss(x_0,\xi))}.
\end{align*}
\endgroup
Plugging this result in \eqref{eq: robustOpt: computations.} and using 
$- \varepsilon' \|\Pb_{t_T}(u) - \Pb_1 \|
\geq
-\varepsilon'r \geq -\varepsilon'$, which results from $\Pb_{t_T}(u) \in \mathcal{B}(\Pb_1,r)$,
we get,
\begin{align*}
    c(x_0,\Pb_1) - \hat{c}(x_0,\Pb_{t_T}(u),t_T) 
&\geq 
    \sqrt{\alpha_{t_T}} \bigg[(1-\eta)c(x_0,u) 
    - \sqrt{\Var_{\Pb_{t_T}(u)}(\loss(x_0,\xi))}
    + \varepsilon 
    -\varepsilon'
    \bigg]\\
&\geq 
    \sqrt{\alpha_{t_T}} \bigg[(1-\eta)c(x_0,u) 
    - \sqrt{\Var_{\Pb_{t_T}(u)}(\loss(x_0,\xi))}
   + 2\varepsilon/3 \bigg]
\end{align*}
where the last inequality is by definition of $\varepsilon'$. As $\Pb_{t_T}(u)$ converges to $\Pb_1$, this inequality becomes
\begin{align*}
    c(x_0,\Pb_1) - \hat{c}(x_0,\Pb_{t_T}(u),t_T) 
&\geq 
    \sqrt{\alpha_{t_T}} 
    \bigg[
    (1-\eta)c(x_0,u) 
    - \sqrt{\Var_{\Pb_1}(\loss(x_0,\xi))} +o(1)
    +2\varepsilon/3
    \bigg] \\
&\geq 
    \sqrt{\alpha_{t_T}} 
    [
    \varepsilon/3 +o(1)
    ] >0
\end{align*}
where the last inequality is by definition of $\Gamma$.
This inequality is uniform in $u$ as
$u \in \Gamma \rightarrow\sqrt{\Var_{\Pb_{t_T}(u)}(\loss(x_0,\xi))}$ converges uniformly to the constant function equal to $\sqrt{\Var_{\Pb_1}(\loss(x_0,\xi))}$, therefore, 
there exists $T_1>T_0$ such that
\begin{equation}\label{eq: main proof: disap in seq t_T.}
c(x_0,\Pb_1) > \hat{c}(x_0,\Pb_{t_T}(u),t_T), 
\quad \forall T \geq T_1,
\; \forall u \in \Gamma.
\end{equation}

Denote $\mathcal{D}_T = \{\Pb'\; : \; c(x_0,\Pb_1)> \hat{c}(x_0,\Pb',T) \}$ the set of disappointing distributions at time $T$ when the true distribution is $\Pb_1$. The inequality \eqref{eq: main proof: disap in seq t_T.} implies that for all $T\geq T_1$, 
$\Pb_1 -(1-\eta)\sqrt{\alpha_{t_T}} \Gamma \subset \mathcal{D}_{t_T}$.
We have therefore
\begin{align}
    \limsup_{T\to\infty} \frac{1}{a_T} \log \Pb^\infty \left( c(x, \Pb_1) > \hat c(x, \hat{\Pb}_T,T)\right)
    &=
    \limsup_{T\to\infty} \frac{1}{a_T} \log \Pb^\infty \left( \hat{\Pb}_T \in \mathcal{D}_T\right) \nonumber\\
    &\geq
    \limsup_{T\to\infty} \frac{1}{a_{t_T}} \log \Pb^\infty \left( \hat{\Pb}_{t_T} \in \mathcal{D}_{t_T}\right) \nonumber\\
    &\geq \limsup_{T\to\infty} \frac{1}{a_{t_T}} \log \Pb^\infty \left( \hat{\Pb}_{t_T} - \Pb_1 \in -\sqrt{\alpha_{t_T}}(1-\eta) \sqrt{2}\Gamma\right) \nonumber \\
    &= \limsup_{T\to\infty} \frac{1}{a_{t_T}} \log \Pb^\infty \left( \hat{\Pb}_{t_T} - \Pb_1 \in \sqrt{\frac{a_{t_T}}{t_T}}\cdot -(1-\eta) \sqrt{2}\Gamma\right) \nonumber\\
    &\geq \liminf_{T\to\infty} \frac{1}{a_{T}} \log \Pb^\infty \left( \hat{\Pb}_{T} - \Pb_1 \in \sqrt{\frac{a_T}{T}}\cdot -(1-\eta) \sqrt{2}\Gamma\right) 
    \nonumber \\ 
    &\geq - \inf_{\Delta \in \Gamma^{\into}} \|(1-\eta)\sqrt{2}\Delta\|_{\Pb_1}^2 
     \label{proof eq: series of ineq 3} \\
    &\geq - (1-\eta)^2 \|\sqrt{2}\varphi(\Pb_1)\|_{\Pb_1}^2 =  - (1-\eta)^2 >-1
    \label{proof eq: series of ineq 4} 
\end{align}
where \eqref{proof eq: series of ineq 3} uses the Moderate Deviation principle (Theorem \ref{thm: MDP finite space}) and  \eqref{proof eq: series of ineq 4} is justified by $\varphi(\Pb_1) \in \Gamma = \Gamma^{\into}$ and $\|\sqrt{2}\varphi(\Pb_1)\|_{\Pb_1}=1$. This inequality contradicts the feasibility of $\hat{c}$ which completes the proof.
\end{proof}

\begin{proof}[Proof of Theorem \ref{thm: strong optimality subexp}]
Let $\hat{c}\in \cC$ be a feasible predictor in the optimal prediction Problem \eqref{eq:optimal-predictor}. Let $(x,\Pb) \in \cX \times \cPin$. The goal is to show that 
$$
\limsup_{T\to \infty} \frac{|\cSVP(x,\Pb,T)-c(x,\Pb)|}{|\hat{c}(x,\Pb,T)-c(x,\Pb)|} \leq 1.
$$
It suffices to show that 
for all sequence $(\beta_T)_{T\geq 1} \in \Re_+^{\integ}$ we have\footnote{This is actually equivalent to the desired result, see Lemma \ref{lemma: equivalence of order notions.}.}
\begin{equation*}
    \limsup_{T\to\infty} \frac{1}{\beta_T} |\hat{c}(x,\Pb,T) - c(x,\Pb)| \geq \limsup_{T\to\infty} \frac{1}{\beta_T} |\cSVP(x,\Pb,T) - c(x,\Pb)|.
\end{equation*}
The result indeed follows with $\beta_T = |\hat{c}(x,\Pb,T)-c(x,\Pb)|$ for all $T$.
We distinguish two cases depending on whether $\Var_{\Pb}(\loss(x,\xi)) = 0$. If $\Var_{\Pb}(\loss(x,\xi)) = 0$, then  $\cSVP(x,\Pb,T) = c(x,\Pb)$ for all $T$ and therefore the desired inequality becomes trivial. 

Suppose now $\Var_{\Pb}(\loss(x,\xi)) > 0$. Denote $\alpha_T = 2a_T/T$ for all $T$.
We have
\begin{equation*}
    \limsup_{T\to\infty} \frac{1}{\beta_T} |\hat{c}(x,\Pb,T) - c(x,\Pb)| 
    =
    \limsup_{T\to\infty}  \frac{\sqrt{\alpha_T}}{\beta_T} \cdot \frac{1}{\sqrt{\alpha_T}} |\hat{c}(x,\Pb,T) - c(x,\Pb)|
\end{equation*}
Proposition \ref{prop: liminf>limsup for a_T.} ensures that 
\begin{align*}
    \liminf_{T\to\infty} \frac{1}{\sqrt{\alpha_T}} |\hat{c}(x,\Pb,T) - c(x,\Pb)|
    \geq
    \limsup_{T\to\infty} \frac{1}{\sqrt{\alpha_T}} |\cSVP(x,\Pb,T) - c(x,\Pb)|
    =
    \sqrt{\Var_{\Pb}(\loss(x,\xi))} >0
\end{align*}
We can apply, therefore, a $\limsup$-$\liminf$ inequality (Lemma \ref{lemma: limsup liminf ineq}) to get
\begin{align*}
    \limsup_{T\to\infty} \frac{\sqrt{\alpha_T}}{\beta_T} \cdot \frac{1}{\sqrt{\alpha_T}} |\hat{c}(x,\Pb,T) - c(x,\Pb)| 
    &\geq 
    \limsup_{T\to\infty} \frac{\sqrt{\alpha_T}}{\beta_T} \cdot
    \liminf_{T\to\infty} \frac{1}{\sqrt{\alpha_T}}  |\hat{c}(x,\Pb,T) - c(x,\Pb)|
\end{align*} 
Using Proposition \ref{prop: liminf>limsup for a_T.}, we have
\begingroup
\allowdisplaybreaks
\begin{align*}
    \limsup_{T\to\infty} \frac{\sqrt{\alpha_T}}{\beta_T} \cdot
    \liminf \frac{1}{\sqrt{\alpha_T}}  |\hat{c}(x,\Pb,T) - c(x,\Pb)|
    &\geq
    \limsup_{T\to\infty} \frac{\sqrt{\alpha_T}}{\beta_T} \cdot
    \limsup_{T\to\infty} \frac{1}{\sqrt{\alpha_T}}  |\cSVP(x,\Pb,T) - c(x,\Pb)| \\
    &= 
    \limsup_{T\to\infty} \frac{\sqrt{\alpha_T}}{\beta_T} \cdot
    \lim_{T\to\infty} \frac{1}{\sqrt{\alpha_T}}  |\cSVP(x,\Pb,T) - c(x,\Pb)|
    \\
    &= 
    \limsup_{T\to\infty} \frac{1}{\beta_T} |\cSVP(x,\Pb,T) - c(x,\Pb)|,
\end{align*}
\endgroup
where the last equality is justified by Lemma \ref{lemma: lim in limsup.} and the fact that $\lim \frac{1}{\sqrt{\alpha_T}} |\cSVP(x,\Pb,T) - c(x,\Pb)| = \sqrt{\Var_{\Pb}(\loss(x,\xi))} \not \in \{0,\infty\}$.
\end{proof}

We point out that although our required out-of-sample guarantee on the predictor is asymptotic (Proposition \ref{prop: feasibility of robust predictor}), SVP does enjoy also finite sample guarantees.

\begin{proposition}[Finite Sample Guarantees]\label{prop: predictor finite sample guarantees.}
Let $x,\Pb \in \cX \times \cPin$. For all $T\in \integ$, the following holds
\begin{align*}
\Pb^{\infty}\left(c(x,\Pb) \leq \cSVP(x,\hat{\Pb}_T,T) +  \frac{7K}{3}\frac{a_T}{T}\right) 
&\geq 1- 2e^{-a_T} \\
\Pb^{\infty}\left(
c(x,\Pb)
\geq
\cSVP(x,\hat{\Pb}_T,T) 
- \sqrt{\frac{8a_T}{T}\Var_{\Pb}(\loss(x,\xi))}
-\frac{7K}{3}\frac{a_T}{T}
\right) 
&\geq 1- 2e^{-a_T}
\end{align*}
where $K = 2\sup_{x\in \cX}\|\loss(x,\cdot)\|_{\infty}$.
\end{proposition}

The proof of these bounds relies essentially on concentration inequalities on the empirical standard deviation shown by \cite{maurer2009empirical} and \cite{audibert2009exploration}. See Appendix \ref{Appendix: proof of finite guarantees} for a full proof.

{\color{black}
While \cite{van2020data} points out that in the exponential regime, only the KL predictor is strongly optimal, we have shown here that SVP and $\chi^2$ divergence DRO are strongly optimal in the subexponential regime. A key question now is whether this optimality in the subexponential regime extends to KL DRO as well, as a Taylor expansion of the KL divergence reveals
$$
I(\Pb',\Pb) = \|\Pb' - \Pb\|_{\Pb}^2 +o\left( \|\Pb' - \Pb\|_{\Pb}^2\right).
$$
We show that this is indeed true. Consider the KL predictor with decreasing radius in $a_T/T$
\begin{equation}\label{eq: KL in subexp}
\hat{c}_{\text{KL}}'(x,\Pb,T) = \sup_{\Pb' \in \cP} \left\{ c(x,\Pb') \; : \; I(\Pb,\Pb') \leq \frac{a_T}{T} \right\}, \quad \forall x \in \cX, \; \forall \Pb \in \cP, \; \forall T \in \integ.
\end{equation}

The SVP predictor and Pearson DRO predictor approximate the KL predictor with a precision scaling with the square root of the KL ball radius.
The proof is in Appendix \ref{App: proof of KL=SVP}.

\begin{proposition}\label{prop: KL equiv to SVP}
We have for all $x \in \cX$ and $\Pb \in \cPin$
\begin{align*}
\sup_{\Pb' \in \cP} \left\{ c(x,\Pb') \; : \; I(\Pb,\Pb') \leq \frac{a_T}{T} \right\}
&=
\sup_{\Pb' \in \cP} \left\{ c(x,\Pb') \; : \; \|\Pb' - \Pb \|_{\Pb}^2 \leq \frac{a_T}{T} \right\} + o\left( \sqrt{\frac{a_T}{T}}\right)\\
&=
\cSVP(x,\Pb,T) +o\left( \sqrt{\frac{a_T}{T}}\right),
\end{align*}
where the assymtptotic notation $o$ is uniform in $x$.
Furthermore, if $\Var_{\Qb}(\loss(x,\xi)) >0$ for all $\Qb \in \cPin$ and $x\in \cX$, then
$\cKL' \equiv \cSVP$ with respect to the partial order in the subexponential regime. 
\end{proposition}

When the KL ball radius is chosen appropriately, the KL predictor also verifies the out-of-sample guarantee in the subexponential regime. The proof is in Appendix \ref{App: Proof of KL feasibility in subexp}.

\begin{proposition}\label{prop: KL feasibility in subexp}
The predictor $\hat{c}_{\text{KL}}'$ verifies the out-of-sample guarantee \eqref{eq: out-of-sample ganrantee} when $a_T \ll T$.
\end{proposition}
}

\section{Optimal Data-Driven Prescription}\label{sec: presc}
In this section we study the problem of optimal data-driven prescription as formalized in Problem (\ref{eq:optimal-prescriptor}).
The key question can informally be stated as one of optimal approximation of the unknown objective function of Problem \eqref{eq: stochastic opt} who's minimum provides the best approximation to its optimal solution.
Typically in machine learning and decision making, predictors (approximating the expectation) are first established with provable prediction guarantees as was done in Section \ref{sec: predictor}.
Subsequently, these guarantees are extended to the prescribed solution using the structure of the decision set $\cX$.
A key questions is whether circumventing the prediction step results in better formulations.
In particular, would such a formulation improve the quality of the cost of the prescribed solution at the expense of a reduced quality of the overall cost prediction of any other decision?
We prove that this is not the case.

We will show that in each regime, the strong optimal predictors identified in Section \ref{sec: predictor} induce strong optimal prescriptors as well. This result suggests that the classical approach in machine learning and decision-making of constructing estimators with guarantees on prediction and then extending such guarantees to the prescription is justified: the optimal predictor also induces an optimal prescriptor. 

{\color{black} We briefly present our results in the exponential and superexponential regime, then we discuss more extensively the subexponential regime which bares much of the novel insights.}

\subsection{The Exponential Regime}\label{sec: exp presc}

Consider the exponential regime in which $a_T \sim rT$, $r>0$. \cite{van2020data} showed that the distributionally robust predictor \eqref{eq: KL predictor} and its associated prescriptor
\begin{align*}
    \xKL{,T}(\Pb) \in & \argmin_{x\in \cX} \cKL(x,\Pb,T), 
    \quad \forall \Pb \in \cP, \; \forall T \in \integ
    \\
  \in & \argmin_{x\in \cX} \sup_{\Pb' \in \cP} \set{c(x,\Pb')}{\D{\Pb} {\Pb'}\leq r}, 
    \quad \forall \Pb \in \cP, \; \forall T \in \integ
\end{align*}
are also optimal among all prescriptors which are not an explicit function of the data size $T$. Notice that the minimizer $\xKL{}$ exists as $\cX$ is compact and $\cKL$ is continuous in the first argument.
We will extend this result and show that also our more general setting where predictors can be a function of $T$ this distributionally robust predictor remains optimal. {\color{black}We note that dependence in $T$ involves further technical hurdles and our proof|in particular of strong optimality|involves finer analysis. Proofs of results of this section are in Appendix \ref{App: proof feasibility Exp presc}.}

\begin{proposition}[Feasibility]\label{prop: feasibility presc exp}
The prescriptor $(\cKL,\xKL{}) \in \cC \times \hat{\cX}$ verifies the prescription out-of-sample guarantee \eqref{eq: prescriptor out-of-sample ganrantee} when $a_T \sim rT$.
\end{proposition}

\begin{theorem}[Strong Optimality]\label{thm: strong opt exponential regime prescriptor.}
Consider the exponential regime in which $a_T\sim rT$. The pair of predictor and prescriptor $(\cKL,\xKL{})\in \cC \times \hat{\cX}$ is feasible in the prescription problem \eqref{eq:optimal-prescriptor} and
for any pair of predictor and prescriptor $(\hat{c},\hat{x})\in \cC \times \hat{\cX}$ satisfying the out-of-sample guarantee \eqref{eq: prescriptor out-of-sample ganrantee}, we have $(\cKL,\xKL{}) \orderpresc (\hat{c},\hat{x})$. That is, $(\cKL,\xKL{})$ is a strong optimal prescriptor in the exponential regime.
\end{theorem}

\subsection{The Superexponential Regime}\label{sec: presc Superexp}
We prove the same result for the superexponential regime where $a_T \gg T$. Even when the out-of-sample guarantee is required only for the prescribed solution, superexponential guarantees require the predictor to cover the worst scenario of the uncertainty, no matter the data size. Consider a prescriptor associated to the robust predictor $\cRob$ defined in \eqref{eq: overly robust predictor}
\begin{align*}
    \xRob{,T}(\Pb) &\in \argmin_{x\in \cX} \cRob(x,\Pb,T), 
    \quad \forall \Pb \in \cP, \; \forall T \in \integ, \\
    &\in \argmin_{x\in \cX} \max_{i\in \Sigma} \,\ell(x,i), 
    \quad \forall \Pb \in \cP, \; \forall T \in \integ.
\end{align*}
This minimizer exists as $\cX$ is compact and $\cRob$ is continuous in the first argument.
\begin{proposition}
The prescriptor $(\cRob,\xRob{}) \in \cC \times \hat{\cX}$ verifies the prescription out-of-sample guarantee \eqref{eq: prescriptor out-of-sample ganrantee} when $a_T \gg T$.
\end{proposition}
\begin{proof}
Let $\Pb \in \cPin$. The guarantee is directly implied by
$$\Pb^\infty 
    \left(
       c(\xRob{,T}(\hat{\Pb}_T), \Pb) > \cRob^{\star}(\hat{\Pb}_T,T)
    \right) = \Pb^\infty 
    \left(
       c(\xRob{,T}(\hat{\Pb}_T), \Pb) > \sup_{\Pb' \in \cP} c(\xRob{,T}(\hat{\Pb}_T),\Pb')
    \right) = 0.
$$%
\end{proof}
\begin{theorem}\label{thm: optimality presc superexp}
Consider the superexponential regime in which $a_T\gg T$. The pair of predictor and prescriptor $(\cRob,\xRob{})\in \cC \times \hat{\cX}$ is feasible in the prescription problem \eqref{eq:optimal-prescriptor} and
for any pair of predictor and prescriptor $(\hat{c},\hat{x})\in \cC \times \hat{\cX}$ satisfying the out-of-sample guarantee \eqref{eq: prescriptor out-of-sample ganrantee}, we have $(\cRob,\xRob{}) \orderpresc (\hat{c},\hat{x})$. That is, $(\cRob,\xRob{})$ is a strong optimal prescriptor in the superexponential regime.
\end{theorem}
We defer the proof to Appendix \ref{Appendix: proofs presc superexp}. The proof is in essence the same as the proof of Theorem \ref{thm: strong opt exponential regime prescriptor.} with $r \to \infty$.

\subsection{The subexponential Regime}\label{sec: prescriptor subexp}

We now turn to the subexponential regime where $a_T \ll T$. Akin the prediction problem, consistency is a necessary condition for optimality in the optimal prescription problem in the subexponential regime.

\begin{proposition}[Consistency of weakly optimal prescriptors]\label{prop: consitency of prescriptors}
Consider the subexponential regime in which $a_T\ll T$. Every weakly optimal pair of predictor-precriptor in \eqref{eq:optimal-prescriptor} is consistent. That is, for every pair of predictor-precriptor $(\hat{x},\hat{c})$ verifying the prescription out-of-sample guarantee \eqref{eq: prescriptor out-of-sample ganrantee}, either $(\hat{c}(\cdot,\cdot,T))_{T\geq1}$ converges point-wise to $c(\cdot,\cdot)$,
or there exists a pair of predictor precriptor $(\hat{x}',\hat{c}')$ verifying the out-of-sample guarantee that is strictly preferred to $(\hat{x},\hat{c})$, ie $(\hat{x}',\hat{c}') \orderpresc (\hat{x},\hat{c})$ and $(\hat{x}',\hat{c}') \not \equiv_{\hat{\cX}} (\hat{x},\hat{c})$. 
\end{proposition}

\begin{proof}[Sketch of proof]
The full proof is deferred to Appendix \ref{Appendix: proof consistence presc}. Suppose $(\hat{x},\hat{c})$ is weakly optimal and not consistent. There exists $x_0\in \mathcal{X}$ and $\Pb_0 \in \mathcal{P}$ such that $\limsup |\hat{c}(x_0,\Pb_0,T) - c(x,\Pb_0)| = \varepsilon>0$. We consider the same exact construction of $\hat{c}'$ as in the proof of Proposition \ref{prop: consitency} (illustrated in Figure \ref{fig: proof of consistency}).
Among the possible prescriptors $\hat{x}'$ associated to $\hat{c}'$, we consider the closest one to the prescriptor $\hat{x}_T$ of $\hat{c}$. We then show that $(\hat{x}',\hat{c}')$ is a feasible pair of predictor prescriptor and that $(\hat{x}',\hat{c}')$ is strictly preferred to $(\hat{x},\hat{c})$. The latter result follows essentially from the construction of $\hat{c}'$: $\hat{c}'$ is less or equal to $\hat{c}$ at every point, hence, its minimum is also lower than the minimum of $\hat{c}$. Furthermore, we show that a feasible predictor is necessarily larger than the true cost $c$ asymptotically, hence, $\hat{c}'$ and its minimum $\hat{c}'^{\star}$ are closer to the true cost $c$ and the optimal cost $c^{\star}$ respectively.

To show feasibility of the pair $(\hat{x}',\hat{c}')$, we prove that they verify the out-of-sample guarantee. In order to do that, we examine the pair in two regions of $\cP$. The first region consist of distributions $\Pb$ such that $(\hat{x}'_T(\Pb),\Pb)$ does not fall in the ball $\mathcal{B}\left((x_0,\Pb_0),\frac{\rho}{2} \right)$ where $\hat{c}$ was perturbed into $\hat{c}'$ (see Figure \ref{fig: proof of consistency}). In this region, $\hat{c}'$ is exactly $\hat{c}$, and as $\hat{x}'$ is chosen as the closest minimizer of $\hat{c}'$ to $\hat{x}$, we prove that $\hat{x}'$ is also exactly $\hat{x}$. Hence, $(\hat{x}',\hat{c}')$ inherits the out-of-sample guarantee of $(\hat{x},\hat{c})$ in this region. The second region is where $(\hat{x}'_T(\Pb),\Pb)$ falls in the ball $\mathcal{B}\left((x_0,\Pb_0),\frac{\rho}{2} \right)$ where $\hat{c}$ was perturbed into $\hat{c}'$. We show that a constant gap with the true cost suffices to verify a prescription subexponential guarantee. As $\hat{c}'$ is constructed such that in this region, it maintains a constant gap with the true cost, it follows that $(\hat{x}',\hat{c}')$ verifies the desired subexponential guarantee.
\end{proof}

\begin{remark}
As pointed out out in Remark \ref{rmk: consistency of estimator}, the consistency of predictors in the sense of Proposition \ref{prop: consitency of prescriptors} implies the consistency of the estimator $(\hat{c}(x,\hat{\Pb}_T,T))_{T\geq 1}$ of the true cost $c(x,\Pb)$. Notice that it also implies the consistency of the estimator $(\hat{c}^{\star}(\hat{\Pb}_T,T))_{T\geq 1}$ of the optimal cost $c^{\star}(\Pb)$, in the prescription problem. In fact, the point-wise convergence of $(\hat{c}(\cdot,\cdot,T))_{T\geq 1}$ to $c(\cdot,\cdot)$ along with equicontinuity of $(\hat{c}(\cdot,\cdot,T))_{T\geq 1}$ implies its uniform convergence. Uniform convergence of $(\hat{c}(\cdot,\cdot,T))_{T\geq 1}$ to $c(\cdot,\cdot)$ implies point-wise convergence of $(\hat{c}^{\star}(\cdot,T))_{T\geq 1}$ to $c^{\star}(\cdot)$ (see Lemma \ref{lemma: convergence of minimum}), which combined with compactness of $\cX$, implies its uniform convergence (see Lemma \ref{lemma: unif convergence of minimum}). Uniform convergence combined with the almost sure convergence of $\hat{\Pb}_T$ to $\Pb$ implies almost sure convergence of $(\hat{c}^{\star}(\hat{\Pb}_T,T))_{T\geq 1}$ to the optimal cost $c^{\star}(\Pb)$.
\end{remark}

Consider a prescriptor associated to the SVP predictor $\cSVP$ defined in Equation \eqref{eq: Robust predictor SVP.}
\begin{align*}
    \xSVP{,T}(\Pb) &\in \argmin_{x\in \cX} \cSVP(x,\Pb,T), 
    \quad \forall \Pb \in \cPin,\; \forall T \in \integ,\\
    &\in \argmin_{x\in \cX} c (x,\Pb)+ \sqrt{\frac{2a_T}{T}\Var_{\Pb}(\loss(x, \xi))}, 
    \quad \forall \Pb \in \cPin,\; \forall T \in \integ. 
\end{align*}
The minimum of $\cSVP(\cdot,\Pb,T)$ is indeed attained as $\cX$ is compact and $\cSVP(\cdot,\Pb,T)$ is continuous. The SVP predictor $\cSVP$ emerges in our frameworks |as we will prove momentarily| as the optimal prescriptor.
Our framework considers only the out-of-sample performance and accuracy of the considered prescriptors.
Nevertheless, the tractability of the resulting prescriptor is a key practical issue.
The following proposition shows that minimizing the SVP predictor $\cSVP$ is \textit{essentially} a convex optimization problem when the loss function is convex.

\begin{proposition}[Convexity of SVP]\label{prop: convexity of robust pred}
Suppose the loss function $x\rightarrow \loss(x,i)$ of each uncertain scenario $i\in \Sigma$ is convex. Let $T\in \integ$ and $\hat{\Pb}_T$ the observed empirical distribution.  If $a_T$ is chosen such that
such that 
$\sqrt{2a_T/T} \geq -\frac{\sqrt{\Var_{\hat{\Pb}_T}(\loss(x,\xi))}}{|\loss(x,i) - c(x,\hat{\Pb}_T)|} $ for all $i\in \Sigma$,
then $x\rightarrow c(x,\hat{\Pb}_T) + \sqrt{\frac{2a_T}{T} \Var_{\hat{\Pb}_T}(\loss(x,\xi))}$ is convex in $\cX$.
\end{proposition}
\begin{proof}
Notice that for all $\Pb' \in \cP$, $x\rightarrow c(x,\Pb)$ is convex as a convex combination of the loss function on each uncertainty.
Under the condition on $a_T$, Proposition \ref{prop: robust predictor is DRO.} implies that the considered function, SVP $\cSVP$, is equal to $\sup_{\Pb'\in \cP} \{c(x,\Pb') \; : \; \|\Pb'-\hat{\Pb}_T \|^2_{\hat{\Pb}_T} \leq a_T/T\}$ which is the supremum of convex functions, and is therefore convex.
\end{proof}

This result is rather surprising. While the empirical expectation is clearly convex when the loss is convex, the empirical standard deviation $x \rightarrow \sqrt{\Var_{\hat{\Pb}_T}(\loss(x,\xi))}$ is in general non-convex even when the loss is convex \cite{maurer2009empirical,duchi2016variance,lam2019recovering}. This implies that SVP is a sum of a convex and a non-convex function which is, in general, not convex. However, Proposition \ref{prop: convexity of robust pred} shows that when the scaling $\sqrt{2a_T/T}$ of the standard deviation is small enough, the SVP predictor becomes convex (see Figure \ref{fig:convexity} for an illustration). 
  This result admits also a probabilistic interpretation which can be traced back to \cite{duchi2016variance}. As the empirical distribution $\hat{\Pb}_T$ is close to $\Pb \in \cPin$ with increasing probability with $T$, and the scaling $\sqrt{2a_T/T}$ converges to $0$ (subexponential regime), the convexity condition is verified with increasingly high probability with $T$. Hence, the SVP predictor is convex with high probability.

\begin{figure}
    \centering
    \includegraphics[width=0.5\linewidth]{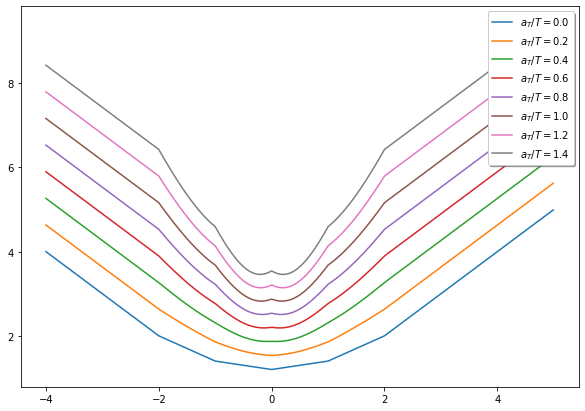}
    \caption{Plot of $x\rightarrow \cSVP(x,\hat{\Pb}_T,T) = c(x,\hat{\Pb}_T) + \sqrt{\frac{2a_T}{T} \Var_{\hat{\Pb}_T}(\loss(x,\xi))}$ for different values of empirical standard deviation scaling $a_T/T$. Here $\loss(x,\xi) = |x-\xi|$ and $\hat{\Pb}_T$ is a uniform distribution on $\{-2,1,0,1,2\}$. The higher curves correspond to higher scaling $a_T/T$. The curves transition from non-convexity to convexity when the scaling decreases illustrating Proposition \ref{prop: convexity of robust pred}.}
    \label{fig:convexity}
\end{figure}

The following proposition quantifies the amount of regularization the prescription of SVP adds. It also highlights how SVP naturally favors solutions with lower variance, and ultimately converges to the optimal solution with the lowest variance.

\begin{proposition}[Regularization of SVP prescription]\label{prop: rate of cv of prescriptor}
For all $\Pb \in \cPin$,
\begin{equation*}
|\cSVP^{\star}(\Pb,T) - c^{\star}(\Pb)| \leq \sqrt{\frac{2a_T}{T}\Var_{\Pb}(\loss(x^{\star}(\Pb),\xi))}, \quad \forall T \in \integ,
\end{equation*}
where $x^{\star}(\Pb)$ is a minimizer of $c(\cdot,\Pb)$ that has the lowest variance, that is, $x^{\star}(\Pb) \in \argmin \{\Var_{\Pb}(\loss(x,\xi)) \; : \; x\in \argmin c(\cdot,\Pb)\}$. Furthermore
\begin{equation*}
\cSVP^{\star}(\Pb,T) - c^{\star}(\Pb) = \sqrt{\frac{2a_T}{T}\Var_{\Pb}(\loss(x^{\star}(\Pb),\xi))} + o\left(\sqrt{\frac{a_T}{T}}\right),
\end{equation*}
and any prescriptor $\xSVP{,T}(\Pb)$ associated to the SVP predictor verifies $$\Var_{\Pb}(\loss(\xSVP{,T}(\Pb),\xi)) \xrightarrow[T\to \infty]{}  \Var_{\Pb}(\loss(x^{\star}(\Pb),\xi)).$$
\end{proposition}
\begin{proof}
Let $\Pb \in \cPin$ and $T\in \integ$. Set $\alpha_T = 2a_T/T$. By minimality of $c^{\star}(\Pb)$, we have
\begin{align*}
    \cSVP^{\star}(\Pb,T) &= c(\xSVP{,T}(\Pb),\Pb) + \sqrt{\alpha_T \Var_{\Pb}(\loss(\xSVP{,T}(\Pb),\xi))} \numberthis \label{proof eq: robust precriptor eq 1}\\
    &\geq c^{\star}(\Pb) + \sqrt{\alpha_T \Var_{\Pb}(\loss(\xSVP{,T}(\Pb),\xi))} \numberthis \label{proof eq: robust precriptor ineq 1}
\end{align*}
Let $x^{\star}(\Pb)$ as defined in the proposition.
By minimality of $\cSVP^{\star}$, we have
\begin{align*}
    \cSVP^{\star}(\Pb,T) &\leq \cSVP(x^{\star}(\Pb),\Pb,T) = c^{\star}(\Pb) + \sqrt{\alpha_T \Var_{\Pb}(\loss(x^{\star}(\Pb),\xi))} \numberthis \label{proof eq: robust precriptor ineq 2}
\end{align*}
Hence, combining \eqref{proof eq: robust precriptor ineq 1} and \eqref{proof eq: robust precriptor ineq 2}, we get,
\begin{equation}\label{proof eq: gendarmes}
\sqrt{\frac{a_T}{T}\Var_{\Pb}(\loss(\xSVP{,T}(\Pb),\xi))} \leq 
\cSVP^{\star}(\Pb,T) - c^{\star}(\Pb)
\leq 
\sqrt{\frac{a_T}{T}\Var_{\Pb}(\loss(x^{\star}(\Pb),\xi))}.
\end{equation}
which proves the desired inequality. It remains to prove the convergence of $(\Var_{\Pb}(\loss(\xSVP{,T}(\Pb),\xi)))_{T\geq 1}$ to $\Var_{\Pb}(\loss(x^{\star}(\Pb),\xi))$ and the asymptotic development of $|\cSVP^{\star}(\Pb,T) - c^{\star}(\Pb)|$. Notice that the convergence of the variance along with \eqref{proof eq: gendarmes} implies the asymptotic development, hence it suffices to prove the converenge of the variance.

Inequality \eqref{proof eq: gendarmes} implies that $\limsup_{T\to \infty} \Var_{\Pb}(\loss(\xSVP{,T}(\Pb),\xi)) \leq \Var_{\Pb}(\loss(x^{\star}(\Pb),\xi))$.  It remains to the show that $\liminf_{T\to \infty} \Var_{\Pb}(\loss(\xSVP{,T}(\Pb),\xi)) \geq \Var_{\Pb}(\loss(x^{\star}(\Pb),\xi))$.
As the sequence $\xSVP{,T}(\Pb)$ lives in the compact set $\cX$, it has accumulation points. Let $(\xSVP{,t_T}(\Pb))_{T\geq 1}$ be subsequence converging to a limit $x_{\infty} \in \cX$. Let us show that $x_{\infty}$ is a minimizer of $c(\cdot,\Pb)$. Inequality \eqref{proof eq: gendarmes} implies that $\cSVP^{\star}(\Pb,T) \xrightarrow[T\to \infty]{} c^{\star}(\Pb)$, therefore using \eqref{proof eq: robust precriptor eq 1}, we have $c(\xSVP{,T}(\Pb),\Pb)\xrightarrow[T\to \infty]{} c^{\star}(\Pb)$. Hence, by continuity of $c(\cdot,\Pb)$, we have $c(x_{\infty},\Pb) = c^{\star}(\Pb)$ which proves that $x_{\infty}$ is a minimizer of $c(\cdot,\Pb)$. As $x^{\star}(\Pb)$ is a minimizer of $c(\cdot,\Pb)$ with the lowest variance by definition, we have $\Var_{\Pb}(\loss(x_{\infty},\xi)) \geq \Var_{\Pb}(\loss(x^{\star}(\Pb),\xi))$. We have therefore shown that every accumulation point $v$ of $\Var_{\Pb}(\loss(\xSVP{,T}(\Pb),\xi))$ verifies $v \geq \Var_{\Pb}(\loss(x^{\star}(\Pb),\xi))$. Hence $\liminf_{T\to \infty} \Var_{\Pb}(\loss(\xSVP{,T}(\Pb),\xi)) \geq \Var_{\Pb}(\loss(x^{\star}(\Pb),\xi))$ which completes the proof.
\end{proof}

We now prove the strong optimality of the SVP prescriptor. As before, we first establish its feasibility.
\begin{proposition}[Feasibility of the SVP Prescriptor]\label{prop: feasibility of prescriptors}
The prescriptor $(\cSVP,\xSVP{}) \in \cC \times \hat{\cX}$ verifies the prescription out-of-sample guarantee \eqref{eq: prescriptor out-of-sample ganrantee} when $a_T \ll T$.
\end{proposition}
\begin{proof}
Let $\Pb \in \cPin$ and $T_0 \in \integ$ be such that for all $T\geq T_0$ the DRO form \eqref{eq: robust predictor DRO.} of $\cSVP$ holds. For $T \in \integ$, we have
$$
\hat{\Pb}_T \in E_T:= \set{\Pb' \in \mathcal{P}}{\norm{\Pb'-\Pb}^2_{\Pb'}\leq \frac{a_T}{T}} \implies
c(\xSVP{,T}(\hat{\Pb}_T), \Pb) \leq
\cSVP(\xSVP{,T}(\hat{\Pb}_T), \hat{\Pb}_T,T) =\cSVP^{\star}(\hat{\Pb}_T,T)
$$
Hence,
$$\frac{1}{a_T} \log \Pb^\infty 
    \left(
        c(\xSVP{,T}(\hat{\Pb}_T), \Pb) > \cSVP^{\star}( \hat{\Pb}_T,T)
    \right)
\leq
    \frac{1}{a_T} \log \Pb^{\infty}
    \left(
      \hat{\Pb}_T \not\in E_T
          \right)
    $$
We have shown in the proof of Proposition \ref{prop: feasibility of robust predictor} that $\limsup_{T\to\infty}
    \frac{1}{a_T}  \log \Pb^{\infty}
    \left(
      \hat{\Pb}_T \not\in E_T
          \right) \leq -1$ which gives the desired result.
\end{proof}

To establish strong optimality in the sub-exponential regime we will need to impose further regularity on problem (\ref{eq: stochastic opt}) which we wish to approximate. In order to establish the following theorem, we suppose that the minimizer $x^{\star}(\Pb)$ of the true cost $c(\cdot,\Pb)$ is unique, for all $\Pb \in \cPin$. This is the  case for instance when the loss of each uncertain scenario $i$, $x \rightarrow \loss(x,i)$ is strictly convex. Notice that the imposed assumption is only on the actual cost $c$ and not on the predictors $\hat{c}$ which can have multiple minima. The imposed restriction is necessary as $x^{\star}(\Pb)$ can behave very erratically at any distribution $\Pb$ where $\argmin c(x, \Pb)$ should fail to be single-valued and consequently seems too be hard to approximate optimally.

\begin{theorem}[Strong optimality]\label{thm: strong optimality mdp prescriptor.}
Consider the subexponential regime in which $a_T\ll T$. The pair of predictor and prescriptor $(\cSVP,\xSVP{})\in \cC \times \hat{\cX}$ is feasible in the prescription problem \eqref{eq:optimal-prescriptor} and
for any pair of predictor and prescriptor $(\hat{c},\hat{x})\in \cC \times \hat{\cX}$ satisfying the out-of-sample guarantee \eqref{eq: prescriptor out-of-sample ganrantee}, we have $(\cSVP,\xSVP{}) \orderpresc (\hat{c},\hat{x})$. That is, $(\cSVP,\xSVP{})$ is a strong optimal prescriptor in the subexponential regime.
\end{theorem}

The key step to prove Theorem \ref{thm: strong optimality mdp prescriptor.} is to show that in order to verify an out-of-sample guarantee with speed $(a_T)_{T\geq1}$, a prescriptor must minimize a cost that necessarily adds a regularization to the empirical cost larger than $\sqrt{2a_T/T}\sqrt{\Var_{\Pb}(\loss(x^{\star}(\Pb),\xi))}$, where $x^{\star}(\Pb)$ is a minimizer of the cost $c(x,\Pb)$. This quantity happens to be exactly the regularization added by the SVP prescriptor by Proposition \ref{prop: rate of cv of prescriptor}.
\begin{proposition}\label{prop: lower bound strong opt presc}
Let $(\hat{c}, \hat{x}) \in \hat{\cX}$ be a pair of predictor prescriptor verifying the out-of-sample guarantee \eqref{eq: prescriptor out-of-sample ganrantee}. The following inequality holds
\begin{equation*}
    \liminf_{T\to \infty} \sqrt{\frac{T}{a_T}}|\hat{c}^{\star}(\Pb,T) - c^{\star}(\Pb)|
    \geq
    \limsup_{T\to \infty} \sqrt{\frac{T}{a_T}}|\cSVP^{\star}(\Pb,T) - c^{\star}(\Pb)| = \sqrt{2\Var_{\Pb}(\loss(x^{\star}(\Pb),\xi))}, \quad \forall \Pb \in \cPin.
\end{equation*}
\end{proposition}
\begin{proof}
Suppose for the sake of argument that there exists $\Pb_0 \in \cPin$ such that   
\begin{equation} \label{proof eq: initial assumption presc}
    \liminf_{T\to \infty} \sqrt{\frac{T}{a_T}}|\hat{c}^{\star}(\Pb_0,T) - c^{\star}(\Pb_0)|
    <
    \limsup_{T\to \infty} \sqrt{\frac{T}{a_T}}|\cSVP^{\star}(\Pb_0,T) - c^{\star}(\Pb_0)|.
\end{equation}

We start the proof by showing how this inequality extends to an open ball. This will allow us then to examine $\hat{c}$ in an open neighborhood. In all what follow, we denote $\alpha_T = 2a_T/T$, for all $T$.

\begin{claim}\label{claim: prescriptor unif ineq}
There exists $\Pb_1 \in \cPin$, $\varepsilon>0$, and an increasing sequence $(t_T)_{T\geq 1} \in \integ^\integ$, such that for all $\varepsilon'>0$, there exists an open ball $\mathcal{B}(\Pb_1,r)$ around $\Pb_1$ of radius $r>0$ such that
$$
\hat{c}^{\star}(\Pb,t_T) 
+\varepsilon  \sqrt{\alpha_{t_T}} 
-\varepsilon' \|\Pb-\Pb_1\| + c(\xSVP{,t_T}(\Pb_1),\Pb_1-\Pb)
   - c(\hat{x}_{t_T}(\Pb),\Pb_1- \Pb)
\leq
\cSVP^{\star}(\Pb,t_T),
$$
for all $\Pb \in \mathcal{B}(\Pb_1,r)$ and $T\in \integ$. 
\end{claim}
\begin{proof}
See Appendix \ref{Appendix: proofs of lower bound presc}.
\end{proof}

Let $\varepsilon$, $\Pb_1$ and $(t_T)_{T \geq 1}$ given by Claim \ref{claim: prescriptor unif ineq}. Set $\varepsilon'>0$ to be specified later. Let $r>0$ given by Claim \ref{claim: prescriptor unif ineq}, such that for all $\Pb \in \mathcal{B}(\Pb_1,r)$ and $T\geq 1$
\begin{equation}\label{eq: robustOpt precriptor: uniformoty ineq. }
    \hat{c}^{\star}(\Pb,t_T) 
    \leq
    \cSVP^{\star}(\Pb,t_T) 
    - \varepsilon \sqrt{\alpha_{t_T}} 
    + \varepsilon' \|\Pb - \Pb_1\|
    +c(\xSVP{,t_T}(\Pb),\Pb_1- \Pb)
    - c(\hat{x}_{t_T}(\Pb_1),\Pb_1-\Pb).
\end{equation}
In the reminder of the proof, we will show that the out-of-sample guarantee \eqref{eq: prescriptor out-of-sample ganrantee} for $\hat{c}$ fails at $\Pb_1$, which contradicts the feasibility of $(\hat{c},\hat{x})$. In order to do that, we will construct a sequence of distributions $(\Pb_T)_{T\geq1}$ converging sufficiently slowly to $\Pb_1$ and where the out-of-sample guarantee always fails. We will then use the Moderate Deviation Principle, Theorem \ref{thm: MDP finite space}.  \\
Let us first recall some notions that will be used in this proof. Recall the function $\varphi$ defined in Proposition \ref{prop: robust predictor is DRO.}. The proposition ensures that $\cSVP(x,\Pb,T) = c(x,\Pb + \sqrt{\alpha_T} \varphi_x(\Pb))$ for all $x,\Pb \in \cX \times \cPin$ and $T\in \integ$. Moreover, immediate computations\footnote{See Lemma \ref{lemma: Cov expression}} imply that $c(x_1,\varphi_{x_2}(\Pb)) = \Cov_{\Pb}(\loss(x_1,\xi),\loss(x_2,\xi)) / \sqrt{\Var_{\Pb}(\loss(x_2,\xi))}$ where 
$\Cov_{\Pb}(\loss(x_1,\xi),\loss(x_2,\xi)):= \Eb_{\Pb}(\loss(x_1,\xi)\loss(x_2,\xi)) - \mathbb{E}_{\Pb}(\loss(x_1,\xi))\mathbb{E}_{\Pb}(\loss(x_2,\xi))$ for all $x_1,x_2 \in \cX$, when $\Var_{\Pb}(\loss(x_2,\xi)) >0$.

We start by setting key ingredients in constructing the sequence $\Pb_T$. $(\hat{x}_T(\Pb_1))_{T\geq 1}$ lives in the compact $\cX$, we can therefore assume without loss of generality that $(\hat{x}_{t_T}(\Pb_1))_{T\geq 1}$ converges to a limit $x_1 \in \cX$.

\begin{claim}\label{claim: optimality of x_0}
$x_1 \in \argmin_{x \in \cX}c(x,\Pb_1)$.
\end{claim}
\begin{proof}
We have $|c(\hat{x}_{t_T}(\Pb_1),\Pb_1)-c^{\star}(\Pb_1)| \leq |c(\hat{x}_{t_T}(\Pb_1), \Pb_1)-\hat{c}(\hat{x}_{t_T}(\Pb_1),\Pb_1,t_T)| + |\hat{c}^{\star}(\Pb_1,t_T) -c^{\star}(\Pb_1)| \xrightarrow[T\to \infty]{} 0$ by uniform convergence of $\hat{c}(\cdot,\cdot,t_T)$ to $c(\cdot,\cdot)$ (see Lemma \ref{lemma: convergence of minimum} for details on the convergence of the second term).  This implies by continuity of $c(\cdot,\Pb_1)$ and convergence of $(\hat{x}_{t_T}(\Pb_1))_{T\geq 1}$ to $x_1$ that $c(x_1,\Pb_1) = c^{\star}(\Pb_1)$ which gives the desired result.
\end{proof}

Let $\eta \in (0,1)$ and $\delta>0$ to be specified later.
Let
$\Gamma = \mathcal{B}(\varphi_{x_1}(\Pb_1), \frac{\delta}{\sup_x \|\loss(x,\cdot)\|})$ the open ball in the set of measures that sum to zero, $\cP_{0,\infty}$, centered around $\varphi_{x_1}(\Pb_1)$ of radius $\frac{\delta}{\sup_x \|\loss(x,\cdot)\|}$.
For all $u \in \Gamma$ and $T \in \integ$, we consider the sequence
$$\Pb_T(u) = \Pb_1 - (1-\eta)\sqrt{\alpha_T} u,$$

for we which we show that $(\hat{x},\hat{c})$ is always disappointing, i.e., the guarantee \eqref{eq: prescriptor out-of-sample ganrantee} fails to hold, in the subsequence $(t_T)_{T\geq 1}$. See Figure \ref{fig: PTu construction} for an illustration of the construction.

Fix $u\in \Gamma$. Let us examine the sign of $c(\hat{x}_{t_T}(\Pb_{t_T}(u)),\Pb_1) - \hat{c}^{\star}(\Pb_{t_T}(u),t_T)$. To simplify notations, we denote $\Pb_T:= \Pb_T(u)$. 
As $\Pb_T \to \Pb_1$, we can assume without loss of generality that $\Pb_{t_T} \in \mathcal{B}(\Pb_1,r)$ for all $T$ by extraction accordingly from $(t_T)_{T\geq 1}$. Using \eqref{eq: robustOpt precriptor: uniformoty ineq. } to bound $\hat{c}^{\star}(\Pb_{t_T},t_T)$, we have for all $T \in \integ$
\begin{align*}
    c(\hat{x}_{t_T}(\Pb_T),\Pb_1) - \hat{c}^{\star}(\Pb_{t_T},t_T)
    &\geq 
     c(\hat{x}_{t_T}(\Pb_{t_T}),\Pb_1) - \cSVP^{\star}(\Pb_{t_T},t_T) \\
    &  \;  + \varepsilon \sqrt{\alpha_{t_T}} - \varepsilon' \|\Pb_{t_T} - \Pb_1\| \\
    &   \;  - c(\xSVP{,t_T}(\Pb_{t_T}), \Pb_1 - \Pb_{t_T})
       +c(\hat{x}_{t_T}(\Pb_1),\Pb_1 - \Pb_{t_T}). \numberthis \label{proof eq: prescriptor main ineq}
\end{align*}

We start by analyzing the first term $c(\hat{x}_{t_T}(\Pb_{t_T}),\Pb_1)
- \cSVP^{\star}(\Pb_{t_T},t_T)$. Using the minimality of $c^{\star}(\Pb_1)$, we have for all $T \in \integ$
\begin{align}\label{proof eq: first term}
c(\hat{x}_{t_T}(\Pb_{t_T}),\Pb_1)
- \cSVP^{\star}(\Pb_{t_T},t_T) 
&\geq
c^{\star}(\Pb_1) - \cSVP^{\star}(\Pb_{t_T},t_T)
\end{align}
In all what follow, the $o$ notation hides constants independent of $u$, therefore, the asymptotic notation is uniform in $u$.

\begin{claim}\label{claim: c_r star - c star P0}
The follow inequalities hold
$$\eta \sqrt{\Var_{\Pb_1}(\loss(\xSVP{,T}(\Pb_T),\xi))} 
- \delta
+ o(1)
\leq
\frac{1}{\sqrt{\alpha_T}}[\cSVP^{\star}(\Pb_T,T) - c^{\star}(\Pb_1)]
\leq
\eta  \sqrt{\Var_{\Pb_1}(\loss(x_1,\xi))} 
+ \delta
+ o(1).
$$
\end{claim}
\begin{proof}
Let $T\in \integ$. Let us first prove the LHS inequality.
\begingroup
\allowdisplaybreaks
\begin{align*}
    \cSVP^{\star}(\Pb_T,T) 
    &=
    c(\xSVP{,T}(\Pb_T),\Pb_T) + \sqrt{\alpha_T} \sqrt{\Var_{\Pb_T}(\loss(\xSVP{,T}(\Pb_T),\xi))} \\
    &= c(\xSVP{,T}(\Pb_T),\Pb_1) 
    - (1-\eta)\sqrt{\alpha_T} c(\xSVP{,T}(\Pb_T),u) 
    +\sqrt{\alpha_T} \sqrt{\Var_{\Pb_T}(\loss(\xSVP{,T}(\Pb_T),\xi))}
    \\
    &\geq c(\xSVP{,T}(\Pb_T),\Pb_1) 
    - (1-\eta)\sqrt{\alpha_T} c(\xSVP{,T}(\Pb_T),\varphi_{x_1}(\Pb_1)) 
    +\sqrt{\alpha_T} \sqrt{\Var_{\Pb_T}(\loss(\xSVP{,T}(\Pb_T),\xi))}
   -\delta  \sqrt{\alpha_T}
    \\
    &\geq c^{\star}(\Pb_1) 
    - (1-\eta)\sqrt{\alpha_T}\frac{\Cov_{\Pb_1}(\loss(\xSVP{,T}(\Pb_T),\xi),\loss(x_1,\xi))}{\sqrt{\Var_{\Pb_1}(\loss(x_1,\xi))}}
    +\sqrt{\alpha_T} \sqrt{\Var_{\Pb_T}(\loss(\xSVP{,T}(\Pb_T),\xi))}
    -\delta  \sqrt{\alpha_T}
    \\
    &\geq c^{\star}(\Pb_1) 
    - (1-\eta)\sqrt{\alpha_T} \sqrt{\Var_{\Pb_1}(\loss(\xSVP{,T}(\Pb_T),\xi))} 
    +\sqrt{\alpha_T} \sqrt{\Var_{\Pb_T}(\loss(\xSVP{,T}(\Pb_T),\xi))}
    -\delta  \sqrt{\alpha_T}
    \\
    &= c^{\star}(\Pb_1) + \eta \sqrt{\alpha_T} \sqrt{\Var_{\Pb_1}(\loss(\xSVP{,T}(\Pb_T),\xi))} 
    -\delta  \sqrt{\alpha_T}
    + o(\sqrt{\alpha_T})
\end{align*}
\endgroup
where the first equality uses the robust predictor's formula \eqref{eq: Robust predictor SVP.}, the first inequality uses the definition of $\Gamma$, the second inequality uses the minimality of $c^{\star}$, the last inequality Cauchy-Schartz inequality  and the last equality uses the fact that $\Var_{\Pb_T}(\loss(x_T,\xi)) = \Var_{\Pb_1}(\loss(x_T,\xi)) + o(1)$ for all $(x_T)_{T\geq 1} \in \cX^\integ$ (see Lemma \ref{lemma: convergence of Var} for details). 

We now turn to the RHS. Similarly, we have
\begingroup
\allowdisplaybreaks
\begin{align*}
    \cSVP^{\star}(\Pb_T,T) 
    \leq 
    \cSVP(x_1,\Pb_T,T)
    &= c(x_1,\Pb_T) + \sqrt{\alpha_T} \sqrt{\Var_{\Pb_T}(\loss(x_1,\xi))}\\
    &= c(x_1,\Pb_1) 
    -(1-\eta)\sqrt{\alpha_T} c(x_1,u)
    + \sqrt{\alpha_T} \sqrt{\Var_{\Pb_T}(\loss(x_1,\xi))}\\
    &\leq c(x_1,\Pb_1) 
    -(1-\eta)\sqrt{\alpha_T} c(x_1,\varphi_{x_1}(\Pb_1))
    + \sqrt{\alpha_T} \sqrt{\Var_{\Pb_T}(\loss(x_1,\xi))}
    + \delta \sqrt{\alpha_T} \\
    &= c(x_1,\Pb_1) 
    -(1-\eta)\sqrt{\alpha_T} \sqrt{\Var_{\Pb_1}(\loss(x_1,\xi))}
    + \sqrt{\alpha_T} \sqrt{\Var_{\Pb_T}(\loss(x_1,\xi))}
    + \delta \sqrt{\alpha_T} 
    \\
    &= c(x_1,\Pb_1) + \eta \sqrt{\alpha_T} \sqrt{\Var_{\Pb_1}(\loss(x_1,\xi))} 
    + \delta \sqrt{\alpha_T} 
    + o(\sqrt{\alpha_T})
\end{align*}
\endgroup
Using Claim \ref{claim: optimality of x_0}, we have $c(x_1,\Pb_1) = c^{\star}(\Pb_1)$, therefore, the last inequality gives the desired result.
\end{proof}

Using the RHS of Claim \ref{claim: c_r star - c star P0} with \eqref{proof eq: first term}, we get
\begin{align*}
c(\hat{x}_{t_T}(\Pb_{t_T}),\Pb_1)
- \cSVP^{\star}(\Pb_{t_T},t_T) 
&\geq
-\eta \sqrt{\alpha_{t_T}} \sqrt{\Var_{\Pb_1}(\loss(x_1,\xi))} 
- \delta  \sqrt{\alpha_{t_T}}
+o(\sqrt{\alpha_{t_T}}). \numberthis \label{proof eq: first term ineq}
\end{align*}

The second term of \eqref{proof eq: prescriptor main ineq} can be written as 
\begin{equation}\label{proof eq: second term}
\varepsilon \sqrt{\alpha_{t_T}} 
- \varepsilon' \|\Pb_{t_T} - \Pb_1\|
= (\varepsilon - \varepsilon'(1-\eta)\|u\|) \sqrt{\alpha_{t_T}}.
\end{equation}
Let us now examine the last term of \eqref{proof eq: prescriptor main ineq}. We have
\begingroup
\allowdisplaybreaks
\begin{align*}
c(\hat{x}_{t_T}(\Pb_1),&\Pb_1 - \Pb_{t_T}) - c(\xSVP{,t_T}(\Pb_{t_T}), \Pb_1 - \Pb_{t_T}) \\
&=
\sqrt{\alpha_{t_T}} 
    \left(
    c(\hat{x}_{t_T}(\Pb_1),u)
    -
    c(\xSVP{,t_T}(\Pb_{t_T}), u)
    \right)\\
&\geq
\sqrt{\alpha_{t_T}} 
    \left(
    c(\hat{x}_{t_T}(\Pb_1),\varphi_{x_1}(\Pb_1))
    -
    c(\xSVP{,t_T}(\Pb_{t_T}), \varphi_{x_1}(\Pb_1))
    - 2\delta
    \right)  \\
&=
    \sqrt{\alpha_{t_T}} 
    \left(
    \sqrt{\Var_{\Pb_1}(\loss(x_1,\xi))}
    -
    \frac{\Cov_{\Pb_1}(\loss(\xSVP{,T}(\Pb_{t_T}),\xi), \loss(x_1,\xi))}{\sqrt{\Var_{\Pb_1}(\loss(x_1,\xi))}}
    \right) 
    - 2\delta \sqrt{\alpha_{t_T}}  
    \\
&\geq
    \sqrt{\alpha_{t_T}} 
    \left(
    \sqrt{\Var_{\Pb_1}(\loss(x_1,\xi))}
    -
    \sqrt{\Var_{\Pb_1}(\loss(\xSVP{,t_T}(\Pb_{t_T}),\xi))}
    \right)
    - 2\delta \sqrt{\alpha_{t_T}} 
\end{align*}
\endgroup
where the first inequality is by definition of $\Gamma$ and the last inequality is due to Cauchy–Schwarz inequality. Hence, using Claim \ref{claim: c_r star - c star P0}, we have\footnote{Notice that we can also show that this quantity is non-positive, giving therefore upper and lower bounds.}
\begin{equation}\label{proof eq: presc third term}
  c(\hat{x}_{t_T}(\Pb_1),\Pb_1 - \Pb_{t_T}) - c(\xSVP{,t_T}(\Pb_{t_T}), \Pb_1 - \Pb_{t_T})
  \geq
  -\left(2\delta/\eta + 2\delta \right) \sqrt{\alpha_{t_T}}  +o(\sqrt{\alpha_{t_T}})
\end{equation}

Combining the lower bounds \eqref{proof eq: first term ineq}, \eqref{proof eq: second term}, \eqref{proof eq: presc third term} on the three terms of \eqref{proof eq: prescriptor main ineq}, we get
\begin{align*}
    c(\hat{x}_{t_T}(\Pb_{t_T}),\Pb_1) - \hat{c}^{\star}(\Pb_{t_T},t_T)
    &\geq
    \sqrt{\alpha_{t_T}} 
    \left(
    \varepsilon
    - 
    \varepsilon'(1-\eta)\|u\|
    -
    \eta \sqrt{\Var_{\Pb_1}(\loss(x_1,\xi))}
    -3\delta - 2\delta/\eta
    + o(1)
    \right)\\
    &\geq
    \sqrt{\alpha_{t_T}} 
    \left(
    \varepsilon
    - 
    \varepsilon'\frac{(1-\eta)\delta}{\sup_x \|\loss(x,\cdot)\|}
    -
    \eta \sqrt{\Var_{\Pb_1}(\loss(x_1,\xi))}
    -3\delta - 2\delta/\eta
    + o(1)
    \right)
\end{align*}
By choosing $\eta>0$ such that $\eta \sup_{x\in \cX}\sqrt{\Var_{\Pb_1}(\loss(x,\xi))}< \varepsilon/4$ (which is possible as $\cX$ is compact, $\loss(\cdot,i)$ is continuous for all $i\in \Sigma$, and therefore the $\sup$ is finite), then $\delta>0$ such that $\delta (3+2/\eta) < \varepsilon/4$ and then $\varepsilon'>0$ such that $\varepsilon'\frac{(1-\eta)\delta}{\sup_x \|\loss(x,\cdot)\|} < \varepsilon/4$ we get 
$$
\varepsilon
    >
    \varepsilon'\frac{(1-\eta)\delta}{\sup_{x\in \cX} \|\loss(x,\cdot)\|}
    -
    \eta \sqrt{\Var_{\Pb_1}(\loss(x_1,\xi))}
    -3\delta - 2\delta/\eta.
$$
Hence, there exist $T_0 \in \integ$ such that for all $T\geq T_0$ and $u\in \Gamma$, we have 
    $\Pb_{t_T}(u) \in \mathcal{D}_{t_T}$
where
    $$
    \mathcal{D}_{T} :=
    \set{\Pb'\in \cP}{c(\hat{x}_{T}(\Pb'),\Pb_1) - \hat{c}^{\star}(\Pb',T) >0}, $$
is the set of disappointing distributions at time $T$. This implies that $\Pb_1 -(1-\eta)\sqrt{\alpha_{t_T}} \cdot\Gamma \subset \mathcal{D}_{t_T}$ for all $T\geq T_0$.
We have
\begingroup
\allowdisplaybreaks
\begin{align}
    \limsup_{T\to\infty} 
    \frac{1}{a_T} \log \Pb^\infty 
    \left(
        c(\hat{x}_T(\hat{\Pb}_T), \Pb_1) > \hat c^{\star}(\hat{\Pb}_T,T)
    \right)
    &=
    \limsup_{T\to\infty} 
    \frac{1}{a_T} \log \Pb^\infty \left( \hat{\Pb}_T \in \mathcal{D}_T\right) \nonumber\\
    &\geq
    \limsup_{T\to\infty} 
    \frac{1}{a_{t_T}} \log \Pb^\infty \left( \hat{\Pb}_{t_T} \in \mathcal{D}_{t_T}\right) \nonumber\\
    &\geq \limsup_{T\to\infty} 
    \frac{1}{a_{t_T}}  \log \Pb^\infty \left( \hat{\Pb}_{t_T} - \Pb_1 \in -\sqrt{\alpha_{t_T}}(1-\eta) \Gamma\right) \nonumber \\
    &= \limsup_{T\to\infty} 
    \frac{1}{a_{t_T}} \log \Pb^\infty \left( \hat{\Pb}_{t_T} - \Pb_1 \in \sqrt{\frac{a_{t_T}}{t_T}}\cdot -(1-\eta) \sqrt{2}\Gamma\right) \nonumber\\
    &\geq \liminf_{T\to\infty} 
    \frac{1}{a_{T}} \log \Pb^\infty \left( \hat{\Pb}_{T} - \Pb_1 \in \sqrt{\frac{a_{T}}{T}}\cdot -(1-\eta) \sqrt{2}\Gamma\right) 
    \nonumber \\ 
    &\geq - \inf_{\Delta \in \Gamma^{\into}} \|(1-\eta)\sqrt{2}\Delta\|_{\Pb_1}^2 
     \label{proof eq: precriptor series of ineq 3} \\
    &\geq - (1-\eta)^2 \|\sqrt{2}\varphi_{x_1}(\Pb_1)\|_{\Pb_1}^2 =  - (1-\eta)^2 >-1     \label{proof eq: precriptor series of ineq 4} 
\end{align}
\endgroup
where \eqref{proof eq: precriptor series of ineq 3} uses the MDP (Theorem \ref{thm: MDP finite space}) and  \eqref{proof eq: precriptor series of ineq 4} is justified by $\varphi_{x_1}(\Pb_1) \in \Gamma = \Gamma^{\into}$ and $\|\sqrt{2}\varphi_{x_1}(\Pb_1)\|_{\Pb_1} =1$. This inequality contradicts the feasibility of $(\hat{x},\hat{c})$ which completes the proof.
\end{proof}

\begin{proof}[Proof of Theorem \ref{thm: strong optimality mdp prescriptor.}]
Let $(\hat{x},\hat{c})\in \cX \times \cC$ be a feasible pair of predictor-prescriptor in \eqref{eq:optimal-predictor}. Let $\Pb \in \cPin$. The goal is to show that 
$$
\limsup_{T\to \infty} \frac{|\cSVP^{\star}(\Pb,T) - c^{\star}(\Pb)|}{|\hat{c}^{\star}(\Pb,T) - c^{\star}(\Pb)|} \leq 1.
$$
It suffices to show that 
for all sequence $(\beta_T)_{T\geq 1} \in \Re_+^{\integ}$ we have\footnote{Similarly to the case of predictors, this is equivalent to the desired result, see Lemma \ref{lemma: equivalence of order notions.}.}
\begin{equation*}
    \limsup_{T\to\infty} \frac{1}{\beta_T} |\hat{c}^{\star}(\Pb,T) - c^{\star}(\Pb)| \geq \limsup_{T\to\infty} \frac{1}{\beta_T} |\cSVP^{\star}(\Pb,T) - c^{\star}(\Pb)|.
\end{equation*}
The result follows with $\beta_T = |\hat{c}^{\star}(\Pb,T) - c^{\star}(\Pb)|$ for all $T$.
Let $\Pb \in \cPin$. Proposition \ref{prop: rate of cv of prescriptor} ensures that $\lim_{T\to \infty} \sqrt{T/a_T}(\cSVP^{\star}(\Pb,T) - c^{\star}(\Pb)) = \sqrt{2\Var_{\Pb}(\loss(x^{\star}(\Pb),\xi))}$ where $x^{\star}(\Pb)$ is a minimizer of $c(\cdot,\Pb)$ that has the lowest variance. When $\sqrt{\Var_{\Pb}(\loss(x^{\star}(\Pb),\xi))}>0$, the proof is analogue to the proof of Theorem \ref{thm: strong optimality subexp} using Proposition \ref{prop: lower bound strong opt presc}. Suppose $\sqrt{\Var_{\Pb}(\loss(x^{\star}(\Pb),\xi))}=0$. For every $x\in \cX$, we have
$
\cSVP(x,\Pb,T) = c(x,\Pb) + \sqrt{2a_T/T}\sqrt{\Var_{\Pb}(\loss(x,\xi))} \geq c(x^{\star}(\Pb),\Pb) = c(x^{\star}(\Pb),\Pb)+ \sqrt{2a_T/T}\sqrt{\Var_{\Pb}(\loss(x^{\star}(\Pb),\xi))} = \cSVP(x^{\star}(\Pb),\Pb,T)$, where the first inequality uses the minimality of $c(x^{\star}(\Pb),\Pb)$ and the second equality uses $\Var_{\Pb}(\loss(x^{\star}(\Pb),\xi))=0$. Hence $\cSVP^{\star}(\Pb,T) = \cSVP(x^{\star}(\Pb),\Pb,T) = c^{\star}(\Pb)$ which implies that the RHS of the desired inequality is 0 and the result is immediate.
\end{proof}

We note that the SVP prescriptor enjoys also finite sample guarantees. These guarantees usually depend on complexity measures of the decision set $\cX$ (function class in machine learning). We refer to Theorem 6 of \cite{maurer2009empirical} for such a finite sample bound.

{\color{black}
Finally, as in the predictor case, we note that the KL prescriptor is asymptotically equivalent to the SVP prescriptor in the subexponential regime, which indicates the optimality of the KL prescriptor both in the exponential and subexponential regimes. This is in contrast with SVP and $\chi^2$ DRO which are optimal only in the subexponential regime. Consider the prescription $\hat{x}'_{\text{KL},T}(\Qb) \in \argmin_{x\in \cX} \hat{c}_{\text{KL}}'(x,\Qb)$ for all $\Qb \in \cP$ associated with KL predictor $\hat{c}_{\text{KL}}'$ with radius $a_T/T$ as defined in \eqref{eq: KL in subexp}.

The following result follows immediately from Proposition \ref{prop: KL equiv to SVP}. The proof is in Appendix \ref{App: proof of KL prescriptor approx}.

\begin{proposition}\label{prop: KL prescriptor approx}
We have for all $\Pb \in \cPin$
\begin{align*}
\min_{x\in\cX} \sup_{\Pb' \in \cP} \left\{ c(x,\Pb') \; : \; I(\Pb,\Pb') \leq \frac{a_T}{T} \right\}
&=
\min_{x\in\cX}
\sup_{\Pb' \in \cP} \left\{ c(x,\Pb') \; : \; \|\Pb' - \Pb \|_{\Pb}^2 \leq \frac{a_T}{T} \right\} + o\left( \sqrt{\frac{a_T}{T}}\right)\\
&=
\cSVP^\star(\Pb,T) +o\left( \sqrt{\frac{a_T}{T}}\right).
\end{align*}
Furthermore, is $\Var_{\Qb}(\loss(x^\star(\Qb),\xi)) >0$ for all $\Qb \in \cPin$, then $(\xSVP,\cSVP) \equiv_{\hat{\cX}} (\hat{x}'_{\text{KL}},\cKL')$.
\end{proposition}

The KL prescriptor is also feasible in the subexponential regime with an appropriately chosen radius. The proof is in Appendix \ref{App: proof of feasibility of KL presc in subsexp}.

\begin{proposition}\label{prop: feasibility of KL presc in subsexp}
    The prescriptor $(\hat{c}_{\text{KL}}', \hat{x}_{\text{KL}}')$ verifies the out-of-sample guarantee \eqref{eq: prescriptor out-of-sample ganrantee} when $a_T \ll T$.
\end{proposition}

}

\section{Discussion and Generalizations}
In this paper we propose a framework to construct optimal data-driven learning and decision-making formulations.
We prove that within our framework such optimal formulations do indeed exist and can be made explicit.
In this section we discuss the limitations of our approach and discuss potential directions of generalization.

Perhaps the most restrictive assumption we make is that the set of possible uncertain scenarios $\Sigma$ is finite.
Our framework holds also for continuous distributions, however, our optimality results use the assumption.
We believe that this assumption is warranted as it avoids intricate topological problems which arise when working with continuous probability measures.
It allows us in particular to state optimal formulations and discuss interesting phase transitions using mostly elementary arguments.
That being said, we suspect that our optimality results can be at least partially generalized to the continuous setting following the strategy proposed by \cite{van2020data}.
Indeed, our constructed optimal predictors and prescriptors have each a natural continuous generalization in all of the considered regimes.
We must point out however, that the subexponential regime considered here requires more delicate arguments than the exponential regime considered in \cite{van2020data}. For instance, our proofs for this regime require using the notions of continuity, differentiability, uniform convergence, gradient and asymptotic development of predictors as functions of finite distributions. If predictors are functions of continuous measures, these topological and analytical properties become delicate and give rise to complex considerations and potential pathological behaviors.
Hence we stop short from claiming that this generalization is straightforward.

We assume in our paper, as is common, that the data samples are independent and identically distributed. We note though that we took some care in our proofs to not exploit this fact explicitly. Indeed, we merely use that the empirical distribution $\hat{\Pb}_T$ constructed enjoys certain deviation principles (Theorem \ref{thm: LDP finite space} and \ref{thm: MDP finite space}).
Hence, our analysis should with some effort remain applicable in case the data is generated by any process for which $\hat{\Pb}_T$ verifies the appropriate deviation principles. We refer to \cite{zeitouni1998large} for several examples of such processes.

Naturally, our optimal formulations depend on the precise notion of optimality considered. The optimality criteria related to the considered out-of-sample guarantee and our considered particular partial order. As for the out-of-sample guarantee, this is a rather standard choice for enforcing out-of-sample performance and is widely adopted by the operations research and machine learning community as we do point out in Section \ref{sec: data-driven decision making}. For the partial order, other choices may be of interest too as we discuss in Section \ref{sec: data-driven decision making}. An interesting question is whether the optimal formulations we identify remain optimal also for slightly different partial order and perhaps more fundamentally whether such alternatives even admit a similar notion of strong optimal formulation. Our chosen order has the benefit of being stronger than some other natural orders, such as bias and $L^1$ error. This implies that at the very least our formulations remain optimal also when considering any such weaker orders.


%
%

\section*{Statements and Declarations}

\paragraph{Funding} No funding to be reported.

\paragraph{Conflicts of interests/Competing interests} Authors do not identify any potential conflict or competing interests relative to this work.

\bibliographystyle{spmpsci}      
\bibliography{References}   

\begin{thebibliography}{10}
\providecommand{\url}[1]{{#1}}
\providecommand{\urlprefix}{URL }
\expandafter\ifx\csname urlstyle\endcsname\relax
  \providecommand{\doi}[1]{DOI~\discretionary{}{}{}#1}\else
  \providecommand{\doi}{DOI~\discretionary{}{}{}\begingroup
  \urlstyle{rm}\Url}\fi

\bibitem{arzela1895sulle}
Arzel{\`a}, C.: Sulle funzioni di linee.
\newblock Gamberini e Parmeggiani (1895)

\bibitem{ascoli1884curve}
Ascoli, G.: Le curve limite di una variet{\`a} data di curve.
\newblock Coi tipi del Salviucci (1884)

\bibitem{audibert2009exploration}
Audibert, J., Munos, R., Szepesv{\'a}ri, C.: Exploration--exploitation tradeoff
  using variance estimates in multi-armed bandits.
\newblock Theoretical Computer Science \textbf{410}(19), 1876--1902 (2009)

\bibitem{bahadur1967rates}
Bahadur, R.R.: Rates of convergence of estimates and test statistics.
\newblock The Annals of Mathematical Statistics \textbf{38}(2), 303--324 (1967)

\bibitem{ben2009robust}
Ben-Tal, A., El~Ghaoui, L., Nemirovski, A.: Robust optimization.
\newblock Princeton university press (2009)

\bibitem{bennouna2022holistic}
Bennouna, A., Van~Parys, B.: Holistic robust data-driven decisions.
\newblock arXiv preprint arXiv:2207.09560  (2022)

\bibitem{bertsimas2018data}
Bertsimas, D., Gupta, V., Kallus, N.: Data-driven robust optimization.
\newblock Mathematical Programming \textbf{167}(2), 235--292 (2018)

\bibitem{bertsimas2018robust}
Bertsimas, D., Gupta, V., Kallus, N.: Robust sample average approximation.
\newblock Mathematical Programming \textbf{171}(1), 217--282 (2018)

\bibitem{besbes2023big}
Besbes, O., Mouchtaki, O.: How big should your data really be? data-driven
  newsvendor: learning one sample at a time.
\newblock Management Science  (2023)

\bibitem{delage2010distributionally}
Delage, E., Ye, Y.: Distributionally robust optimization under moment
  uncertainty with application to data-driven problems.
\newblock Operations research \textbf{58}(3), 595--612 (2010)

\bibitem{duchi2016variance}
Duchi, J., Namkoong, H.: Variance-based regularization with convex objectives.
\newblock arXiv preprint arXiv:1610.02581  (2016)

\bibitem{duchi2021statistics}
Duchi, J.C., Glynn, P.W., Namkoong, H.: Statistics of robust optimization: A
  generalized empirical likelihood approach.
\newblock Mathematics of Operations Research \textbf{46}(3), 946--969 (2021)

\bibitem{dunford1958linear}
Dunford, N., Schwartz, J.: Linear Operators: General Theory, vol.~1.
\newblock Wiley-Interscience (1958)

\bibitem{gao2017wasserstein}
Gao, R., Chen, X., Kleywegt, A.J.: Wasserstein distributional robustness and
  regularization in statistical learning.
\newblock arXiv e-prints pp. arXiv--1712 (2017)

\bibitem{gao2016distributionally}
Gao, R., Kleywegt, A.: Distributionally robust stochastic optimization with
  wasserstein distance.
\newblock arXiv preprint arXiv:1604.02199  (2016)

\bibitem{garivier2011kl}
Garivier, A., Capp{\'e}, O.: The kl-ucb algorithm for bounded stochastic
  bandits and beyond.
\newblock In: Proceedings of the 24th annual conference on learning theory, pp.
  359--376. JMLR Workshop and Conference Proceedings (2011)

\bibitem{van2016estimation}
van~de Geer, S.: Estimation and testing under sparsity.
\newblock Springer (2016)

\bibitem{gotoh2018robust}
Gotoh, J.y., Kim, M.J., Lim, A.E.: Robust empirical optimization is almost the
  same as mean--variance optimization.
\newblock Operations research letters \textbf{46}(4), 448--452 (2018)

\bibitem{gotoh2021calibration}
Gotoh, J.y., Kim, M.J., Lim, A.E.: Calibration of distributionally robust
  empirical optimization models.
\newblock Operations Research \textbf{69}(5), 1630--1650 (2021)

\bibitem{gupta2015near}
Gupta, V.: Near-optimal bayesian ambiguity sets for distributionally robust
  optimization.
\newblock Management Science \textbf{65}(9), 4242--4260 (2019)

\bibitem{hoeffding1994probability}
Hoeffding, W.: Probability inequalities for sums of bounded random variables.
\newblock In: The Collected Works of Wassily Hoeffding, pp. 409--426. Springer
  (1994)

\bibitem{hoerl1970ridge1}
Hoerl, A., Kennard, R.: Ridge regression: applications to nonorthogonal
  problems.
\newblock Technometrics \textbf{12}(1), 69--82 (1970)

\bibitem{hoerl1970ridge2}
Hoerl, A., Kennard, R.: Ridge regression: Biased estimation for nonorthogonal
  problems.
\newblock Technometrics \textbf{12}(1), 55--67 (1970)

\bibitem{jongeneel2021efficient}
Jongeneel, W., Sutter, T., Kuhn, D.: Efficient learning of a linear dynamical
  system with stability guarantees.
\newblock arXiv preprint arXiv:2102.03664  (2021)

\bibitem{jongeneel2021topological}
Jongeneel, W., Sutter, T., Kuhn, D.: Topological linear system identification
  via moderate deviations theory.
\newblock IEEE Control Systems Letters  (2021)

\bibitem{koltchinskii2011nuclear}
Koltchinskii, V., Lounici, K., Tsybakov, A.: Nuclear-norm penalization and
  optimal rates for noisy low-rank matrix completion.
\newblock The Annals of Statistics \textbf{39}(5), 2302--2329 (2011)

\bibitem{kuhn2019wasserstein}
Kuhn, D., Esfahani, P., Nguyen, V., Shafieezadeh-Abadeh, S.: Wasserstein
  distributionally robust optimization: Theory and applications in machine
  learning.
\newblock In: Operations Research \& Management Science in the Age of
  Analytics, pp. 130--166. INFORMS (2019)

\bibitem{lam2016robust}
Lam, H.: Robust sensitivity analysis for stochastic systems.
\newblock Mathematics of Operations Research \textbf{41}(4), 1248--1275 (2016)

\bibitem{lam2019recovering}
Lam, H.: Recovering best statistical guarantees via the empirical
  divergence-based distributionally robust optimization.
\newblock Operations Research \textbf{67}(4), 1090--1105 (2019)

\bibitem{lam2021impossibility}
Lam, H.: On the impossibility of statistically improving empirical
  optimization: A second-order stochastic dominance perspective.
\newblock arXiv preprint arXiv:2105.13419  (2021)

\bibitem{lehmann2006theory}
Lehmann, E., Casella, G.: Theory of point estimation.
\newblock Springer Science \& Business Media (2006)

\bibitem{li1986asymptotic}
Li, K.C.: Asymptotic optimality of cl and generalized cross-validation in ridge
  regression with application to spline smoothing.
\newblock The Annals of Statistics pp. 1101--1112 (1986)

\bibitem{Markowitz1952protfolio}
Markowitz, H.: Portfolio selection.
\newblock The Journal of Finance \textbf{7}(1), 77--91 (1952)

\bibitem{maurer2009empirical}
Maurer, A., Pontil, M.: Empirical bernstein bounds and sample variance
  penalization.
\newblock arXiv preprint arXiv:0907.3740  (2009)

\bibitem{michaud1989markowitz}
Michaud, R.: The {M}arkowitz optimization enigma: Is `optimized' optimal?
\newblock Financial Analysts Journal \textbf{45}(1), 31--42 (1989)

\bibitem{mulvey1995robust}
Mulvey, J., Vanderbei, R., Zenios, S.: Robust optimization of large-scale
  systems.
\newblock Operations research \textbf{43}(2), 264--281 (1995)

\bibitem{puhalskii2007large}
Puhalskii, A.A., Vladimirov, A.A.: A large deviation principle for join the
  shortest queue.
\newblock Mathematics of Operations Research \textbf{32}(3), 700--710 (2007)

\bibitem{shapiro2003monte}
Shapiro, A.: Monte carlo sampling methods.
\newblock Handbooks in operations research and management science \textbf{10},
  353--425 (2003)

\bibitem{smith2006optimizer}
Smith, J., Winkler, R.: The optimizer's curse: Skepticism and postdecision
  surprise in decision analysis.
\newblock Management Science \textbf{52}(3), 311--322 (2006)

\bibitem{sutter2020general}
Sutter, T., Van~Parys, B.P., Kuhn, D.: A general framework for optimal
  data-driven optimization.
\newblock arXiv preprint arXiv:2010.06606  (2020)

\bibitem{tibshirani1996regression}
Tibshirani, R.: Regression shrinkage and selection via the lasso.
\newblock Journal of the Royal Statistical Society: Series B (Methodological)
  \textbf{58}(1), 267--288 (1996)

\bibitem{tikhonov1943stability}
Tikhonov, A.: On the stability of inverse problems.
\newblock Doklady Akademii Nauk SSSR \textbf{39}(5), 195--198 (1943)

\bibitem{van2020data}
Van~Parys, B.P., Esfahani, P., Kuhn, D.: From data to decisions:
  Distributionally robust optimization is optimal.
\newblock Management Science  (2020)

\bibitem{van2016generalized}
Van~Parys, B.P., Goulart, P., Kuhn, D.: Generalized gauss inequalities via
  semidefinite programming.
\newblock Mathematical Programming \textbf{156}(1-2), 271--302 (2016)

\bibitem{vapnik2013nature}
Vapnik, V.: The nature of statistical learning theory.
\newblock Springer science \& business media (2013)

\bibitem{vapnik2015uniform}
Vapnik, V., Chervonenkis, A.: On the uniform convergence of relative
  frequencies of events to their probabilities.
\newblock In: Measures of complexity, pp. 11--30. Springer (2015)

\bibitem{wiesemann2014distributionally}
Wiesemann, W., Kuhn, D., Sim, M.: Distributionally robust convex optimization.
\newblock Operations Research \textbf{62}(6), 1358--1376 (2014)

\bibitem{xu2009robustness}
Xu, H., Caramanis, C., Mannor, S.: Robustness and regularization of support
  vector machines.
\newblock Journal of machine learning research \textbf{10}(7) (2009)

\bibitem{zeitouni1998large}
Zeitouni, A., Dembo, O.: Large deviations techniques and applications (1998)

\end{thebibliography}

\newpage

\appendix

\allowdisplaybreaks

\section{Topological notions and results}
\begin{definition}[Uniform boundedness]\label{def: unif boundedness}
A sequence of real valued functions $(f_T)_{T\geq 1}$ on a topological space $\mathcal{Y}$ is said to be uniformly bounded if there exists $K>0$ such that for all $T \in \integ$, $\|f_T\|_{\infty}\leq K$.
\end{definition}

\begin{definition}[Equicontinuity]\label{def: Equicontinuity}
A sequence of real valued functions $(f_T)_{T\geq 1}$ on a topological space $\mathcal{Y}$ is said to be equicontinuous if for all $y \in \mathcal{Y}$ and every $\varepsilon>0$, $y$ has a neighborhood $U_y$ such that
$$
\forall z \in U_y, \; \forall T \in \integ, \; |f_T(y)-f_T(z)| < \varepsilon.
$$
It is said to be uniformally equicontinuous when $U_y$ does not depend on $y$. 
\end{definition}

\begin{definition}[Uniform convergence]
A sequence of real valued functions $(f_T)_{T\geq 1}$ on a set $\mathcal{Y}$ is said to be uniformly convergent to a function $f: \mathcal{Y} \longrightarrow \Re$ if
$$ \forall \varepsilon >0, \; \exists t \in \integ, \; 
\forall T \geq t, \;
\forall x \in \mathcal{Y}, \quad
|f_T(x)-f(x)| \leq \varepsilon
$$
\end{definition}

\begin{theorem}[Arzelà–Ascoli \cite{arzela1895sulle,ascoli1884curve}, (\cite{dunford1958linear}, IV.6.7)]
\label{thm: Arzela–Ascoli}
Let $\mathcal{Y}$ be a compact Hausdorff space. Let $f_T: \mathcal{Y} \rightarrow \Re$ be a sequence of continuous functions. If the sequence $(f_T)_{T\geq 1}$ is equicontinuous and uniformly bounded, then $(f_T)_{T\geq 1}$ admits a sub-sequence that converges uniformly.
\end{theorem}

\begin{lemma}[Uniform equicontinuity]\label{lemma: Uniform equicontinuity}
Let $f_T: \mathcal{Y} \times \mathcal{Z} \longrightarrow \Re$, $T\geq 1$, be an equicontinuous real valued functions where $\mathcal{Y},\mathcal{Z}$ are metric spaces. If $\mathcal{Y}$ is compact, then the following uniform continuity holds: for all $\varepsilon>0$, for all $z_0\in \mathcal{Z}$, there exists $\delta>0$ such that
$$
\forall T \in \integ, \; \forall y \in \mathcal{Y}, \; \forall z\in \mathcal{Z},  \quad  d_{\mathcal{Z}}(z_0,z) \leq \delta \implies |f_T(y,z_0)-f_T(y,z)| \leq \varepsilon,
$$
where $d_{\mathcal{Z}}$ is the distance associated to $\mathcal{Z}$.
\end{lemma}
\begin{proof}
Suppose this claim is not true. There exists $\varepsilon>0$ and $z_0 \in \mathcal{Z}$ such that for all $k\in \integ$, there exists $T_k\in \integ$, $y_k \in \mathcal{Y}$ and $z_k \in \mathcal{Z}$ such that $d_{\mathcal{Z}}(z_0,z_k) \leq 1/k$ and $|f_{T_k}(y_k,z_0)-f_{T_k}(y_k,z_k)| > \varepsilon$. Here, we used the contraposition of the previous claim with $\delta=1/k$. As $\mathcal{Y}$ is compact, there exists a sub-sequence $(y_{k_n})_{n\geq 1}$ of $(y_{k})_{k\geq 1}$ that converges to some $y_{\infty} \in \mathcal{Y}$. By using the equicontinuity of $(f_T)_{T\geq 1}$ in $(z_0,y_{\infty})$, there exists $\delta_{z_0},\delta_{y_{\infty}}>0$ such that for all $y,z \in \mathcal{Y}\times \mathcal{Z}$
\begin{equation}\label{proof eq: unif equicont}
d_{\mathcal{Y}}(z_0,z) \leq \delta_{z_0} \; \text{and}\;  d_{\mathcal{Z}}(y_{\infty},y) \leq \delta_{y_{\infty}} 
\implies 
\forall T \in \integ, \; |f_T(y_{\infty},z_0)-f_T(y,z)| \leq \varepsilon/2
\end{equation}
As $(y_{k_n})_{n\geq 1}$ converges to $y_{\infty}$ and $(z_k)_{k\geq1}$ converges to $z_0$, there exists $k'\geq 1$ such that $d_{\mathcal{Y}}(y_{\infty},y_{k'}) \leq \delta_{y_{\infty}}$ and $d_{\mathcal{Z}}(z_0,z_{k'}) \leq \delta_{z_0}$. Hence, using \eqref{proof eq: unif equicont}, we have 
\begin{align*}
    |f_{T_{k'}}(y_{k'},z_0)-f_{T_{k'}}(y_{k'},z_{k'})|
    &\leq 
    |f_{T_{k'}}(y_{k'},z_0)-f_{T_{k'}}(y_{\infty},z_0)|
    +
    |f_{T_{k'}}(y_{\infty},z_0)-f_{T_{k'}}(y_{k'},z_{k'})| \\
    &\leq 
    \varepsilon/2 + \varepsilon/2 = \varepsilon
\end{align*}
which contradicts the assumption.
\end{proof}

\begin{lemma}[Uniform continuity]
\label{lemma: uniform continuity}
Let $f: \mathcal{Y} \times \mathcal{Z} \longrightarrow \Re$, $T\geq 1$, be a real valued functions where $\mathcal{Y},\mathcal{Z}$ are metric spaces. If $\mathcal{Y}$ is compact, then the following uniform continuity holds: for all $\varepsilon>0$, for all $z_0\in \mathcal{Z}$, there exists $\delta>0$ such that
$$
\forall y \in \mathcal{Y}, \; \forall z\in \mathcal{Z},  \quad  d_{\mathcal{Z}}(z_0,z) \leq \delta \implies |f(y,z_0)-f(y,z)| \leq \varepsilon,
$$
where $d_{\mathcal{Z}}$ is the distance associated to $\mathcal{Z}$.
\end{lemma}
\begin{proof}
This is a special case of Lemma \ref{lemma: uniform continuity} by taking the constant sequence $f_T = f$ for all $T$.
\end{proof}

\begin{lemma}[Equicontinuity of minimum]
\label{lemma: hein for equicontinuity}
Let $\mathcal{Y}$, $\mathcal{Z}$ be metric spaces. If $f_T: \mathcal{Y} \times \mathcal{Z} \longrightarrow \Re$, $T\geq 1$, is an equicontinuous sequence of real valued functions, and $\mathcal{Y}$ is compact, then $(f^{\star}_T)_{T\geq 1}$ is equicontinuous, where $f^{\star}_T(z) = \inf_{y \in \mathcal{Y}} f_T(y,z)$ for all $T\in \integ$ and $z \in \mathcal{Z}$.
\end{lemma}
\begin{proof}
Denote $d_{\mathcal{Y}}$ and $ d_{\mathcal{Z}}$ the distance metrics associated to $\mathcal{Y}$ and $\mathcal{Z}$ respectively.
Let $\varepsilon>0$ and $z_0\in \mathcal{Z}$. Using Lemma \ref{lemma: Uniform equicontinuity}, there exists $\delta>0$ such that for all $z \in \mathcal{Z}$ verifying $d_{\mathcal{Z}}(z_0,z) \leq \delta$, for all $y \in \mathcal{Y}$ and $T\in \integ$, we have $|f_T(y,z_0)-f_T(y,z)| \leq \varepsilon$. Let $T\in \integ$, $z\in \mathcal{Z}$, $y_{0,T} \in \argmin_{y\in \mathcal{Y}}f_T(y,z_0)$ and $y_{1,T} \in \argmin_{y\in \mathcal{Y}}f_T(y,z)$. We have
\begin{align*}
    f_T^{\star}(z_0) \leq f_T(y_{1,T},z_0) \leq f_T(y_{1,T},z)+\varepsilon = f_T^{\star}(z) +\varepsilon \\
    f_T^{\star}(z) \leq f_T(y_{0,T},z) \leq f_T(y_{0,T},z_0)+\varepsilon = f_T^{\star}(z_0)+\varepsilon
\end{align*}
Hence $|f_T^{\star}(z)-f_T^{\star}(z_0)| \leq \varepsilon$ which proves the equicontinuity.
\end{proof}

\begin{lemma}[Continuity of minimum]
\label{lemma: continutiy of minumum}
Let $\mathcal{Y}$, $\mathcal{Z}$ be metric spaces. If $f: \mathcal{Y} \times \mathcal{Z} \longrightarrow \Re$ is a continuous real valued function, and $\mathcal{Y}$ is compact, then $f^{\star}$ is continuous, where $f^{\star}(z) = \inf_{y \in \mathcal{Y}} f(y,z)$ for all $z \in \mathcal{Z}$.
\end{lemma}
\begin{proof}
This is a special case of Lemma \ref{lemma: hein for equicontinuity} by taking the constant sequence $f_T = f$ for all $T$.
\end{proof}

\begin{lemma}[Continuity of a unique minimizer]
\label{lemma: Continuity of a unique minimizer}
Let $\mathcal{Y}$, $\mathcal{Z}$ be metric spaces. Let $f: \mathcal{Y} \times \mathcal{Z} \longrightarrow \Re$ be a continuous real valued function, where $\mathcal{Y}$ is a compact. Suppose for all $z \in \mathcal{Z}$, the minimizer $y^{\star}(z) \in \argmin_{y \in \mathcal{Y}}f(y,z)$ is unique. Then $y^{\star}: \mathcal{Z} \rightarrow \mathcal{Y}, \; z \rightarrow y^{\star}(z)$ is continuous.
\end{lemma}
\begin{proof}
Let $\varepsilon>0$ and $z \in \mathcal{Z}$. As $y^{\star}(z)$ is unique, there exists $\delta>0$ and $\varepsilon >\tilde{\varepsilon}>0$ such that
\begin{equation}\label{proof eq: continuity of min 1}
\forall y' \in \mathcal{Y} \setminus \mathcal{B}(y^{\star}(z),\tilde{\varepsilon}), \quad |f(y',z)-f^{\star}(z)|>\delta,
\end{equation}
where  $f^{\star}(z) = f(y^{\star}(z),z) = \min_{z\in \mathcal{Z}}f(y,z)$, and $\mathcal{B}(z,\tilde{\varepsilon})$ is the ball of center $z$ and radius $\tilde{\varepsilon}$ in the topology of $\mathcal{Z}$. By uniform continuty of $f$ (see Lemma \ref{lemma: uniform continuity}) and continuity of $f^{\star}$ (see Lemma \ref{lemma: continutiy of minumum}), there exists $\eta>0$ such that
\begin{equation}\label{proof eq: continuity of min 2}
\forall z'\in \mathcal{Z},  \quad  d_{\mathcal{Z}}(z,z') \leq \eta \implies
\begin{cases}
|f(y',z)-f(y',z')| \leq \delta/4, \; \forall y' \in \mathcal{Y} \\
|f^{\star}(z)-f^{\star}(z')| \leq \delta/4
\end{cases}
\end{equation}
where $d_{\mathcal{Z}}$ is the distance of $\mathcal{Z}$ topology. Let $y' \in \mathcal{Y} \setminus \mathcal{B}(y^{\star}(z),\tilde{\varepsilon})$. Using \eqref{proof eq: continuity of min 1} and \eqref{proof eq: continuity of min 2} we get
$$
\forall z' \in \mathcal{B}(z,\eta),
\quad |f(y',z') - f^{\star}(z')|>\delta/2.
$$
Hence $y^{\star}(z') \in \mathcal{B}(y^{\star}(z),\tilde{\varepsilon}) \subset \mathcal{B}(y^{\star}(z),\varepsilon)$ for all $z' \in \mathcal{B}(z,\eta)$, which completes the proof of continuity.
\end{proof}

\begin{lemma}[Convergence of minimum]\label{lemma: convergence of minimum}
Let $\mathcal{Y}$ be a compact and $\mathcal{Z}$ be metric spaces. Let $f_T: \mathcal{Y} \times \mathcal{Z} \longrightarrow \Re$, $T\geq 1$, be a sequence of continuous real valued functions. If $(f_T)_{T\geq 1}$ converges uniformally to $f$, then
$(f^{\star}_T)_{T\geq 1}$ converges point-wise to $f^{\star}$, where $f^{\star}_T(z) = \inf_{y \in \mathcal{Y}} f_T(y,z)$ for all $T\in \integ$ and $z \in \mathcal{Z}$, and $f^{\star}(z) = \inf_{y \in \mathcal{Y}} f(y,z)$ for all $z \in \mathcal{Z}$.
\end{lemma}
\begin{proof}
Let $z\in \mathcal{Z}$. Let $\epsilon>0$. By uniform convergence, there exists $T_0 \in \integ$ such that for all $T\geq T_0$, for all $y\in \mathcal{Y}$, $|f_T(y,z) - f(y,z)| \leq \epsilon$. Let $y^*(z) \in \argmin_{y\in \mathcal{Y}} f(y,z)$ and $y_T^*(z) \in \argmin_{y\in \mathcal{Y}} f_T(y,z)$. We have for all $T\geq T_0$, $f_T^*(z) \leq f_T(y^*(z),z) \leq f(y^*(z),z) + \epsilon = f^*(z) + \epsilon$. Furthermore $f^*(z) \leq f(y^*_T(z),z) \leq f_T(y^*_T(z),z) +\epsilon = f_T^*(z) + \epsilon$. Hence $|f^*_T(z) - f^*(z)| \leq \epsilon$, which proves the convergence.
\end{proof}

\begin{lemma}[Uniform convergence of minimum]\label{lemma: unif convergence of minimum}
Let $\mathcal{Y}$ be a compact and $\mathcal{Z}$ be metric spaces. Let $f_T: \mathcal{Y} \times \mathcal{Z} \longrightarrow \Re$, $T\geq 1$, be a sequence of continuous real valued functions. If $(f_T)_{T\geq 1}$is equicontinuous and converges uniformally to $f$, then
$(f^{\star}_T)_{T\geq 1}$ converges uniformally to $f^{\star}$, where $f^{\star}_T(z) = \inf_{y \in \mathcal{Y}} f_T(y,z)$ for all $T\in \integ$ and $z \in \mathcal{Z}$, and $f^{\star}(z) = \inf_{y \in \mathcal{Y}} f(y,z)$ for all $z \in \mathcal{Z}$.
\end{lemma}
\begin{proof}
Lemma \ref{lemma: convergence of minimum} provides the point-wise convergence of $(f^{\star}_T)_{T\geq 1}$. Lemma \ref{lemma: hein for equicontinuity} provides the equicontinuity of $(f^{\star}_T)_{T\geq 1}$. These properties combined provide uniform convergence.
\end{proof}

\section{Large Deviation Theory}\label{Appendix: LDT}

\begin{definition}[Relative entropy] 
  \label{def:relative_entropy}
  The relative entropy of an estimator realization $\Pb'\in\cP$ with respect to a model $\Pb\in\cP$ is defined as
  \begin{align*}
    \D{\Pb'}{\Pb} = \sum_{i\in\Sigma} \Pb'(i) \log\left(\frac{\Pb'(i)}{\Pb(i)}\right),
  \end{align*}
  where we use the conventions $0 \log(0/p)=0$ for any $p\geq 0$ and $p' \log(p'/0)=\infty$ for any $p'> 0$.
\end{definition}

As key result in our analysis of the exponential and superexponential regime is the \textit{Large Deviation Principle} (LDP) which the empirical distribution $\hat{\Pb}_T$ obeys.
\begin{theorem}[Large Deviation Principle]\label{thm: LDP finite space}
For all $\Pb \in \mathcal{P}$ and $\Gamma$ a Borel set of $\mathcal{P}$, the sequence of empirical distributions verifies the following inequalities
\begin{align*}
        -\inf_{\Pb'\in \Gamma^{\into}}
             I(\Pb',\Pb)
    \leq &
        \liminf_{T\rightarrow \infty } 
        \frac{1}{T} \log \Pb^\infty
        \left(
            \hat{\Pb}_T \in \Gamma
        \right) \\
    & \limsup_{T\rightarrow \infty } 
        \frac{1}{T} \log \Pb^\infty
        \left(
            \hat{\Pb}_T \in \Gamma
        \right)
        \leq 
        -\inf_{\Pb'\in \bar{\Gamma}}
             I(\Pb',\Pb)
\end{align*}
where $\Gamma^{\into}$ and $\bar{\Gamma}$ denote respectively the interior and the closure of $\Gamma$ in the weak topology on $\cP$ and $I(\Pb', \Pb)$ is relative entropy of $\Pb'$ with respect to $\Pb$.
\end{theorem}
We refer to \cite{zeitouni1998large}, Theorem 6.2.10, for a proof and further details. A discussion in the context of data-driven decision-making can be found in \cite{van2020data}.

In the subexponential regime, the key result in our analysis will be the \textit{Moderate Deviation Principle} (MDP).
For a given probability distribution $\Pb \in \cPin$, consider the norm $\|\cdot\|_{\Pb}$ associated to $\Pb$ defined as
\begin{equation*}
    \|\Delta\|_{\Pb}^2 := \frac{1}{2} \sum_{i\in \Sigma} \frac{1}{\Pb(i)}\Delta_i^2, \quad \forall \Delta \in \Re^d.
\end{equation*}
As the relative entropy $I(\cdot,\Pb)$ is the right notion of distance when analyzing the assymptotic behavior of the empirical distribution in the exponential regime, the norm $\|\cdot\|_{\Pb}$ induces the right notion of distance in the subexponential regime.
Denote $\cP_{0,\infty} = \set{\Delta \in \Re^d}{e^\top\Delta =0}$, where $e = (1,\ldots,1)^\top$, the hyper-plane containing differences of distributions.
\begin{theorem}[Moderate Deviation Principle]\label{thm: MDP finite space}
Let $(a_T)_{T\geq 1} \in \Re_+^{\integ}$ be a sequence of increasing numbers such that $a_T \to \infty$ and $a_T/T \to 0$. For all $\Pb \in \cPin$ and measurable set $\Gamma \subset \mathcal{P}_{0,\infty}$, the following inequalities holds
\begin{align*}
        -\inf_{\Delta\in \Gamma^{\into}}
             \norm{\Delta}^2_\Pb
    \leq &
        \liminf_{T\rightarrow \infty } 
        \frac{1}{a_T} \log \Pb^\infty
        \left(
            \hat{\Pb}_T-\Pb \in \sqrt{\frac{a_T}{T}} \cdot\Gamma
        \right) \\
    & \limsup_{T\rightarrow \infty } 
        \frac{1}{a_T} \log \Pb^\infty
        \left(
            \hat{\Pb}_T-\Pb \in \sqrt{\frac{a_T}{T}} \cdot\Gamma
        \right) \leq 
        -\inf_{\Delta\in \bar{\Gamma}}
             \norm{\Delta}^2_\Pb
\end{align*}
where $\Gamma^{\into}$ and $\bar{\Gamma}$ denote respectively the interior and the closure of $\Gamma$ in the induced topology of $\mathcal{P}_{0,\infty}$.
\end{theorem}
We refer to \cite{zeitouni1998large}, Theorem 3.7.1, for a proof and further details.

\section{Lemmas related to the partial order}\label{App: Partial order}
The proof of the following lemma uses results and notions stated in advance stages of the paper, namely Definition \ref{def: regular pred} and Proposition \ref{prop: liminf>limsup for a_T.}. 
{\color{black}
\begin{lemma}\label{lemma: order implies less bias and L1 error.}
Let $\hat{c}_1,\hat{c}_2 \in \cC$ be predictors verifying the out-of-sample guarantee \eqref{eq: out-of-sample ganrantee} with $a_T \to \infty$. Let $x \in \cX$.
Suppose
$$
\limsup_{T\to \infty} \frac{|\hat{c}_1(x,\Pb,T)-c(x,\Pb)|}{|\hat{c}_2(x,\Pb,T)-c(x,\Pb)|} \leq 1,
\quad \forall \Pb \in \cPin.
$$
If
$$
\Pb' \rightarrow \frac{|\hat{c}_1(x,\Pb',T) - c(x,\Pb')|}{|\hat{c}_2(x,\Pb',T) - c(x,\Pb')|}
\quad
\text{and}
\quad
\Pb' \rightarrow \frac{\hat{c}_2(x,\Pb',T) - c(x,\Pb')}{\Eb_{\Pb^\infty}(\hat{c}_2(x,\hat{\Pb}_T,T) - c(x,\hat{\Pb}_T))}
$$
are uniformly bounded in $\cP$,
and $(\sqrt{T/a_T} |\hat{c}_i(x,\cdot,T) - c(x,\cdot)|)_{T\geq 1}$ is equicontinuous for $i \in \{1,2\}$,
then the estimator $(\hat{c}_1(x,\hat{\Pb}_T,T))_{T\geq1}$ has less asymptotic bias than $(\hat{c}_2(x,\hat{\Pb}_T,T))_{T\geq1}$, i.e.,
$$
\limsup_{T\to \infty}\frac{\Eb_{\Pb^\infty}(\hat{c}_1(x,\hat{\Pb}_T,T) - c(x,\Pb))}{\Eb_{\Pb^\infty}(\hat{c}_2(x,\hat{\Pb}_T,T) - c(x,\Pb))} \leq 1,
\quad \forall \Pb \in \cPin,
$$
and less $L^1$ error, i.e.,
$$
\limsup_{T\to \infty}\frac{\Eb_{\Pb^\infty}(|\hat{c}_1(x,\hat{\Pb}_T,T) - c(x,\Pb)|)}{\Eb_{\Pb^\infty}(|\hat{c}_2(x,\hat{\Pb}_T,T) - c(x,\Pb)|)} \leq 1,
\quad \forall \Pb \in \cPin.
$$
\end{lemma}
}
\begin{proof}
Let $\hat{c}_1, \hat{c}_2\in \cC$ verifying the assumptions of the lemma. 
Let $x,\Pb \in \cX \times \cPin$. 
We first show that\footnote{Notice that we can further drop the absolute value on the denominator (using \eqref{proof eq: L1 error 3}) and show that less bias implies less $L^1$ error when $\hat{c}_1$ and $\hat{c}_2$ verify the asymptotic guarantee and the proposition's assumption.}
$$
\limsup_{T\to \infty}\frac{\Eb_{\Pb^\infty}(|\hat{c}_1(x,\hat{\Pb}_T,T) - c(x,\Pb)|)}{\Eb_{\Pb^\infty}(|\hat{c}_2(x,\hat{\Pb}_T,T) - c(x,\Pb)|)}
\leq 
\limsup_{T\to \infty}\frac{\Eb_{\Pb^\infty}(|\hat{c}_1(x,\hat{\Pb}_T,T) - c(x,\hat{\Pb}_T)|)}{\Eb_{\Pb^\infty}(|\hat{c}_2(x,\hat{\Pb}_T,T) - c(x,\hat{\Pb}_T)|)}.
$$
Using successively the triangle inequality and Cauchy-Schwartz, we have for all $T \in \integ$

\begin{align*}
\Eb_{\Pb^\infty}(|\hat{c}_1(x,\hat{\Pb}_T,T) - c(x,\Pb)|) 
&\leq 
   \Eb_{\Pb^\infty}(|\hat{c}_1(x,\hat{\Pb}_T,T) - c(x,\hat{\Pb}_T)|) + \Eb_{\Pb^\infty}(|c(x,\hat{\Pb}_T) - c(x,\Pb)|) \\
&\leq
    \Eb_{\Pb^\infty}(|\hat{c}_1(x,\hat{\Pb}_T,T) - c(x,\hat{\Pb}_T)|) + \sqrt{\Eb_{\Pb^\infty}((c(x,\hat{\Pb}_T) - c(x,\Pb))^2)} \\
&=
    \Eb_{\Pb^\infty}(|\hat{c}_1(x,\hat{\Pb}_T,T) - c(x,\hat{\Pb}_T)|) 
    + \sqrt{\frac{1}{T}\Var_{\Pb}(\loss(x,\xi))}.\numberthis \label{proof eq: L1 error 1}
\end{align*}
We next show that the second term in negligible compared to the first term.
$\hat{c}_1$ verifies an out-of-sample guarantee \eqref{eq: out-of-sample ganrantee} for some $(a_T)_{T\geq 1}$, with $a_T \to \infty$. We can assume WLOG that $a_T\ll T$ as a stronger guarantee implies a weaker guarantee. 
{\color{black}Using Fatou's Lemma we have
$$
\liminf_{T\to \infty}\sqrt{\frac{T}{a_T}}\Eb_{\Pb^\infty}\!(|\hat{c}_1(x,\hat{\Pb}_T,T) - c(x,\hat{\Pb}_T)|) 
\geq  
\Eb_{\Pb^\infty}\!\left(\liminf_{T\to \infty}\sqrt{\frac{T}{a_T}}|\hat{c}_1(x,\hat{\Pb}_T,T) - c(x,\hat{\Pb}_T)|\right).
$$
As $(\sqrt{T/a_T}|\hat{c}_1(x,\cdot,T) - c(x,\cdot)|)_{T\geq 1}$ is equicontinuous, then for any sequence $\Pb'_T \to \Pb$, we have $\liminf_{T\to\infty} \sqrt{T/a_T}|\hat{c}_1(x,\Pb'_T,T) - c(x,\Pb'_T)| = \liminf_{T\to\infty} \sqrt{T/a_T}|\hat{c}_1(x,\Pb,T) - c(x,\Pb)|$. As $\hat{\Pb}_T$ converges almost surely to $\Pb$, this implies that 
\begin{align*}
\Eb_{\Pb^\infty}\!\left(\liminf_{T\to \infty}\sqrt{\frac{T}{a_T}}|\hat{c}_1(x,\hat{\Pb}_T,T) - c(x,\hat{\Pb}_T)|\right) 
&\!=\!
\Eb_{\Pb^\infty}\!\left(\liminf_{T\to \infty}\sqrt{\frac{T}{a_T}}|\hat{c}_1(x,\Pb,T) - c(x,\Pb)|\right) 
\!=\!
\liminf_{T\to \infty}\sqrt{\frac{T}{a_T}}|\hat{c}_1(x,\Pb,T) - c(x,\Pb)|
\end{align*}
which can be lowerbounded by $\sqrt{2\Var_{\Pb}(\loss(x,\xi))}$ using Proposition \ref{prop: liminf>limsup for a_T.}. Hence, we have
$$
\liminf_{T\to \infty}\sqrt{\frac{T}{a_T}}\Eb_{\Pb^\infty}(|\hat{c}_1(x,\hat{\Pb}_T,T) - c(x,\hat{\Pb}_T)|) 
\geq
\sqrt{2\Var_{\Pb}(\loss(x,\xi))}
$$
which implies}
\begin{equation}\label{proof eq: L1 error 2}
\liminf_{T\to \infty} \frac{\Eb_{\Pb^\infty}(|\hat{c}_1(x,\hat{\Pb}_T,T) - c(x,\hat{\Pb}_T)|)}{\sqrt{\Var_{\Pb}(\loss(x,\xi))/T}} = \infty.
\end{equation}
Using the inequality \eqref{proof eq: L1 error 1} along with \eqref{proof eq: L1 error 2}, we get 
\begin{align*}
\limsup_{T\to \infty}\frac{\Eb_{\Pb^\infty}(|\hat{c}_1(x,\hat{\Pb}_T,T) - c(x,\Pb)|)}{\Eb_{\Pb^\infty}(|\hat{c}_2(x,\hat{\Pb}_T,T) - c(x,\Pb)|)}
&\leq 
\limsup_{T\to \infty}\frac{\Eb_{\Pb^\infty}(|\hat{c}_1(x,\hat{\Pb}_T,T) - c(x,\hat{\Pb}_T)|)}{\Eb_{\Pb^\infty}(|\hat{c}_2(x,\hat{\Pb}_T,T) - c(x,\Pb)|)}.
\end{align*}

{\color{black}We now use the same arguments for $\hat{c}_2$. Using the triangular inequality, we have
\begin{align*}
\Eb_{\Pb^\infty}(|\hat{c}_2(x,\hat{\Pb}_T,T) - c(x,\Pb)|)
&\geq
\Eb_{\Pb^\infty}(|\hat{c}_2(x,\hat{\Pb}_T,T) - c(x,\hat{\Pb}_T)|)
-
\Eb_{\Pb^\infty}(|c(x,\hat{\Pb}_T) - c(x,\Pb)|)\\
&\geq
\Eb_{\Pb^\infty}(|\hat{c}_2(x,\hat{\Pb}_T,T) - c(x,\hat{\Pb}_T)|)
-
\sqrt{\frac{1}{T}\Var_{\Pb}(\loss(x,\xi))}
\end{align*}
Again, this implies with Proposition \ref{prop: liminf>limsup for a_T.} that
\begin{align*}
\limsup_{T\to \infty}\frac{\Eb_{\Pb^\infty}(|\hat{c}_1(x,\hat{\Pb}_T,T) - c(x,\Pb)|)}{\Eb_{\Pb^\infty}(|\hat{c}_2(x,\hat{\Pb}_T,T) - c(x,\Pb)|)}
&\leq 
\limsup_{T\to \infty}\frac{\Eb_{\Pb^\infty}(|\hat{c}_1(x,\hat{\Pb}_T,T) - c(x,\hat{\Pb}_T)|)}{\Eb_{\Pb^\infty}(|\hat{c}_2(x,\hat{\Pb}_T,T) - c(x,\hat{\Pb}_T)|)}.
\end{align*}
We now prove that the RHS term is bounded by 1. Denote $\rho_T(\Pb') = |\hat{c}_1(x,\Pb',T) - c(x,\Pb')|/|\hat{c}_2(x,\Pb',T) - c(x,\Pb')|$ and $\delta_T(\Pb') = \hat{c}_2(x,\Pb',T) - c(x,\Pb')$ for all $T \in \integ$ and $\Pb' \in \cPin$. 
By assumption, $(\rho_T(\cdot)|\delta_T(\cdot)|/\Eb_{\Pb^\infty}(|\delta_T(\hat{\Pb}_T)|))_{T\geq 1}$ is uniformly bounded.
Let $A>0$ be the uniform bound. Let $\varepsilon>0$.
\begin{align*}
    \limsup_{T\to \infty}\frac{\Eb_{\Pb^\infty}(|\hat{c}_1(x,\hat{\Pb}_T,T) - c(x,\hat{\Pb}_T)|)}{\Eb_{\Pb^\infty}(|\hat{c}_2(x,\hat{\Pb}_T,T) - c(x,\hat{\Pb}_T)|)}
    &=
    \limsup_{T\to \infty}
    \Eb_{\Pb^\infty}\left(
    \frac{|\hat{c}_1(x,\hat{\Pb}_T,T) - c(x,\hat{\Pb}_T)|}{|\hat{c}_2(x,\hat{\Pb}_T,T) - c(x,\hat{\Pb}_T)|}
    \frac{|\hat{c}_2(x,\hat{\Pb}_T,T) - c(x,\hat{\Pb}_T)|}{\Eb_{\Pb^\infty}(|\hat{c}_2(x,\hat{\Pb}_T,T) - c(x,\hat{\Pb}_T)|)}
    \right)
    \\
    &=
    \limsup_{T\to \infty}
    \Eb_{\Pb^\infty}
    \left[
    \rho_T(\hat{\Pb}_T) \frac{|\delta_T(\hat{\Pb}_T)|}{\Eb_{\Pb^\infty}(|\delta_T(\hat{\Pb}_T)|)}
    \right] \\
    &\leq
     \limsup_{T\to \infty}
    \Eb_{\Pb^\infty}
    \left[
    (1+\varepsilon) \mathbb{1}_{\rho_T(\hat{\Pb}_T)\leq1+\varepsilon} \frac{|\delta_T(\hat{\Pb}_T)|}{\Eb_{\Pb^\infty}(|\delta_T(\hat{\Pb}_T)|)}
    \right] \\
    & \quad \quad+
    \Eb_{\Pb^\infty}
    \left[
    A \mathbb{1}_{\rho_T(\hat{\Pb}_T)>1+\varepsilon}
    \right] \\
    &\leq
    (1+\varepsilon)+
    A\limsup_{T\to \infty} 
    \Pb^\infty(\rho_T(\hat{\Pb}_T)>1+\varepsilon)
\end{align*}
We have by assumption $\limsup_{T\to\infty} \rho_T(\Pb') \leq 1$ for all $\Pb' \in \cPin$, hence, as convergence almost surely implies convergence in probability, $\limsup_{T\to \infty} 
    \Pb^\infty(\rho_T(\hat{\Pb}_T)>1+\varepsilon) = 0$. This implies
\begin{align*}
    \limsup_{T\to \infty}\frac{\Eb_{\Pb^\infty}(|\hat{c}_1(x,\hat{\Pb}_T,T) - c(x,\hat{\Pb}_T)|)}{\Eb_{\Pb^\infty}(|\hat{c}_2(x,\hat{\Pb}_T,T) - c(x,\hat{\Pb}_T)|)}
    &\leq 1+\epsilon
\end{align*}
for all $\epsilon>0$ and hence proves the desired inequality.
We have shown therefore that
$$
\limsup_{T\to \infty}\frac{\Eb_{\Pb^\infty}(|\hat{c}_1(x,\hat{\Pb}_T,T) - c(x,\Pb)|)}{\Eb_{\Pb^\infty}(|\hat{c}_2(x,\hat{\Pb}_T,T) - c(x,\Pb)|)}
\leq 1+ \varepsilon
$$
for all $\varepsilon>0$, which completes the proof.\\
Let us now prove the result for the bias. We show that $\limsup_{T\to \infty} \frac{\Eb_{\Pb^\infty}(|\delta_T(\hat{\Pb}_T)|)}{\Eb_{\Pb^\infty}(\delta_T(\hat{\Pb}_T))} = 1$, which implies
\begin{align*}
    \limsup_{T\to \infty}\frac{\Eb_{\Pb^\infty}(\hat{c}_1(x,\hat{\Pb}_T,T) - c(x,\Pb))}{\Eb_{\Pb^\infty}(\hat{c}_2(x,\hat{\Pb}_T,T) - c(x,\Pb))} 
&=
\limsup_{T\to \infty}\frac{\Eb_{\Pb^\infty}(\hat{c}_1(x,\hat{\Pb}_T,T) - c(x,\hat{\Pb}_T))}{\Eb_{\Pb^\infty}(\hat{c}_2(x,\hat{\Pb}_T,T) - c(x,\hat{\Pb}_T))} \\
&=
\limsup_{T\to \infty}\frac{\Eb_{\Pb^\infty}(\hat{c}_1(x,\hat{\Pb}_T,T) - c(x,\hat{\Pb}_T))}{\Eb_{\Pb^\infty}(|\hat{c}_2(x,\hat{\Pb}_T,T) - c(x,\hat{\Pb}_T)|)} \\
&\leq 
\limsup_{T\to \infty}\frac{\Eb_{\Pb^\infty}(|\hat{c}_1(x,\hat{\Pb}_T,T) - c(x,\hat{\Pb}_T)|)}{\Eb_{\Pb^\infty}(|\hat{c}_2(x,\hat{\Pb}_T,T) - c(x,\hat{\Pb}_T)|)}.
\end{align*}
We have that $\left(\frac{\delta_T(\cdot)}{\Eb_{\Pb^\infty}(\delta_T(\hat{\Pb}_T))}\right)_{T>1}$ is uniformly bounded by assumption. Denote with $B>0$ its bound.
We have
\begin{align*}
    \limsup_{T\to \infty} \frac{\Eb_{\Pb^\infty}(|\delta_T(\hat{\Pb}_T)|)}{\Eb_{\Pb^\infty}(\delta_T(\hat{\Pb}_T))}
    &= 
    \limsup_{T\to \infty} \frac{
    \Eb_{\Pb^\infty}(\delta_T(\hat{\Pb}_T)\mathbb{1}_{\delta_T(\hat{\Pb}_T)\geq 0})
    -
    \Eb_{\Pb^\infty}(\delta_T(\hat{\Pb}_T)\mathbb{1}_{\delta_T(\hat{\Pb}_T)< 0})
    }{
    \Eb_{\Pb^\infty}(\delta_T(\hat{\Pb}_T))} \\
    &= 
    \limsup_{T\to \infty} 1-\frac{
    2\Eb_{\Pb^\infty}(\delta_T(\hat{\Pb}_T)\mathbb{1}_{\delta_T(\hat{\Pb}_T)< 0})
    }{
    \Eb_{\Pb^\infty}(\delta_T(\hat{\Pb}_T))} \\
    &=\limsup_{T\to \infty}
    1-
    \Eb_{\Pb^\infty}\left(
        \frac{2\delta_T(\hat{\Pb}_T)}{\Eb_{\Pb^\infty}(\delta_T(\hat{\Pb}_T))} 
        \mathbb{1}_{\delta_T(\hat{\Pb}_T)< 0}
    \right) \\
    &=
    1-2B \cdot \mathcal{O}\left(\limsup_{T\to \infty} \Pb^\infty(\delta_T(\hat{\Pb}_T)< 0)\right) 
    =
    1 - \mathcal{O}\left(\limsup_{T\to \infty} \Pb^\infty\left(\sqrt{\frac{T}{a_T}}\delta_T(\hat{\Pb}_T)< 0\right) \right)
\end{align*}
If $\Var_{\Pb}(\loss(x,\xi))=0$, then $c(x,\hat{\Pb}_T)=c(x,\Pb)$ almost surely, therefore $\Pb^\infty(\delta_T(\hat{\Pb}_T)<0) = \Pb^\infty(c(x,\Pb)>\hat{c}(x,\hat{\Pb}_T,T))$ which converges to zero as $\hat{c}_2$ verifies the out-of-sample guarantee. Suppose now $\Var_{\Pb}(\loss(x,\xi))>0$. Then, $\Var_{\Pb'}(\loss(x,\xi))>0$ for all $\Pb'\in \cPin$.
Proposition \ref{prop: liminf>limsup for a_T.} implies that $\liminf_{T\to \infty} \sqrt{T/a_T} \delta_T(\Pb') \geq  \sqrt{2\Var_{\Pb'}(\loss(x,\xi))}>0$, hence the almost sure convergence implies convergence in probability and we get $\limsup_{T\to \infty} \Pb^\infty\left(\sqrt{T/a_T}\delta_T(\hat{\Pb}_T)< 0\right) = 0$. Hence
\begin{equation}\label{proof eq: L1 error 3}
    \limsup_{T\to \infty} \frac{\Eb_{\Pb^\infty}(|\delta_T(\hat{\Pb}_T)|)}{\Eb_{\Pb^\infty}(\delta_T(\hat{\Pb}_T))} = 1.
\end{equation}
}
\end{proof}

{\color{black}
\begin{lemma}\label{lemma: order on prescriptors}
    Let $(\hat{c}_1,\hat{x}_1),(\hat{c}_2,\hat{x}_2) \in \hat{\cX}$ be prescriptors verifying the out-of-sample guarantee \eqref{eq: prescriptor out-of-sample ganrantee} with $a_T \to \infty$. Let $x \in \cX$.
Suppose
$$
\limsup_{T\to \infty} \frac{|\hat{c}_1^\star(\Pb,T)-c^\star(\Pb)|}{|\hat{c}_2^\star(\Pb,T)-c^\star(\Pb)|} \leq 1,
\quad \forall \Pb \in \cPin.
$$
Assume
$
\left(\frac{|\hat{c}_1^\star(\cdot,T)-c^\star(\cdot)|}{\Eb_{\Pb^\infty}(|\hat{c}_2^\star(\hat{\Pb}_T,T) - c^\star(\hat{\Pb}_T)|)}\right)_{T\geq 1}
$ is uniformly bounded, $(\sqrt{T/a_T}|\hat{c}_i(x,\cdot,T) - c(x,\cdot)|)_{T\geq 1}$ is equicontinuous for all $x \in \cX$, for $i\in \{1,2\}$ and $\Var_{\Pb}(\loss(x^{\star}(\Pb),\xi)) >0$ for all $\Pb \in \cP^\into$.
We have
$$
\limsup_{T\to \infty}\frac{\Eb_{\Pb^\infty}(|\hat{c}_1^\star(\hat{\Pb}_T,T) - c^\star(\Pb)|)}{\Eb_{\Pb^\infty}(|\hat{c}_2^\star(\hat{\Pb}_T,T) - c^\star(\Pb)|)} \leq 1,
\quad \forall \Pb \in \cPin.
$$
\end{lemma}
\begin{proof}
Let $(\hat{c}_1,\hat{x}_1),(\hat{c}_2,\hat{x}_2) \in \hat{\cX}$ verifying the conditions of the lemma. We first show that
$$
\limsup_{T\to \infty}\frac{\Eb_{\Pb^\infty}(|\hat{c}_1^\star(\hat{\Pb}_T,T) - c^\star(\Pb)|)}{\Eb_{\Pb^\infty}(|\hat{c}_2^\star(\hat{\Pb}_T,T) - c^\star(\Pb)|)} 
\leq 
\limsup_{T\to \infty}\frac{\Eb_{\Pb^\infty}(|\hat{c}_1^\star(\hat{\Pb}_T,T) - c^\star(\hat{\Pb}_T)|)}{\Eb_{\Pb^\infty}(|\hat{c}_2^\star(\hat{\Pb}_T,T) - c^\star(\hat{\Pb}_T)|)} 
\quad \forall \Pb \in \cPin.
$$
Using the triangular inequality, we have
\begin{align*}
\Eb_{\Pb^\infty}(|\hat{c}_1^\star(\hat{\Pb}_T,T) - c^\star(\Pb)|)
&\leq
\Eb_{\Pb^\infty}(|\hat{c}_1^\star(\hat{\Pb}_T,T) - c^\star(\hat{\Pb}_T)|)
+
\Eb_{\Pb^\infty}(|c^\star(\hat{\Pb}_T) - c^\star(\Pb)|)
\end{align*}
Let us now bound the second term. Using the optimality of $c^\star(\hat{\Pb}_T)$ for the first inequality, we have 
\begin{align*}
c^\star(\hat{\Pb}_T) - c^\star(\Pb)
&\leq
c(x^\star(\Pb),\hat{\Pb}_T) - c^\star(\Pb)\\
&=
\Eb_{\hat{\Pb}_T}[\loss(x^\star(\Pb), \xi)] - \Eb_{\Pb}[\loss(x^\star(\Pb),\xi)]\\
&=
\int \loss(x^\star(\Pb),\xi) d (\hat{\Pb}_T - \Pb)(\xi)\\
&\leq L\|\hat{\Pb}_T - \Pb\|_1 
\end{align*}
where $L$ is an upper bound on $\ell$.
Using the same argument with the optimality of $c^\star(\hat{\Pb})$, we get 
\begin{align*}
c^\star(\hat{\Pb}_T) - c^\star(\Pb)
&\geq
\Eb_{\hat{\Pb}_T}[\loss(x^\star(\hat{\Pb}_T), \xi)) - \Eb_{\Pb}(\loss(x^\star(\hat{\Pb}_T),\xi)]
\geq -L \|\hat{\Pb}_T - \Pb\|_1 
\end{align*}
Hence, 
$$
\Eb_{\Pb}(|c^\star(\hat{\Pb}_T) - c^\star(\Pb)|)
\leq
L \Eb_{\Pb}(\|\hat{\Pb}_T - \Pb\|_1) = \mathcal{O}\left(\frac{1}{\sqrt{T}}\right).
$$
Plugging this result, we get 
\begin{equation}\label{proof eq: order strength 1}
\Eb_{\Pb^\infty}(|\hat{c}_1^\star(\hat{\Pb}_T,T) - c^\star(\Pb)|)
\leq
\Eb_{\Pb^\infty}(|\hat{c}_1^\star(\hat{\Pb}_T,T) - c^\star(\hat{\Pb}_T)|)
+
\mathcal{O}\left(\frac{1}{\sqrt{T}}\right) 
\end{equation}
%
Using Fatou Lemma, we have
$$
\liminf_{T\to \infty} \sqrt{\frac{T}{a_T}} \Eb_{\Pb^\infty}(|\hat{c}_1^\star(\hat{\Pb}_T,T) - c^\star(\hat{\Pb}_T)|) 
\geq
 \Eb_{\Pb^\infty}\left(\liminf_{T\to \infty} \sqrt{\frac{T}{a_T}}|\hat{c}_1^\star(\hat{\Pb}_T,T) - c^\star(\hat{\Pb}_T)|\right) 
$$
The equicontinuty of $(\sqrt{T/a_T}|\hat{c}_1(x,\cdot,T) - c(x,\cdot)|)_{T\geq 1}$ for all $x \in \cX$ implies the equicontinuity of $(\sqrt{T/a_T}|\hat{c}_1^\star(\cdot,T) - c^\star(\cdot)|)_{T\geq 1}$ by Lemma \ref{lemma: hein for equicontinuity}. Hence, for all sequence $\Pb'_T \to \Pb$, we have $\liminf \sqrt{T/a_T}|\hat{c}_1^\star(\Pb'_T,T) - c^\star(\Pb'_T)| = \liminf \sqrt{T/a_T}|\hat{c}_1^\star(\Pb,T) - c^\star(\Pb)|$. As $\hat{\Pb}_T$ converges to $\Pb$ almost surely, we have 
\begin{align*}
 \Eb_{\Pb^\infty}\left(\liminf_{T\to \infty} \sqrt{\frac{T}{a_T}}|\hat{c}_1^\star(\hat{\Pb}_T,T) - c^\star(\hat{\Pb}_T)|\right) 
 &=
  \Eb_{\Pb^\infty}\left(\liminf_{T\to \infty} \sqrt{\frac{T}{a_T}}|\hat{c}_1^\star(\Pb,T) - c^\star(\Pb)|\right) \\
 &= \liminf_{T\to \infty} \sqrt{\frac{T}{a_T}}|\hat{c}_1^\star(\Pb,T) - c^\star(\Pb)|
\end{align*}
This last quantity is lowerbounded by $\sqrt{2\Var_{\Pb}(\loss(x^{\star}(\Pb),\xi))}$ using Proposition \ref{prop: lower bound strong opt presc}. Hence, we have
$$\liminf_{T\to \infty} \sqrt{\frac{T}{a_T}} \Eb_{\Pb^\infty}(|\hat{c}_1^\star(\hat{\Pb}_T,T) - c^\star(\hat{\Pb}_T)|) \geq \sqrt{2\Var_{\Pb}(\loss(x^{\star}(\Pb),\xi))} = \Omega(1).$$
Combining this results with \eqref{proof eq: order strength 1}, we get
$$
\Eb_{\Pb^\infty}(|\hat{c}_1^\star(\hat{\Pb}_T,T) - c^\star(\Pb)|)
\leq
\Eb_{\Pb^\infty}(|\hat{c}_1^\star(\hat{\Pb}_T,T) - c^\star(\hat{\Pb}_T)|)
+
o\left(\Eb_{\Pb^\infty}(|\hat{c}_1^\star(\hat{\Pb}_T,T) - c^\star(\hat{\Pb}_T)|)\right)
$$
We now apply a similar reasoning to $\hat{c}_2$. Using the triangular inequality, we have
\begin{align*}
\Eb_{\Pb^\infty}(|\hat{c}_2^\star(\hat{\Pb}_T,T) - c^\star(\Pb)|)
&\geq
\Eb_{\Pb^\infty}(|\hat{c}_2^\star(\hat{\Pb}_T,T) - c^\star(\hat{\Pb}_T)|)
-
\Eb_{\Pb^\infty}(|c^\star(\hat{\Pb}_T) - c^\star(\Pb)|),
\end{align*}
and similarly we obtain
\begin{align*}
\Eb_{\Pb^\infty}(|\hat{c}_2^\star(\hat{\Pb}_T,T) - c^\star(\Pb)|)
&\geq
\Eb_{\Pb^\infty}(|\hat{c}_2^\star(\hat{\Pb}_T,T) - c^\star(\hat{\Pb}_T)|)
+ 
o\left(
\Eb_{\Pb^\infty}(|\hat{c}_2^\star(\hat{\Pb}_T,T) - c^\star(\hat{\Pb}_T)|)
\right)
\end{align*}
Combining the results for $\hat{c}_1$ and $\hat{c}_2$, we get
\begin{align*}
\limsup_{T\to \infty}\frac{\Eb_{\Pb^\infty}(|\hat{c}_1^\star(\hat{\Pb}_T,T) - c^\star(\Pb)|)}{\Eb_{\Pb^\infty}(|\hat{c}_2^\star(\hat{\Pb}_T,T) - c^\star(\Pb)|)} 
&\leq
\limsup_{T\to \infty}\frac{
    \Eb_{\Pb^\infty}(|\hat{c}_1^\star(\hat{\Pb}_T,T) - c^\star(\hat{\Pb}_T)|)
    + 
    o\left( \Eb_{\Pb^\infty}(|\hat{c}_1^\star(\hat{\Pb}_T,T) - c^\star(\hat{\Pb}_T)|)\right)
}{
    \Eb_{\Pb^\infty}(|\hat{c}_2^\star(\hat{\Pb}_T,T) - c^\star(\hat{\Pb}_T)|) 
    +
    o\left( \Eb_{\Pb^\infty}(|\hat{c}_2^\star(\hat{\Pb}_T,T) - c^\star(\hat{\Pb}_T)|) \right)
}
\\
&=
\limsup_{T\to \infty}\frac{\Eb_{\Pb^\infty}(|\hat{c}_1^\star(\hat{\Pb}_T,T) - c^\star(\hat{\Pb}_T)|)}{\Eb_{\Pb^\infty}(|\hat{c}_2^\star(\hat{\Pb}_T,T) - c^\star(\hat{\Pb}_T)|)}
\end{align*}
We now prove that this term is bounded by 1. Denote $\rho_T(\Pb') = |\hat{c}_1^\star(\Pb',T) - c^\star(\Pb')|/|\hat{c}_2^\star(\Pb',T) - c^\star(\Pb')|$ and $\delta_T(\Pb') = |\hat{c}_2^\star(\Pb',T) - c^\star(\Pb')|$ for all $T \in \integ$ and $\Pb' \in \cPin$. By assumption, $(\rho_T(\cdot)\delta_T(\cdot)/\Eb_{\Pb}(\delta_T(\hat{\Pb}_T)))_{T\geq 1}$ is uniformly bounded.
Let $A>0$ be the uniform bound. Let $\varepsilon>0$.
\begin{align*}
    \limsup_{T\to \infty}\frac{\Eb_{\Pb^\infty}(|\hat{c}_1^\star(\hat{\Pb}_T,T) - c^\star(\hat{\Pb}_T)|)}{\Eb_{\Pb^\infty}(|\hat{c}_2^\star(\hat{\Pb}_T,T) - c^\star(\hat{\Pb}_T)|)}
&=
    \limsup_{T\to \infty}
    \Eb_{\Pb^\infty}\left(
    \frac{|\hat{c}_1^\star(\hat{\Pb}_T,T) - c^\star(\hat{\Pb}_T)|}{|\hat{c}_2^\star(\hat{\Pb}_T,T) - c^\star(\hat{\Pb}_T)|}
    \frac{|\hat{c}_2^\star(\hat{\Pb}_T,T) - c^\star(\hat{\Pb}_T)|}{\Eb_{\Pb^\infty}(|\hat{c}_2^\star(\hat{\Pb}_T,T) - c^\star(\hat{\Pb}_T)|)}
    \right)\\
&=
    \limsup_{T\to \infty}
    \Eb_{\Pb^\infty}
    \left[
    \rho_T(\hat{\Pb}_T) \frac{\delta_T(\hat{\Pb}_T)}{\Eb_{\Pb^\infty}(\delta_T(\hat{\Pb}_T))}
    \right] \\
&\leq
     \limsup_{T\to \infty}
    \Eb_{\Pb^\infty}
    \left[
    (1+\varepsilon) \mathbb{1}_{\rho_T(\hat{\Pb}_T)\leq1+\varepsilon} \frac{\delta_T(\hat{\Pb}_T)}{\Eb_{\Pb^\infty}(\delta_T(\hat{\Pb}_T))}
    \right] \\
    & \quad \quad+
    \Eb_{\Pb^\infty}
    \left[
    A \mathbb{1}_{\rho_T(\hat{\Pb}_T)>1+\varepsilon}
    \right] \\
&\leq
    (1+\varepsilon)
    +
    \limsup_{T\to \infty} 
    \Pb^\infty(\rho_T(\hat{\Pb}_T)>1+\varepsilon)
\end{align*}
As convergence almost surely of $\rho_T(\hat{\Pb}_T)$ imply convergence in probability, we have $\limsup_{T\to \infty} 
    \Pb^\infty(\rho_T(\hat{\Pb}_T)>1+\varepsilon) = 0$ which completes the proof.
\end{proof}
}

{\color{black}
\begin{lemma}\label{lemma: exact cov counter example}
Consider the family of predictors
    $$
  \hat{c}_{M,p,q}(x,\hat{\Pb}_T,T) := M \left(1-2 \cdot \mathbb{1}\left( T\hat{\Pb}_T(i_0) \; \text{mod} \; q \in \{0, \ldots, p-1\})\right)\right), \quad \forall x,T,
  $$
  for $M>0$ , $\tfrac{p}{q}$ a rational number, and $i_0 \in \Sigma$ is arbitrary.
  Then $\hat{c}_{M,p,q}$ verifies an exact out-of-sample guarantee at level $p/q$, i.e.,
  $$
    \lim_{T \to \infty}
    \Pb^{\infty}\left(
    c(x,\Pb) 
    > 
    \hat{c}_{M,p,q}(x,\hat{\Pb}_T,T)
    \right) = p/q
    $$
    for any $M$ large enough. Furthermore,
    for all $x\in \mathcal X,T\geq 1$ and $\Qb \in \cP^\circ$ we have
    $$
    \lim_{M\to\infty}|\hat{c}_{M,p,q}(x,\Qb,T)
    -
    c(x,\Qb)|
    \xrightarrow[M \to \infty]{} \infty
    $$
    as well as for all $x\in \mathcal X,T\geq 1$ we get
    $$
    \Eb_{\Pb^\infty}\left[ 
    \left|\hat{c}_{M,p,q}(x,\hat{\Pb}_T,T)
    -
    c(x,\Pb)
    \right|
    \right] \xrightarrow[M \to \infty]{} \infty.
    $$
\end{lemma}
\begin{proof}
    Notice first that $T\hat{\Pb}_T(i_0)$ follows a binomial distribution with parameters $T$ and $\Pb(i_0)>0$. One can easily show that for all $k \in \integ$
    $$
    \lim_{T\to \infty} 
    \Pb^\infty\left(T\hat{\Pb}_T(i_0) \; \text{mod} \; q = k \right) = 1/q.
    $$
    Hence, 
    $
    \lim_{T\to \infty} 
    \Pb^\infty\left( T\hat{\Pb}_T(i_0) \; \text{mod} \; q \in \{0, \ldots, p-1\} \right) = p/q$. Let $A_T:=\{T\hat{\Pb}_T(i_0) \; \text{mod} \; q \in \{0, \ldots, p-1\}\}$ and take
    $M$ be larger than a strict uniform upper bound on $|c(x,\Pb)|$.
    We have
    \begin{align*}
    \Pb^{\infty}\left(
        c(x,\Pb) 
        > 
        \hat{c}_{M,p,q}(x,\hat{\Pb}_T,T)
    \right)
&=
    \Pb^{\infty}\left(
            c(x,\Pb) 
            > 
            \hat{c}_{M,p,q}(x,\hat{\Pb}_T,T)
        , ~A_T
    \right)
     +
    \Pb^{\infty}\left(
            c(x,\Pb) 
            > 
            \hat{c}_{M,p,q}(x,\hat{\Pb}_T,T)
         , ~A_T^c
    \right)\\
&=
    \Pb^{\infty}\left(A_T
    \right)
=
    p/q.
    \end{align*}

Furthermore, for all $x,\Qb,T$, we have $|\hat{c}_{M,p,q}(x,\Qb,T) - c(x,\Qb)| \geq |\hat{c}_{M,p,q}(x,\Qb,T)| - |c(x,\Qb)| = M - |c(x,\Qb)| \xrightarrow[M \to \infty]{} \infty$. Similarly, for all $x,T$, $\Eb_{\Pb^\infty}\left[ 
    \left|\hat{c}_{M,p,q}(x,\hat{\Pb}_T,T)
    -
    c(x,\Pb)
    \right|
    \right]
\geq 
\Eb_{\Pb^\infty}\left[ 
    |\hat{c}_{M,p,q}(x,\hat{\Pb}_T,T)|
    -
    |c(x,\Pb)|
    \right]
= \Eb_{\Pb^\infty}\left[ 
    M
    -
    |c(x,\Pb)|
    \right] \xrightarrow[M \to \infty]{} \infty
$.
\end{proof}
}

\begin{lemma}\label{lemma: hat c > true c}
Let $\hat{c} \in \cC$ a predictor verifying the out-of-sample guarantee \eqref{eq: out-of-sample ganrantee} for any $(a_T)_{T\geq 1} \in \Re_+^{\integ}$ such that $a_T \to \infty$. We have for all $\Pb\in \cPin$ and $x\in \cX$,
\begin{equation*}
    \liminf_{T\to\infty} \hat{c}(x,\Pb,T) - c(x,\Pb) \geq 0.
\end{equation*}
\end{lemma}
\begin{proof}
Suppose for the sake of argument that there exists $x_0,\Pb_0 \in \cX \times \cPin$ such that $\liminf_{T\to\infty} \hat{c}(x_0,\Pb_0,T) - c(x,\Pb) < 0$. By definition of the limit inferior, there exists $(t_T)_{T\geq 1}\in \integ^{\integ}$ and $\varepsilon>0$ such that $\hat{c}(x_0,\Pb_0,T) < c(x_0,\Pb_0) - \varepsilon$ for all $T \in \integ$. By equicontinuity of $\hat{c}$, there exists an open set $U\subset \cPin$ containing $\Pb_0$ such that for all $\Pb' \in U$ and $T\in \integ$, we have $\hat{c}(x_0,\Pb',T) < c(x_0,\Pb_0) - \varepsilon/2$. Denote $\Gamma = U - \Pb_0$. Let $(b_T)_{T\geq 1} \in (\Re_+)^{\integ}$ be such that $b_T \ll a_T$, $b_T\ll T$ and $b_T \to \infty$ (take for example $b_T = \min( a_T/\log a_T, T/\log T)$). We have
\begin{align*}
    \limsup_{T \to \infty} \frac{1}{b_T}
    \log \Pb_0^{\infty}(c(x_0,\Pb_0)> \hat{c}(x_0,\hat{\Pb}_T,T))
    &\geq 
    \limsup_{T \to \infty} \frac{1}{b_{t_T}}
    \log\Pb_0^{\infty}(c(x_0,\Pb_0)> \hat{c}(x_0,\hat{\Pb}_{t_T},t_T)) \\
    &\geq
    \limsup_{T \to \infty} \frac{1}{b_{t_T}}
    \log \Pb_0^{\infty}(\hat{\Pb}_{t_T} \in U)\\
    &\geq
    \liminf_{T \to \infty} \frac{1}{b_T}
    \log \Pb_0^{\infty}(\hat{\Pb}_{T} \in U)\\
    &=
    \liminf_{T \to \infty} \frac{1}{b_T}
    \log \Pb_0^{\infty}(\hat{\Pb}_{T} -\Pb_0 \in \Gamma)\\
    &\geq \liminf_{T \to \infty} \frac{1}{b_T}
    \log \Pb_0^{\infty}\left(\hat{\Pb}_{T} -\Pb_0 \in \sqrt{\frac{b_T}{T}}\Gamma \right)
\end{align*}

where the last inequality is implied by the fact that $b_T \leq T$ eventually. Using the MDP (Theorem \ref{thm: MDP finite space}), and the fact that the distribution $0$ |that puts a zero weight on each uncertainty| is included in $\Gamma^{\into}$, we get
\begin{align*}
    \limsup_{T \to \infty} \frac{1}{b_T}
    \log \Pb_0^{\infty}(c(x_0,\Pb_0)> \hat{c}(x_0,\hat{\Pb}_T,T))
    &\geq \liminf_{T \to \infty} \frac{1}{b_T}
    \log \Pb_0^{\infty}\left(\hat{\Pb}_{T} -\Pb_0 \in \sqrt{\frac{b_T}{T}}\Gamma \right)\\
    &\geq - \inf_{\Delta \in \Gamma^{\into}} \|\Delta\|^2_{\Pb_0} \geq 0
\end{align*}
Hence
\begin{align*}
     \limsup_{T \to \infty} \frac{1}{a_T}
    \log \Pb_0^{\infty}(c(x_0,\Pb_0)> \hat{c}(x_0,\hat{\Pb}_T,T))
    &=
    \limsup_{T \to \infty} \frac{b_T}{a_T} \frac{1}{b_T}
    \log \Pb_0^{\infty}(c(x_0,\Pb_0)> \hat{c}(x_0,\hat{\Pb}_T,T)) \\
    &\geq \limsup_{T \to \infty} \frac{b_T}{a_T} \times 0 = 0
\end{align*}
as $\lim b_T/a_T =0$ by construction of $(b_T)_{T\geq 1}$. This last inequality contradicts the out-of-sample guarantee \eqref{eq: out-of-sample ganrantee} which completes the proof.
\end{proof}

\begin{lemma}\label{lemma: equivalence of order notions.}
Consider predictors $\hat{c}_1,\hat{c}_2 \in \cC$.
Let $x,\Pb \in \cX\times \cP$. We have
$$
\forall (\beta_T)_{\geq 1} \in \Re_+^{\integ}, \quad 
\limsup_{T\to \infty} \beta_T|\hat{c}_1(x,\Pb,T)-c(x,\Pb)|
\leq 
\limsup_{T\to \infty} \beta_T|\hat{c}_2(x,\Pb,T)-c(x,\Pb)|
$$
if and only if
$$
\limsup_{T\to \infty} \frac{|\hat{c}_1(x,\Pb,T)-c(x,\Pb)|}{|\hat{c}_2(x,\Pb,T)-c(x,\Pb)|} \leq 1,
$$
where the form $\frac{0}{0}$ is by convention considered $1$.
\end{lemma}
\begin{proof}
Suppose the first property is true. It suffices to chose $\beta_T = \frac{1}{|\hat{c}_2(x,\Pb,T)-c(x,\Pb)|}$ to get the second property. 
Let $x,\Pb \in \cX \times \cP$. 
Let us now show the reverse implication. Suppose
$$\limsup_{T\to \infty} \frac{|\hat{c}_1(x,\Pb,T)-c(x,\Pb)|}{|\hat{c}_2(x,\Pb,T)-c(x,\Pb)|} \leq 1$$ and let $(\beta_T)_{T\geq 1} \in \Re_+^{\integ}$. We have 
\begin{align*}
    \limsup_{T \to \infty} \beta_T |\hat{c}_1(x,\Pb,T)-c(x,\Pb)|
    &= \limsup_{T \to \infty} \beta_T |\hat{c}_2(x,\Pb,T)-c(x,\Pb)| \frac{|\hat{c}_1(x,\Pb,T)-c(x,\Pb)|}{|\hat{c}_2(x,\Pb,T)-c(x,\Pb)|} \\
    &\leq \limsup_{T \to \infty} \beta_T |\hat{c}_2(x,\Pb,T)-c(x,\Pb)| \limsup_{T \to \infty} \frac{|\hat{c}_1(x,\Pb,T)-c(x,\Pb)|}{|\hat{c}_2(x,\Pb,T)-c(x,\Pb)|}  \\
    &\leq \limsup_{T \to \infty} \beta_T |\hat{c}_2(x,\Pb,T)-c(x,\Pb)|.
\end{align*}
\end{proof}

\section{Omitted proofs of Section \ref{sec: predictor}: Optimal 
Prediction}
\subsection{Omitted proofs of Subsection \ref{sec: pred exp}: Predictors in the Exponential Regime}
\subsubsection{Proof of Proposition \ref{prop: frasibility pred exp}: Feasibility}\label{App: proof feasibility pred exp}
\begin{proof}[Proof of Proposition \ref{prop: frasibility pred exp}]
We first show that $\cKL$ verifies the out-of-sample guarantee. Let $x,\Pb \in \cX \times \cPin$. Denote $\Gamma = \set{\Pb' \in \mathcal{P}}{I(\Pb',\Pb) > r}$. Observe that, by definition of $\cKL$ \eqref{eq: KL predictor}, for all $T\in \integ$ we have
  \[
    \hat{\Pb}_T \not \in \Gamma \implies c(x,\Pb) \leq \cKL(x,\hat{\Pb}_T,T).
  \]
 Hence, we have
  \begin{align*}
    \limsup_{T\to\infty}
    \frac{1}{a_T} \log \Pb^{\infty}
        \left(
                c(x,\Pb) > \cKL(x,\hat{\Pb}_T,T)
        \right)
    \leq& 
        \limsup_{T\to\infty}
    \frac{1}{rT} \log \Pb^{\infty}
    \left(
      \hat{\Pb}_T \in \Gamma
          \right)
    \leq -\frac{1}{r}\inf_{\Pb'\in \bar{\Gamma}}
             I(\Pb',\Pb),
  \end{align*}
where the last equality uses the Large Deviation Principle, Theorem \ref{thm: LDP finite space}. Notice that $I(\cdot,\cdot)$ is convex and hence continuous on $\cPin \times \cPin$. Hence,
\begin{align*}
\overline{\Gamma} 
&\subset
\set{\Pb' \in \mathcal{P}}{I(\Pb',\Pb) \geq r+1} \cup 
\overline{\set{\Pb' \in \mathcal{P}}{r+1>I(\Pb',\Pb) > r}} \\
&= 
\set{\Pb' \in \cP}{I(\Pb',\Pb) \geq r+1} 
\cup 
\overline{\set{\Pb' \in \cPin}{r+1>I(\Pb',\Pb) > r}}\\
&=
\set{\Pb' \in \mathcal{P}}{I(\Pb',\Pb) \geq r+1} \cup 
\set{\Pb' \in \mathcal{P}}{r+1 \geq I(\Pb',\Pb) \geq r} \\
&= \set{\Pb' \in \mathcal{P}}{I(\Pb',\Pb) \geq r}.
\end{align*}
This implies that 
$\inf_{\Pb'\in \bar{\Gamma}}
             I(\Pb',\Pb) \geq r$,
and therefore, using the previous inequality, we get
\begin{align*}
    \limsup_{T\to\infty}
    \frac{1}{a_T} \log \Pb^{\infty}
        \left(
                c(x,\Pb) > \cKL(x,\hat{\Pb}_T,T)
        \right) \leq -1
\end{align*}

We now show that $\cKL \in \cC$.
The equicontinuty is trivial as the predictor does not depend on $T$. It suffices to prove the continuity and differentiability. The uniform boundedness then is directly implied from the continuity on a compact, and the non-dependence on $T$.
 We will prove that the predictor $\Pb\mapsto \cKL(x, \Pb, T)$ is continuous and differentiable at any $\Pb$ in $\cPin$. Denote with $\gamma(x) = \max_{i\in \Sigma} \loss(x, i)$.

  \textbf{Case I:} Let $\min_{i\in \Sigma} \loss(x, i) = \max_{i\in \Sigma} \loss(x, i)= \gamma(x)$. In this case $\Pb\mapsto \cKL(x, \Pb, T) = \gamma(x)$ which is constant in $\Pb$ and hence differentiable.

  \textbf{Case II:} Let $\min_{i\in \Sigma} \loss(x, i) < \max_{i\in \Sigma} \loss(x, i)$. Following \cite[Proposition 5]{van2020data} it follows that the predictor can be characterized equivalently as the convex continuously differentiable minimization problem
  \begin{equation}
    \label{eq:strong-dual-final-finite}
    \cKL(x,\mb P, T)= \min_{\alpha\geq \gamma(x)} ~\{f(\alpha; x,\Pb, T)\defn \alpha - e^{-r} \exp(\textstyle\sum_{i\in \Sigma} \log \left(\alpha-\loss(x, i) \right) {\mb P(i)})\}
  \end{equation}
  whenever $\Pb'\in \cPin$.   Denote with $\alpha^\star(\Pb)$ its optimal solution.  We have that $\alpha^\star(\Pb)>\gamma(x)$ as 
  \begin{align*}
    & \lim_{\alpha\downarrow \gamma(x)} f'(\alpha; x,\mb P, T)\\
    = & 1-\lim_{\alpha\downarrow \gamma(x)} e^{-r}\exp(\textstyle\sum_{i\in \Sigma} \log \left(\alpha-\loss(x, i) \right) {\mb P(i)}) \cdot \sum_{i\in \Sigma} \frac{{\mb P(i)}}{\alpha-\loss(x, i)}\\
    = & 1- \lim_{\alpha\downarrow \gamma(x)} e^{-r} \prod_{i\in \Sigma}\left(\alpha-\loss(x, i) \right)^{\mb P(i)}\cdot \textstyle\sum_{i\in \Sigma} \frac{{\mb P(i)}}{\alpha-\loss(x, i)}\\
    \leq & 1- \lim_{\alpha\downarrow \gamma(x)} e^{-r}\left(\alpha-\gamma(x) \right)^{\mb P(\Sigma^\star(x))} \cdot \prod_{i\in \Sigma\setminus \Sigma^\star(x)}\left(\alpha-\loss(x, i) \right)^{\mb P(i)} \cdot \textstyle\sum_{i\in \Sigma} \frac{{\mb P(i)}}{\alpha-\gamma(x)}\\
    = & 1- \lim_{\alpha\downarrow \gamma(x)}e^{-r} \left(\alpha-\gamma(x) \right)^{\mb P(\Sigma^\star(x))-1} \prod_{i\in \Sigma\setminus \Sigma^\star(x)}\left(\alpha-\loss(x, i) \right)^{\mb P(i)} < 0
  \end{align*}
  where we denote $\Sigma^\star(x) \defn \set{i\in \Sigma}{\loss(x, i)=\gamma(x)}$ and use that $\alpha - \loss(x,i) \geq \alpha - \gamma(x)$ for all $i\in \Sigma$.
  A standard convex optimization result is hence that the optimal solution $\alpha^\star(\Pb)$ is characterized by the vanishing gradient condition
  \[
    f'(\alpha^\star(\Pb); x,\mb P, T) = 1-e^{-r}\exp(\textstyle\sum_{i\in \Sigma} \log \left(\alpha^\star(\Pb)-\loss(x, i) \right) {\mb P(i)}) \cdot \sum_{i\in \Sigma} \frac{{\mb P(i)}}{\alpha^\star(\Pb)-\loss(x, i)} = 0.
  \]
  Moreover, we have that
  \begin{align*}
    f''(\alpha^\star(\Pb)) & = e^{-r}\exp(\textstyle\sum_{i\in \Sigma} \log \left(\alpha^\star(\Pb)-\loss(x, i) \right) {\mb P(i)}) \left(\sum_{i\in \Sigma}\frac{{\mb P(i)}}{(\alpha^\star(\Pb)-\loss(x, i))^2} - \left(\sum_{i\in \Sigma} \frac{{\mb P(i)}}{\alpha^\star(\Pb)-\loss(x, i)}\right)^2 \right) > 0
  \end{align*}
  where strict positivity is a direct consequence of the Cauchy-Schwarz inequality, i.e.,
  \(
  \norm{a}^2 \norm{b}^2 > (a\tpose b)^2
  \)
  applied to the vectors $a_i\defn\sqrt{\Pb(i)}$ and $b_i = \sqrt{\mb P(i)}/(\alpha^\star(\Pb)-\gamma(x, i))$ for all $i\in \Sigma$.
  Remark indeed that the vectors $a$ and $b$ are not scalar multiples as $\min_{i\in \Sigma} \loss(x, i) < \max_{i\in \Sigma} \loss(x, i)$.
  By the implicit function theorem we have now that $\alpha^\star(\Pb)$ is differentiable at $\Pb$. Hence, the composition $\cKL(x,\mb P, T)= f(\alpha^\star(\Pb); x,\Pb, T)$ is differentiable as well at $\Pb$.
\end{proof}

\subsection{Proof of Theorem \ref{thm: strong opt exponential regime predictor.}: Strong Optimality}\label{Appendix: Claims of exp pred thm}

{\color{black}
\begin{proof}[Proof of Theorem \ref{thm: strong opt exponential regime predictor.}]
Proposition \ref{prop: frasibility pred exp} ensures feasiblity.
Assume that $\cKL$ is not strong optimal. Then, there must exist a predictor $\hat{c}\in \cC$ verifying the out-of-sample guarantee \eqref{eq: out-of-sample ganrantee} and $x_0,\Pb_0 \in \cX \times \cPin$ such that 
\begin{equation}\label{proof eq: dro ineq 1}
\limsup_{T\to \infty} \frac{|\cKL(x_0,\Pb_0,T)-c(x_0,\Pb_0)|}{|\hat{c}(x_0,\Pb_0,T)-c(x_0,\Pb_0)|} > 1
\end{equation}
with the convention $\frac{0}{0} =1$. From the definition of superior limit there must exist an increasing sequence $(t_T)_{T\geq 1} \in \integ^{\integ}$ and $\varepsilon>0$ such that 
\begin{equation}\label{proof eq: dro ineq 2}
|\cKL(x_0,\Pb_0,t_T)-c(x_0,\Pb_0)| \geq (1+\varepsilon)|\hat{c}(x_0,\Pb_0,t_T)-c(x_0,\Pb_0)|, \quad \forall T \in \integ.
\end{equation} 

Let $\bar{\Pb} \in \argmax_{\Pb'\in\cP} \{c(x_0,\Pb') \; : \; I(\Pb_0,\Pb') \leq r\}$ which exists as $\cP$ is compact and $\Pb'\mapsto I(\Pb_0,\Pb')$ lower semicontinuous. We have $\cKL(x_0,\Pb_0,T) = c(x_0,\bar{\Pb})$ for all $T\in \integ$ and $I(\Pb_0,\bar{\Pb}) \leq r$. Using this equality in inequality \eqref{proof eq: dro ineq 2} and the fact that $c(x_0,\bar{\Pb}) \geq c(x_0,\Pb_0)$ by definition of $\bar{\Pb}$, we get
\begin{equation}\label{proof eq: dro ineq 3}
c(x_0,\bar{\Pb}) \geq \hat{c}(x_0,\Pb_0,t_T) + \varepsilon |\hat{c}(x_0,\Pb_0,t_T) - c(x_0,\Pb_0)|,
\quad \forall T \in \integ.
\end{equation}
We use this inequality to prove successively the following claims.
\begin{claim}\label{claim: dro proof}
There exists $\varepsilon_1>0$ and $(l_T)_{T\geq 1} \in \integ^{\integ}$ such that $c(x_0,\bar{\Pb}) \geq \hat{c}(x_0,\Pb_0,l_T) + \varepsilon_1$, for all $T\in \integ$.
\end{claim}
\begin{proof}
    We distinguish two cases. If $c(x_0,\bar{\Pb}) > \limsup_{T\in \integ} \hat{c}(x_0,\Pb_0,t_T)$ then the result is immediate. Suppose now $c(x_0,\bar{\Pb}) \leq \limsup_{T\in \integ} \hat{c}(x_0,\Pb_0,t_T)$.
We have $ c(x_0,\bar{\Pb}) \geq c(x_0,\Pb_0)$ by definition of $\bar{\Pb}$, and $| c(x_0,\bar{\Pb}) - c(x_0,\Pb_0)|>0$ as otherwise the LHS of \eqref{proof eq: dro ineq 1} is zero or one and inequality \eqref{proof eq: dro ineq 1} would fail to hold. Therefore,  $0< c(x_0,\bar{\Pb}) - c(x_0,\Pb_0) \leq \limsup_{T \in \integ} \hat{c}(x_0,\Pb_0,t_T) - c(x_0,\Pb_0)$. Hence, by definition of the limit superior, there exists $\delta>0$ and $(l_T)_{T\geq 1}$, a sub-sequence of $(t_T)_{T\geq 1}$, such that $|\hat{c}(x_0,\Pb_0,l_T) - c(x_0,\Pb_0)| \geq \delta$ for all $T$. Plugging this inequality in \eqref{proof eq: dro ineq 3}, we get the desired result with $\varepsilon_1 = \varepsilon \delta$.
\end{proof}

\begin{claim}\label{claim: dro proof 2}
There exists $\bar{\Pb}_1 \in \cPin$ verifying $I(\Pb_0,\bar{\Pb}_1)<r$, and an open set $U \subset \cPin$ containing $\Pb_0$ such that for all $\Pb'\in U$ and $T \in \integ$, we have $c(x_0,\bar{\Pb}_1) > \hat{c}(x_0,\Pb', l_T)$.
\end{claim}
\begin{proof}
    Let $\varepsilon_1 >0$ and $(l_T)_{T\geq 1} \in \integ^{\integ}$ given by Claim \ref{claim: dro proof}. By continuity of $\Pb\mapsto c(x_0,\Pb)$, there exists $1\geq \lambda>0$ such that 
$\bar{\Pb}_1 = \lambda \Pb_0 + (1-\lambda) \bar{\Pb}$ verifies $c(x_0,\bar{\Pb}_1) \geq c(x_0,\bar{\Pb}) - \varepsilon_1/2$. Hence, by Claim \ref{claim: dro proof}, we have $c(x_0,\bar{\Pb}_1) \geq \hat{c}(x_0,\Pb_0,l_T) + \varepsilon_1/2$ for all $T\in \integ$. 
Using the equicontinuity of $\hat{c}$, there exists an open set $U \subset \cPin$ containing $\Pb_0$ such that for all $\Pb'\in U$ and $T \in \integ$, we have $c(x_0,\bar{\Pb}_1) > \hat{c}(x_0,\Pb', l_T)$. Furthermore, by convexity of the relative entropy, we have 
\(
I(\Pb_0,\bar{\Pb}_1) \leq \lambda I(\Pb_0,\Pb_0) + (1-\lambda)I(\Pb_0,\bar{\Pb}) \leq (1-\lambda)r <r.
\)
\end{proof}

The proof of the second claim uses the continuity of $c$ and equicontinuity of $\hat{c}$ to perturb $\bar{\Pb} \in \cP$ into $\bar{\Pb}_1 \in \cPin$ verifying $I(\Pb_0,\bar{\Pb}_1)<r$, and $\Pb_0$ into an open neighborhood $U$, while only losing a gap less than $\varepsilon_1$ in the inequality of the first claim.
Using the previous claim and then the Large Deviation Principle (Theorem \ref{thm: LDP finite space}) we get
\begin{align*}
    \limsup_{T\to\infty} \frac{1}{rT} \log \bar{\Pb}_1^\infty \left( c(x_0, \bar{\Pb}_1) > \hat c(x_0, \hat{\Pb}_T,T)\right)
    &\geq 
    \limsup_{T\to\infty} \frac{1}{rl_T} \log \Pb_1^\infty 
    \left( 
    \hat{\Pb}_{l_T} \in U
    \right) \\
    & \geq 
    \liminf_{T\to\infty} \frac{1}{rT} \log \Pb_1^\infty 
    \left( 
    \hat{\Pb}_{T} \in U
    \right)\\
    &\geq 
    - \frac{1}{r}\inf_{\Pb' \in U^{\into}}I(\Pb',\Pb_1) 
    \geq - \frac{1}{r}I(\Pb_0,\bar{\Pb}_1) > -1
\end{align*}
where the last two inequalities use $\Pb_0 \in U^{\into}$ and $I(\Pb_0,\bar{\Pb}_1)<r$ from Claim \ref{claim: dro proof 2}. This inequality contradicts the feasibility of $\hat{c}$ as it does not verify the out-of-sample guarantee \eqref{eq: out-of-sample ganrantee}, which completes the proof.
\end{proof}}

\subsection{Omitted proofs of Subsection \ref{sec: superexp pred}: Predictors in the Superexponential Regime}
\subsubsection{Proof of Theorem \ref{thm: overly strong predictor}: Strong Optimality}
\begin{proof}[Proof of Theorem \ref{thm: overly strong predictor}]
Proposition \ref{prop: feasibility or predictor} ensures the feasibility of $\cRob$. Let the predictor $\hat{c}\in \cC$ be a feasible solution in the optimal prediction Problem \eqref{eq:optimal-predictor} with $a_T \gg T$. Let us show that $\cRob \preceq_{\cC} \hat{c}$. Remark that the predictor $\hat{c}$ verifies the out-of-sample guarantee with $a_T \gg T$. Hence, it therefore also verifies the guarantee in the exponential case ($a_T \sim rT$) for all $r>0$. Let $r>0$ be arbitrary and consider the corresponding DRO predictor $\cKL$ defined in \eqref{eq: KL predictor}. Theorem \ref{thm: strong opt exponential regime predictor.} ensures that $\cKL \preceq_{\cC} \hat{c}$. Hence, for every $x,\Pb \in \cX \times \cPin$
\begin{equation*}
\limsup_{T\to \infty} \frac{|\cKL(x,\Pb,T)-c(x,\Pb)|}{|\hat{c}(x_0,\Pb,T)-c(x,\Pb)|} \leq 1
\end{equation*}
Let $x,\Pb \in \cX \times \cPin$.
As $\cKL$ is independent of $T$, we can denote $\cKL(x,\Pb) := \cKL(x,\Pb,T)$ for all $x,\Pb$ and $T$. Moreover, from the definition of the predictor $\cKL(x,\Pb)$ as the supremum $c(x,\cdot)$ over an ambiguity set containing $\Pb$ it follows that $\cKL(x,\Pb)-c(x,\Pb) \geq 0$. Hence we have
\begin{equation*}
 \frac{\cKL(x,\Pb)-c(x,\Pb)}{\liminf_{T\to \infty}|\hat{c}(x,\Pb,T)-c(x,\Pb)|} \leq 1.
\end{equation*}
Let $\Pb' \in \cPin$ be arbitrary. As $\Pb \in \cPin$, there exists $r>0$ such that $I(\Pb,\Pb')<r$. Therefore, for this choice of $r$, we have $c(x,\Pb') \leq \cKL(x,\Pb)$ which implies
\begin{equation*}
 \frac{c(x,\Pb')-c(x,\Pb)}{\liminf_{T\to \infty}|\hat{c}(x,\Pb,T)-c(x,\Pb)|} \leq 1.
\end{equation*}
Taking the supremum over all $\Pb' \in \cPin$, which equals the supremum over $\Pb'\in \cP$ by the continuity of $c$, we get
\begin{equation*}
 \frac{\cRob(x,\Pb)-c(x,\Pb)}{\liminf_{T\to \infty}|\hat{c}(x,\Pb,T)-c(x,\Pb)|} \leq 1.
\end{equation*}
The previous inequality can be rewritten as 
\begin{equation*}
 \limsup_{T\to \infty}\frac{|\cRob(x,\Pb)-c(x,\Pb)|}{|\hat{c}(x,\Pb,T)-c(x,\Pb)|} \leq 1.
\end{equation*}
We have proven this for all $x,\Pb \in \cX\times \cPin$. Hence $\cRob \preceq_{\cC} \hat{c}$ completing the proof.
\end{proof}\label{Appendix: proof of thm superexp pred}

\subsubsection{Self contained proof of Theorem \ref{thm: overly strong predictor}: Strong Optimality}\label{Appendix: alternative proof of strong opt superexp pred}
\begin{proof}[Proof of Theorem \ref{thm: overly strong predictor}]
Suppose for the sake of contradiction that there exists a predictor $\hat{c} \in \cC$ verifying the out-of-sample guarantee, i.e., feasible in \eqref{eq:optimal-predictor} such that $\hat{c}_{R} \not \orderleq \hat{c}$.
There must hence exists $x_0\in \cX$ and $\Pb_0\in \cP$ such that 
\begin{equation}\label{proof eq: strong guarantee 1}
\limsup_{T\to \infty} \frac{|\hat{c}_{R}(x_0,\Pb_0,T)-c(x_0,\Pb_0)|}{|\hat{c}(x_0,\Pb_0,T)-c(x_0,\Pb_0)|} > 1.
\end{equation}
with the convention $\frac{0}{0} =1$. From the definition of superior limit there must exist an increasing sequence $(t_T)_{T\geq 1} \in \integ^{\integ}$ and $\varepsilon>0$ such that 
\begin{equation}\label{proof eq: strong guarantee 2}
|\hat{c}_{R}(x_0,\Pb_0,t_T)-c(x_0,\Pb_0)| \geq (1+\varepsilon)|\hat{c}(x_0,\Pb_0,t_T)-c(x_0,\Pb_0)|, \quad \forall T \in \integ.
\end{equation} 
Let $\bar{\Pb} \in \argmax_{\Pb'\in \cP} c(x_0,\Pb')$. We have $\hat{c}_{R}(x_0,\Pb_0,t_T) = c(x_0,\bar{\Pb})$ for all $T$.
By substituting this equality in Equation \eqref{proof eq: strong guarantee 2} and dropping the absolute values as $c(x_0,\bar{\Pb})\geq c(x_0,\Pb_0)$, we get
\begin{equation}\label{proof eq: overly robust ineq}
c(x_0,\bar{\Pb}) \geq \hat{c}(x_0,\Pb_0,t_T) + \varepsilon |\hat{c}(x_0,\Pb_0,t_T) - c(x_0,\Pb_0)|,
\quad \forall T \in \integ.
\end{equation}
We next use this inequality to prove the following claim.
\begin{claim}
There exists $\varepsilon_1>0$ and a sub-sequence $(l_T)_{T\geq 1}$ such that $
c(x_0,\bar{\Pb}) \geq \hat{c}(x_0,\Pb_0,l_T) + \varepsilon_1
$ for all $T\in \integ$. 
\end{claim}
\begin{proof}
We distinguish two cases. If 
$ c(x_0,\bar{\Pb}) > \limsup_{T \in \integ} \hat{c}(x_0,\Pb_0,t_T)
$, then the result follows immediately. Suppose now $c(x_0,\bar{\Pb}) \leq \limsup_{T \in \integ} \hat{c}(x_0,\Pb_0,t_T)$. 
We have $ c(x_0,\bar{\Pb}) \geq c(x_0,\Pb_0)$ by definition of $\bar{\Pb}$, and $| c(x_0,\bar{\Pb}) - c(x_0,\Pb_0)|>0$ as otherwise the numerator in the LHS of \eqref{proof eq: strong guarantee 1} is zero and \eqref{proof eq: strong guarantee 1} would fail to hold. Therefore,  $0< c(x_0,\bar{\Pb}) - c(x_0,\Pb_0) \leq \limsup_{T \in \integ} \hat{c}(x_0,\Pb_0,t_T) - c(x_0,\Pb_0)$. Hence, by definition of the limit superior, there exists $\delta>0$ and $(l_T)_{T\geq 1}$, a sub-sequence of $(t_T)_{T\geq 1}$, such that $|\hat{c}(x_0,\Pb_0,t_T) - c(x_0,\Pb_0)| \geq \delta$ for all $T$. Plugging this inequality in \eqref{proof eq: overly robust ineq}, we get the desired result with $\varepsilon_1 = \varepsilon \delta$.
\end{proof}

By continuity of $c(x_0,\cdot)$ there exists $\Pb_1 \in \cPin$ such that $
c(x_0,\Pb_1) \geq \hat{c}(x_0,\Pb_0,l_T) + \varepsilon_1/2
$ for all $T$. By equicontinuity of the sequence $(\hat{c}(x_0,\cdot,l_T))_{T\geq 1}$, there exists an open set $U\subset \cP$ containing $\Pb_0$ such that for all $\Pb' \in U$ and $T\in \integ$, we have $c(x_0,\Pb_1)> \hat{c}(x_0,\Pb',l_T)$.
Using the LDP (Theorem \ref{thm: LDP finite space}) we get
\begin{align*}
    \limsup_{T\to\infty} \frac{1}{T} \log \Pb_1^\infty \left( c(x_0, \Pb_1) > \hat c(x_0, \hat{\Pb}_T,T)\right)
    &\geq 
    \limsup_{T\to\infty} \frac{1}{l_T} \log \Pb_1^\infty 
    \left( 
    \hat{\Pb}_{l_T} \in U
    \right) \\
    & \geq 
    \liminf_{T\to\infty} \frac{1}{T} \log \Pb_1^\infty 
    \left( 
    \hat{\Pb}_{T} \in U
    \right)
    \geq 
    - \inf_{\Pb' \in U^{\into}}I(\Pb',\Pb_1) > -\infty,
\end{align*}
therefore,
\begin{align*}
    \limsup_{T\to\infty} \frac{1}{a_T} \log \Pb_1^\infty \left( c(x_0, \Pb_1) > \hat c(x_0, \hat{\Pb}_T,T)\right)
    =
    \limsup_{T\to\infty} \frac{T}{a_T} \frac{1}{T} \log \bar{\Pb}^\infty \left( c(x_0, \Pb_1) > \hat c(x_0, \hat{\Pb}_T,T)\right)
    =
    0,
\end{align*}
as $\lim T/a_T =0$. This implies that $\hat{c}$ violates the out-of-sample guarantee which contradicts our assumption. 
\end{proof}

\subsection{Omitted proofs of Subsection \ref{sec: subexp pred}: Predictors in the subexponential Regime}
\subsubsection{Proof of Proposition \ref{prop: consitency}: Consistency of weakly optimal predictors}
\begin{proof}[Proof of Proposition \ref{prop: consitency}]
Suppose that there exists a predictor $\hat{c} \in \cC$ which satisfies the out-of-sample guarantee and does not converge point-wise to $c$. That is, there must exist $x_0\in \mathcal{X}$ and distribution $\Pb_0 \in \mathcal{P}$ such that $\limsup_{T\to\infty} |\hat{c}(x_0,\Pb_0,T) - c(x_0,\Pb_0)| = \varepsilon>0$. Let $\delta(x,\Pb,T) = \hat{c}(x,\Pb,T) - c(x,\Pb)$ for all $x\in \cX$, $\Pb \in \cP$ and $T \in \integ$. 
Let $\mathcal{T} = \set{T \in \integ}{\delta(x_0,\Pb_0,T) >\frac{\varepsilon}{2}}$.
The out-of-sample guarantee \eqref{eq: out-of-sample ganrantee} implies
that $\liminf_{T\to\infty} \hat{c}(x_0,\Pb_0,T) - c (x_0,\Pb_0) \geq 0$ (see Lemma \ref{lemma: hat c > true c}).
Hence, we can drop the absolute values and conclude that $\limsup_{T\to\infty} \hat{c}(x_0,\Pb_0,T) - c(x,\Pb_0) = \varepsilon$. The previous observation also implies that $\mathcal{T}$ is a set of infinite cardinality.
By equicontinuity of $\hat{c}$ (as $\hat{c}\in \cC$) there exists $\rho>0$ such that the closed ball\footnote{Here the ball is taken for the product topology.} $\mathcal{B}((x_0,\Pb_0),\rho)$ centered around $(x_0,\Pb_0)$ of radius $\rho$ verifies
\begin{equation}\label{proof eq: cal T def}
\forall T \in \mathcal{T}, \; 
\forall (x,\Pb) \in \mathcal{B}((x_0,\Pb_0),\rho), \quad
\delta(x,\Pb,T) > \frac{\varepsilon}{4}
.
\end{equation}

Let the variation $\eta: \cX \times \cP \longrightarrow [0,\frac{\varepsilon}{8}]$ be an infinitely differentiable function\footnote{For example, take the bump function $x \rightarrow \exp(-1/(1-x^2))\mathbb{1}_{x\in (-1,1)}$ scaled accordingly.} of support $\mathcal{B}((x_0,\Pb_0),\frac{\rho}{2})$ such that $\eta(x_0,\Pb_0) = \frac{\varepsilon}{8}$.

Consider the predictor $\hat{c}'$ defined as $\hat{c}'(x,\Pb,T) = \hat{c}(x,\Pb,T) - \eta(x,\Pb) \mathbf{1}_{T \in \mathcal{T}}$. Figure \ref{fig: proof of consistency} illustrates this construction. We will show that $\hat{c}'$ is feasible in Problem \eqref{eq:optimal-predictor} and is strictly preferred to $\hat{c}$ hence establishing the claim.


Let us first show that the derived predictor $\hat{c}'$ is strictly preferred to $\hat{c}$. Let $(t_T)_{T\geq 1}$ be the increasing sequence of elements in $\mathcal{T}$ and $(l_T)_{T\geq 1}$ be the increasing sequence of elements in its complement $\integ \setminus \mathcal{T}$. 
We have the following chain of equalities
$\limsup_{T\to\infty} \hat{c}'(x_0,\Pb_0,T)-c(x_0,\Pb_0) 
= \limsup_{T\to\infty} \delta(x_0,\Pb_0,T)-\eta(x_0,\Pb_0) \mathbf{1}_{T \in \mathcal{T}}
= \max(
\limsup_{T\to\infty} \delta(x_0,\Pb_0,t_T)-\eta(x_0,\Pb_0),
\limsup_{T\to\infty} \delta(x_0,\Pb_0,l_T)
)
\leq 
\max(
\varepsilon-\varepsilon/8,
\varepsilon/2
)$ (see Lemma \ref{lemma: limsup of complementary sequences} for details on the second equality) which is strictly smaller than $ \limsup_{T\to\infty} \hat{c}(x_0,\Pb_0,T)-c(x_0,\Pb_0) =\varepsilon 
$.
Hence, $\hat{c}' \not \equiv \hat{c}$.
Furthermore, the positivity of the considered variation $\eta$ implies that 
$\limsup_{T\to\infty} |\hat{c}'(x,\Pb,T)-c(x,\Pb)|/
|\hat{c}(x,\Pb,T)-c(x,\Pb)| \leq 1
$ for all $(x,\Pb) \in \cX \times \cP$. Hence $\hat{c}'\orderleq \hat{c}$. It remains to show feasibility of $\hat{c}'$, i.e., the derived predictor $\hat{c}'$ is regular and verifies the required out-of-sample guarantee.

Notice first that the regularity of $\eta$ implies $\hat{c}' \in \cC$. In fact, as $\eta$ and $\hat{c}$ are differentiable, $\hat{c}'$ is also differentiable. Moreover, as $\eta$ is continuous and does not depend on $T$, subtracting it from $\hat{c}$ does not affect the uniform boundedness and equicontinuity of the predictor or its derivatives.

Let $(x,\Pb) \in \cX \times \cP$ be arbitrary. Let us verify the out-of-sample guarantee \eqref{eq: out-of-sample ganrantee} at $(x,\Pb)$. Let $p(x,\Pb,T) \defn  \frac{1}{a_T} \log \Pb^\infty ( c(x, \Pb) > \hat{c}'(x, \hat{\Pb}_T,T))$ for all $x,\Pb,T$. We distinguish two cases.

\textbf{Case I:} Suppose $(x,\Pb) \in \mathcal{B}((x_0,\Pb_0),\rho)$. The variation $\eta$ is bounded by $\varepsilon/8$, therefore, using inequality \eqref{proof eq: cal T def}, we have for all $T \in \mathcal{T}$, $\hat{c}'(x,\Pb,T)-c(x,\Pb)\geq \delta(x,\Pb,T) - \eta(x,\Pb) > \varepsilon/4 -\varepsilon/8 = \varepsilon/8$. By equicontinuity of $\hat{c}'$, there exists an open neighborhood $U$ of $\Pb$ independent of $T$ such that for all $\Pb' \in U$, for all $T\in \mathcal{T}$, $\hat{c}'(x,\Pb',T)>c(x,\Pb)$. Hence, 
$\limsup_{T\to\infty} p(x,\Pb,t_T) 
\leq
\limsup_{T\to\infty} \frac{1}{a_{t_T}} \log \Pb^\infty \left( \hat{\Pb}_{t_T} \notin U\right)$.
By the Large Deviation Principle (Theorem \ref{thm: LDP finite space}), 
$$\limsup_{T\to\infty} \frac{1}{t_T} \log \Pb^\infty \left( \hat{\Pb}_{t_T} \notin U\right) \leq \limsup_{T\to\infty} \frac{1}{T} \log \Pb^\infty \left( \hat{\Pb}_{T} \notin U\right) \leq - \inf_{\Pb' \in U^c} I(\Pb',\Pb)< 0$$
as $\Pb \not \in U^c$.
Furthermore, $a_T\ll T \implies \lim_{T\to\infty} T/a_T = \infty$. Therefore, we have that the limit $\limsup_{T\to\infty} \frac{1}{a_{t_T}} \log \Pb^\infty \left( \hat{\Pb}_{t_T} \notin U\right) = - \infty$ diverges.
Hence, $\limsup_{T\to\infty} p(x,\Pb,t_T)= -\infty < -1$ diverges as well. 
For all $T \notin \mathcal{T}$, $\hat{c}'(\cdot,\cdot,T) = \hat{c}(\cdot,\cdot,T)$, therefore, by feasibility of $\hat{c}$ we have $$\limsup_{T\to\infty} p(x,\Pb,l_T) = \frac{1}{a_{l_T}} \log \Pb^\infty \left( c(x, \Pb) > \hat{c}(x, \hat{\Pb}_{l_T},l_T)\right) \leq -1.$$ Combining the results on $p$ for both the sequences $(l_T)_{T\geq 1}$ and $(t_T)_{T\geq 1}$, we get the desired guarantee
$\limsup_{T \to \infty} p(x,\Pb,T) \leq \max(\limsup_{T \to \infty} p(x,\Pb,t_T), \limsup_{T \to \infty} p(x,\Pb,l_T)) <-1 $ (see Lemma \ref{lemma: limsup of complementary sequences} for details on the first inequality).

\textbf{Case II:} Suppose $(x,\Pb) \notin \mathcal{B}((x_0,\Pb_0),\rho)$. Denote with $U$ the compliment of  $\mathcal{B}((x_0,\Pb_0),3\rho/4)$. Notice that $U$ is open and that the perturbation function $\eta$ takes the value zero on $U$. Hence, for all $(x',\Pb') \in U$ and $T \in \integ$, $\hat{c}'(x',\Pb',T) = \hat{c}(x',\Pb',T)$. Furthermore, for $(x,\Pb) \in U$ there exists an open set $U_x$ containing $\Pb$ such that for all $\Pb'\in U_x$, $x,\Pb' \in U$. We have
\begin{align*}
    \frac{1}{a_T} \log \Pb^\infty \left( c(x, \Pb) > \hat{c}'(x, \hat{\Pb}_T,T)\right)
    & \leq 
    \frac{1}{a_T} \log \left [
    \Pb^\infty \left( c(x, \Pb) > \hat{c}'(x, \hat{\Pb}_T,T) 
    ~\&~
    \hat{\Pb}_T \in U_x
    \right)
    + \Pb^{\infty}(\hat{\Pb}_T \notin U_x)
    \right]\\
    &\leq 
    \frac{1}{a_T} \log \left [
    \Pb^\infty \left( c(x, \Pb) > \hat{c}(x, \hat{\Pb}_T,T) 
    \right)
    + \Pb^{\infty}(\hat{\Pb}_T \notin U_x)
    \right]
\end{align*}
Let $\mu = \inf_{\Pb' \in U^c_x} I(\Pb',\Pb)>0$ which is positive as $\Pb \not \in U^c_x$. Using the Large Deviations Principle (Theorem \ref{thm: LDP finite space}), $\Pb^{\infty}(\hat{\Pb}_T \notin U_x) \leq e^{-\mu T + o(T)}$. By feasibility of the predictor $\hat{c}$,  we have furthermore that
$\Pb^\infty ( c(x, \Pb) > \hat{c}(x, \hat{\Pb}_T,T) 
) \leq e^{-a_T + o(a_T)}$
. Hence,
\begin{align*}
    \limsup_{T\to\infty} \frac{1}{a_T} \log \Pb^\infty \left( c(x, \Pb) > \hat{c}'(x, \hat{\Pb}_T,T)\right)
    & \leq 
    \limsup_{T\to\infty}
    \frac{1}{a_T} \log \left [
    e^{-a_T + o(a_T)}
    + e^{-\mu T + o(T)}
    \right] = -1
\end{align*}%
as $a_T\ll T\implies 
    e^{-\mu T + o(T)} = o(e^{-a_T + o(a_T)})$.
\end{proof}

\subsubsection{
Proof  and generalizaiton of Proposition \ref{prop: robust predictor is DRO.}: Robust interpretation}
\label{Appendix: robust-interpr}
In the following theorem only, $\Sigma$ can be a continuous set and $\cP$ is the set of measures (possibly continuous) over $\Sigma$. We denote $\mathcal{B}(\Sigma)$ as the set of all events over $\Sigma$.
\begin{theorem}
  Let $\Pb \in \cP$, $x\in \cX$ and $r>0$. Suppose $\Var_{\Pb}(\loss(x,\xi)) \neq 0$.
  We have
  \begin{align*}
    \sup \left\{ \Eb_{\Pb'}[\loss(x,\xi)] 
      \; : \; 
      \Pb'\in \bar{\cP}, \;
      \int_{\Sigma}\frac{1}{2}\left(\frac{d\Pb'}{d\Pb} -1\right)^2 d\Pb \leq r \right\}
    =\Eb_{\Pb}(\loss(x,\xi))
      + 
      \sqrt{2r \Var_{\Pb}(\loss(x,\xi))}
  \end{align*}
  where $\Var_{\Pb}(\loss(x,\xi)) = \Eb_{\Pb}[(\loss(x,\xi) - \Eb_{\Pb}(\loss(x,\xi)))^2]$ and $\bar{\cP}$ is the set of signed measures summing to $1$.
  Furthermore, the optimal solution of the supremum is attained in the distribution
  $$
  \Pb'(A) = \Pb(A) + \sqrt{\frac{2r}{\Var_{\Pb}(\loss(x,\xi))}}\left( \int_{A}\loss(x,\xi) d\Pb(\xi) - \Eb_{\Pb}(\loss(x,\xi)) \Pb(A) \right),
  \quad \forall A \in \mathcal{B}(\Sigma).
  $$
  In particular, the equality is also true with $\cP$, the set of probability measures, instead of $\bar{\cP}$ when $r$ verifies for all events $A \in \mathcal{B}(\Sigma)$
    $$
    \sqrt{2r} \left( \int_{A}\loss(x,\xi)d\Pb(\xi) - \Eb_{\Pb}[\loss(x,\xi)]\Pb(A) \right) \geq 
    -\Pb(A)
    \sqrt{\Var_{\Pb}(\loss(x,\xi))}.
    $$
\end{theorem}
\begin{remark}
Let $C>0$ be an upper bound on $|\loss(x,\cdot) - \Eb_{\Pb}[\loss(x,\xi)]|$. Then any $r$ verifying the condition
$$
r 
\leq 
\frac{\Var_{\Pb}(\loss(x,\xi))}{2C^2}
$$
verifies the condition of Theorem \ref{Appendix: robust-interpr}. In particular, the set of $r$ verifying the condition of Theorem \ref{Appendix: robust-interpr} is non-empty.

In fact, denoting $\mathbb{1}_A(\xi) := \mathbb{1}(\xi \in A)$ for all $\xi \in \Sigma$ and event $A$, we have
\begin{align*}
    \left( \int_{A}\loss(x,\xi)d\Pb(\xi) - \Eb_{\Pb}[\loss(x,\xi)]\Pb(A) \right)^2 
    &=
    \left( \int (\loss(x,\cdot)- \Eb_{\Pb}[\loss(x,\xi)])\mathbb{1}_A(\cdot) d\Pb \right)^2 
    \leq C^2 \Pb(A)^2
\end{align*}
Hence, if $r \leq 
\frac{\Var_{\Pb}(\loss(x,\xi))}{2C^2}$, then 
$
2r \left( \int_{A}\loss(x,\xi)d\Pb(\xi) - \Eb_{\Pb}[\loss(x,\xi)]\Pb(A) \right)^2 \leq 
    \frac{\Var_{\Pb}(\loss(x,\xi))}{C^2} C^2\Pb(A)^2 = \Var_{\Pb}(\loss(x,\xi)) \Pb(A)^2$.
\end{remark}



\begin{proof}[Proof of Theorem \ref{Appendix: robust-interpr}]
  The LHS can be written explicitly as
  \begin{equation}\label{proof eq: SVP 1}
    \sup \left\{ \Eb_{\Pb'}[\loss(x,\xi)] 
      \; : \; 
      \Pb'\in \cP, \;
      \int_{\Sigma}\frac{1}{2}\left(\frac{d\Pb'}{d\Pb} -1\right)^2 d\Pb \leq r \right\}
  \end{equation}
  We will exhibit a feasible solution to this supremum problem that attains the RHS, then show that the cost of each feasible solution is no larger than the RHS.

  \textbf{Constructing a feasible solution attaining the RHS.} Consider the solution $\Pb' \in \bar{\cP}$ defined as
  $$
  \Pb'(A) = \Pb(A) + \sqrt{\frac{2r}{\Var_{\Pb}(\loss(x,\xi))}}\left( \int_{A}\loss(x,\xi) d\Pb(\xi) - \Eb_{\Pb}(\loss(x,\xi)) \Pb(A) \right),
  \quad \forall A \in \mathcal{B}(\Sigma).
  $$
  Let us verify the feasibility of the solution. We have 
  \begin{align*}
    \Pb'(\Sigma) 
    &=
      \Pb(\Sigma) + \sqrt{\frac{2r}{\Var_{\Pb}(\loss(x,\xi))}}\left( \int_{\Sigma}\loss(x,\xi) d\Pb(\xi) - \Eb_{\Pb}(\loss(x,\xi)) \Pb(\Sigma) \right) \\
    &= 
      1 + \sqrt{\frac{2r}{\Var_{\Pb}(\loss(x,\xi))}}\left( \Eb_{\Pb}(\loss(x,\xi)) - \Eb_{\Pb}(\loss(x,\xi)) \right)  =1
  \end{align*}

  Hence, $\Pb'$ is measure summing to 1, ie $\Pb'\in \bar{\cP}$. Furthermore, if the stated condition on $r$ is verified, we have for all events $A$
  \begin{align*}
      \sqrt{\frac{2r}{\Var_{\Pb}(\loss(x,\xi))}}\left| \int_{\Sigma}\loss(x,\xi) d\Pb(\xi) - \Eb_{\Pb}(\loss(x,\xi)) \Pb(\Sigma) \right| \leq \Pb(A)
  \end{align*}
  and
    \begin{align*}
      \sqrt{\frac{2r}{\Var_{\Pb}(\loss(x,\xi))}}\left| \int_{\Sigma}\loss(x,\xi) d\Pb(\xi) - \Eb_{\Pb}(\loss(x,\xi)) \Pb(\Sigma) \right| \leq 1-\Pb(A).
  \end{align*}
 These two inequalities imply that $\Pb'(A)\geq 0$ and $\Pb'(A)\leq 1$ respectively. Hence, $\Pb' \in \cP$.
  
  Let us now verify the second constraint. We have
  \begin{align*}
    \int_{\Sigma}\frac{1}{2}\left(\frac{d\Pb'}{d\Pb} -1\right)^2 d\Pb
    &=
      \frac{r}{\Var_{\Pb}(\loss(x,\xi))}
      \int_{\Sigma}\left(\loss(x,\cdot) - \Eb_{\Pb}(\loss(x,\xi)) \right)^2 d\Pb
      =
      r
  \end{align*}
  which concludes the proof of feasibility. We now compute the cost of the solution. We have
  \begin{align*}
    \Eb_{\Pb'}[\loss(x,\xi)] 
    &=
      \Eb_{\Pb}(\loss(x,\xi))
      +
      \int_{\Sigma}\loss(x,\xi) \sqrt{\frac{2r}{\Var_{\Pb}(\loss(x,\xi))}}\left( \loss(x,\xi) - \Eb_{\Pb}(\loss(x,\xi)) \right) d\Pb(\xi)\\
    &=
      \Eb_{\Pb}(\loss(x,\xi))
      +
      \sqrt{\frac{2r}{\Var_{\Pb}(\loss(x,\xi))}}
      \int_{\Sigma}\loss(x,\xi) \left( \loss(x,\xi) - \Eb_{\Pb}(\loss(x,\xi))  \right) d\Pb(\xi)\\
    &=
      \Eb_{\Pb}(\loss(x,\xi))
      +
      \sqrt{\frac{2r}{\Var_{\Pb}(\loss(x,\xi))}}
      \left(\Eb_{\Pb}( \loss(x,\xi)^2) - \Eb_{\Pb}(\loss(x,\xi))^2\right)\\
    &=
      \Eb_{\Pb}(\loss(x,\xi))
      +
      \sqrt{2r\Var_{\Pb}(\loss(x,\xi))}
  \end{align*}
  Hence $\Pb'$ is a feasible solution with cost the RHS which proves that LHS$\geq$RHS.

  \textbf{Proving LHS$\leq$RHS.} Let $\Pb' \in \cP$ be a feasible solution to the supremum problem \eqref{proof eq: SVP 1}. $\Pb'$ is absolutely continuous with respect to $\Pb$ by feasibility. Hence, we have
  \begin{align*}
    \Eb_{\Pb'}[\loss(x,\xi)]
    &=
      \Eb_{\Pb}(\loss(x,\xi))
      +
      \int_{\Sigma} \loss(x,\xi) d(\Pb' - \Pb)(\xi) \\
    &= 
      \Eb_{\Pb}(\loss(x,\xi))
      +
      \int_{\Sigma} [\loss(x,\xi)-\Eb_{\Pb}(\loss(x,\xi))] d(\Pb' - \Pb)(\xi) \\
    &= 
      \Eb_{\Pb}(\loss(x,\xi))
      +
      \int_{\Sigma} [\loss(x,\xi)-\Eb_{\Pb}(\loss(x,\xi))] \left(\frac{d\Pb'}{d\Pb}-1\right)(\xi)d\Pb(\xi) \\
    &\leq 
      \Eb_{\Pb}(\loss(x,\xi))
      +
      \sqrt{\int_{\Sigma} \left(\loss(x,\xi)-\Eb_{\Pb}(\loss(x,\xi))\right)^2d\Pb(\xi) }
      \sqrt{\int_{\Sigma}\left(\frac{d\Pb'}{d\Pb}-1\right)^2d\Pb} \\
    &\leq 
      \Eb_{\Pb}(\loss(x,\xi))
      +
      \sqrt{\Var_{\Pb}(\loss(x,\xi))}
      \sqrt{2r}
  \end{align*}
  where the second equality is justified by $\int_{\Sigma}\Pb'-\Pb =0$ , the first inequality is by Cauchy-Schwartz and the last inequality uses the constraint verified by $\Pb'$ in \eqref{proof eq: SVP 1}. Hence, any feasible solution of the supremum problem \eqref{proof eq: SVP 1} has cost no larger than the RHS, which completes the proof.
\end{proof}
\begin{proof}[Proof of Proposition \ref{prop: robust predictor is DRO.}]
The proof follows immediately from the previous theorem by choosing $r=a_T/T$, $\varphi_x(\Pb) = \Pb' - \Pb$ and identifying the condition on $a_T/T$ such that $0\leq \Pb'(i)\leq 1$ for all $i$.
The identities of $\varphi_x$ follows from the following more general lemma.
\begin{lemma}\label{lemma: Cov expression}
Let $\Pb \in \cPin$ and $x_1,x_2 \in \cX$. We have
\begin{align*}
    c(x_1,\varphi_{x_2}(\Pb)) 
    &=
    \Cov_{\Pb}(\loss(x_1,\xi),\loss(x_2,\xi)) / \sqrt{\Var_{\Pb}(\loss(x_2,\xi))}
\end{align*}
where 
$\Cov_{\Pb}(\loss(x_1,\xi),\loss(x_2,\xi)):= \Eb(\loss(x_1,\xi)\loss(x_2,\xi)) - \mathbb{E}_{\Pb}(\loss(x_1,\xi))\mathbb{E}_{\Pb}(\loss(x_2,\xi))$ and $\varphi_{x_1}(\Pb)$, $\varphi_{x_2}(\Pb)$ are defined in Proposition \ref{prop: robust predictor is DRO.}.
\end{lemma}
\begin{proof}
We have
\begin{align*}
    c(x_1,\varphi_{x_2}(\Pb)) 
    &=
    \frac{1}{\sqrt{\Var(\loss(x_2,\xi))}}
    \left(
        c(x_1,\loss(x_2,\cdot) \odot \Pb)
        -
        c(x_1,\Pb)c(x_2,\Pb)
    \right)\\
    &=
    \frac{1}{\sqrt{\Var(\loss(x_2,\xi))}}
    \left(
        \sum_{i=1}^{d} \loss(x_1,i)\loss(x_2,i)\Pb(i)
        -
        c(x_1,\Pb)c(x_2,\Pb)
    \right)\\
    &=
    \frac{1}{\sqrt{\Var(\loss(x_2,\xi))}}
    \left(
        \Eb_{\Pb}(\loss(x_1,\xi)\loss(x_2,\xi))
        -
        \Eb_{\Pb}(\loss(x_1,\xi))\Eb_{\Pb}(\loss(x_2,\xi)
    \right)\\
    &= \Cov_{\Pb}(\loss(x_1,\xi),\loss(x_2,\xi))
\end{align*}%
\end{proof}%
\end{proof}%

\subsubsection{Proof of Proposition \ref{prop: regularity of SVP}: Regularity of SVP}\label{Appendix: proof of SVP regularity}
\begin{proof}[Proof of Proposition \ref{prop: regularity of SVP}]
\textit{Unifrom boundedness:} $\loss(\cdot,\cdot)$ is bounded, therefore, both its expectation and variance are bounded. Moreover, $a_T/T \to 0$, hence $(a_T/T)_{T\geq 1}$ is uniformally bounded. The predictor $\cSVP$ is a sum and product of the expectation, the square root of the variance and $a_T/T$, and is therefore uniformally bounded. \textit{Equicontinuity:} It is clear that for each $T$, $\cSVP(\cdot,\cdot,T)$ is continuous as it is defined as the elementary composition of continuous functions. Let $\varepsilon>0$. Let $K>0$ be a bound on the standard deviation $(x,\Pb) \rightarrow \sqrt{\Var_{\Pb}(\ell(x,\xi))}$. Let $T_0 \in \integ$ be such that for all $T\geq T_0$, $\sqrt{a_T/T} \leq \varepsilon /(4K)$. Denote $D$ the distance compatible with the product topology of $\cX \times \cP$. Let $\delta>0$ be such that for all $x_1,x_2 \in \cX$ and $\Pb_1,\Pb_2 \in \cP$ such that $D((x_1,\Pb_1),(x_2,\Pb_2))\leq \delta$, we have  $|c(x_1,\Pb_1) - c(x_2,\Pb_2)| \leq \varepsilon/2$ and $|\cSVP(x_1,\Pb_1,T) - \cSVP(x_2,\Pb_2,T)| \leq \varepsilon$ for all $T < T_0$. Such $\delta$ exists as the finite number of functions $c, \cSVP(\cdot,\cdot, 1), \ldots, \cSVP(\cdot,\cdot, T_0 -1)$ are continuous. For all $x_1,x_2 \in \cX$ and $\Pb_1,\Pb_2 \in \cP$ such that $D((x_1,\Pb_1),(x_2,\Pb_2))\leq \delta$, we have for $T\geq T_0$
\begin{align*}
    |\cSVP(x_1,\Pb_1,T) - \cSVP(x_2,\Pb_2,T)|
    &\leq
    |c(x_1,\Pb_1)-c(x_2,\Pb_2)|
    +
    \sqrt{\frac{a_T}{T}} \left| \sqrt{\Var_{\Pb_1}(\loss(x_1,\xi))} 
    -
    \sqrt{\Var_{\Pb_2}(\loss(x_2,\xi))} \right| \\
    &\leq 
    \varepsilon/2 + \frac{\varepsilon}{4K} 2K
    = \varepsilon.
\end{align*}
Hence, for all $T \in \integ$, $|\cSVP(x_1,\Pb_1,T) - \cSVP(x_2,\Pb_2,T)| \leq \varepsilon$ which proves the equicontinuity. \textit{Differentiable:} Both the expectation and the variance are infinitely differentiable in $\Pb$, therefore, $\cSVP$ is infinitely differentiable in $\Pb$. Hence, $\cSVP$ is differentiable with continuous derivatives. \textit{Regularity of derivatives:} We have $\nabla \cSVP(x,\cdot,T)(\Pb) = \ell_x + \sqrt{a_T/T} \nabla \sqrt{\Var_{(\cdot)}(\ell(x,\xi))} (\Pb)$ for all $x\in \cX$ and $T\in \integ$, where $\ell_x$ is the vector $(\ell(x,1),\ldots,\ell(x,d))^\top$. The function $\Pb \rightarrow \nabla \sqrt{\Var_{(\cdot)}(\ell(x,\xi))} (\Pb)$ is continuous and therefore bounded on the compact $\cP$, hence, we can apply the same proof as of boundedness and equicontinuity of $\cSVP$ to get boundedness and equicontinuity of the derivative $\Pb \rightarrow \nabla \cSVP(x,\cdot,T)(\Pb)$.
\end{proof}

\subsubsection{Proof of Claim \ref{claim: set inclusions for Gammas}: Convergence of Ellipsoids}\label{Appendix: Convergence of Ellipsoids}
\begin{proof}[Proof of Claim \ref{claim: set inclusions for Gammas}]
We show that $\sqrt{1-\varepsilon_T}\Gamma^c \subset \Gamma_T^c \subset \sqrt{1+\varepsilon_T}\Gamma^c$.

Notice first that $\Gamma^c$ and $\Gamma_T^c$ are bounded in infinity norm. In fact, $\Gamma^c$ is a bounded ellipsoid, therefore, there exists $B>0$ such that $\|\Delta\|_{\infty}\leq B$ for all $\Delta \in \Gamma$. Moreover, for all $T \in \integ$ and $\Delta \in \Gamma_T^c$. 
\begin{align*}
    1 &\geq 
    \frac{1}{2}\sum_{i\in \Sigma}\frac{\Delta_i^2}{\Pb(i) + \Delta_i \sqrt{a_T/T}} 
    \geq  
    \frac{1}{2}\sum_{i\in \Sigma}\frac{\Delta_i^2}{1} 
    \geq 
    \frac{1}{2} \|\Delta\|_{\infty}^2.
\end{align*}

where we used $\Delta \in \Gamma_T^c \implies \sqrt{a_T/T}\Delta \in \cP_0(\Pb) \implies \Pb + \sqrt{a_T/T}\Delta \in \cP \implies  \|\Pb + \sqrt{a_T/T}\Delta\|_{\infty} \leq 1$.
Hence, $\|\Delta\|_{\infty} \leq 2$ independently of $T$. Let $K=\max(B,2)$. Notice that all elements of $\Gamma^c$ and $\Gamma^c_T$ are bounded in infinity norm by $K$.
Let $\varepsilon_T = (K\sqrt{a_T/T})/\min_{i\in \Sigma} \Pb(i)$ for all $T$. We start by showing the second inclusion. Let $\Delta \in \Gamma^c_T$. We have
\begin{align*}
    1 \geq 
    \frac{1}{2}\sum_{i\in \Sigma}\frac{\Delta_i^2}{\Pb(i) + \Delta_i\sqrt{a_T/T}}
    &\geq 
    \frac{1}{2}\sum_{i\in \Sigma}\frac{\Delta_i^2}{\Pb(i) +  K\sqrt{a_T/T}}\\
    &= 
    \frac{1}{2}\sum_{i\in \Sigma}\frac{\Delta_i^2}{\Pb(i) + \varepsilon_T \min_{j\in \Sigma} \Pb(j)}\\
    & \geq 
    \frac{1}{2}\sum_{i\in \Sigma}\frac{\Delta_i^2}{\Pb(i) + \varepsilon_T \Pb(i)}  \\
    &=
    \frac{1}{2} \frac{1}{1 + \varepsilon_T}\sum_{i\in \Sigma}\frac{\Delta_i^2}{\Pb(i)}
\end{align*}
This implies that $\frac{1}{\sqrt{1 + \varepsilon_T}} \Delta \in \Gamma^c$. We have shown therefore that $\Gamma^c_T \subset (\sqrt{1 + \varepsilon_T})\Gamma^c$. 

We now show the first inclusion. Let $\Delta \in \mathcal{P}_{0,\infty}$ such that $\Delta \in \sqrt{1-\varepsilon_T}\Gamma^c$. We have
\begin{align*}
    1 &\geq 
    \frac{1}{2}\sum_{i\in \Sigma}
    \frac{\Delta_i^2}
    {(1-\varepsilon_T)\Pb(i)}\\
    &= 
    \frac{1}{2}\sum_{i\in \Sigma}\frac{\Delta_i^2}{\Pb(i) - (K\Pb(i)\sqrt{a_T/T}) /\min_{j\in \Sigma} \Pb(j)}\\
    & \geq 
    \frac{1}{2}\sum_{i\in \Sigma}\frac{\Delta_i^2}{\Pb(i) + \Delta_i\sqrt{a_T/T}}
\end{align*}
To complete the proof of the inclusion in $\Gamma_T^c$, it remains to prove that $\Delta \in \sqrt{T/a_T}\mathcal{P}_0(\Pb)$. As $\Pb \in \cPin$, $\mathcal{P}_0(\Pb)$ contains a non empty ball around $0\in \mathcal{P}_0(\Pb)$ for the norm infinity in the topology of $\cP$. Furthermore, $T/a_T \to \infty$, therefore, there exists $T_1 \in \integ$ such that for all $T\geq T_1$, $\sqrt{T/a_T}\mathcal{P}_0(\Pb)$ contains the ball of norm infinity, around $0$ of radius $K$. Hence, as $\|\Delta\|_{\infty} \leq K$, independently of $T$, for all $T\geq T_1$, $\Delta \in \sqrt{T/a_T}\mathcal{P}_0(\Pb)$. This completes the proof of the inclusion.
\end{proof}

\subsubsection{Omitted proofs in the proof of Proposition \ref{prop: feasibility of robust predictor}: Strong Optimality}\label{Appendix: proof of long claim storng optimality pred}
\begin{proof}[Proof of Claim \ref{claim: unif conv.}]
Let 
$$R_{\text{inf}}(\hat{c},x_0,\Pb_0) =\liminf_{T\to\infty} \frac{1}{\sqrt{\alpha_T}}(\hat{c}(x_0,\Pb_0,T)-c(x_0,\Pb_0))$$
and
$$R_{\text{sup}}(\cSVP,x_0,\Pb_0) = \lim_{T\to\infty} \frac{1}{\sqrt{\alpha_T}}|\cSVP(x_0,\Pb_0,T)-c(x_0,\Pb_0)|.$$
Let $\varepsilon = (R_{\text{inf}}(\cSVP, x_0, \Pb_0) - R_{\text{sup}}(\hat c, x_0, \Pb_0)) /4 >0$ which is positive by assumption of the proof. We have
$R_{\text{inf}}(\hat c, x_0, \Pb_0) +\varepsilon 
<
R_{\text{sup}}(\cSVP, x_0, \Pb_0)$.
As $R_{\text{sup}}(\cSVP, x_0, \Pb_0) = \lim_{T\to\infty} \frac{1}{\sqrt{\alpha_T}}|\cSVP(x,\Pb,T)-c(x,\Pb)| = \sqrt{\Var_{\Pb}(\loss(x, \xi))}$ there exists a sub-sequence $(l_T)_{T \geq 1}$ (corresponding to the inferior limit in the definition of $R_{\text{inf}}(\hat c, x_0, \Pb_0)$) such that for all $T \in \integ$,
$$
\frac{1}{\sqrt{\alpha_{l_T}}}(\hat{c}(x_0,\Pb_0,l_T) -c(x_0,\Pb_0))
+ \varepsilon 
\leq 
\frac{1}{\sqrt{\alpha_{l_T}}}|\cSVP(x_0,\Pb_0,l_T) -c(x_0,\Pb_0)|.
$$
We can drop the absolute values in the right hand-side as the SVP predictor is always greater than the true cost by construction. Therefore, for all $T \in \integ$
\begin{equation}\label{eq: claim epsilon gap.}
\hat{c}(x_0,\Pb_0,l_T)  
+ \varepsilon \sqrt{\alpha_{l_T}} 
\leq \cSVP(x_0,\Pb_0,l_T).
\end{equation}

We now consider the subsequence of $(\hat{c}(x_0,\cdot,l_T))_{T\geq 1}$ which enjoys the desired regularity the complete the analysis. 
This sequence is equicontinuous and uniformly bounded (as $\hat{c} \in \cC$), therefore, by Arzelà–Ascoli theorem (Theorem \ref{thm: Arzela–Ascoli}), there exists a sub-sequence of $(\hat{c}(x_0,\cdot,l_T))_{T\geq 1}$ that converges uniformly.
This sub-sequence, in turn, has equicontinuous and uniformally bounded sequence of gradients in $\Pb$, as $\hat{c} \in \cC$. By Arzelà–Ascoli theorem, we can extract a sub-sequence such that the corresponding sequence of gradients of this sub-sequence converges uniformly.
Denote by $(t_T)_{T\geq 1}$ the corresponding sequence of indices of this sub-sequence, and $\hat{c}_{\infty}(x_0,\cdot)$ the limit of $(\hat{c}(x_0,\cdot,t_T))_{T\geq 1}$. Uniform convergence of the derivatives further imply that $(\nabla\hat{c}(x_0,\cdot,t_T))_{T\geq 1}$ converges uniformly to $\nabla \hat{c}_{\infty}(x_0,\cdot)(\cdot): \cP \rightarrow \Re^{d}$.

For all $x,\Pb,T$, let $\hat{\delta}(x,\Pb,T) = \hat{c}(x,\Pb,T) - c(x,\Pb)$ and $\hat{\delta}_{\infty}(x,\Pb) = \hat{c}_{\infty}(x,\Pb) - c(x,\Pb)$ its limit. Similarly, denote $\deltaSVP(x,\Pb,T) = \cSVP(x,\Pb,T) - c(x,\Pb)$.
We consider two cases.

\textbf{Case I:} Consider the case where $\nabla \hat{\delta}_{\infty}(x_0,\cdot)(\Pb_0) \neq 0$. \\
$(\hat{\delta}(\cdot,\cdot,t_T))_{T\geq1}$ clearly inherits the regularity properties of $(\hat{c}(\cdot,\cdot,t_T))_{T\geq 1}$ (Definition \ref{def: regular pred}).
By uniform convergence of $(\nabla \hat{\delta}(x_0,\cdot,t_T))_{T\geq 1}$, there exists $T_0$ such that
\begin{equation}\label{proof eq: gradient ineq 2}
   \|\nabla \hat{\delta}(x_0,\cdot,t_T)(\Pb)
-
\nabla \hat{\delta}_{\infty}(x_0,\cdot)(\Pb) \|
\leq \|\nabla \hat{\delta}_{\infty}(x_0,\cdot)(\Pb_0)\|/4, 
\quad \forall \Pb \in \cPin, \; \forall T>T_0.   
\end{equation}
By continuity of $\nabla \hat{\delta}_{\infty}(x_0,\cdot)$ (inherited from the continuity of $\nabla \hat{\delta}(x_0,\cdot,t_T)$ by the uniform convergence), 
there exists an open ball $\mathcal{B}(\Pb_0,r_0)$ around $\Pb_0$ such that for all $\Pb \in \mathcal{B}(\Pb_0,r_0)$,
$\|\nabla \hat{\delta}_{\infty}(x_0,\cdot)(\Pb)
-
\nabla \hat{\delta}_{\infty}(x_0,\cdot)(\Pb_0) \|
\leq \|\nabla \hat{\delta}_{\infty}(x_0,\cdot)(\Pb_0)\|/4$, which implies with the previous inequality \eqref{proof eq: gradient ineq 2}
\begin{equation}\label{proof eq: gradient ineq 1}
\|\nabla \hat{\delta}(x_0,\cdot,t_T)(\Pb) 
-
\nabla \hat{\delta}_{\infty}(x_0,\cdot)(\Pb_0) \|
\leq 
\|\nabla \hat{\delta}_{\infty}(x_0,\cdot)(\Pb_0)\|/2, 
\quad 
\forall \Pb \in \mathcal{B}(\Pb_0,r_0),
\; \forall T \geq T_0
.
\end{equation}

Choose now $\Pb_1 = \Pb_0 - r_0 \frac{\nabla \hat{\delta}_{\infty}(x_0,\cdot)(\Pb_0)}{ \|\nabla \hat{\delta}_{\infty}(x_0,\cdot)(\Pb_0)\|}$.
Notice that as the gradient is of a function defined on the simplex $\cP$, we can chose $r_0$ small enough such that $\Pb_1\in \cPin$. 
Using the mean value theorem, for all $T\geq T_0$, there exists 
$\Pb'_T\in [\Pb_0,\Pb_1] \subset \mathcal{B}(\Pb_0,r_0)$ such that 
\begin{align*}
& \hat{\delta}(x_0,\Pb_1,t_T) - \hat{\delta}(x_0,\Pb_0,t_T)\\
=&
\nabla \hat{\delta}(x_0,\cdot,t_T)(\Pb'_T)^\top (\Pb_1-\Pb_0) \\
=& -r_0  
\frac{
    \nabla \hat{\delta}_{\infty}(x_0,\cdot)(\Pb_0)^\top
    \nabla \hat{\delta}(x_0,\cdot,t_T)(\Pb'_T)}
    { \|\nabla \hat{\delta}_{\infty}(x_0,\cdot)(\Pb_0)\|} \\
=&
-r_0  
\frac{
    \|\nabla \hat{\delta}_{\infty}(x_0,\cdot)(\Pb_0)\|^2 
    + \nabla \hat{\delta}_{\infty}(x_0,\cdot)(\Pb_0)^\top
        (\nabla \hat{\delta}(x_0,\cdot,t_T)(\Pb'_T)
        -
        \nabla \hat{\delta}_{\infty}(x_0,\cdot)(\Pb_0))
    }
    { \|\nabla \hat{\delta}_{\infty}(x_0,\cdot)(\Pb_0)\|}
\end{align*}
Using Cauchy-Schwarz and then inequality \eqref{proof eq: gradient ineq 1}, we can bound the previous term as
\begin{align*}
&\leq 
-r_0  
\frac{
    \|\nabla  \hat{\delta}_{\infty}(x_0,\cdot)(\Pb_0)\|^2 
    -  \|\nabla \hat{\delta}_{\infty}(x_0,\cdot)(\Pb_0)\|
        \|\nabla \hat{\delta}(x_0,\cdot,t_T)(\Pb'_T)
        -
        \nabla \hat{\delta}_{\infty}(x_0,\cdot)(\Pb_0)\|
    }
        { \|\nabla \hat{\delta}_{\infty}(x_0,\cdot)(\Pb_0)\|} \\ 
&\leq 
-r_0  
\frac{
    \|\nabla \hat{\delta}_{\infty}(x_0,\cdot)(\Pb_0)\|^2 
    -  \|\nabla\hat{\delta}_{\infty}(x_0,\cdot)(\Pb_0)\| \frac{1}{2}\|\nabla\hat{\delta}_{\infty}(x_0,\cdot)(\Pb_0)\|
    }
        { \|\nabla \hat{\delta}_{\infty}(x_0,\cdot)(\Pb_0)\|} \\ 
&=
-\frac{r_0}{2}  
\|\nabla\hat{\delta}_{\infty}(x_0,\cdot)(\Pb_0)\|
\end{align*}
Hence we get for all $T\geq T_0$
\begin{equation}\label{proof eq: opt-pred: first case ineq}
\hat{\delta}(x_0,\Pb_1,t_T) - \hat{\delta}(x_0,\Pb_0,t_T)
\leq 
-\frac{r_0}{2}  
\|\nabla\hat{\delta}_{\infty}(x_0,\cdot)(\Pb_0)\| 
:= -\tilde{\varepsilon} <0
\end{equation}

Inequality \eqref{eq: claim epsilon gap.} implies that $\lim_{T\to \infty}\hat{\delta}(x_0,\Pb_0,t_T) = 0$ and the consistency of the predictor $\cSVP$ ensures that $\lim_{T\to \infty}\deltaSVP(x_0,\Pb_1,t_T)=0$.

Hence there exists $T_1 \geq T_0$ such that for all $T\geq T_1$,
$\hat{\delta}(x_0,\Pb_0,t_T) - \deltaSVP(x_0,\Pb_1,t_T) \leq \tilde{\varepsilon}/2$. For $T\geq T_1$, using successively inequality \eqref{proof eq: opt-pred: first case ineq} and the previous inequality, we have
\begin{align*}
    \hat{c}(x_0,\Pb_1,t_T) 
    &= c(x_0,\Pb_1) + \hat{\delta}(x_0,\Pb_1,t_T)\\
    &\leq c(x_0,\Pb_1) + \hat{\delta}(x_0,\Pb_0,t_T) -\tilde{\varepsilon} \\
    &= \cSVP(x_0,\Pb_1,t_T) - \deltaSVP(x_0,\Pb_1,t_T) + \hat{\delta}(x_0,\Pb_0,t_T) -\tilde{\varepsilon}\\
    & \leq \cSVP(x_0,\Pb_1,t_T) - \tilde{\varepsilon}/2
\end{align*}
Using the equicontinuity of $\hat{c}$ and $\cSVP$ in $\Pb_1$, the previous inequality implies that there exists an open ball $\mathcal{B}(\Pb_1,r_1)$, with $r_1<r_0$, such that for all $\Pb \in \mathcal{B}(\Pb_1,r)$ and $T\geq T_1$, 
$\hat{c}(x_0,\Pb,t_T) \leq \cSVP(x_0,\Pb,t_T) - \tilde{\varepsilon}/4$. This inequality is stronger than the desired result. In fact, fix $\varepsilon'>0$ and let $r_2<r_1$ and $T_2>T_1$ such that for all $T\geq T_2$, $\tilde{\varepsilon}/4 \geq \varepsilon \sqrt{\alpha_{t_T}} - \varepsilon' r_2$. We have for all $\Pb \in \mathcal{B}(\Pb_1,r')$ and $T\geq T_2$, $\hat{c}(x_0,\Pb,t_T) \leq \cSVP(x_0,\Pb,t_T) - \tilde{\varepsilon}/4 \leq \cSVP(x_0,\Pb,t_T) - \varepsilon \sqrt{\alpha_{t_T}} + \varepsilon' \|\Pb - \Pb_1\|$  which is the desired result.

\textbf{Case II:} We now turn to the case where $\nabla \hat{\delta}_{\infty}(x_0,\cdot)(\Pb_0) = 0$.\\
We show the result with $\Pb_1:=\Pb_0$. Fix $\varepsilon'>0$. Notice that for all $\Pb \in \cPin$ and $T \in \integ$, $\nabla \deltaSVP(x_0,\cdot,t_T)(\Pb) = \sqrt{\alpha_{t_T}} \nabla\sqrt{\Var_{(\cdot)}(\loss(x_0,\xi))}(\Pb)$ and the gradient of the standard deviation is bounded, therefore, the sequence $(\nabla \deltaSVP(x_0,\cdot,t_T)(\cdot))_{T\geq 1}$ converges uniformly to zero. Hence, there exists $r>0$ and $T_0\in \integ$ such that for all $\Pb \in \mathcal{B}(\Pb_0,r)$, for all $T\geq T_0$, we have 
$\|\nabla \hat{\delta}(x_0,\cdot,t_T)(\Pb)\| \leq \varepsilon'/2$ 
and 
$\|\nabla \deltaSVP(x_0,\cdot,t_T)(\Pb)\| \leq \varepsilon'/2$.
Using successively the mean value theorem and Cauchy-Schwarz on $\hat{\delta}$, inequality \eqref{eq: claim epsilon gap.}, then the mean value theorem for $\deltaSVP$, we have for all $\Pb \in \mathcal{B}(\Pb_0,r)$ and $T \geq T_0$
\begin{align*}
    \hat{\delta}(x_0,\Pb,t_T) 
    &\leq \hat{\delta}(x_0,\Pb_0,t_T) + \varepsilon'/2 \|\Pb -\Pb_0\| \\
    &\leq \deltaSVP(x_0,\Pb_0,t_T) - \varepsilon\sqrt{\alpha_{t_T}} + \varepsilon'/2 \|\Pb -\Pb_0\| \\
    & \leq \deltaSVP(x_0,\Pb,t_T) - \varepsilon\sqrt{\alpha_{t_T}} + \varepsilon' \|\Pb -\Pb_0\|
\end{align*}
which implies directly the desired result, with $\Pb_1 := \Pb_0$, by adding $c(x_0,\Pb)$ on both sides. Notice finally than we can assume WLOG that the result is true for all $T\in \integ$ (and not starting at some threshold $T_0$) as we can simply appropriately modify the sequence $(t_T)_{T\geq 1}$.
\end{proof}

\subsubsection{Proof of Proposition \ref{prop: predictor finite sample guarantees.}: SVP finite sample guarantees}\label{Appendix: proof of finite guarantees}
\begin{proof}[Proof of Proposition \ref{prop: predictor finite sample guarantees.}]
Let us prove the second result. Let $\delta>0$. Bennett's inequality (\cite{hoeffding1994probability}, Theorem 3 in \cite{maurer2009empirical}) implies that with probability at least $1-\delta$
$$
\Eb_{\Pb}(\loss(x,\xi)) - \Eb_{\hat{\Pb}_T}(\loss(x,\xi))
\geq
-\sqrt{\frac{2\ln 1/\delta}{T}\Var_{\Pb}(\loss(x,\xi))} 
-\frac{K\ln 1/\delta}{3T}.
$$
Hence,
\begin{align*}
\Eb_{\Pb}(\loss(x,\xi)) - &\Eb_{\hat{\Pb}_T}(\loss(x,\xi)) - 
 \sqrt{\frac{2\ln 1/\delta}{T}\Var_{\hat{\Pb}_T}(\loss(x,\xi))}
\geq \\
&-\sqrt{\frac{2\ln 1/\delta}{T}\Var_{\Pb}(\loss(x,\xi))} 
- \sqrt{\frac{2\ln 1/\delta}{T}\Var_{\hat{\Pb}_T}(\loss(x,\xi))}
- \frac{K\ln 1/\delta}{3T}.
\end{align*}
Theorem 10 in \cite{maurer2009empirical} implies that with probability at least $1-\delta$
$$
\sqrt{\Var_{\hat{\Pb}_T}(\loss(x,\xi))} 
\leq 
\sqrt{\Var_{\Pb}(\loss(x,\xi))} + \sqrt{\frac{2K\ln1/\delta}{T}}.
$$
Hence, with probability at least $1-2\delta$,
$$
\Eb_{\Pb}(\loss(x,\xi)) - \Eb_{\hat{\Pb}_T}(\loss(x,\xi)) - \sqrt{\frac{2\ln 1/\delta}{T}\Var_{\hat{\Pb}_T}(\loss(x,\xi))}
\geq
-\sqrt{\frac{8\ln 1/\delta}{T}\Var_{\Pb}(\loss(x,\xi))} 
- \frac{7K\ln 1/\delta}{3T}.
$$
The result follows by chosing $\ln \delta = a_T$ and substituting $c$ and $\cSVP$ by their expressions.

The first inequality follows from the same arguments. Notice that we can directly get a similar result to the first inequality from Theorem 1 in \cite{audibert2009exploration}. Theorem 4 in \cite{maurer2009empirical} is also a similar result with the non-biased empirical variance.
\end{proof}

{\color{black}
\subsubsection{Proof Proposition \ref{prop: KL equiv to SVP}: Approximation of the KL predictor by the $\chi^2$ predictor}\label{App: proof of KL=SVP}
\begin{proof}[Proof Proposition \ref{prop: KL equiv to SVP}]
Denote $E^{\text{KL}}_T = \{ \Pb' \in \cP \; : \; I(\Pb,\Pb') \leq \frac{a_T}{T} \}$
and $E^{\chi^2}_T = \{ \Pb' \in \cP \; : \; \|\Pb' - \Pb \|_{\Pb}^2 \leq \frac{a_T}{T}\}$. \\
Denote $\mathcal{E}^{\text{KL}}_T = \sqrt{a_T/T}(E^{\text{KL}}_T - \Pb) = \{\Delta \in \sqrt{T/a_T} \mathcal{P}_0(\Pb) : I(\Pb, \Pb+\sqrt{a_T/T}\Delta) \leq a_T/T\}$, where $\mathcal{P}_0(\Pb) = \set{\Pb - \Pb'}{\Pb' \in \mathcal{P}}$.\\
Similarly, denote
$\mathcal{E}^{\chi^2}_T 
=
\sqrt{a_T/T}(E^{\chi^2}_T - \Pb) 
=
\{\Delta \in \sqrt{T/a_T} \mathcal{P}_0(\Pb) : \|\Delta\|_{\Pb}^2 \leq 1\}$. To simplify notations, denote $\alpha_T = \sqrt{a_T/T}$ for all $T \in \integ$.

Notice first that $\bigcup_{T\geq 1}\mathcal{E}^{\text{KL}}_T$ is bounded. Indeed, for $T\in \integ$ and $\Delta \in \mathcal{E}^{\text{KL}}_T$, we have using Pinsker's inequality
\begin{align*}
    \alpha_T^2 \geq I(\Pb, \Pb+\alpha_T\Delta) 
    \geq \frac{1}{2} \|\Pb -  \Pb-\alpha_T\Delta \|_1^2 = \frac{1}{2}\alpha_T^2 \|\Delta \|_1^2
\end{align*}
Hence, $\|\Delta \|_1 \leq \sqrt{2}$. Similarly, $\bigcup_{T\geq 1}\mathcal{E}^{\chi^2}_T$ is clearly bounded as it is included in the ellipsoid $\{\Delta \in \mathcal{P}_0(\Pb) : \|\Delta\|_{\Pb}^2 \leq 1\}$.

Let $\Delta \in \bigcup_{T\geq 1}\mathcal{E}^{\text{KL}}_T \cup \bigcup_{T\geq 1}\mathcal{E}^{\chi^2}_T$.
Using a Taylor expansion of the log, we have
\begin{align*}
    I(\Pb, \Pb+\alpha_T\Delta)
    &=
    \sum_{i \in \Sigma} \Pb(i) \log\left( \frac{\Pb(i)}{\Pb(i)+\alpha_T\Delta_i}\right)\\
    &=
    -\sum_{i \in \Sigma} \Pb(i) \log\left( 1+ \frac{\alpha_T\Delta_i}{\Pb(i)}\right)\\
    &= 
    -\sum_{i \in \Sigma} \Pb(i) 
    \left(\frac{\alpha_T\Delta_i}{\Pb(i)} - \frac{1}{2}\left(\frac{\alpha_T\Delta_i}{\Pb(i)}\right)^2 + o( \alpha_T^2)\right)\\
    &= 
    -\sum_{i \in \Sigma} 
    \alpha_T\Delta_i - \frac{1}{2}\frac{\left(\alpha_T\Delta_i\right)^2}{\Pb(i)} + o( \alpha_T^2)\\
    &= 
    \|\alpha_T \Delta \|_{\Pb}^2 + o( \alpha_T^2)
\end{align*}
where the asymptotic notation $o$ is independent of $\Delta \in \bigcup_{T\geq 1}\mathcal{E}^{\text{KL}}_T \cup \bigcup_{T\geq 1}\mathcal{E}^{\chi^2}_T$, as $\bigcup_{T\geq 1}\mathcal{E}^{\text{KL}}_T \cup \bigcup_{T\geq 1}\mathcal{E}^{\chi^2}_T$ is bounded. Hence, there exists a sequence $(\varepsilon_T)_{\geq 1}$, independent on $\Delta$, converging to $0$, such that 
$$
(1- \varepsilon_T)\alpha_T^2 \|\Delta \|_{\Pb}^2 
\leq
I(\Pb, \Pb+\alpha_T\Delta) 
\leq (1+ \varepsilon_T)\alpha_T^2 \|\Delta \|_{\Pb}^2 
$$
for all $\Delta \in \bigcup_{T\geq 1}\mathcal{E}^{\text{KL}}_T \cup \bigcup_{T\geq 1}\mathcal{E}^{\chi^2}_T$.
Let $T_0$ such that for all $T\geq T_0$, $\varepsilon_T < 1$. Then we can write 
$$
\alpha_T^2 \|\sqrt{1- \varepsilon_T}\Delta \|_{\Pb}^2
\leq
I(\Pb, \Pb+\alpha_T\Delta)
\leq 
\alpha_T^2 \|\sqrt{1+ \varepsilon_T}\Delta \|_{\Pb}^2.
$$
Hence $ \sqrt{1- \varepsilon_T} \mathcal{E}^{\text{KL}}_T \subset  \mathcal{E}^{\chi^2}_T \subset \sqrt{1+ \varepsilon_T} \mathcal{E}^{\text{KL}}_T$.

Let us now show the asymptotic equality of the DRO formulations. Let $\varepsilon'_T = \sqrt{1+\varepsilon_T}-1$ for all $T\geq T_0$. We have $\varepsilon'_T \to 0$ and $\mathcal{E}^{\chi^2}_T \subset (1+ \varepsilon'_T)\mathcal{E}^{\text{KL}}_T $.
We have for all $T\geq T_0$
\begin{align*}
\sup_{\Pb' \in \cP} \left\{ c(x,\Pb+\alpha_T \Delta) \; : \;  \|\Pb' - \Pb \|_{\Pb}^2 \leq \alpha_T^2 \right\}
&=
\sup \left\{ c(x,\Pb+\alpha_T \Delta) \; : \;  \Delta \in \mathcal{E}^{\chi^2}_T \right\}\\
&\leq 
\sup \left\{ c(x,\Pb+\alpha_T \Delta) \; : \;  \frac{1}{1+ \varepsilon'_T}\Delta \in \mathcal{E}^{\text{KL}}_T\right\}\\
&= 
\sup \left\{ c(x,\Pb+\alpha_T (1+ \varepsilon'_T)\Delta) \; : \;  \Delta \in \mathcal{E}^{\text{KL}}_T \right\}\\
&= 
\sup \left\{ c(x,\Pb+\alpha_T \Delta) + \varepsilon'_T\alpha_T c(x, \Delta) \; : \;  \Delta \in \mathcal{E}^{\text{KL}}_T \right\}\\
&= 
\sup \left\{ c(x,\Pb+\alpha_T \Delta) + \varepsilon'_T\alpha_T \mathcal{O}(1) \; : \;  \Delta \in \mathcal{E}^{\text{KL}}_T \right\}\\
&=
\sup \left\{ c(x,\Pb+\alpha_T \Delta) \; : \;  \Delta \in \mathcal{E}^{\text{KL}}_T \right\} + o(\alpha_T)\\
&=
\sup_{\Pb' \in \cP} \left\{ c(x,\Pb') \; : \; I(\Pb,\Pb') \leq \alpha_T^2 \right\} + o(\alpha_T)
\end{align*}
where the assymptotic $o$ is uniform in $x$, as $\mathcal{E}^{\chi^2}_T$ is bounded, and $\ell$ is uniformally bounded. We can show similarly the reverse inequality.

Finally, it follows immediately that $\cKL' \equiv \cSVP$ as for all $x \in \cX$ and $\Qb \in \cPin$
$$
\frac{\cKL'(x,\Qb,T) - c(x,\Qb)}{\cSVP(x,\Qb,T) - c(x,\Qb)}
=
\frac{\cSVP(x,\Qb,T) - c(x,\Qb) + o(\sqrt{a_T/T})}{\cSVP(x,\Qb,T)- c(x,\Qb)}
=
\frac{\sqrt{(a_T/T)\Var_{\Qb}(\loss(x,\xi))} + o(\sqrt{a_T/T})}{\sqrt{(a_T/T)\Var_{\Qb}(\loss(x,\xi))}} = 1 + o(1)
$$
\end{proof}

\subsubsection{Proof of Proposition \ref{prop: KL feasibility in subexp}: Feaibility of the KL predictor in the subexponential regime}\label{App: Proof of KL feasibility in subexp}
\begin{proof}[Proof of Proposition \ref{prop: KL feasibility in subexp}]
    We will follow closely the proof of Proposition \ref{prop: feasibility of robust predictor}. Let $\Pb \in \cPin$ and $x \in \mathcal{X}$. Observe that for all $T\geq 1$
  \[
    \hat{\Pb}_T \in E_T :=\set{\Pb' \in \mathcal{P}}{I(\Pb',\Pb)\leq \frac{a_T}{T}} 
    \implies c(x,\Pb) \leq \cKL'(x,\hat{\Pb}_T,T).
  \]
  Hence, we have
  \begin{align*}
    \limsup_{T\to\infty}
    \frac{1}{a_T} \log \Pb^{\infty}
        \left(
                c(x,\Pb) > \cKL'(x,\hat{\Pb}_T,T)
        \right)
    \leq& 
        \limsup_{T\to\infty}
    \frac{1}{a_T} \log \Pb^{\infty}
    \left(
      \hat{\Pb}_T \not\in E_T
          \right)\\
    \leq & \limsup_{T\to\infty}
    \frac{1}{a_T} \log \Pb^{\infty}
    \left(
      \hat{\Pb}_T - \Pb \in \sqrt{ \frac{a_T}{T}}\Gamma_T
          \right)           
  \end{align*}
  where $\Gamma_T := \sqrt{T/a_T} (E_T^c -\Pb) = \{\Delta \in \sqrt{T/a_T} \mathcal{P}_0(\Pb) : I(\Pb+\sqrt{a_T/T}\Delta, \Pb)>a_T/T\}$, $\mathcal{P}_0(\Pb) = \set{\Pb - \Pb'}{\Pb' \in \mathcal{P}}$. Similar to the proof of Proposition \ref{prop: feasibility of robust predictor}, we show that $\Gamma_T$ is asymptotically close to $\Gamma := \set{\Delta \in \mathcal{P}_{0,\infty}}{\norm{\Delta}^2_{\Pb} > 1}$ where $\mathcal{P}_{0,\infty}:= \{ \Delta \in \Re^d\; : \; e^\top \Delta =0 \}$.
  
\begin{claim}\label{claim: KL = Pearson}
There exists a sequences $(\varepsilon_T)_{T\geq 1}\in \Re^{\integ}_+$ decreasing to $0$ and $T_0\in \integ$ such that $\sqrt{1-\varepsilon_T}\Gamma^c_T 
\subset \Gamma^c
\subset \sqrt{1+\varepsilon_T}\Gamma^c_T$ for all $T \geq T_0$. Furthermore, $(\Gamma^c_T)_{T\geq 1}$ are uniformally bounded.
\end{claim}
\begin{proof}
To simplify notations, denote $\alpha_T = \sqrt{a_T/T} \xrightarrow[T \to \infty]{} 0$.
    Notice first that $\Gamma^c$ and $\bigcup_{T\geq 1}\Gamma_T^c$ are bounded.  In fact, $\Gamma^c$ is a bounded ellipsoid. For all $\Delta \in \Gamma_T^c$, $T \in \integ$, we have
    \begin{align*}
    \alpha_T^2 \geq
    I(\Pb+\alpha_T \Delta,\Pb)
    &\geq 
    \frac{1}{2} \|\Pb+\alpha_T \Delta - \Pb \|_1^2
    =
    \frac{1}{2} \|\alpha_T \Delta\|_1^2
    \end{align*}
    where we used here Pinsker’s inequality between the KL divergence and the total variations distance. This implies that $\|\Delta\|_1 \leq \sqrt{2}$.

    Using a Taylor expansion of the $\log$, we get 
    \begin{align*}
        I(\Pb+\alpha_T \Delta,\Pb)
        &=
        \sum_{i \in \Sigma} (\Pb(i)+\alpha_T \Delta_i) \log\left( \frac{\Pb(i)+\alpha_T \Delta_i}{\Pb(i)}\right) \\
        &=
        \sum_{i \in \Sigma} (\Pb(i)+\alpha_T \Delta_i) \log\left( 1+\frac{\alpha_T \Delta_i}{\Pb(i)}\right) \\
        &=
        \sum_{i \in \Sigma} (\Pb(i)+\alpha_T \Delta_i)
        \left( \frac{\alpha_T \Delta_i}{\Pb(i)} - \frac{1}{2}\left(\frac{\alpha_T \Delta_i}{\Pb(i)}\right)^2 + o\left( \alpha_T^2\right)\right)\\
        &=
        \sum_{i \in \Sigma}
        \alpha_T \Delta_i + \frac{(\alpha_T \Delta_i)^2}{2\Pb(i)} + o\left( \alpha_T^2\right)\\
        &= \sum_{i \in \Sigma}
        \frac{(\alpha_T \Delta_i)^2}{2\Pb(i)} + o\left( \alpha_T^2\right)
        = \alpha_T^2 \|\Delta\|_{\Pb}^2 + o\left( \alpha_T^2\right)
    \end{align*}
Here the asymptotic $o$ is uniformally in $\Delta \in \bigcup_{T=1}^\infty \Gamma_T^c$, as this set is bounded. Hence, there exists $(\varepsilon_T)_{T\geq 1}$, $\varepsilon_T \to 0$ independent of $\Delta$, such that for all $\Delta \in \bigcup_{T=1}^\infty \Gamma_T^c$
\begin{align*}
\alpha_T^2 \|\Delta\|_{\Pb}^2 - \varepsilon_T \alpha_T^2 \leq
        I(\Pb+\alpha_T \Delta,\Pb)
        \leq \alpha_T^2 \|\Delta\|_{\Pb}^2 + \varepsilon_T \alpha_T^2
    \end{align*}
There exists $T_0$ such that for all $T \geq T_0$, $\varepsilon_T <1$. Hence, for all $T\geq T_0$,
$$
\alpha_T^2 \|\sqrt{1-\varepsilon_T} \Delta\|_{\Pb}^2 \leq
I(\Pb+\alpha_T \Delta,\Pb) \leq \alpha_T^2 \|\sqrt{1+\varepsilon_T} \Delta\|_{\Pb}^2
$$
Therefore, for $T\geq T_0$,
$\Delta \in \Gamma^c_T 
\implies
I(\Pb+\alpha_T \Delta,\Pb) \leq \alpha_T^2 
\implies 
\|\sqrt{1-\varepsilon_T} \Delta\|_{\Pb}^2 \leq 1 \implies \sqrt{1-\varepsilon_T}\Delta \in \Gamma^c$. We have shown $\sqrt{1-\varepsilon_T}\Gamma^c_T 
\subset \Gamma^c$. Similarly, $\Gamma^c \subset \sqrt{1+\varepsilon_T}\Gamma^c_T$. Therefore
$$
\sqrt{1-\varepsilon_T}\Gamma^c_T 
\subset \Gamma^c
\subset \sqrt{1+\varepsilon_T}\Gamma^c_T
$$


\end{proof}

Let $(\varepsilon_T)_{T\geq 1}$ and $T_0$ be given by the previous claim. We have $\sqrt{1+\varepsilon_T} \Gamma_T \subset \Gamma$ for all $T \geq T_0$. Therefore 
$ \Gamma_T \subset \sqrt{1-\frac{\varepsilon_T}{1+\varepsilon_T}}\Gamma \subset \sqrt{1+ \inf_{t\geq T}\frac{-\varepsilon_T}{1+\varepsilon_t}}\Gamma$, as we have $a \Gamma \subset b \Gamma$ for all $a\geq b$. Denote $\varepsilon'_T = \inf_{t\geq T}\frac{-\varepsilon_T}{1+\varepsilon_t}$ for all $T$. We have $\Gamma_T \subset \sqrt{1+\varepsilon'_T}\Gamma$ for all $T$, $(\sqrt{1+\varepsilon'_T}\Gamma)_{T\geq 1}$ is a decreasing sequence of sets, and $\varepsilon'_T$ is increasing to $0$.
  Hence, using the MDP, Theorem \ref{thm: MDP finite space}, we have for all $t\geq T_0$
  \begin{align*}
    \limsup_{T\to\infty}
      \frac{1}{a_T} \log \Pb^{\infty}
      \left(
      c(x,\Pb) > \cSVP(x,\hat{\Pb}_T,T)
      \right)
    \leq & \limsup_{T\to\infty}
    \frac{1}{a_T} \log \Pb^{\infty}
    \left(
      \hat{\Pb}_T - \Pb \in \sqrt{\frac{a_T}{T}}\Gamma_{T}
          \right) \\
    \leq & \limsup_{T\to\infty}
    \frac{1}{a_T} \log \Pb^{\infty}
    \left(
      \hat{\Pb}_T - \Pb \in \sqrt{\frac{a_T}{T}}\sqrt{1+\varepsilon'_T}\Gamma
          \right) \\
    \leq & \limsup_{T\to\infty}
    \frac{1}{a_T} \log \Pb^{\infty}
    \left(
      \hat{\Pb}_T - \Pb \in \sqrt{\frac{a_T}{T}}\sqrt{1+\varepsilon'_t}\Gamma
          \right) \\
    \leq& -\inf_{\Delta\in \sqrt{1+\varepsilon'_t}\bar\Gamma}\norm{\Delta}^2_{\Pb} 
    = -(1+\varepsilon'_t) \inf_{\Delta\in \bar\Gamma}\norm{\Delta}^2_{\Pb}
    = -(1+\varepsilon'_t)
  \end{align*}
  This holds for all $t$, therefore,
    \(
    \limsup_{T\to\infty}
      \frac{1}{a_T} \log \Pb^{\infty}
      \left(
      c(x,\Pb) > \cSVP(x,\hat{\Pb}_T,T)
      \right)
    \leq -1.
  \)
\end{proof}
}

\section{Omitted proofs of Section \ref{sec: presc}: Optimal Prescription}
\subsection{Omitted proofs of Subsection \ref{sec: exp presc}: Prescriptors in the Exponential Regime}

\subsubsection{Proof of Proposition \ref{prop: feasibility presc exp}: Feasibility}\label{App: proof feasibility Exp presc}

{\color{black}
\begin{proof}[Proof of Proposition \ref{prop: feasibility presc exp}]
Let $\Pb \in \cPin$. Denote $\Gamma = \set{\Pb' \in \mathcal{P}}{I(\Pb',\Pb) > r}$. Observe that, by definition of $\cKL$ \eqref{eq: KL predictor}, for all $T\geq T_0$ we have
  \[
    \hat{\Pb}_T \not \in \Gamma \implies c(\xKL{,T},\Pb) \leq \cKL(\xKL{,T}(\Pb),\hat{\Pb}_T,T) = \cKL^{\star}(\Pb,T).
  \]
 Hence, we have
  \begin{align*}
    \limsup_{T\to\infty}
    \frac{1}{T} \log \Pb^{\infty}
        \left(
                c(\xKL{,T}(\Pb),\Pb) > \cKL^{\star}(\hat{\Pb}_T,T)
        \right)
    \leq& 
        \limsup_{T\to\infty}
    \frac{1}{T} \log \Pb^{\infty}
    \left(
      \hat{\Pb}_T \in \Gamma
          \right)
    \leq -\inf_{\Pb'\in \bar{\Gamma}}
             I(\Pb',\Pb)
  \end{align*}
where the last equality uses the Large Deviation Principle, Theorem \ref{thm: LDP finite space}. 
The convexity of the continuity of the relative entropy $I(\cdot,\cdot)$ in $\cPin \times \cPin$ implies that $\bar{\Gamma} \subset \set{\Pb' \in \mathcal{P}}{I(\Pb',\Pb) \geq r}$ (see Lemma \ref{lemma: I() continuty}).
Hence
$\inf_{\Pb'\in \bar{\Gamma}}
             I(\Pb',\Pb) \geq r$,
and therefore, using the previous inequality, we get
\[
    \limsup_{T\to\infty}
    \frac{1}{T} \log \Pb^{\infty}
        \left(
                c(\xKL{,T}(\Pb),\Pb) > \cKL^{\star}(\hat{\Pb}_T,T)
        \right)
    \leq
        -r.
\]%
\end{proof}
}

\begin{lemma}[KL Divergence property]\label{lemma: I() continuty}
For $r>0$ and $\Pb \in \cPin$, we have 
$\overline{\set{\Pb' \in \mathcal{P}}{I(\Pb',\Pb) > r}} 
\subset
\set{\Pb' \in \mathcal{P}}{I(\Pb',\Pb) \geq r}$.
\end{lemma}
\begin{proof}
Denote $\Gamma = \set{\Pb' \in \mathcal{P}}{I(\Pb',\Pb) \geq r}$. $I(\cdot,\cdot)$ is convex and hence continuous on $\cPin \times \cPin$. Hence,
\begin{align*}
\overline{\Gamma} 
&\subset
\set{\Pb' \in \mathcal{P}}{I(\Pb',\Pb) \geq r+1} \cup 
\overline{\set{\Pb' \in \mathcal{P}}{r+1>I(\Pb',\Pb) > r}} \\
&= 
\set{\Pb' \in \cP}{I(\Pb',\Pb) \geq r+1} 
\cup 
\overline{\set{\Pb' \in \cPin}{r+1>I(\Pb',\Pb) > r}}\\
&=
\set{\Pb' \in \mathcal{P}}{I(\Pb',\Pb) \geq r+1} \cup 
\set{\Pb' \in \mathcal{P}}{r+1 \geq I(\Pb',\Pb) \geq r} \\
&= \set{\Pb' \in \mathcal{P}}{I(\Pb',\Pb) \geq r}.
\end{align*}
\end{proof}

\subsubsection{Proof of Theorem \ref{thm: strong opt exponential regime prescriptor.}: Strong Optimality}\label{Appendix: proof of Strong opt exp presc}

{\color{black}
\begin{proof}[Proof of Theorem \ref{thm: strong opt exponential regime prescriptor.}]
Suppose for the sake of contradiction that there exists $(\hat{c},\hat{x})$ feasible in \eqref{eq:optimal-prescriptor} such that $(\cKL,\xKL{}) \not \orderpresc (\hat{c},\hat{x})$. There must exist $\Pb_0 \in \cPin$ such that 

\begin{equation}\label{proof eq: dro prescriptor 1}
\limsup_{T\to \infty} \frac{|\cKL^{\star}(\Pb_0,T) - c^{\star}(\Pb_0)|}{|\hat{c}^{\star}(\Pb_0,T) - c^{\star}(\Pb_0)|} > 1,
\end{equation}
with the convention $\frac{0}{0} =1$. 

As $\cKL$ does not depend on $T$, we denote $\cKL(x,\Pb) := \cKL(x,\Pb,T)$, and $\cKL^{\star}(\Pb) := \inf_{x\in\cX}\cKL(x,\Pb)$ for all $x,\Pb \in \cX \times\cP$ and $T\in \integ$. 
Inequality \eqref{proof eq: dro prescriptor 1} implies that there exists $\varepsilon>0$ and $(t_T) \in \integ^{\integ}$ such that 
$|\cKL^{\star}(\Pb_0) - c^{\star}(\Pb_0)| 
\geq 
(1+\varepsilon) |\hat{c}^{\star}(\Pb_0,t_T) - c^{\star}(\Pb_0)|$. This inequality can be rewritten as
\begin{equation}\label{proof eq: dro prescriptor 2}
\cKL^{\star}(\Pb_0) \geq \hat{c}^{\star}(\Pb_0,t_T) + \varepsilon |\hat{c}^{\star}(\Pb_0,t_T) - c^{\star}(\Pb_0)|
\end{equation}
where we used $\cKL^{\star}(\Pb_0) = \inf_{x\in \cX} \sup_{\Pb': I(\Pb_0,\Pb')\leq r} c(x,\Pb') \geq \inf_{x\in \cX} c(x,\Pb_0) = c^{\star}(\Pb_0)$ to drop the absolute values. We use this inequality to derive the following claim.

\begin{claim}\label{claim: dro prescriptor}
There exists $\varepsilon_1>0$ and $(l_T)_{T\geq 1} \in \integ^{\integ}$ such that $\cKL^{\star}(\Pb_0) \geq \hat{c}^{\star}(\Pb_0,l_T) +\varepsilon_1$, for all $T\in\integ$.
\end{claim}
\begin{proof}
We distinguish two cases. If $\cKL^{\star}(\Pb_0) > \limsup_{T\in \integ}\hat{c}^{\star}(\Pb_0,t_T)$, then the result follows immediately. Suppose $\cKL^{\star}(\Pb_0) \leq \limsup_{T\in \integ}\hat{c}^{\star}(\Pb_0,t_T)$. We have $\cKL^{\star}(\Pb_0) \geq c^{\star}(\Pb_0)$, and $|\cKL^{\star}(\Pb_0) - c^{\star}(\Pb_0)|>0$ as otherwise \eqref{proof eq: dro prescriptor 1} does not hold. Therefore,  $0< \cKL^{\star}(\Pb_0) - c^{\star}(\Pb_0) \leq \limsup_{T\in \integ}\hat{c}^{\star}(\Pb_0,T) - c^{\star}(\Pb_0)$. Hence, by definition of the limit superior, there exists $\delta>0$ and $(l_T)_{T\geq 1}$, a sub-sequence of $(t_T)_{T\geq 1}$, such that $|\hat{c}^{\star}(\Pb_0,l_T) - c^{\star}(\Pb_0)| \geq \delta$ for all $T\in \integ$. Plugging this inequality in \eqref{proof eq: dro prescriptor 2}, we get the desired result with $\varepsilon_1 = \varepsilon \delta$.
\end{proof}

Let $\varepsilon_1$ and $(l_T)_{T\geq1}$ given by the previous claim.
Using a probabilistic characterisation of the compactness of $\cX$ (Lemma \ref{lemma: prob compactness}), there exists $x_{\infty} \in \cX$ such that for all $\rho>0$
\begin{equation}\label{proof eq: dro precriptor 4}
    \limsup_{T \to \infty} \Pb_0^{\infty} 
    \left( 
    \| \hat{x}_T(\Pb_0) - x_{\infty}\| \leq \rho
    \right) >0.
\end{equation}

Let $\bar{\Pb} \in \cP$ be the maximizer in the definition of $\cKL$ \eqref{eq: KL predictor} such that $\cKL(x_{\infty},\Pb_0) = c(x_{\infty},\bar{\Pb})$ and $I(\Pb_0,\bar{\Pb})\leq r$. 
By continuity of $c(x_{\infty},\cdot)$, we can perturb $\bar{\Pb}$ into $\bar{\Pb}_1 \in \cPin$ such that $\cKL(x_{\infty},\Pb_0) \leq c(x_{\infty},\bar{\Pb}_1) + \varepsilon_1/2$ and $I(\Pb_0,\bar{\Pb}_1)<r$.
The minimality of $\cKL^{\star}(\Pb_0)$ implies that $\cKL^{\star}(\Pb_0) \leq \cKL(x_{\infty},\Pb_0) \leq c(x_{\infty},\bar{\Pb}_1) + \varepsilon_1/2$. 
Combining this result with Claim \ref{claim: dro prescriptor}, we get $\hat{c}^{\star}(\Pb_0,l_T) + \varepsilon_1/2 \leq c(x_{\infty},\bar{\Pb}_1)$ for all $T \in \integ$. Finally, by the continuity of $c(\cdot, \bar{\Pb}_1)$ and the equicontinuity of $\hat{c}^{\star}$ (due to the compactness of $\cX$ and equicontinuity of $\hat{c}$, see Lemma \ref{lemma: hein for equicontinuity}), there exists $\rho>0$ and an open set $U\subset \cPin$ containing $\Pb_0$ such that
\begin{equation}\label{proof eq: dro prescriptor 3}
    \hat{c}^{\star}(\Pb',l_T) + \varepsilon_1/3 \leq c(x,\bar{\Pb}_1),
    \quad \forall T \in \integ,
    \; \forall x \in \cX \; : \; \|x-x_{\infty}\| \leq \rho,
    \; \forall \Pb' \in U.
\end{equation}
Armed with these results, we will now prove that $\hat{c}$ violates the out-of-sample guarantee \eqref{eq: prec out-of-sample guarantee prob} in $\bar{\Pb}_1$. We have
\begin{align*}
    \limsup_{T\to \infty} 
    &\frac{1}{T} \log \bar{\Pb}_1^{\infty} 
    \left(
    c(\hat{x}_T(\hat{\Pb}_T),\bar{\Pb}_1) > \hat{c}^{\star}(\hat{\Pb}_T,T)
    \right) \\
    & \geq 
    \limsup_{T\to \infty} 
    \frac{1}{T} \log \bar{\Pb}_1^{\infty} 
    \left(
    c(\hat{x}_T(\hat{\Pb}_T),\bar{\Pb}_1) > \hat{c}^{\star}(\hat{\Pb}_T,T)
    \; \cap \; 
    \|\hat{x}_T(\hat{\Pb}_T) - x_{\infty}\|\leq \rho
    \right) \\
    &\geq
    \limsup_{T\to \infty} 
    \frac{1}{l_T} \log \bar{\Pb}_1^{\infty} 
    \left(
    c(\hat{x}_T(\hat{\Pb}_{l_T}),\bar{\Pb}_1) > \hat{c}^{\star}(\hat{\Pb}_{l_T},l_T)
    \; \cap \; 
    \|\hat{x}_{l_T}(\hat{\Pb}_{l_T}) - x_{\infty}\|\leq \rho
    \right) \\
    &\geq
    \limsup_{T\to \infty} 
    \frac{1}{l_T} \log \bar{\Pb}_1^{\infty} 
    \left(
    \hat{\Pb}_{l_T} \in U
    \; \cap \; 
    \|\hat{x}_{l_T}(\hat{\Pb}_{l_T}) - x_{\infty}\|\leq \rho
    \right)
\end{align*}
where the last inequality uses \eqref{proof eq: dro prescriptor 3}. Using an exponential change of measure, Lemma \ref{lemma: change of probability}, we have
\begin{align*}
    \limsup_{T\to \infty} 
    &\frac{1}{l_T} \log \bar{\Pb}_1^{\infty} 
    \left(
    \hat{\Pb}_{l_T} \in U
    \; \cap \; 
    \|\hat{x}_{l_T}(\hat{\Pb}_{l_T}) - x_{\infty}\|\leq \rho
    \right) \\
    &\geq
    - I(\Pb_0,\bar{\Pb}_1) + 
    \limsup_{T\to \infty} 
    \frac{1}{l_T} \log \Pb_0^{\infty} 
    \left(
    \hat{\Pb}_{l_T} \in U
    \; \cap \; 
    \|\hat{x}_{l_T}(\hat{\Pb}_{l_T}) - x_{\infty}\|\leq \rho
    \right)
\end{align*}
We show now that the second term of the LHS is zero.
We have $\lim_{T\to \infty} \Pb_0^{\infty} (\hat{\Pb}_{l_T} \in U) =1$ as $\Pb_0 \in U^{\into}$ and $\limsup_{T\in \integ}\Pb_0^{\infty} (\|\hat{x}_{l_T}(\hat{\Pb}_{l_T}) - x_{\infty}\|\leq \rho)>0$ by \eqref{proof eq: dro precriptor 4}, therefore,
\begin{align*}
    \limsup_{T\to \infty} 
    &\frac{1}{l_T} \log \Pb_0^{\infty} 
    \left(
    \hat{\Pb}_{l_T} \in U
    \; \cap \; 
    \|\hat{x}_{l_T}(\hat{\Pb}_{l_T}) - x_{\infty}\|\leq \rho
    \right)\\
    &\geq 
        \limsup_{T\to \infty} 
    \frac{1}{l_T} \log 
    \left[
        \Pb_0^{\infty} 
        (
        \hat{\Pb}_{l_T} \in U)
        + 
        \Pb_0^{\infty} 
        (
        \|\hat{x}_{l_T}(\hat{\Pb}_{l_T}) - x_{\infty}\|\leq \rho)
        -1
    \right] = 0.
\end{align*}
Moreover we have $I(\Pb_0,\bar{\Pb}_1)<r$ by construction of $\bar{\Pb}_1$. Hence, we have shown that 
\begin{align*}
     \limsup_{T\to \infty} 
    \frac{1}{T} \log \bar{\Pb}_1^{\infty} 
    \left(
    c(\hat{x}_T(\hat{\Pb}_T),\bar{\Pb}_1) > \hat{c}^{\star}(\hat{\Pb}_T,T)
    \right)
    \geq 
    -I(\Pb_0,\bar{\Pb}_1)>-r.
\end{align*}
This implies that $\hat{c}$ violates the prescription out-of-sample guarantee \eqref{eq: prec out-of-sample guarantee prob} which contradicts our feasibility assumption. 
\end{proof}
}

\begin{lemma}[Probabilistic compactness characterisation]\label{lemma: prob compactness}
Let $\Pb \in \cPin$. If a random variable $z_T$, $T \in \integ$ takes values in a compact $Z\subset \Re^n$, then there exists a deterministic vector $z_{\infty} \in Z$ such that $\limsup_{T \to \infty} \Pb[\|z_T - z_{\infty}\| < \rho] >0$ for all $\rho>0$.
\end{lemma}
\begin{proof}
See Lemma B.2 in \cite{sutter2020general}.
\end{proof}

\begin{lemma}[Distribution Shift]\label{lemma: change of probability}
Let $\Pb_1,\Pb_2 \in \cPin$ and $\Pb_1^{\infty},\Pb_2^{\infty}$ their respective induced probability distribution on the data $(\xi_T)_{T\geq 1}$. Let $(A_T)_{T\geq 1}$ be a sequence of events on the empirical distribution $\hat{\Pb}_T$. 
We have
$$
\limsup_{T \to \infty} \frac{1}{T}\log \Pb_2^{\infty}(A_T) 
\geq 
-I(\Pb_1,\Pb_2) + \limsup_{T \to \infty} \frac{1}{T}\log \Pb_1^{\infty}(A_T).
$$
More precisely, for all $\delta>0$ for all $T\in \integ$ 
$$
\Pb_2^{\infty}(A_T \cap \hat{\Pb}_T \in \mathcal{B}(\Pb_1,\delta)) \geq \exp(-T(I(\Pb_1,\Pb_2) - c\delta)) \Pb_1^{\infty}(A_T),
$$
where $\mathcal{B}(\Pb_1,\delta)$ is the ball around $\Pb_1$ of radius $\delta$ and $c$ is a constant depending only on $\Pb_1$ and $\Pb_2$. 
\end{lemma}
\begin{proof}
We first show the second result.
Let $T\in \integ$. Denote $\Lambda_T  = \{(\frac{\alpha_1}{T},\ldots,\frac{\alpha_d}{T})^\top \; : \; \alpha_1+\ldots+\alpha_d = T, \; \alpha_1,\ldots,\alpha_d \in [0,T]\}$ the set of possible empirical distributions with data size $T$. We have
\begin{align*}
    \Pb_2^{\infty}(A_T) 
    = \Eb_{\Pb_2} 
        \left[ 
            \mathbb{1}(A_T)
        \right]
    &= \sum_{p \in \Lambda_T} \prod_{i=1}^d \Pb_2(i)^{Tp(i)}  \mathbb{1}(p \in A_T) \\
    &= \sum_{p \in \Lambda_T} \prod_{i=1}^d \Pb_1(i)^{Tp(i)} 
    \cdot \prod_{i=1}^d  \left(\frac{\Pb_2(i)}{\Pb_1(i)} \right)^{Tp(i)}
    \mathbb{1}(p \in A_T) \\
    &= \Eb_{\Pb_2} 
        \left[ 
            \prod_{i=1}^d   \left(\frac{\Pb_2(i)}{\Pb_1(i)} \right)^{T\hat{\Pb}_T(i)}
            \mathbb{1}(A_T)
        \right] \\
    &=  \Eb_{\Pb_1} 
        \left[ 
        \exp\left(
            \sum_{i=1}^d  T\hat{\Pb}_T(i) \log \left(\frac{\Pb_2(i)}{\Pb_1(i)} \right) \right)
            \mathbb{1}(A_T)
        \right] \\
    &= \exp\left(
            -TI(\Pb_1,\Pb_2) \right) 
    \Eb_{\Pb_1} 
        \left[ 
        \exp\left(
            \sum_{i=1}^d  T(\hat{\Pb}_T(i)-\Pb_1(i)) \log \left(\frac{\Pb_2(i)}{\Pb_1(i)} \right) \right)
            \mathbb{1}(A_T)
        \right]
\end{align*}
Under the event $\hat{\Pb}_T \in \mathcal{B}(\Pb_1,\delta)$ we have 
$$
\sum_{i=1}^d  (\hat{\Pb}_T(i)-\Pb_1(i)) \log \left(\frac{\Pb_2(i)}{\Pb_1(i)} \right)
\geq 
-\|\hat{\Pb}_T - \Pb_1\| \sqrt{\sum_{i=1}^d \log \left(\frac{\Pb_2(i)}{\Pb_1(i)} \right)^2} \leq -c \delta 
$$
where $c = \sqrt{\sum_{i=1}^d \log \left(\frac{\Pb_2(i)}{\Pb_1(i)} \right)^2}$. Hence we have
\begin{align*}
    \Pb_2^{\infty}(A_T \cap \hat{\Pb}_T \in \mathcal{B}(\Pb_1,\delta))
    \geq 
    \exp\left(
            -TI(\Pb_1,\Pb_2) -cT\delta \right) 
    \Eb_{\Pb_1} 
        \left[
            \mathbb{1}(A_T)
        \right]
    =
        \exp\left(
            -TI(\Pb_1,\Pb_2) -cT\delta \right) 
    \Pb_1^{\infty}(A_T)
\end{align*}
where $c = \sum_{i=1}^d \log \frac{\Pb_2(i)}{\Pb_1(i)}$.

We now show the first result. We have
\begin{align*}
\limsup_{T\to \infty} \frac{1}{T}\log \Pb_2^{\infty}(A_T)
&\geq
\limsup_{T\to \infty} \frac{1}{T}\log
    \Pb_2^{\infty}(A_T \cap \hat{\Pb}_T \in \mathcal{B}(\Pb_1,\delta)))\\
&\geq
\limsup_{T\to \infty} \frac{1}{T}\log
\left[
    \exp\left(
            -TI(\Pb_1,\Pb_2) -cT\delta \right) 
    \Pb_1^{\infty}(A_T)
\right] \\
& = -I(\Pb_1,\Pb_2) -c\delta + \limsup_{T\to \infty} \frac{1}{T}\log \Pb_1^{\infty}(A_T)
\end{align*}
This is true for all $\delta>0$ which gives the desired result with $\delta \to 0$.
\end{proof}

\subsubsection{Omitted proofs of Section \ref{sec: presc Superexp}: Prescriptors in the Superexponential Regime}\label{Appendix: proofs presc superexp}
\begin{proof}[Proof of Theorem \ref{thm: optimality presc superexp}]
We note that this proof is essentially the same as the proof of Theorem \ref{thm: strong opt exponential regime prescriptor.} with $r = \infty$.

Suppose for the sake of contradiction that there exists $(\hat{c},\hat{x})$ feasible in \eqref{eq:optimal-prescriptor} such that $(\cKL,\xRob{}) \not \orderpresc (\hat{c},\hat{x})$. There must exist $\Pb_0 \in \cP$ such that 

\begin{equation}\label{proof eq: Rob prescriptor 1}
\limsup_{T\to \infty} \frac{|\cRob^{\star}(\Pb_0,T) - c^{\star}(\Pb_0)|}{|\hat{c}^{\star}(\Pb_0,T) - c^{\star}(\Pb_0)|} > 1,
\end{equation}
with the convention $\frac{0}{0} =1$. 

As $\cRob$ does not depend on $\Pb$ and $T$, we denote $\cRob(x) := \cRob(x,\Pb,T)$, and $\cRob^{\star} := \inf_{x\in\cX}\cRob(x)$ for all $x,\Pb \in \cX \times\cP$ and $T\in \integ$. 
\eqref{proof eq: Rob prescriptor 1} imply that there exists $\varepsilon>0$ and $(t_T) \in \integ^{\integ}$ such that 
$|\cRob^{\star} - c^{\star}(\Pb_0)| 
\geq 
(1+\varepsilon) |\hat{c}^{\star}(\Pb_0,t_T) - c^{\star}(\Pb_0)|$. This inequality can be rewritten as
\begin{equation}\label{proof eq: Rob prescriptor 2}
\cRob^{\star} \geq \hat{c}^{\star}(\Pb_0,t_T) + \varepsilon |\hat{c}^{\star}(\Pb_0,t_T) - c^{\star}(\Pb_0)|
\end{equation}
where we used $\cRob^{\star} = \inf_{x\in \cX} \sup_{\Pb'\in \cP} c(x,\Pb') \geq \inf_{x\in \cX} c(x,\Pb_0) = c^{\star}(\Pb_0)$ to drop the absolute values. We use this inequality to derive the following claim.

\begin{claim}\label{claim: Rob prescriptor}
There exists $\varepsilon_1>0$ and $(l_T)_{T\geq 1} \in \integ^{\integ}$ such that $\cRob^{\star} \geq \hat{c}^{\star}(\Pb_0,l_T) +\varepsilon_1$, for all $T\in\integ$.
\end{claim}
\begin{proof}
We distinguish two cases. If $\cRob^{\star} > \limsup_{T\in \integ}\hat{c}^{\star}(\Pb_0,t_T)$, then the result follows immediately. Suppose $\cRob^{\star} \leq \limsup_{T\in \integ}\hat{c}^{\star}(\Pb_0,t_T)$. We have $\cRob^{\star} \geq c^{\star}(\Pb_0)$, and $|\cRob^{\star} - c^{\star}(\Pb_0)|>0$ as otherwise \eqref{proof eq: Rob prescriptor 1} does not hold. Therefore,  $0< \cRob^{\star} - c^{\star}(\Pb_0) \leq \limsup_{T\in \integ}\hat{c}^{\star}(\Pb_0,T) - c^{\star}(\Pb_0)$. Hence, by definition of the limit superior, there exists $\delta>0$ and $(l_T)_{T\geq 1}$, a sub-sequence of $(t_T)_{T\geq 1}$, such that $|\hat{c}^{\star}(\Pb_0,l_T) - c^{\star}(\Pb_0)| \geq \delta$ for all $T\in \integ$. Plugging this inequality in \eqref{proof eq: Rob prescriptor 2}, we get the desired result with $\varepsilon_1 = \varepsilon \delta$.
\end{proof}

Let $\varepsilon_1$ and $(l_T)_{T\geq1}$ given by the previous claim.
Using a probabilistic characterisation of the compactness of $\cX$ (Lemma \ref{lemma: prob compactness}), there exists $x_{\infty} \in \cX$ such that for all $\rho>0$
\begin{equation}\label{proof eq: Rob precriptor 4}
    \limsup_{T \to \infty} \Pb_0^{\infty} 
    \left( 
    \| \hat{x}_T(\Pb_0) - x_{\infty}\| \leq \rho
    \right) >0.
\end{equation}

Let $\bar{\Pb} \in \cP$ reaching the max in the definition of $\cRob$ \eqref{eq: overly robust predictor} such that $\cRob(x_{\infty}) = c(x_{\infty},\bar{\Pb})$. 
By continuity of $c(x_{\infty},\cdot)$, we can perturb $\bar{\Pb}$ into $\bar{\Pb}_1 \in \cPin$ such that $\cRob(x_{\infty}) \leq c(x_{\infty},\bar{\Pb}_1) + \varepsilon_1/2$. 
The minimality of $\cRob^{\star}$ implies that $\cRob^{\star} \leq \cRob(x_{\infty}) \leq c(x_{\infty},\bar{\Pb}_1) + \varepsilon_1/2$. 
Combining this result with Claim \ref{claim: Rob prescriptor}, we get $\hat{c}^{\star}(\Pb_0,l_T) + \varepsilon_1/2 \leq c(x_{\infty},\bar{\Pb}_1)$ for all $T \in \integ$. Finally, by the continuity of $c(\cdot, \bar{\Pb}_1)$ and the equicontinuity of $\hat{c}^{\star}$ (due to the compactness of $\cX$ and equicontinuity of $\hat{c}$, see Lemma \ref{lemma: hein for equicontinuity}), there exists $\rho>0$ and an open set $U\subset \cPin$ containing $\Pb_0$ such that
\begin{equation}\label{proof eq: Rob prescriptor 3}
    \hat{c}^{\star}(\Pb',l_T) + \varepsilon_1/3 \leq c(x,\bar{\Pb}_1),
    \quad \forall T \in \integ,
    \; \forall x \in \cX \; : \; \|x-x_{\infty}\| \leq \rho,
    \; \forall \Pb' \in U.
\end{equation}
Armed with these results, we will now prove that $\hat{c}$ violates the out-of-sample guarantee \eqref{eq: prec out-of-sample guarantee prob} in $\bar{\Pb}_1$. We have
\begin{align*}
    \limsup_{T\to \infty} 
    &\frac{1}{T} \log \bar{\Pb}_1^{\infty} 
    \left(
    c(\hat{x}_T(\hat{\Pb}_T),\bar{\Pb}_1) > \hat{c}^{\star}(\hat{\Pb}_T,T)
    \right) \\
    & \geq 
    \limsup_{T\to \infty} 
    \frac{1}{T} \log \bar{\Pb}_1^{\infty} 
    \left(
    c(\hat{x}_T(\hat{\Pb}_T),\bar{\Pb}_1) > \hat{c}^{\star}(\hat{\Pb}_T,T)
    \; \cap \; 
    \|\hat{x}_T(\hat{\Pb}_T) - x_{\infty}\|\leq \rho
    \right) \\
    &\geq
    \limsup_{T\to \infty} 
    \frac{1}{l_T} \log \bar{\Pb}_1^{\infty} 
    \left(
    c(\hat{x}_T(\hat{\Pb}_{l_T}),\bar{\Pb}_1) > \hat{c}^{\star}(\hat{\Pb}_{l_T},l_T)
    \; \cap \; 
    \|\hat{x}_{l_T}(\hat{\Pb}_{l_T}) - x_{\infty}\|\leq \rho
    \right) \\
    &\geq
    \limsup_{T\to \infty} 
    \frac{1}{l_T} \log \bar{\Pb}_1^{\infty} 
    \left(
    \hat{\Pb}_{l_T} \in U
    \; \cap \; 
    \|\hat{x}_{l_T}(\hat{\Pb}_{l_T}) - x_{\infty}\|\leq \rho
    \right)
\end{align*}
where the last inequality uses \eqref{proof eq: Rob prescriptor 3}. Using a distribution shift, Lemma \ref{lemma: change of probability}, we have
\begin{align*}
    \limsup_{T\to \infty} 
    &\frac{1}{l_T} \log \bar{\Pb}_1^{\infty} 
    \left(
    \hat{\Pb}_{l_T} \in U
    \; \cap \; 
    \|\hat{x}_{l_T}(\hat{\Pb}_{l_T}) - x_{\infty}\|\leq \rho
    \right) \\
    &\geq
    - I(\Pb_0,\bar{\Pb}_1) + 
    \limsup_{T\to \infty} 
    \frac{1}{l_T} \log \Pb_0^{\infty} 
    \left(
    \hat{\Pb}_{l_T} \in U
    \; \cap \; 
    \|\hat{x}_{l_T}(\hat{\Pb}_{l_T}) - x_{\infty}\|\leq \rho
    \right)
\end{align*}
We show now that the second term of the LHS is zero.
We have $\lim_{T\to \infty} \Pb_0^{\infty} (\hat{\Pb}_{l_T} \in U) =1$ as $\Pb_0 \in U^{\into}$ and $\limsup_{T\in \integ}\Pb_0^{\infty} (\|\hat{x}_{l_T}(\hat{\Pb}_{l_T}) - x_{\infty}\|\leq \rho)>0$ by \eqref{proof eq: Rob precriptor 4} therefore
\begin{align*}
    \limsup_{T\to \infty} 
    &\frac{1}{l_T} \log \Pb_0^{\infty} 
    \left(
    \hat{\Pb}_{l_T} \in U
    \; \cap \; 
    \|\hat{x}_{l_T}(\hat{\Pb}_{l_T}) - x_{\infty}\|\leq \rho
    \right)\\
    &\geq 
        \limsup_{T\to \infty} 
    \frac{1}{l_T} \log 
    \left[
        \Pb_0^{\infty} 
        (
        \hat{\Pb}_{l_T} \in U)
        + 
        \Pb_0^{\infty} 
        (
        \|\hat{x}_{l_T}(\hat{\Pb}_{l_T}) - x_{\infty}\|\leq \rho)
        -1
    \right] = 0.
\end{align*}
Moreover we have $I(\Pb_0,\bar{\Pb}_1)< \infty$ as $\Pb_0, \bar{\Pb}_1 \in \cPin$. Hence, we have shown that 
\begin{align*}
     \limsup_{T\to \infty} 
    \frac{1}{T} \log \bar{\Pb}_1^{\infty} 
    \left(
    c(\hat{x}_T(\hat{\Pb}_T),\bar{\Pb}_1) > \hat{c}^{\star}(\hat{\Pb}_T,T)
    \right)
    \geq 
    -I(\Pb_0,\bar{\Pb}_1)>-\infty.
\end{align*}
As $a_T/T \to \infty$, we have therefore
\begin{align*}
     \limsup_{T\to \infty} 
    \frac{1}{a_T} \log \bar{\Pb}_1^{\infty} 
    &\left(
    c(\hat{x}_T(\hat{\Pb}_T),\bar{\Pb}_1) > \hat{c}^{\star}(\hat{\Pb}_T,T)
    \right) \\
    &=
    \limsup_{T\to \infty} 
    \frac{T}{a_T} \frac{1}{T}\log \bar{\Pb}_1^{\infty} 
    \left(
    c(\hat{x}_T(\hat{\Pb}_T),\bar{\Pb}_1) > \hat{c}^{\star}(\hat{\Pb}_T,T)
    \right) = 0 >-1.
\end{align*}

This implies that $\hat{c}$ violates the prescription out-of-sample guarantee \eqref{eq: prec out-of-sample guarantee prob} which contradicts our feasibility assumption. 
\end{proof}

\subsection{Omitted proofs of Section \ref{sec: prescriptor subexp}: Prescriptors in the subexponential Regime}

\subsubsection{Proof of Proposition \ref{prop: consitency of prescriptors}: Consistency of weakly optimal prescriptors}\label{Appendix: proof consistence presc}
\begin{proof}[Proof of Proposition \ref{prop: consitency of prescriptors}]
Suppose $(\hat{x},\hat{c})$ is weakly optimal and not consistent. There exists $x_0\in \mathcal{X}$ and $\Pb_0 \in \mathcal{P}$ such that $\limsup |\hat{c}(x_0,\Pb_0,T) - c(x,\Pb_0)| = \varepsilon>0$. Let $\delta(x,\Pb,T) = \hat{c}(x,\Pb,T) - c(x,\Pb)$ for all $x\in \cX$, $\Pb \in \cP$ and $T \in \integ$. 
Consider the same exact construction of $\hat{c}'$ as in the proof of Proposition \ref{prop: consitency} (illustrated in Figure \ref{fig: proof of consistency}). Among the possible prescriptors of $\hat{c}'$, we consider the closest one to the prescriptor $\hat{x}_T$ of $\hat{c}$,
\begin{equation}\label{eq: proof def prescriptor}
\hat{x}'_T(\Pb) := \argmin \set{\|x' - \hat{x}_T(\Pb)\|}
{ x' \in \argmin_{x\in \cX} \hat{c}(x,\Pb,T)},
\end{equation}

which exists as the set of minimizers is compact and the norm is continuous.
Let us show that $(\hat{c}',\hat{x}')$ is feasible.

Let $\Pb \in \cP$. Let us verify the out-of-sample guarantee \eqref{eq: prescriptor out-of-sample ganrantee} in $\Pb$.
Denote 
\begin{align*}
\mathcal{D}'_T(\Pb):= \set{\Pb'\in \mathcal{P}}{c(\hat{x}'_T(\Pb'), \Pb) > \hat{c}'^{\star}( \Pb',T)},
\\
\mathcal{D}_T(\Pb):= \set{\Pb'\in \mathcal{P}}{c(\hat{x}_T(\Pb'), \Pb) > \hat{c}^{\star}( \Pb',T)},
\end{align*}
the set of disappointing distributions of $\hat{c}'$ and $\hat{c}$ respectively.

Recall $\mathcal{T}$, $(l_T)_{T\geq 1}$, $(t_T)_{T\geq 1}$ and $\rho$ defined in construction of $\hat{c}'$ in the proof of Proposition \ref{prop: consitency}. We examine the guarantee when $T\notin \mathcal{T}$, ie the sequence $(l_T)_{T\geq 1}$, and when $T\in \mathcal{T}$, ie the sequence $(t_T)_{T\geq 1}$. When $T\notin \mathcal{T}$, we have $\hat{c}(x,\Pb,T) = \hat{c}'(x,\Pb,T)$ for all $x,\Pb \in \cX \times \cP$, and by definition of $\hat{x}'_T$ \eqref{eq: proof def prescriptor}, we have $\hat{x}'_T(\Pb) = \hat{x}_T(\Pb)$. Hence,
$
    \Pb^{\infty}(\hat{\Pb}_{l_T} \in \mathcal{D}'_{l_T}(\Pb))
    = 
    \Pb^{\infty}(\hat{\Pb}_{l_T} \in \mathcal{D}_{l_T}(\Pb)) 
$.
As $\hat{c}$ is feasible, the last equality implies that
\begin{equation}\label{proof eq: consistency, feasbility for l_T}
\limsup_{T \to \infty}
    \frac{1}{a_T} \log
    \Pb^{\infty}(\hat{\Pb}_{l_T} \in \mathcal{D}'_{l_T}(\Pb))
    \leq -1.
\end{equation}
Let us now examine the out-of-sample guarantee for the subsequence $(t_T)_{T\geq 1}$. Let $T \in \mathcal{T}$.
Consider the set of distributions where the pair of prescription distribution is in the ball where $\hat{c}$ differs from $\hat{c}'$, i.e., where the perturbation $\eta$ is non-negative, 
$$
A_T := \set{\Pb' \in \cP}{(\hat{x}'_T(\Pb'), \Pb') \in \mathcal{B}\left((x_0,\Pb_0),\frac{\rho}{2} \right)}.
$$
We have
\begin{align*}
     \Pb^{\infty}(\hat{\Pb}_T \in \mathcal{D}'_T(\Pb))
     &=
     \Pb^{\infty}(\hat{\Pb}_T \in \mathcal{D}'_T(\Pb) \cap \hat{\Pb}_T \notin A_T)
     +
     \Pb^{\infty}(\hat{\Pb}_T \in \mathcal{D}'_T(\Pb) \cap \hat{\Pb}_T \in A_T).
\end{align*}

We will examine each of the two terms separately. When $\hat{\Pb}_T \notin A_T$, $\eta(\hat{x}'_T(\hat{\Pb}_T), \hat{\Pb}_T) =0$, therefore, 
$
\hat{c}'^{\star}(\hat{\Pb}_T,T) 
=
\hat{c}'(\hat{x}'_T(\hat{\Pb}_T),\hat{\Pb}_T,T) 
= 
\hat{c}(\hat{x}'_T(\hat{\Pb}_T),\hat{\Pb}_T,T) \geq 
\hat{c}^{\star}(\hat{\Pb}_T,T)
$. 
Moreover, recall that 
$
\hat{c}'(\cdot,\cdot,T) 
= 
\hat{c}(\cdot,\cdot,T) - \eta(\cdot,\cdot) \mathbf{1}_{T \in \mathcal{T}} 
\leq \hat{c}(\cdot,\cdot,T)
$.
Therefore, 
$
\hat{c}'^{\star}(\hat{\Pb}_T,T) 
\leq 
\hat{c}^{\star}(\hat{\Pb}_T,T)
$. 
Hence 
$
\hat{c}'^{\star}(\hat{\Pb}_T,T)
= \hat{c}^{\star}(\hat{\Pb}_T,T)
$.
This implies by definition of $\hat{x}'_T$, see Equation \eqref{eq: proof def prescriptor}, that $\hat{x}_T(\hat{\Pb}_T) = \hat{x}'_T(\hat{\Pb}_T)$. 
Hence 
$$ 
c(\hat{x}'_T(\hat{\Pb}_T),\Pb) > \hat{c}'^{\star}(\hat{\Pb}_T,T)
\; \text{and} \; \hat{\Pb}_T \not \in A_T
\implies 
c(\hat{x}_T(\hat{\Pb}_T),\Pb) > \hat{c}^{\star}(\hat{\Pb}_T,T).
$$
Therefore,
\begin{equation}\label{proof eq: first part of consistency}
\Pb^{\infty}(\hat{\Pb}_T \in \mathcal{D}'_T(\Pb) \cap \hat{\Pb}_T \notin A_T) 
\leq 
\Pb^{\infty}(\hat{\Pb}_T \in \mathcal{D}_T(\Pb)).
\end{equation}

Suppose now $\hat{\Pb}_T \in A_T$. As $T\in \mathcal{T}$, we have by \eqref{proof eq: cal T def}, $\hat{c}(\hat{x}'_T(\hat{\Pb}_T),\hat{\Pb}_T,T) \geq c(\hat{x}'_T(\hat{\Pb}_T),\hat{\Pb}_T) + \varepsilon/4$. As $\eta$ is bounded by $\varepsilon/8$, we have therefore $\hat{c}'(\hat{x}'_T(\hat{\Pb}_T),\hat{\Pb}_T,T) \geq c(\hat{x}'_T(\hat{\Pb}_T),\hat{\Pb}_T) + \varepsilon/4 -\varepsilon/8 = c(\hat{x}'_T(\hat{\Pb}_T),\hat{\Pb}_T) + \varepsilon/8$.

The loss function $x \rightarrow \ell(x,i)$ is continuous in the compact $\cX$ for all $i\in \Sigma$, therefore, it is bounded. There exists $K>0$ such that $\|\loss(x,\cdot) \| <K$ for all $x \in \cX$. Consider the open set $U:= \set{\Pb'\in \cPin}{\|\Pb - \Pb'\| < \frac{\varepsilon}{8K}}$.  
If $\hat{\Pb}_T \in U$, then
\begin{align*}
\hat{c}'(\hat{x}'_T(\hat{\Pb}_T),\hat{\Pb}_T,T) 
&\geq
c(\hat{x}'_T(\hat{\Pb}_T),\hat{\Pb}_T) + \frac{\varepsilon}{8} \\
&=
c(\hat{x}'_T(\hat{\Pb}_T),\Pb)
+  c(\hat{x}'_T(\hat{\Pb}_T), \hat{\Pb}_T-\Pb)
+ \frac{\varepsilon}{8}\\
&\geq 
c(\hat{x}'_T(\hat{\Pb}_T),\Pb) - \|\loss(\hat{x}'_T(\hat{\Pb}_T),\cdot)\| \|\hat{\Pb}_T - \Pb\| + \frac{\varepsilon}{8}\\
&\geq 
c(\hat{x}'_T(\hat{\Pb}_T),\Pb) - K \frac{\varepsilon}{8K} + \frac{\varepsilon}{8} = c(\hat{x}'_T(\hat{\Pb}_T),\Pb)
\end{align*}
where the second inequality uses Cauchy–Schwarz inequality. This implies that $\hat{\Pb}_T \notin \mathcal{D}'_T(\Pb)$. We have, therefore, $\hat{\Pb}_T \in A_T \cap \mathcal{D}'_T(\Pb) \implies \hat{\Pb}_T \in A_T, \hat{\Pb}_T\notin U$. Hence 
\begin{equation}\label{proof eq: second part of consistency}
  \Pb^{\infty}(\hat{\Pb}_T \in \mathcal{D}'_T(\Pb) \cap \hat{\Pb}_T \in A_T)
  \leq 
\Pb^{\infty}(\hat{\Pb}_T \in A_T \cap \hat{\Pb}_T \notin U)
\leq 
\Pb^{\infty}(\hat{\Pb}_T \notin U).
\end{equation}
Combining \eqref{proof eq: first part of consistency} and \eqref{proof eq: second part of consistency}, we get,
     \begin{align*}
     \limsup_{T \to \infty}
    \frac{1}{a_{t_T}} \log 
    \Pb^{\infty}(\hat{\Pb}_{t_T} \in \mathcal{D}'_{t_T}(\Pb))
     &\leq 
    \limsup_{T \to \infty} \frac{1}{a_{t_T}} \log
    \left (
        \Pb^{\infty}(\hat{\Pb}_{t_T} \in \mathcal{D}_{t_T}(\Pb))
         +
         \Pb^{\infty}(\hat{\Pb}_{t_T} \notin U) 
     \right)
     \\
    &\leq
    \limsup_{T \to \infty} \frac{1}{a_{t_T}} \log
    \left (
        e^{-a_{t_T} + o(a_{t_T})}
        + e^{-\inf_{\Pb'\in U^c}
             I(\Pb',\Pb) t_T + o(t_T)}
     \right) = -1
\end{align*}
where the first part of the second inequality comes from the feasibility of $\hat{c}$ and the second part from the LDP (Theorem \ref{thm: LDP finite space}). The last equality is justified by $\inf_{\Pb'\in U^c} I(\Pb',\Pb) > 0$, as $\Pb \in U$, and $a_T \ll T$.

Combining the last result with \eqref{proof eq: consistency, feasbility for l_T}, we get the out-of-sample guarantee for $\hat{c}'$.
\end{proof}


\subsubsection{Proof of Proposition \ref{prop: lower bound strong opt presc}: Lower bound on regularity}\label{Appendix: proofs of lower bound presc}

\begin{proof}[Proof of Claim \ref{claim: prescriptor unif ineq}]
The inequality \eqref{proof eq: initial assumption presc} implies that there exists a subsequence $(l_T)_{T\geq 1}$ and $\varepsilon>0$ such that
$$ 
\frac{1}{\sqrt{\alpha_{l_T}}}|\hat{c}^{\star}(\Pb_0,l_T) -c^{\star}(\Pb_0)| + \varepsilon \leq \frac{1}{\sqrt{\alpha_{l_T}}}|\cSVP^{\star}(\Pb_0,l_T) -c^{\star}(\Pb_0)|, 
\quad \forall T \in \integ.
$$
We can drop the absolute values in the right hand-side as $\cSVP$ is always greater than the true cost by definition, therefore, for all $T \in \integ$
\begin{equation}\label{proof eq: gap between hat c* and hat c*r.}
\hat{c}^{\star}(\Pb_0,l_T)  + \varepsilon \sqrt{\alpha_{l_T}} \leq \cSVP^{\star}(\Pb_0,l_T).
\end{equation}

Similarly as in the proof of Claim \ref{claim: unif conv.}, using the Arzelà–Ascoli theorem (Theorem \ref{thm: Arzela–Ascoli}) twice, there exists a sub-sequence of the sequence of functions $(\hat{c}(\cdot,\cdot,l_T))_{T\geq 1}$ that converges uniformly and its sequence of gradients in $\Pb$ converges uniformly (in $x$ and $\Pb$). Let $(\hat{c}(\cdot,\cdot,t_T))_{T\geq 1}$ be this sub-sequence and $\hat{c}_{\infty}(\cdot,\cdot)$ its limit. Uniform convergence implies that the sequence of gradients of $\hat{c}$, $(x,\Pb \rightarrow\nabla\hat{c}(x,\cdot,t_T)(\Pb))_{T\geq 1}$ converges to $x,\Pb \rightarrow\nabla\hat{c}_{\infty}(x,\cdot)(\Pb)$ uniformly.

For all $x,\Pb,T$, let $\hat{\delta}(x,\Pb,T) = \hat{c}(x,\Pb,T) - c(x,\Pb)$ and $\hat{\delta}_{\infty}(x,\Pb) = \hat{c}_{\infty}(x,\Pb) - c(x,\Pb)$ its limit. Similarly, denote $\deltaSVP(x,\Pb,T) = \cSVP(x,\Pb,T) - c(x,\Pb)$.

As $(\hat{x}_{t_T}(\Pb_0))_{T\geq 1}$ lives in the compact $\cX$, we can assume WLOG, by extracting again from $(t_T)_{T\geq 1}$ that $(\hat{x}_{t_T}(\Pb_0))_{T\geq 1}$ converges to some $x_0 \in \cX$.
We consider two cases.

\textbf{Case I:} First, consider the case where $\nabla \hat{\delta}_{\infty}(x_0,\cdot)(\Pb_0) = 0$.\\
We show the result with $\Pb_1:=\Pb_0$. Fix $\varepsilon'>0$. By equicontinity of 
$(\nabla \hat{\delta}(x_0,\cdot,t_T))_{T\geq 1}$, and uniform convergence to $\nabla \hat{\delta}_{\infty}(x_0,\cdot)$,
there exists $r>0$ and $T_0\in \integ$ such that for all $\Pb \in \mathcal{B}(\Pb_0,r)$, for all $T\geq T_0$, we have 
$\|\nabla \hat{\delta}(x_0,\cdot,t_T)(\Pb)\| \leq \varepsilon'/4$.
As $(\hat{x}_{t_T}(\Pb_0))_{T\geq1}$ converges to $x_0$, using the equicontinuity property in $x$, there exists $T_1 \geq T_0$ such that for all $T\geq T_1$ and $\Pb \in \mathcal{B}(\Pb_0,r)$, $\|\nabla \hat{\delta}(\hat{x}_{t_T}(\Pb_0),\cdot,t_T)(\Pb)\| \leq \varepsilon'/2$.
Using successively the mean value theorem for $\hat{\delta}$ and inequality \eqref{proof eq: gap between hat c* and hat c*r.}, we have for all $\Pb \in \mathcal{B}(\Pb_0,r)$ and $T \geq T_1$
\begin{align*}
    \hat{\delta}(\hat{x}_{t_T}(\Pb_0),\Pb,t_T) 
    &\leq \hat{\delta}(\hat{x}_{t_T}(\Pb_0),\Pb_0,t_T) + \varepsilon'/2 \|\Pb -\Pb_0\| \\
    &= \hat{c}^{\star}(\Pb_0,t_T) 
    - c(\hat{x}_{t_T}(\Pb_0),\Pb_0) 
    + \varepsilon'/2 \|\Pb -\Pb_0\|\\
    &\leq \cSVP^{\star}(\Pb_0,t_T) 
    - \varepsilon \sqrt{\alpha_{t_T}} 
    - c(\hat{x}_{t_T}(\Pb_0),\Pb_0) 
    + \varepsilon'/2 \|\Pb -\Pb_0\|
\end{align*}
Using the minimality of $\cSVP^{\star}$, the previous inequality leads to
\begin{align*}
\hat{\delta}(\hat{x}_{t_T}(\Pb_0),\Pb,t_T) 
    &\leq 
    \cSVP(\xSVP{,t_T}(\Pb), \Pb_0,t_T) 
    - \varepsilon \sqrt{\alpha_{t_T}} 
    - c(\hat{x}_{t_T}(\Pb_0),\Pb_0) 
    + \varepsilon'/2 \|\Pb -\Pb_0\|\\
    &= \deltaSVP(\xSVP{,t_T}(\Pb),\Pb_0,t_T)
    + c(\xSVP{,t_T}(\Pb),\Pb_0)
    - \varepsilon \sqrt{\alpha_{t_T}} 
    - c(\hat{x}_{t_T}(\Pb_0),\Pb_0) 
    + \varepsilon'/2 \|\Pb -\Pb_0\|
\end{align*}

Notice that for all $\Pb \in \cPin$, $x\in \cX$ and $T \in \integ$, $\nabla \deltaSVP(x,\cdot,t_T)(\Pb) = \sqrt{\alpha_{t_T}} \nabla\Var_{(\cdot)}(\loss(x,\xi))(\Pb)$. We have $x,\Pb \rightarrow \nabla\Var_{(\cdot)}(\loss(x,\xi))(\Pb)$ bounded in $\cX \times \cP$, therefore, $\nabla \deltaSVP(\cdot,\cdot,t_T)$ converges uniformly to $0$. 
Hence, by equicontinuty of $\nabla \deltaSVP(\cdot,\cdot,t_T)$, we can chose $T_1$ such that for all $T\geq T_1$ and $\Pb \in \mathcal{B}(\Pb_0,r)$,
$\|\nabla \deltaSVP(\xSVP{,t_T}(\Pb_0),\cdot,t_T)(\Pb)\| \leq \varepsilon'/2$.
Using this result and the mean value theorem for $\deltaSVP$, the previous chain of inequalities leads to
\begin{align*}
\hat{\delta}(\hat{x}_{t_T}(\Pb_0),\Pb,t_T) 
    &\leq
    \deltaSVP(\xSVP{,t_T}(\Pb),\Pb,t_T)
    + c(\xSVP{,t_T}(\Pb),\Pb_0)
    - \varepsilon \sqrt{\alpha_{t_T}} 
    - c(\hat{x}_{t_T}(\Pb_0),\Pb_0) 
    + \varepsilon' \|\Pb -\Pb_0\|\\
    &=
    \cSVP^{\star}(\Pb,t_T)
    + c(\xSVP{,t_T}(\Pb),\Pb_0-\Pb)
    - \varepsilon \sqrt{\alpha_{t_T}} 
    - c(\hat{x}_{t_T}(\Pb_0),\Pb_0) 
    + \varepsilon' \|\Pb -\Pb_0\|
\end{align*}
which gives the desired result by substracting $c(\hat{x}_{t_T}(\Pb_0),\Pb)$ in both sides and noticing that $\hat{c}^{\star}(\Pb,t_T) \leq \hat{c}(\hat{x}_{t_T}(\Pb_0),\Pb,t_T) $ by minimality of $\hat{c}^{\star}$.

\textbf{Case II:} We now turn to the case where $\nabla \hat{\delta}_{\infty}(x_0,\cdot)(\Pb_0) \neq 0$.\\
Using the same arguments as the proof of Claim \ref{claim: unif conv.} (case I), there exists $\bar{r}_0>0$ such that for all $\bar{r}_0 \geq r_0>0$, $\Pb_1 = \Pb_0 - r_0 \frac{\nabla \hat{\delta}_{\infty}(x_0,\cdot)(\Pb_0)}{ \|\nabla \hat{\delta}_{\infty}(x_0,\cdot)(\Pb_0)\|}$ verifies for all $T\geq T_1$
\begin{align*}
\hat{\delta}(x_0,\Pb_1,t_T) - \hat{\delta}(x_0,\Pb_0,t_T)
&\leq
-\frac{r_0}{2}  
\|\nabla\hat{\delta}_{\infty}(x_0,\cdot)(\Pb_0)\| 
:= -\tilde{\varepsilon} <0
\end{align*}
As $(\hat{x}_{t_T}(\Pb_0))_{T\geq 1}$ converges to $x_0$, by equicontinuity of $\hat{\delta}$, there exists $T_1\geq T_0$ such that for all $T\geq T_1$
\begin{align}\label{proof eq: gap delta}
    \hat{\delta}(\hat{x}_{t_T}(\Pb_0),\Pb_1,t_T) - \hat{\delta}(\hat{x}_{t_T}(\Pb_0),\Pb_0,t_T)
    \leq -\tilde{\varepsilon}
\end{align}

In what follow, the assymptotic notation $o$ is in $T \to \infty$. Let $T\geq T_1$. Using successively the minimality of $\hat{c}^{\star}$, \eqref{proof eq: gap delta}, \eqref{proof eq: gap between hat c* and hat c*r.} and then the minimality of $\cSVP^{\star}$ we have
\begin{align*}
\hat{c}^{\star}(\Pb_1,t_T)
&\leq 
\hat{c}(\hat{x}_{t_T}(\Pb_0), \Pb_1,t_T)\\
& = 
\hat{\delta}(\hat{x}_{t_T}(\Pb_0), \Pb_1,t_T) + c(\hat{x}_{t_T}(\Pb_0), \Pb_1) \\
&\leq
\hat{c}^{\star}(\Pb_0,t_T)
-c(\hat{x}_{t_T}(\Pb_0), \Pb_0)
- \tilde{\varepsilon} + c(\hat{x}_{t_T}(\Pb_0), \Pb_1) \\
&\leq 
\cSVP^{\star}(\Pb_0,t_T) 
- \varepsilon \sqrt{\alpha_{t_T}}
+ c(\hat{x}_{t_T}(\Pb_0), \Pb_1-\Pb_0)
- \tilde{\varepsilon} \\
& \leq
\cSVP(\xSVP{,t_T}(\Pb_1), \Pb_0,t_T) 
+ c(\hat{x}_{t_T}(\Pb_0), \Pb_1 - \Pb_0)
- \tilde{\varepsilon} +o(1)\\
& =
\deltaSVP(\xSVP{,t_T}(\Pb_1), \Pb_0,t_T) + c(\xSVP{,t_T}(\Pb_1), \Pb_0)
+ c(\hat{x}_{t_T}(\Pb_0), \Pb_1 - \Pb_0)
- \tilde{\varepsilon} + o(1)
\end{align*}

As seen in the previous case, $(x,\Pb \rightarrow \nabla \deltaSVP(x, \cdot,t_T)(\Pb)\|)_{T\geq 1}$ converges uniformly to $0$, therefore, by the mean value theorem
$|\deltaSVP(\xSVP{,t_T}(\Pb_1), \Pb_0,t_T) - \deltaSVP(\xSVP{,t_T}(\Pb_1), \Pb_1,t_T)| = o(1)\|\Pb_0-\Pb_1\| = o(1)$. Hence, the previous chain of inequalities leads to
\begin{align*}
\hat{c}^{\star}(\Pb_1,t_T)
&\leq 
\deltaSVP(\xSVP{,t_T}(\Pb_1), \Pb_1,t_T)
+ c(\xSVP{,t_T}(\Pb_1), \Pb_0)
+ c(\hat{x}_{t_T}(\Pb_0), \Pb_1 - \Pb_0)
- \tilde{\varepsilon} + o(1)\\
&=
\cSVP^{\star}(\Pb_1,t_T) 
+ c(\xSVP{,t_T}(\Pb_1), \Pb_0-\Pb_1)
+ c(\hat{x}_{t_T}(\Pb_0), \Pb_1 - \Pb_0)
- \tilde{\varepsilon} + o(1)
\end{align*}
By extracting again from $(t_T)_{T\geq 1}$ we can assume WLOG that $(\xSVP{,t_T}(\Pb_1))_{T\geq 1}$ converges to some $x_1 \in \cX$. Moreover, recall $\hat{x}_{t_T}(\Pb_0) \to x_0$. Hence, by continuity of $c(\cdot,\cdot)$ in the first argument, we have
\begin{align*}
\hat{c}^{\star}(\Pb_1,t_T)
&\leq
\cSVP^{\star}(\Pb_1,t_T) 
+ c(x_1, \Pb_0-\Pb_1)
+ c(x_0, \Pb_1 - \Pb_0)
- \tilde{\varepsilon} + o(1)\\
&=
\cSVP^{\star}(\Pb_1,t_T) 
+ (\loss(x_1,\cdot) - \loss(x_0,\cdot))^{\top}(\Pb_0-\Pb_1)
- \tilde{\varepsilon} + o(1)
\end{align*}
where $\loss(x,\cdot)$ is the vector $(\loss(x,1),\ldots,\loss(x,d))^\top$ for all $x\in \cX$. Recall that the minimizer $\{x^{\star}(\Pb)\} = \argmin_{x\in \cX} c(x,\Pb)$ is unique. This implies that $\Pb \rightarrow x^{\star}(\Pb)$ is continuous (see Lemma \ref{lemma: Continuity of a unique minimizer}).
By continuity of $\cSVP$ and $\hat{c}$, we have $x_1 \in \argmin_{x\in \cX} c(x,\Pb_1) = \{x^{\star}(\Pb_1)\}$ and $x_0 \in \argmin_{x\in \cX} c(x,\Pb_0) = \{x^{\star}(\Pb_0)\}$. Therefore, we can chose $\Pb_1$ sufficiently close to $\Pb_0$ (ie $r_0$ sufficiently small) such that
$$
\|\loss(x_1,\cdot) - \loss(x_0,\cdot)\| \leq \frac{1}{4} \|\nabla\hat{\delta}_{\infty}(x_0,\cdot)(\Pb_0)\|
$$
which implies by Cauchy-Schwartz
$$
(\loss(x_1,\cdot) - \loss(x_0,\cdot))^{\top}(\Pb_0-\Pb_1) 
\leq 
\|\loss(x_1,\cdot) - \loss(x_0,\cdot)\|\|\Pb_0-\Pb_1\| \leq 
\frac{1}{4} \|\nabla\hat{\delta}_{\infty}(x_0,\cdot)(\Pb_0)\| r_0 = \tilde{\varepsilon}/2.
$$
Hence, the inequality on $\hat{c}^{\star}(\Pb_1,t_T)$ becomes
\begin{equation}\label{proof eq: claim prescriptor 2}
\hat{c}^{\star}(\Pb_1,t_T)
\leq 
\cSVP^{\star}(\Pb_1,t_T) 
- \tilde{\varepsilon}/2 + o(1).
\end{equation}
By equicontinuity of $\cSVP^{\star}$ and $\hat{c}^{\star}$ (see Lemma \ref{lemma: continutiy of minumum}), we can chose $r'>0$ and $T_1\in \integ$ such that for all $\Pb \in \mathcal{B}(\Pb_1,r')$ and $T \geq T_1$, we have
$\hat{c}^{\star}(\Pb,t_T)
\leq 
\cSVP^{\star}(\Pb,t_T) 
- \tilde{\varepsilon}/4$. This result is stronger than desired result. In fact, for $\varepsilon'>0$, we get the desired result by choosing $T_2\geq T_1$ and $r>0$ sufficiently small with $r<r'$ such that for all $\Pb \in \mathcal{B}(\Pb_1,r)$ and $T\geq T_2$ 
$$
\varepsilon  \sqrt{\alpha_{t_T}} 
-\varepsilon' \|\Pb-\Pb_1\| 
+ c(\xSVP{,t_T}(\Pb_1),\Pb_1-\Pb)
   - c(\hat{x}_{t_T}(\Pb),\Pb_1- \Pb) \leq \tilde{\varepsilon}/4.
$$
This is possible as this quantity converges to $0$ uniformly when $\Pb \to \Pb_1$ and $T \to \infty$.

\end{proof}

\begin{lemma}\label{lemma: convergence of Var}
Let $\Pb \in \cP$ and $(x_T)_{T \geq 1} \in \cX^\integ$. Let $(\Pb_T)_{T\geq 1}$ such that $\Pb_T \to \Pb$. We have asymptotically in $T\to \infty$
$$
\Var_{\Pb_T}(\loss(x_T,\xi)) = \Var_{\Pb}(\loss(x_T,\xi)) +o(1).
$$
\end{lemma}
\begin{proof}
We have
\begin{align*}
    \Var_{\Pb_T}(\loss(x_T,\xi)) - \Var_{\Pb}(\loss(x_T,\xi))
    &=
    \Eb_{\Pb_T}(\loss(x_T,\xi)^2)-\Eb_{\Pb}(\loss(x_T,\xi)^2)
    +
    \Eb_{\Pb_T}(\loss(x_T,\xi))^2-\Eb_{\Pb}(\loss(x_T,\xi))^2.
\end{align*}
Denote $K = \sup_{x\in \cX} \|\loss(x,\cdot)\|_{\infty}<\infty$. This supremum is finite as $\cX$ is compact and the loss is continuous. We have
$
\Eb_{\Pb_T}(\loss(x_T,\xi)^2)-\Eb_{\Pb}(\loss(x_T,\xi)^2)
=
\sum_{i=1}^d \loss(x_T,i)^2 (\Pb_T(i)-\Pb(i)) 
\leq 
K^2 \sum_{i=1}^d |\Pb_T(i)-\Pb(i)| = o(1).
$
Moreover
\begin{align*}
    \Eb_{\Pb_T}(\loss(x_T,\xi))^2-\Eb_{\Pb}(\loss(x_T,\xi))^2
    &=
    \sum_{1\leq i,j\leq d} \loss(x_T,i)\loss(x_T,j) (\Pb_T(i)\Pb_T(j) - \Pb(i)\Pb(j)) \\
    &\leq
    K^2 \sum_{1\leq i,j\leq d} |\Pb_T(i)\Pb_T(j) - 
    \Pb_T(i)\Pb(j) +
    \Pb_T(i)\Pb(j) -
    \Pb(i)\Pb(j)| \\
    &\leq
    K^2 \sum_{1\leq i,j\leq d} \Pb_T(i)|\Pb_T(j) - \Pb(j)| 
    +
    \Pb(j)|\Pb_T(i) -\Pb(i)| = o(1)
\end{align*}
Hence, $\Var_{\Pb_T}(\loss(x_T,\xi)) - \Var_{\Pb}(\loss(x_T,\xi))=o(1)$.
\end{proof}

{\color{black}
\subsubsection{Proof of Proposition \ref{prop: KL prescriptor approx}: Approximation of KL prescription}\label{App: proof of KL prescriptor approx}
\begin{proof}[Proof of Proposition \ref{prop: KL prescriptor approx}]
The equality follows immediately from Proposition \ref{prop: KL equiv to SVP} (as it is uniform in $x$). The equivalence follows from combining this result with Proposition \ref{prop: rate of cv of prescriptor}. Indeed, for all $x \in \cX$ and $\Qb \in \cPin$
$$
\frac{\cKL'^\star(\Qb,T) - c^\star(\Qb)}{\cSVP^\star(\Qb,T) - c^\star(\Qb)}
=
\frac{\cSVP^\star(\Qb,T)- c^\star(\Qb) + o(\sqrt{a_T/T})}{\cSVP^\star(\Qb,T) - c^\star(\Qb)}
=
\frac{\sqrt{(a_T/T)\Var_{\Qb}(\loss(x^\star(\Qb),\xi)} + o(\sqrt{a_T/T})}{\sqrt{(a_T/T)\Var_{\Qb}(\loss(x^\star(\Qb),\xi)} + o(\sqrt{a_T/T})} = 1 + o(1)
$$
\end{proof}

\subsubsection{Proof of Proposition \ref{prop: feasibility of KL presc in subsexp}: Feasibility of KL prescriptor}\label{App: proof of feasibility of KL presc in subsexp}
\begin{proof}[Proof of Proposition \ref{prop: feasibility of KL presc in subsexp}]
    Let $\Pb \in \cPin$ and $T_0 \in \integ$. For $T \in \integ$, we have
$$
\hat{\Pb}_T \in E_T :=\set{\Pb' \in \mathcal{P}}{I(\Pb',\Pb)\leq \frac{a_T}{T}} \implies
c(\hat{x}_{\text{KL},T}'(\hat{\Pb}_T), \Pb) \leq
\cKL'(\hat{x}_{\text{KL},T}'(\hat{\Pb}_T), \hat{\Pb}_T,T) =\cKL'^{\star}(\hat{\Pb}_T,T)
$$
Hence,
$$\frac{1}{a_T} \log \Pb^\infty 
    \left(
c(\hat{x}_{\text{KL},T}'(\hat{\Pb}_T), \Pb) > \cKL'^{\star}(\hat{\Pb}_T,T)
    \right)
\leq
    \frac{1}{a_T} \log \Pb^{\infty}
    \left(
      \hat{\Pb}_T \not\in E_T
          \right)
    $$
We have shown in the proof of Proposition \ref{prop: KL feasibility in subexp} that $\limsup_{T\to\infty}
    \frac{1}{a_T}  \log \Pb^{\infty}
    \left(
      \hat{\Pb}_T \not\in E_T
          \right) \leq -1$ which gives the desired result.
\end{proof}
}

\section{Limit superior lemmas}
\begin{lemma}\label{lemma: lim in limsup.}
For any two sequences of non-negative real numbers $(u_T)_{T \geq 1}$ and $(v_T)_{T \geq 1}$ such that $\lim u_T$ exists, the equality
$$
\limsup_{T \to \infty} u_Tv_T 
=
\lim_{T \to \infty} u_T \cdot \limsup_{T \to \infty} v_T
$$
holds whenever the right hand-side is not of the form $0\cdot \infty$.
\end{lemma}

\begin{lemma}\label{lemma: limsup liminf ineq}
For any two sequences of non-negative real numbers $(u_T)_{T \geq 1}$ and $(v_T)_{T \geq 1}$, the inequality
$$
\limsup_{T \to \infty} u_Tv_T \geq \limsup_{T \to \infty} u_T \cdot \liminf_{T \to \infty} v_T
$$
holds whenever the right hand-side is not of the form $0\cdot \infty$. 
\end{lemma}
\begin{proof}
Let $(k_T)_{T\geq 1}$ be a sequence increasing to infinity such that $\limsup u_T = \lim u_{k_T}$. We have
\begin{align*}
    \limsup u_Tv_T \geq \limsup u_{k_T}v_{k_T}
\end{align*}
Suppose first $\lim u_{k_T} \limsup v_{k_T}$ is not of the form $0\cdot \infty$. Using Lemma \ref{lemma: lim in limsup.}, we have
\begin{align*}
\limsup u_{k_T}v_{k_T}
    = \lim u_{k_T} \limsup v_{k_T}
    = \limsup u_{T} \limsup v_{k_T} 
    \geq \limsup u_{T} \liminf v_{T},
\end{align*}
which proves the result.

Now assume $\lim u_{k_T} =0$ and $\limsup v_{k_T} = \infty$. This implies that $\limsup u_T =0$. As $\limsup_{T \to \infty} u_T \cdot \liminf_{T \to \infty} v_T$ is not of the form $0 \cdot \infty$, we have $\liminf v_T < \infty$, therefore, $\limsup_{T \to \infty} u_T \cdot \liminf_{T \to \infty} v_T = 0$ and the lemma's inequality holds trivially.

Assume $\lim u_{k_T} = \infty$ and $\limsup v_{k_T} = 0$. Then $\liminf v_T =0$. As $\limsup_{T \to \infty} u_T \cdot \liminf_{T \to \infty} v_T$ is not of the form $0\cdot \infty$, we have $\limsup u_T < \infty$, therefore, $\limsup_{T \to \infty} u_T \cdot \liminf_{T \to \infty} v_T = 0$ and the lemma's inequality holds trivially.
\end{proof}

\begin{lemma}\label{lemma: limsup of complementary sequences}
Let $\mathcal{T} \subset \integ $ be an infinite subset of $\integ$ and $u_T \in \Re^{\integ}$ be a real sequence. Let $(t_T)_{T\geq 1}$ be the increasing sequence of elements in $\mathcal{T}$ and $(l_T)_{T\geq 1}$ be the increasing sequence of elements in its complement $\integ \setminus \mathcal{T}$. We have
$$
\limsup_{T \to \infty} u_T = \max \left( 
    \limsup_{T \to \infty } u_{t_T}, \limsup_{T \to \infty } u_{l_T}
    \right).
$$
\end{lemma}
\begin{proof}
We have 
\begin{align*}
\limsup_{T \in \integ } u_T 
=
\lim_{T\to \infty} \sup_{t \geq T}u_t 
&=
\lim_{T\to \infty} 
\max\left(
\sup_{t \geq T, t \in \mathcal{T}}u_t ,
\sup_{t \geq T, t \not\in \mathcal{T}}u_t
\right)\\
&=
\max\left(
\lim_{T\to \infty}  \sup_{t \geq T, t \in \mathcal{T}}u_t ,
\lim_{T\to \infty} \sup_{t \geq T, t \not\in \mathcal{T}}u_t
\right)\\
&=\max \left( 
    \limsup_{T \to \infty} u_{t_T}, \limsup_{T \to \infty} u_{l_T}
    \right),
\end{align*}
where the inversion of $\max$ and limits is true as both the limits in the $\max$ exist.
\end{proof}

\end{document}